\documentclass{article}
\usepackage[utf8]{inputenc}
\usepackage{amssymb}
\usepackage{amsmath}
\usepackage{amsthm}
\allowdisplaybreaks[4]
\usepackage{fullpage}
\usepackage{algorithm}  
\usepackage{algorithmicx}  
\usepackage{algpseudocode}
\usepackage{booktabs}
\usepackage{array}
\usepackage{bm}
\usepackage{graphicx}
\usepackage{authblk}

\newtheorem{theorem}{Theorem}[section]
\newtheorem{lemma}[theorem]{Lemma}
\newtheorem{assumption}[theorem]{Assumption}
\newtheorem{proposition}[theorem]{Proposition}

\def\grad{\nabla}
\def\hat{\widehat}
\newcommand*{\affaddr}[1]{#1}
\newcommand*{\affmark}[1][*]{\textsuperscript{#1}}
\newcommand*{\email}[1]{\texttt{#1}}

\begin{document}
\title{Optimal Estimation of Off-Policy Policy Gradient via Double Fitted Iteration}
\author{
Chengzhuo Ni\affmark[1], Ruiqi Zhang\affmark[2], Xiang Ji\affmark[1], Xuezhou Zhang\affmark[1], and Mengdi Wang\affmark[1]\\
\affaddr{\affmark[1]Department of Electrical and Computer Engineering, Princeton University}\\
\email{\{cn10,xiangj,xz7392,mengdiw\}@princeton.edu}\\
\affaddr{\affmark[2]School of Mathematical Science, Peking University}\\
\email{1800010777@pku.edu.cn}
}
\date{}
\maketitle

\begin{abstract}
    Policy gradient (PG) estimation becomes a challenge when we are not allowed to sample with the target policy but only have access to a dataset generated by some unknown behavior policy. Conventional methods for off-policy PG estimation often suffer from either significant bias or exponentially large variance. In this paper, we propose the double Fitted PG estimation (FPG) algorithm. FPG can work with an arbitrary policy parameterization, assuming access to a Bellman-complete value function class. In the case of linear value function approximation, we provide a tight finite-sample upper bound on policy gradient estimation error, that is governed by the amount of distribution mismatch measured in feature space. We also establish the asymptotic normality of FPG estimation error with a precise covariance characterization, which is further shown to be statistically optimal with a matching Cramer-Rao lower bound. Empirically, we evaluate the performance of FPG on both policy gradient estimation and policy optimization, using either softmax tabular or ReLU policy networks. Under various metrics, our results show that FPG significantly outperforms existing off-policy PG estimation methods based on importance sampling and variance reduction techniques.
\end{abstract}

\section{Introduction}
Policy gradient plays a key role in policy-based reinforcement learning (RL). We focus on the estimation of policy gradient in off-policy reinforcement learning. In the off-policy setting, we are given episodic trajectories that were generated by some unknown behavior policy. Our goal is to estimate the {\it single} policy gradient of a target policy $\pi_\theta$, i.e., $\nabla_\theta v^{\pi_{\theta}}$, based on the off-policy data only. This is motivated by applications such as medical diagnosis and ICU management, in which sampling data with a proposed policy is prohibitive or extremely costly. In these applications, one may not expect to learn the full optimal policy from limited data, but rather learn a single gradient vector for directions of improvement. To handle the distribution mismatch between behavior and target policy, a classic approach is importance sampling (IS) \cite{jie2010connection}. However, IS is known to be sample-expensive and unstable, as the importance sampling weight can grow exponentially with respect to time horizon and causing uncontrollably large variances.

In this work, we design an algorithm to avoid the high variance of importance sampling by utilizing a \textit{good} (to be defined in Sec. \ref{sec:ass}) value function approximation should they be available. The key idea is to perform PG estimation in an iterative way, similar to the well-known Fitted Q Iteration (FQI) algorithm. We propose the double Fitted Policy Gradient (FPG) estimation algorithm, which conducts iterative regression to estimate $Q$ functions and $\nabla_\theta Q$ functions jointly. The FPG algorithm is able to provide an accurate estimation under mild data coverage assumption and without the knowledge of the behavior policy, in contrast to vanilla IS which must know the behavior policy. 

When the function approximator is linear, we show that FPG is equivalent to a model-based plugin estimator and can give an $\varepsilon$-close PG estimator using a sample size of $N = O\left(CH^5/\varepsilon^2\right)$, where $H$ is the horizon length and $C$ is a constant to be specified that measures the distribution shift between behavior policy and target policy. Notably, this distribution shift $C=O(1+\chi_{\mathcal{F}}^2(\mu^\theta,\bar\mu))$ can be bounded by a form of relative condition number or a restricted chi-square divergence, measuring the mismatch between the behavior and target policy in feature space. We additionally establish the asymptotic normality of our FPG estimator with closed form variance expression. We also provide a matching information-theoretic Cramer-Rao lower bound, showing that our estimator is in fact asymptotically optimal. See Table \ref{table_off_policy_PG_methods} for a summary of theoretical results for off-policy PG estimation. FPG can be easily applied as a plug-in PG estimator in any off-policy PG algorithm. Under standard assumptions, a PG algorithm with FPG estimator can find an $\varepsilon$-stationary policy using at most $N = \tilde O\left(\textrm{dim}(\Theta)^2/\varepsilon^2\right)$ samples. If the policy optimization landscape happens to satisfy the \textit{Polyak-lojasiewicz condition} \cite{polyak1963gradient,bhandari2019global}, the sample complexity can be further improved to $N = \tilde O\left({\textrm{dim}(\Theta)^2/\varepsilon}\right)$ for finding an $\varepsilon$-optimal policy. 

\section{Problem Definitions}
\paragraph{Markov Decision Process}
An instance of MDP is defined by the tuple $(\mathcal{S},\mathcal{A},p,r,\xi,H)$ where $\mathcal{S}$ and $\mathcal{A}$ are the state and action spaces, $H\in\mathbb{N}_+$ is the horizon, $p_h: \mathcal{S}\times\mathcal{A}\rightarrow \Delta_{\mathcal{S}}, h\in[H]$ is the transition probability (where $\Delta_\mathcal{S}$ denotes the probability simplex over $\mathcal{S}$), $r_h:\mathcal{S}\times\mathcal{A}\rightarrow[0,1], h\in[H]$ is the reward function and $\xi\in\Delta_{\mathcal{S}}$ is the initial state distribution. Given an MDP, a policy $\pi_h:\mathcal{S}\rightarrow\Delta_{\mathcal{A}},h\in[H]$ is a distribution over the action space given the state $s$ and time step $h$. At each time step $h$, the agent observes $s_h$ and action $a_h$ according to its behavior policy $\pi$. The agent then observes a reward $r_h(s_h, a_h)$ and the next state $s_{h+1}$ sampled according to $s_{h+1}\sim p_h(\cdot\vert s_h, a_h)$. A policy $\pi$ is measured by the Q function $Q^\pi$ and the value $v^\pi$, defined by $Q^\pi_h(s, a)=\mathbb{E}^\pi[\sum_{h^{\prime}=h}^Hr_{h^\prime}(s_{h^{\prime}}, a_{h^{\prime}})\vert s_h=s,a_h=a],\forall h\in[H],(s,a)\in\mathcal{S}\times\mathcal{A}$ and $v^{\pi}= \mathbb{E}^\pi[\sum_{h=1}^Hr_h(s_h, a_h)\vert s_1\sim\xi]$, where $\mathbb{E}^\pi$ denotes the expectation over trajectories by following policy $\pi$. The optimal policy of the MDP is defined as $\pi^*:=\arg\max_{\pi}v^\pi$. 

\paragraph{Off-Policy Policy Gradient Estimation}
Direct policy optimization methods are popular in RL due to their effectiveness and generalizability. Among them, the classic Policy Gradient (PG) method represents policies via a parametric function approximation and perform gradient ascent on the policy parameters \cite{sutton2000policy}.

Denote a parametrized policy as $\pi_{\theta}$, where $\theta\in\Theta$ is the policy parameters. Policy Gradient is defined as the gradient of policy value $v_{\theta}$ with respect to the policy parameter $\theta$:
$\nabla_\theta v_\theta=\nabla_\theta \mathbb{E}^{\pi_{\theta}}[\sum_{h=1}^Hr_h(s_h,a_h)\vert s_1\sim\xi]$. With policy gradients, one may directly search in the policy parameter space $\Theta$ using gradient ascent iterations, giving rise to the class of PG algorithms. However, directly differentiating through the value function is very difficult, especially when we do not have access to the transition probability of the MDP. The policy gradient theorem \cite{sutton2000policy} provides a convenient formula for estimating PG using Monte Carlo sampling:
\begin{align*}
    \nabla_\theta v_\theta=\mathbb{E}^{\pi_\theta}\left[\left.\sum_{h=1}^H\left(\sum_{h^\prime=h}^H r_{h^\prime}\right)\nabla_\theta\log\pi_{\theta, h}\left(a_h\vert s_h\right)\right\vert s_1\sim\xi\right].
\end{align*}
In the online RL setting, one can interact with the environment directly with target policy $\pi_{\theta}$ and directly estimate the PG by averaging over sample trajectories \cite{degris2012off, kakade2001natural, peters2008natural,sutton2000policy, williams1992simple}. 

We focus on the more challenging offline RL setting, where we are not allowed to interact with the environment with the target policy $\pi_\theta$. Instead, we only have access to offline logged data, $\mathcal{D}=\{(s^{(k)}_h,a^{(k)}_h,s^{(k)}_{h+1},r^{(k)}_h)\}_{h\in[H],k\in[K]}$, which consists of $K$ \textit{i.i.d.} trajectories, each of length $H$ and is generated from an unknown behavior policy $\bar{\pi}$. The goal of off-policy PG estimation is to construct an estimator $\widehat{\nabla_{\theta} v_{\theta}}$ based solely on the off-policy data $\mathcal{D}$ that approximates the true gradient with low sample and computational complexity.

\paragraph{Notations}
Let $\pi_\theta$ be a policy parameterized by $\theta\in\Theta\subseteq\mathbb{R}^m$, where $\Theta$ is compact and $m=\textrm{dim}(\theta)$. Let $\theta^*=\arg\max_{\theta\in\Theta}v_\theta$. Denote for short that $Q^\theta_h:=Q^{\pi_\theta}_h, v_\theta := v^{\pi_\theta}$. Define the transition operator $\mathcal{P}_\theta$ by
\begin{align*}
	\left(\mathcal{P}_{\theta, h} f\right)(s,a)&:=\mathbb{E}^{\pi_\theta}\left[f(s_{h+1}, a_{h+1})\vert s_h=s, a_h=a\right],\quad\forall f:\mathcal{S}\times\mathcal{A}\rightarrow\mathbb{R},h\in[H],s\in\mathcal{S},a\in\mathcal{A}.
\end{align*}
where $[N]$ is the set of integer $1,2,\ldots,N$. 
Given a real-valued function class $\mathcal{F}$ and vector-valued function $u:\mathcal{S}\times\mathcal{A}\rightarrow\mathbb{R}^{m}$, we say $u\in\mathcal{F}$ if $u_j\in\mathcal{F},\ \forall j\in[m]$. For any matrix $E\in\mathbb{R}^{d_1\times d_2}$ (which includes scalars and vectors as special cases), we define its Jacobian as $\nabla_\theta E_\theta = (\nabla_\theta^1 E_\theta, \nabla_\theta^2 E_\theta, \ldots, \nabla_\theta^m E_\theta)\in\mathbb{R}^{d_1\times md_2}$, where $\nabla_\theta^j$ is the partial derivative w.r.t. the $j$th entry, i.e., $\nabla_\theta^j := \frac{\partial}{\partial\theta_j}$.

\begin{table*}[htb!]
    \centering
    \begin{tabular}{m{3.5cm}<{\centering} m{3.5cm}<{\centering}  m{1.5cm}<{\centering} m{3cm}<{\centering} m{2cm}<{\centering}}
        \toprule
        Algorithm & Variance  & Require Known Behavior Policy? & Required Estimators & Finite Sample\\
        \hline
        REINFORCE \cite{kakade2001natural,shelton2013policy}  & $2^{\Theta(H)} \Theta(\frac{1}{K})$  & Yes & None & Yes\\
        \hline
        GPOMDP \cite{kakade2001natural,shelton2013policy} & $2^{\Theta(H)} \Theta(\frac{1}{K})$  & Yes & $\widehat{Q}_h^\theta$ & Yes\\
        \hline
        EOOPG \cite{kallus2020statistically} & $\Theta(\frac{H^4}{K}\max_{s,a,h}\frac{\mu_h^\theta(s,a)}{\bar{\mu}_h(s,a)})$ & Yes  & $\widehat{\mu}_h^\theta,\widehat{\nabla_\theta\mu_h^\theta},\widehat{Q}_h^\theta,\widehat{\nabla_\theta Q_h^\theta}$ & No\\
        \hline
        \textbf{FPG} (This Paper, with $d$-dim Features) & $\Theta(\frac{H^4}{K}\min\{C_1d,H\}(1+\chi^2_{\mathcal{F}}(\mu^\theta,\bar{\mu})))$ & No & None & Yes\\
        \bottomrule
    \end{tabular}
    \caption{\textbf{Comparison of Off-Policy PG Estimation Methods.} Both REINFORCE and GPOMDP suffers from exponential variance in the worst case. EOOPG's bound scales with the maximum density ratio $\max_{s,a,h}\frac{\mu_h^\theta(s,a)}{\bar{\mu}_h(s,a)}$, whereas our bound only scales with the $\chi^2$-divergence which can be much smaller. For example, the density ratio between two standard Gaussian distribution is infinite, but the $\chi^2$-divergence is always bounded. In addition, our method also does not require the knowledge of the behavior policy or assume access to any high-performing value or gradient estimators, which may not be available in practice.} 
	\label{table_off_policy_PG_methods} 
\end{table*}

\section{Related Work}
When it comes to off-policy PG estimation, one demanding challenge is the distribution shift between the possibly unknown behavior policy and target policy \cite{agarwal2021theory}. The basic Importance Sampling (IS) estimator for off-policy PG, which is still the most common approach used in practice, is
\begin{align*}
    \widehat{\nabla_\theta v_\theta^{IS}}:=\frac{1}{K}\sum_{k=1}^K w_k\sum_{h=1}^H \left(\sum_{h^\prime =h}^H r_{h^\prime}^{(k)}\right)\nabla_\theta\log\pi_{\theta, h}\left(\left. a_h^{(k)}\right\vert s_h^{(k)}\right),
\end{align*}
where $w_k=\prod_{h=1}^H \frac{\pi_{\theta, h}\left(\left.a_h^{(k)}\right\vert s_h^{(k)}\right)}{\bar \pi_h\left(\left.a_h^{(k)}\right\vert s_h^{(k)}\right)}$ is the IS weight. Classical PG methods including REINFORCE and GPOMDP \cite{sutton2000policy, williams1992simple} are all based on this idea or its modifications \cite{degris2012off, kakade2001natural, peters2008natural}. A severe drawback of the IS method is 
its huge variance that can be as large as $2^{\Theta(H)}$, resulting in ill-behaved gradient steps in practice. IS also requires prior knowledge of $\bar{\pi}$ to compute the IS weights, which often is not available. \cite{kallus2020statistically} proposes a meta-algorithm called (EOOPG) that performs doubly robust off-policy PG estimation, assuming access to a number of nuisance estimators 
including Q function and density estimators. They show that if the state-action density ratio function $\mu^{\pi}/\bar\mu$ can be estimated with error rate $K^{-1/2}$, the EOOPG would be asymptotically efficient with a limit variance $\Theta(H^4/K)$. \cite{xu2021doubly} extends the doubly robust approach to the case of discounted MDP and with a finite sample guarantee. However, both work require the density ratio be precisely estimated, which is arguably an even harder problem. Note that density ratio estimation requires learning a function that maps from the raw state space, which can be arbitrarily high dimensional and complex, whereas policy gradient estimation only requires estimating a vector of length being the number of policy parameters.
They did not provide a guarantee on the error of such an estimator and leave the estimation error in the final result as a irreducible term. \cite{morimura2010derivatives} 
proposes a temporal difference method and estimates the policy gradient via linear function approximation of the stationary state distribution, but does not provide a formal statistical guarantee. Several other methods for off-policy PG, including Non-parametric OPPG \cite{tosatto2020nonparametric} and Q-Prop \cite{gu2016q}, are found to be empirically effective but no theoretical guarantee is provided. In general, theoretical understanding for off-policy PG remains rather limited. We summarize known variance bounds for off-policy PG estimation in Table \ref{table_off_policy_PG_methods}.

Off-policy PG estimation is closely related to PG-based policy optimization. For example, even in online policy optimization, one can use past data for more efficient PG estimation. Several works \cite{papini2018stochastic, xu2020improved, xu2019sample} combines IS with variance reduction technique, but their theories are based on the assumption that the variance of the IS estimator is bounded at some controllable level instead of grow exponentially \cite{jiang2016doubly, degris2012off, kallus2020statistically} or that Lipschitz continuity holds \cite{zhang2021convergence}. \cite{tosatto2020nonparametric} provides a non-parametric OPPG method with some error analysis. \cite{liu2019off, gu2016q} combine off policy PG estimation with actor-critic/policy gradient schemes. \cite{zhang2020variational} generalizes the notion of policy gradient to RL with general utilities and shows that such PG can be estimated by solving a stochastic saddle point.

Another closely related topic is the Offline Policy Evaluation (OPE), i.e., to estimate the target policy's value given offline data generated by some behavior policy $\bar{\pi}$. Various methods, from importance sampling to doubly robust estimators have been proposed \cite{tokdar2010importance,10.5555/645529.658134, jiang2016doubly, thomas2016data}. A marginalized importance sampling \cite{xie2019towards} for tabular MDP and a fitted Q evaluation \cite{duan2020minimax} approach for linear MDP provably achieve minimax-optimal error bound with matching information-theoretic lower bounds. The fitted Q evaluation method was later shown to work with bootstrapping \cite{hao2021bootstrapping}, kernel function approximation \cite{duan2021optimal}, third-order differentiable function approximation \cite{zhang2022off} and ReLU networks \cite{ji2022sample}.

Another line of works use the pessimism principle to design algorithms that can perform stable estimation even under weaker coverage assumption \cite{JinYing2020IPPE,RashidinejadParia2021BORL,zanette2021provable,zhang2021corruption,ChangJonathanD2021MCSi,yin2022near}.
However, to the best of our knowledge, these algorithms do not achieve the minimax optimal rate for OPE and it's unclear how to apply them to PG estimation. 

\section{Assumptions}\label{sec:ass}
In this paper, we focus on a setting where $Q^\theta$ and $\nabla_\theta Q^\theta$ can both be represented within a function class $\mathcal{F}$. Assume without loss of generality that $\bm{1}\in\mathcal{F}$.
\begin{assumption}[Bellman Completeness]
    \label{fclass}
    For any $f\in\mathcal{F}$ and $h\in[H]$, we have $\mathcal{P}_{\theta, h} f\in\mathcal{F}$, and we suppose $r_h\in\mathcal{F},\ \forall h\in[H]$. It follows that $Q^\theta_h\in\mathcal{F},\ \forall h\in[H],\theta\in\Theta$.
\end{assumption}

The \textit{Bellman Completeness} assumption has been commonly made in the theoretical offline RL literature  \cite{xie2021bellman,duan2020minimax}. It requires $\mathcal{F}$ to be closed under the transition operator $\mathcal{P}_\theta$, so that the function approximation incurs zero Bellman error. In fact, it is known both theoretically \cite{wang2020statistical} and empirically \cite{wang2021instabilities} that without such assumption FQI can diverge. We similarly assume that the gradient map also belongs to $\mathcal{F}$.
\begin{assumption}
    \label{gclass}
    $\nabla_\theta Q^\theta_h\in\mathcal{F},\ \forall h\in[H], \theta\in\Theta$. 
\end{assumption}
In the theoretical results, we will focus on the tractable case where $\mathcal{F}$ is a linear function class, since even OPE with general nonlinear function class remains an open problem. However, we remark that our algorithm (see Alg.~\ref{alg1}) applies to any function class, including neural networks.

\paragraph{Linear function approximation}
Let $\phi:\mathcal{S}\times\mathcal{A}\rightarrow\mathbb{R}^d$ be a state-action feature map. Let $\mathcal{F}$ be the class of linear functions given by $\mathcal{F}=\{\phi(\cdot,\cdot)^\top w\vert w\in\mathbb{R}^d\}$. Then for any policy $\pi_\theta$ and $h\in[H]$, Assumption \ref{fclass} implies there exist $w_r\in\mathbb{R}^d$ and $w_h^\theta\in\mathbb{R}^d$ such that
\begin{align*}
    r_h(s,a)=\phi(s,a)^\top w_{r,h},\quad Q_h^\theta(s,a)=\phi(s,a)^\top w_h^\theta. 
\end{align*}
Furthermore, we show that Assumption \ref{fclass} alone is sufficient to ensure the expressiveness of $\mathcal{F}$ for PG estimation in case of the linear function class. 
\begin{proposition}
    \label{lin_rep}
    If $\mathcal{F}=\{\phi(\cdot,\cdot)^\top w\vert w\in\mathbb{R}^d\}$, Assumption \ref{fclass} implies Assumption \ref{gclass}. In particular, we have $w_h^\theta$ is differentiable w.r.t. $\theta$ and
    \begin{align*}
        \nabla_\theta Q_h^\theta(s,a)=\phi(s,a)^\top\nabla_\theta w_h^\theta,\quad\forall h\in[H]. 
    \end{align*}
\end{proposition}
In other word, as long as one can use linear function approximation for policy evaluation, the same feature map automatically allows linear function approximation of $\nabla_\theta Q_h^\theta$.

\section{Algorithm}
In this section, we describe our double Fitted Policy Gradient iteration (FPG) algorithm, designed to estimate the policy gradient $\nabla_\theta v_\theta$ from an arbitrary batch data $\mathcal{D}$. 

\subsection{Policy Gradient Bellman Equation}
Notice that by Bellman's equation, we have
\begin{align*}
    Q_h^\theta(s,a)=r_h(s,a)+\int_{\mathcal{S}\times\mathcal{A}}p_h(s^\prime\vert s,a)\pi_{\theta, h+1}(a^\prime\vert s^\prime)Q_{h+1}^\theta(s^\prime, a^\prime)\mathrm{d}s^\prime\mathrm{d}a^\prime.
\end{align*}
Differentiating on both sides w.r.t. $\theta$, we get
\begin{align*}
    \nabla_\theta Q_h^\theta(s,a)=&\int_{\mathcal{S}\times\mathcal{A}}p_h(s^\prime\vert s,a)\left(\left(\nabla_\theta\pi_{\theta,h+1}(a^\prime\vert s^\prime)\right) Q_{h+1}^\theta(s^\prime,a^\prime)+\pi_{\theta,h+1}(a^\prime\vert s^\prime)\nabla_\theta Q_{h+1}^\theta(s^\prime,a^\prime) \right)\mathrm{d}s^\prime\mathrm{d}a^\prime\\
    =&\mathbb{E}^{\pi_\theta}\left[\left(\nabla_\theta\log\pi_{\theta,h+1}(a_{h+1}\vert s_{h+1})\right)Q_{h+1}^\theta(s_{h+1},a_{h+1})+\nabla_\theta Q_{h+1}^\theta(s_{h+1},a_{h+1})\vert s_h=s,a_h=a\right].
\end{align*}
Here we use the convention that the gradient of $\nabla_\theta Q_h^\theta$ or $\nabla_\theta \pi_{\theta,h}$ is a function from $\mathcal{S}\times\mathcal{A}$ to a row vector in $\mathbb{R}^{1\times m}$. Thus, we get the {\it Policy Gradient Bellman equation}, given by
\begin{align}
    \label{bel}
    Q_h^\theta=r_h+\mathcal{P}_{\theta,h}Q_{h+1}^\theta,\quad\nabla_\theta Q_h^\theta=\mathcal{P}_{\theta,h}\left(\left(\nabla_\theta\log\Pi_{\theta,h+1}\right)Q_{h+1}^\theta+\nabla_\theta Q^\theta_{h+1}\right),
\end{align}
where we define the operator $\nabla_\theta\log \Pi_{\theta,h}$ by
\begin{align*}
	\left(\left(\nabla_\theta\log\Pi_{\theta,h}\right)f\right)(s,a):=\left(\nabla_\theta\log\pi_{\theta,h}(a\vert s)\right)f(s,a). 
\end{align*}
Once we get the estimations of $Q_1^\theta$ and $\nabla_\theta Q_1^\theta$, we can calculate the policy gradient $\nabla_\theta v_\theta$ using the formula 
\begin{align*}
    \nabla_\theta v_\theta=\int_{\mathcal{S}\times\mathcal{A}}\xi(s)\pi_{\theta,1}(a\vert s)\big(\nabla_\theta Q_1^\theta(s,a)+\left(\nabla_\theta\log\pi_{\theta,1}(a\vert s)\right)Q_1^\theta(s,a)\big)\mathrm{d}s\mathrm{d}a. 
\end{align*}

\subsection{Double Fitted Policy Gradient Iteration}
In a similar spirit to Fitted Q Iteration (FQI), we develop our PG estimator based on the gradient Bellman equations \eqref{bel}.  We derive our estimator by applying regression iteratively: Let $\widehat{Q}_{H+1}^{\theta,\textrm{FPG}}=\widehat{\nabla_\theta^j Q_{H+1}^{\theta,\textrm{FPG}}}=0,\ \forall j\in[m]$. For $h=H,H-1,\ldots,1$ and $j\in[m]$, let
\begin{align}
    \label{Q_rec}
    &\widehat{Q}_h^{\theta,\textrm{FPG}}=\mathop{\arg\min}_{f\in\mathcal{F}}\left[\sum_{k=1}^K\bigg(f\left(s^{(k)}_{h},a^{(k)}_{h}\right)-r^{(k)}_h-\int_{\mathcal{A}}\pi_{\theta,h+1}\left(a^\prime\left\vert s^{(k)}_{h+1}\right.\right)\widehat{Q}_{h+1}^{\theta,\textrm{FPG}}\left(s^{(k)}_{h+1}, a^\prime\right)\mathrm{d}a^\prime\bigg)^2+\lambda\rho(f)\right]\\
    &\widehat{\nabla_\theta^j Q_h^{\theta,\textrm{FPG}}}=\mathop{\arg\min}_{f\in\mathcal{F}}\Bigg[\sum_{k=1}^K\bigg(f\left(s^{(k)}_h,a^{(k)}_h\right)-\int_{\mathcal{A}}\pi_{\theta,h+1}\left(a^\prime\left\vert s^{(k)}_{h+1}\right.\right)\bigg(\left(\nabla_\theta^j\log\pi_{\theta, h+1}\left(a^\prime\left\vert s^{(k)}_{h+1}\right.\right)\right)\widehat{Q}_{h+1}^{\theta,\textrm{FPG}}\left(s^{(k)}_{h+1},a^\prime\right)\nonumber\\
    \label{gQ_rec}
    &+\widehat{\nabla_\theta^j Q_{h+1}^{\theta,\textrm{FPG}}}\left(s^{(k)}_{h+1}, a^\prime\right)\bigg)\mathrm{d}a^\prime\bigg)^2+\lambda\rho(f)\Bigg]
\end{align}
After the computation of $\widehat{Q}_h^{\theta,\textrm{FPG}},\widehat{\nabla_\theta Q_h^{\theta,\textrm{FPG}}}$, the policy gradient can be estimated straightforwardly. The full algorithm is summarized in Algorithm \ref{alg1}.
\begin{algorithm}[htb!]
\caption{Fitted PG Algorithm}
\label{alg1}
	\begin{algorithmic}[1] 
		\State\textbf{Input:} Dataset $\mathcal{D}$, target policy $\pi_\theta$, initial state distribution $\xi$.
		\State\textbf{Initialize } $\widehat{Q_{H+1}^{\theta,\textrm{FPG}}}=0$ and $\widehat{\nabla_\theta^j Q_{H+1}^{\theta,\textrm{FPG}}}=0,\ \forall j\in[m]$.
		\For{$h=H,H-1,\ldots,1$}   
		\State Calculate $\widehat{Q}_h^{\theta,\textrm{FPG}},\widehat{\nabla_\theta Q_h^{\theta,\textrm{FPG}}}$ by solving \eqref{Q_rec} and \eqref{gQ_rec}. 
		\EndFor
		\State\textbf{Return} 
		\begin{align*}
		    \widehat{\nabla_\theta v_\theta^{\textrm{FPG}}}=&\int_{\mathcal{S}\times\mathcal{A}} \xi(s)\pi_{\theta,1}(a\vert s)\bigg(\widehat{\nabla_\theta Q_1^{\theta,\textrm{FPG}}}(s,a)+\widehat{Q}_1^{\theta,\textrm{FPG}}(s,a)\nabla_\theta\log\pi_{\theta,1}(a\vert s)\bigg)\mathrm{d}s\mathrm{d}a.
		\end{align*}
	\end{algorithmic}
\end{algorithm}

\subsection{Equivalence to a Model-based Plug-in Estimator}
Next we show that the FPG estimator is equivalent to a model-based plugin estimator. 
Define the model-based reward estimate $\widehat{r}$ and transition operator estimate $\widehat{\mathcal{P}}_\theta$ as followings: for any (possibly vector-valued) function $f$ on $\mathcal{S}\times\mathcal{A}$ and $h\in[H]$,
\begin{align*}
    \widehat{r}_h&:=\mathop{\arg\min}_{f^\prime\in\mathcal{F}}\left[\sum_{k=1}^K\left(f^\prime\left(s^{(k)}_h,a^{(k)}_h\right)-r^{(k)}_h\right)^2+\lambda\rho(f^\prime)\right],\\
    \widehat{\mathcal{P}}_{\theta,h}f&:=\mathop{\arg\min}_{f^\prime\in\mathcal{F}}\left[\sum_{k=1}^K\bigg(f^\prime\left(s_h^{(k)},a_h^{(k)}\right)-\int_{\mathcal{A}}\pi_{\theta,h+1}\left(a^\prime\left\vert s_h^{(k)}\right.\right)f\left(s_{h+1}^{(k)},a^\prime\right)\mathrm{d}a^\prime\bigg)^2+\lambda\rho(f^\prime)\right].
\end{align*}
Plugging $\widehat{\mathcal{P}}_\theta$ and $\widehat{r}$ into \eqref{bel}, we may calculate the policy gradient associated with the estimated model. Let $\widehat{Q}_{H+1}^{\theta,\textrm{MB}}=\widehat{\nabla_\theta^j Q_{H+1}^{\theta,\textrm{MB}}}=0, j\in[m]$. For $h=H,H-1,\ldots,1$, let
\begin{align*}
    \widehat{Q}_h^{\theta,\textrm{MB}}&=\widehat{r}_h+\widehat{\mathcal{P}}_{\theta,h}\widehat{Q}_{h+1}^{\theta,\textrm{MB}},\\
    \widehat{\nabla_\theta^j Q_h^{\theta,\textrm{MB}}}&= \widehat{\mathcal{P}}_{\theta,h}\left(\left(\nabla_\theta^j\log\Pi_{\theta,h+1}\right)\widehat{Q}_{h+1}^{\theta,\textrm{MB}}+\widehat{\nabla_\theta^j Q^{\theta,\textrm{MB}}_{h+1}}\right).
\end{align*}
Then the model-based gradient estimator is 
\begin{align*}
    \widehat{\nabla_\theta v_\theta^{\textrm{MB}}}=&\int_{\mathcal{S}\times\mathcal{A}} \xi(s)\pi_{\theta,1}(a\vert s)\bigg(\widehat{\nabla_\theta Q_1^{\theta,\textrm{MB}}}(s,a)+\widehat{Q}_1^{\theta,\textrm{MB}}(s,a)\nabla_\theta\log\pi_{\theta,1}(a\vert s)\bigg)\mathrm{d}s\mathrm{d}a.
\end{align*}
Note that the model-based plug-in approach makes intuitive sense, but is intractable to implement.

Remarkably, we show that the model-based plug-in estimator $\widehat{\nabla_\theta v_\theta^{\textrm{MB}}}$ is essentially equivalent to the fitted PG estimator, when $\mathcal{F}$ is the class of linear functions.

\begin{proposition}
\label{equiv_mb}
When $\mathcal{F} = \{\phi(\cdot,\cdot)^\top w\vert w\in\mathbb{R}^d\}$ and the regulator $\rho$ is chosen to be $\rho(\phi^\top w) = \Vert w\Vert^2$, we have 
\begin{itemize}
    \item {\small$\widehat{Q}_h^{\theta,\mathrm{FPG}} = \widehat{Q}_h^{\theta,\mathrm{MB}}, \widehat{\nabla_\theta Q_h^{\theta,\mathrm{FPG}}}=\widehat{\nabla_\theta Q_h^{\theta,\mathrm{MB}}}, \forall h\in[H]$;}
    \item {\small$\widehat{\nabla_\theta v_\theta^{\mathrm{FPG}}}=\widehat{\nabla_\theta v_\theta^{\mathrm{MB}}}$.}
\end{itemize}
\end{proposition}
In the remainder, we focus on linear $\mathcal{F}$ and let $\rho(\phi^\top w) = \Vert w\Vert^2$. We will omit the superscript $\textrm{FPG}$ and $\textrm{MB}$ , and simply denote $\widehat{Q}_h^\theta, \widehat{\nabla_\theta Q_h^\theta}, \widehat{\nabla_\theta v_\theta}$ as our estimators. 

\subsection{FPG with Linear Function Approximation}
Define the empirical covariance matrix: For $h\in[H]$, 
\begin{align*}
    \widehat{\Sigma}_h=\frac{1}{K}\left(\lambda I_d+\sum_{k=1}^K\phi\left(s_h^{(k)}, a_h^{(k)}\right)\phi\left(s_h^{(k)}, a_h^{(k)}\right)^\top\right),
\end{align*}
where $I_d\in\mathbb{R}^{d\times d}$ is the identity matrix. In the case of linear function class, one could write down the expression of $\widehat{r}$ and $\widehat{\mathcal{P}}_\theta$ explicitly:
\begin{align}
    \label{wr}
    &\widehat{r}_h(\cdot,\cdot)=\phi(\cdot,\cdot)^\top\widehat{\Sigma}_h^{-1}\frac{1}{K}\sum_{k=1}^K\phi\left(s_h^{(k)},a_h^{(k)}\right)r_h^{(k)}=:\phi(\cdot,\cdot)^\top\widehat{w}_{r,h},\\
    &\left(\widehat{\mathcal{P}}_{\theta,h} f\right)(\cdot,\cdot)=\phi(\cdot,\cdot)^\top\widehat{\Sigma}^{-1}_h\frac{1}{K}\sum_{k=1}^K\phi\left(s_h^{(k)},a_h^{(k)}\right)\int_{\mathcal{A}}\pi_{\theta,h+1}\left(a^\prime\left\vert s_{h+1}^{(k)}\right.\right)f\left(s_{h+1}^{(k)}, a^\prime\right)\mathrm{d}a^\prime.\nonumber
\end{align}
For $f(\cdot,\cdot)=\phi(\cdot,\cdot)^\top w\in\mathcal{F}$, the above become concise closed forms: 
\begin{align*}
    \left(\widehat{\mathcal{P}}_{\theta,h} f\right)(\cdot,\cdot)&=\phi(\cdot,\cdot)^\top\widehat{M}_{\theta,h} w,\\
    \left(\widehat{\mathcal{P}}_{\theta,h}\left(\nabla_\theta\log\Pi_{\theta,h+1}\right)f\right)(\cdot,\cdot)&=\phi(\cdot,\cdot)^\top\widehat{\nabla_\theta M_{\theta,h}}\left(I_m \otimes w\right).
\end{align*}
where the notation $\otimes$ is used to denote the Kronecker product between two matrices, $\widehat{M_{\theta,h}}\in\mathbb{R}^{d\times d}, \widehat{\nabla_\theta M_{\theta,h}}\in\mathbb{R}^{d\times md}$ are defined by
\begin{align}
    \label{M}
	&\widehat{M}_{\theta,h}:=\widehat{\Sigma}_h^{-1}\frac{1}{K}\sum_{k=1}^K\phi\left(s_h^{(k)}, a_h^{(k)}\right)\int_{\mathcal{A}}\pi_{\theta,h+1}\left(a^\prime\left\vert s_{h+1}^{(k)}\right.\right)\phi\left(s_{h+1}^{(k)}, a^\prime\right)^\top\mathrm{d}a^\prime,\\
	\label{gM}
	&\widehat{\nabla_\theta M_{\theta,h}}:=\nabla_\theta\widehat{M}_{\theta,h}=\widehat{\Sigma}_h^{-1}\frac{1}{K}\sum_{k=1}^K\phi\left(s_h^{(k)},a_h^{(k)}\right)\int_{\mathcal{A}}\phi\left(s_{h+1}^{(k)}, a^\prime\right)^\top\left(\nabla_\theta\pi_{\theta,h+1}\left(a^\prime\left\vert s_{h+1}^{(k)}\right.\right)\otimes I_d\right)\mathrm{d}a^\prime.
\end{align}
In this way, one can easily compute $\widehat{Q}_h^\theta$ and $\widehat{\nabla_\theta Q_h^\theta}$ in a matrix recursive form, which we illustrate in Algorithm \ref{alg2}.
\begin{algorithm}[htb!]
\caption{FPG Estimation with Linear Approximation}
\label{alg2}
	\begin{algorithmic}[1] 
		\State\textbf{Input:} Dataset $\mathcal{D}$, target policy $\pi_\theta$, initial state distribution $\xi$.
		\State Calculate $\widehat{w}_{r,h}, \widehat{M}_{\theta,h}, \widehat{\nabla_\theta M_{\theta,h}},h\in[H]$ according to \eqref{wr}, \eqref{M}, \eqref{gM}
		\State Let $\widehat{w}_{H+1}^\theta=\bm{0}_d$ and $\widehat{W}_{H+1}^\theta=\bm{0}_{d\times m}.$
		\For{$h=H,H-1,\ldots,1$}   
		\State Set $\widehat{w}_h^{\theta}=\widehat{w}_{r,h}+\widehat{M}_{\theta,h}\widehat{w}_{h+1}^\theta,\widehat{W}^\theta_h=\widehat{\nabla_\theta M_{\theta,h}}(I_m\otimes\widehat{w}_{h+1}^\theta)+\widehat{M_{\theta,h}}\widehat{W}^\theta_{h+1}.$
		\EndFor
		\State Return $\widehat{\nabla_\theta v_\theta}=\int_{\mathcal{S}\times\mathcal{A}}\xi(s)\pi_{\theta,1}(a\vert s)\phi(s,a)^\top\big(\widehat{W}_1^\theta+\widehat{w}_1^\theta\nabla_\theta\log\pi_{\theta,1}(a\vert s)\big)\mathrm{d}s \mathrm{d}a.$
	\end{algorithmic}
\end{algorithm}
\paragraph{Runtime Complexity} Algorithm \ref{alg2} is computationally very efficient. Suppose that caculating integral against action distribution takes time $O(1)$. In Algorithm 2, the calculation of $\widehat{w}_{r,h}, \widehat{M}_{\theta,h}$ and $\widehat{\nabla_\theta M_{\theta,h}}$ require at most $O(KHmd^2)$ numeric operations. The recursive function fitting steps at line 3-5 require at most $O(Hmd^2)$ numeric operations. Thus the total runtime is only $O(KHmd^2)$. 

\section{Main Results}
In this section we study the statistical properties of the FPG estimator with linear function approximation. Define the population covariance matrix as $\Sigma_h:=\mathbb{E}\left[\phi\left(s^{(1)}_h,a^{(1)}_h\right)\phi\left(s^{(1)}_h,a^{(1)}_h\right)^\top\right],\ h\in[H]$, where $\mathbb{E}$ represents the expectation over the data generating distribution by the behavior policy.
\begin{assumption}[Boundedness Conditions]\label{Boundedness_Conditions}
	Assume for any $h\in[H]$, $\Sigma_h$ is invertible. There exist absolute constants $C_1, G$ such that for any $h\in[H]$ and $(s,a)\in \mathcal{S}\times\mathcal{A},j\in[m]$, we have $\phi(s,a)^\top\Sigma^{-1}_h\phi(s,a)\leq C_1 d$ and $\left\vert\nabla_\theta^j\log\pi_{\theta,h}(a\vert s)\right\vert\leq G$.
\end{assumption}
Assumption \ref{Boundedness_Conditions} requires the data generating distribution to have a full-rank covariance matrix, effectively covering all $d$ directions in the feature space. Note that this is a much weaker condition compared to the uniform coverage condition ($\max_{s,a,h}\frac{\mu_h^\theta(s,a)}{\bar{\mu}_h(s,a)})<\infty$) made in prior works \cite{kallus2020statistically}, which requires coverage on all $(s,a)$ pairs. Define $\nu^\theta_h:=\mathbb{E}^{\pi_\theta}[\phi(s_h,a_h)\vert s_1\sim\xi]$ and $\Sigma_{\theta,h}:=\mathbb{E}^{\pi_\theta}[\phi(s_h,a_h)\phi(s_h,a_h)^\top\vert s_1\sim\xi]$. 

\subsection{Finite-Sample Variance-Aware Error Bound}
Let us first consider finite-sample analysis of our estimator. We present a variance-aware error bound. Denote $\phi_{\theta,h}(s):=\mathbb{E}^{\pi_\theta}[\phi(s^\prime,a^\prime)\vert s^\prime=s]$, $\varepsilon_{h,k}^\theta:=Q_h^\theta\left(s_h^{(k)},a_h^{(k)}\right)-r_h^{(k)}-\mathbb{E}^{\pi_\theta}[Q_{h+1}^\theta\left(s^\prime,a^\prime\right)\vert s^\prime=s_{h+1}^{(k)}]$, and $\Lambda_\theta:=\sum_{h=1}^H\textrm{Cov}\left[\nabla_\theta\left(\varepsilon^\theta_{h,1}\phi\left(s_h^{(1)},a_h^{(1)}\right)^\top\Sigma^{-1}_h\nu_h^\theta\right)\right]$.

\begin{theorem}[Finite Sample Guarantee] 
    \label{thm2_var}
    For any $t\in\mathbb{R}^m$, when $K\geq 36\kappa_1(4+\kappa_2+\kappa_3)^2C_1dH^2\log\frac{8dmH}{\delta}$ and $\lambda\leq C_1d\min_{h\in[H]}\sigma_{\min}(\Sigma_h)\log\frac{8dmH}{\delta}$, with probability $1-\delta$, we have,
    \begin{align*}
        \vert\langle t, \widehat{\nabla_\theta v_\theta}-\nabla_\theta v_\theta\rangle \vert\leq  \sqrt{\frac{2t^\top\Lambda_\theta t}{K}\cdot \log\frac{8}{\delta}}+\frac{C_\theta\Vert t\Vert\log\frac{72mdH}{\delta}}{K},
    \end{align*}
    where $C_\theta=240C_1d\sqrt{m}H^3\kappa_1(5+\kappa_2+\kappa_3)(\max_{j\in[m]}\Vert\Sigma_{\theta,1}^{-\frac{1}{2}}\nabla_\theta^j\nu^\theta_1\Vert+HG\Vert\Sigma_{\theta,1}^{-\frac{1}{2}}\nu^\theta_1\Vert)$ and
    \begin{align*}
        &\kappa_1=\max_{h\in[H]}\frac{\sigma_{\max}\left(\Sigma_h^{-\frac{1}{2}}\Sigma_{\theta,h}\Sigma_h^{-\frac{1}{2}}\right)}{\sigma_{\min}\left(\Sigma_{h+1}^{-\frac{1}{2}}\Sigma_{\theta,h+1}\Sigma_{h+1}^{-\frac{1}{2}}\right)\wedge 1},\quad\kappa_2=\max_{h\in[H]}\left\Vert\Sigma_{h+1}^{-\frac{1}{2}}\mathbb{E}\left[\phi_{\theta,h+1}\left(s_{h+1}^{(1)}\right)\phi_{\theta,h+1}\left(s_{h+1}^{(1)}\right)^\top\right]\Sigma_{h+1}^{-\frac{1}{2}}\right\Vert^{\frac{1}{2}},\\
        &\kappa_3=\frac{1}{G}\max_{j\in[m],h\in[H]}\left\Vert\Sigma_{h+1}^{-\frac{1}{2}}\mathbb{E}\left[\left(\nabla_\theta^j \phi_{\theta,h+1}\left(s_{h+1}^{(1)}\right)\right)\left(\nabla_\theta^j \phi_{\theta,h+1}\left(s_{h+1}^{(1)}\right)\right)^\top\right]\Sigma_{h+1}^{-\frac{1}{2}}\right\Vert^{\frac{1}{2}}.
    \end{align*}
\end{theorem}
Theorem \ref{thm2_var} shows that the finite-sample FPG error is largely determined by $\sqrt{\frac{t^T \Lambda_{\theta}t}{K}}.$ Here $\Lambda_{\theta}$ gives a precise characterization of the error's covariance. 

\subsection{Worst-Case Error Bound and Distribution Shift}
Next we derive a worst-case error bound that depends only on the distribution shift but not on reward/variance properties. The following theorem provides a worst-case guarantee under arbitrary choice of the reward function. 
\begin{theorem}[Finite Sample Guarantee - Reward Free]
    \label{thm2}
    Let the conditions in Theorem \ref{thm2_var} hold, with probability $1-\delta$, we have $\forall j\in[m]$, 
    \begin{align*}
        &\left\vert\widehat{\nabla_\theta^j v_\theta}-\nabla_\theta^j v_\theta\right\vert\leq 4b_\theta\sqrt{\frac{\min\{C_1d,H\}\log\frac{8m}{\delta}}{K}}+\frac{C_\theta\log\frac{72mdH}{\delta}}{K},
    \end{align*}
    where $b_\theta=H^2G\max_{h\in[H]}\Vert\Sigma_h^{-\frac{1}{2}}\nu_h^\theta\Vert+H\max_{h\in[H]}\Vert\Sigma_h^{-\frac{1}{2}}\nabla_\theta^j\nu_h^\theta\Vert$ and $C_\theta$ is the same as that in Theorem \ref{thm2_var}. If we in addition have $\phi(s^\prime,a^\prime)^\top\Sigma_{h}^{-1}\phi(s,a)\geq 0,\forall (s,a),(s^\prime,a^\prime)\in\mathcal{S}\times\mathcal{A},h\in[H]$, we have
    \begin{align*}
        \left\vert\widehat{\nabla_\theta^j v_\theta}-\nabla_\theta^j v_\theta\right\vert\leq 4H^2G\sqrt{\frac{\min\{C_1d,H\}\log\frac{8m}{\delta}}{K}}\max_{h\in[H]}\left\Vert\Sigma_h^{-\frac{1}{2}}\nu_h^\theta\right\Vert+\frac{2C_\theta\log\frac{72mdH}{\delta}}{K}, \quad\forall j\in[m].
    \end{align*}
\end{theorem}
The complete proofs of Theorem \ref{thm2_var} and Theorem \ref{thm2} are deferred to Appendix \ref{pfthm2_var} and \ref{pfthm2}. To further simplify the expression in Theorem \ref{thm2}, we define a variant of $\chi^2$-divergence restricted to the family $\mathcal{F}$: for any two groups of probability distributions $p_1=\{p_{1,h}\}_{h=1}^H, p_2=\{p_{2,h}\}_{h=1}^H$, define
\begin{align*}
    \chi^2_{\mathcal{F}}(p_1,p_2):&=\max_{h\in[H]}\sup_{f\in\mathcal{F}}\frac{\mathbb{E}_{p_{1,h}}\left[f(x)\right]^2}{\mathbb{E}_{p_{2,h}}\left[f(x)^2\right]}-1=\max_{h\in[H]}\nu_{p_{1,h}}^\top\Sigma^{-1}_{p_{2,h}}\nu_{p_{1,h}}-1,
\end{align*}
where $\nu_{p}=\mathbb{E}_p[\phi(s,a)],\ \Sigma_p=\mathbb{E}_p[\phi(s,a)\phi(s,a)^\top]$. Let $\bar{\mu}=\{\bar{\mu}_h\}_{h=1}^H$ be the occupancy distribution of observation $(s_h^{(1)},a_h^{(1)})$. Let $\mu^\theta=\{\mu^\theta_h\}_{h=1}^H$ be the occupancy distribution of $(s_h,a_h)$ under policy $\pi_\theta$. When we have $\phi(s^\prime,a^\prime)^\top\Sigma_{h}^{-1}\phi(s,a)\geq 0,\ \forall (s,a),(s^\prime,a^\prime)\in\mathcal{S}\times\mathcal{A},h\in[H]$, the result of Theorem \ref{thm2} implies $\forall j\in[m]$, 
\begin{align*}
    \vert\widehat{\nabla_\theta^j v_\theta}-\nabla_\theta^j v_\theta\vert\leq 4H^2G\sqrt{\frac{\min\{C_1d,H\}\log\frac{8m}{\delta}}{K}(1+\chi_{\mathcal{F}}^2(\mu^\theta,\bar{\mu}))}+\tilde{O}\left(\frac{1}{K}\right).
\end{align*}
The result of Theorem \ref{thm2} matches the asymptotic bound provided in \cite{kallus2020statistically}, but holds in finite sample regime and requires less stringent conditions. 
\paragraph{The case of tabular MDP.} In the tabular case, the condition $\phi(s,a)^\top\Sigma_h^{-1}\phi(s^\prime,a^\prime)\geq 0$ automatically holds. 
Furthermore, we have the following simplified guarantee:
\begin{theorem}[Upper bound in tabular case]
    \label{thm_tabular}
    \label{tabular}
    In the tabular case with $\mathcal{F}=\mathbb{R}^{\mathcal{S}\times\mathcal{A}}$, if $K$ is sufficiently large and $\lambda=0$, then with probability at least $1-\delta$, $\forall j\in[m]$
    \begin{align*}
        &\vert\widehat{\nabla_\theta^j v_\theta}-\nabla_\theta^j v_\theta\vert\leq 4H^2G\sqrt{\frac{\log\frac{8m}{\delta}}{K}}\sqrt{\min\left\{\max_{h\in[H],s\in\mathcal{S},a\in\mathcal{A}}\frac{\mu^\theta_h(s,a)}{\bar{\mu}_h(s,a)}, C_1d\max_{h\in[H]}\mathbb{E}^{\pi_\theta}\left[\frac{\mu^\theta_h(s_h,a_h)}{\bar{\mu}_h(s_h,a_h)}\right]\right\}}+\tilde{O}\left(\frac{1}{K}\right).
    \end{align*}
\end{theorem}

\subsection{Asymptotic Normality and Cramer-Rao Lower Bound}
Next we show that FPG is an asymptotically normal and efficient estimator. 
\begin{theorem}[Asymptotic Normality]
    \label{thm1}
    The FPG estimator given by Algorithm \ref{alg2} is asymptotically normal:
    \begin{align*}
        \sqrt{K}\left(\widehat{\nabla_\theta v_\theta}-\nabla_\theta v_\theta\right)\stackrel{d}{\rightarrow}\mathcal{N}(0,\Lambda_\theta).
    \end{align*}
\end{theorem}
An obvious corollary of Theorem \ref{thm1} is that for any $t \in \mathbb{R}^m$, 
\begin{equation*}
    \sqrt{K}\left\langle t,  \widehat{\nabla_{\theta} v_{\theta}} -\nabla_{\theta} v_{\theta} \right\rangle \stackrel{d}{\rightarrow} \mathcal{N}\left(0, t^{\top} \Lambda_{\theta} t\right).
\end{equation*}
An asymptotically efficient estimator has the minimal variance among all the unbiased estimators. The next theorem states the Cramer-Rao lower bound for PG estimation. 
\begin{theorem} [Cramer-Rao Lower Bound]
    \label{thm4}
    Let Assumption \ref{fclass} hold. For any vector $t\in\mathbb{R}^m$, the variance of any unbiased estimator for $\langle t,\nabla_{\theta}v_{\theta}\rangle$ is lower bounded by $\frac{1}{K}t^\top \Lambda_\theta t$.
\end{theorem}
The proofs of Theorem \ref{thm1} and \ref{thm4} are deferred to Appendix \ref{pfthm1}, \ref{pfthm4}. Theorem \ref{thm4} along with Theorem \ref{thm1} show that the FQI estimator is statistically optimal.

\subsection{FPG for Policy Optimization}
Lastly we briefly consider the use of FPG for off-policy policy optimization. Assume in the ideal setting we can reliably estimate the PG for all policies, obtaining $\widehat{\nabla_{\theta}v_{\theta}}$ for all $\theta\in\Theta$. Then we can simply set $\widehat{\nabla_{\theta}v_{\theta}} = 0$, identify all the stationary solutions, and pick the best one. For MDP with Lipschitz continuous policy gradients, we show that a policy with $\widehat{\nabla_{\theta}v_{\theta}} = 0$ would be nearly stationary/optimal. 
\begin{assumption}\label{Lipschitz_cont}
    Suppose the parameter space $\Theta$ is bounded and the policy gradient is $L$-Lipschitz continuous and $\chi^2_{\mathcal{F}}$ is $L^\prime$-Lipschitz continuous, i.e., 
    \begin{align*}
        \Vert\nabla_{\theta_1}v_{\theta_1}-\nabla_{\theta_2}v_{\theta_2}\Vert\leq L\Vert\theta_1-\theta_2\Vert,\quad\left\vert\chi_{\mathcal{F}}^2\left(\mu^{\theta_1},\bar{\mu}\right)-\chi_{\mathcal{F}}^2\left(\mu^{\theta_2},\bar{\mu}\right)\right\vert\leq L^\prime\Vert\theta_1-\theta_2\Vert,\quad\forall\theta_1,\theta_2\in\Theta.
    \end{align*}
\end{assumption}
\begin{proposition}
    \label{union_bd}
    Suppose assumption \ref{Lipschitz_cont} and the condition of Theorem \ref{thm2} hold. When $K$ is sufficiently large, we have with probability at least $1-\delta$, 
    \begin{align*}
        \Vert\nabla_\theta v_\theta-\widehat{\nabla_\theta v_\theta}\Vert\leq 64H^2Gm\sqrt{\min\{C_1d,H\}}\sqrt{1+\chi^2_{\mathcal{F}}(\mu^{\theta},\bar{\mu})}\sqrt{\frac{\log\frac{24DKLL^\prime}{\delta HG}}{K}},\quad\forall\theta\in\Theta.
    \end{align*}
    where $D$ is the diameter of $\Theta$. In addition, if the Polyak-Łojasiewicz condition holds, i.e., there exists a constant $c > 0$ such that for any $\theta \in \Theta$, $\frac{1}{2}\Vert\nabla_\theta v_\theta\Vert^2 \geq c(v_{\theta^*}-v_\theta)$, then for any $\widehat{\theta}$ such that $\widehat{\nabla_\theta v_{\widehat{\theta}}}=0$, we have $v_{\theta^*}-v_{\widehat{\theta}}\leq \tilde{\mathcal{O}}\left(\frac{m^2H^4\min\{C_1d,H\} G^2}{K}\right)$. 
\end{proposition}
In general, Proposition \ref{union_bd} implies a $O(1/\varepsilon^2)$ sample complexity for finding $\varepsilon$-stationary policies. This off-policy sample efficiency is remarkably better than the best know $O(1/\varepsilon^3)$ on-policy sample efficiency obtained by variance-reduced PG algorithm \cite{zhang2021convergence}, as long as distribution shift is uniformly bounded. This improvement is due to that FPG makes full usage of data to evaluate PG at every $\theta$. We remark that the discussion in this section is more of a stylish observation than a practically sound algorithm. How to incorporate FPG into policy gradient algorithms is an important future direction.

\section{Experiments}
We empirically evaluate the performance of FPG using the OpenAI gym FrozenLake and  CliffWalking environment. For FrozenLake, we use softmax tabular policy parameterization and $H=100$. For CliffWalking, we use softmax on top of a two-layer ReLU network for policy parameterization. We pick the target policy to be a fixed near-optimal policy, and test using dataset generated from different behavior policies. For comparison, we compute the true gradient using the policy gradient theorem and on-policy Monte Carlo simulation.

\paragraph{FPG's data efficiency}
\label{sec:exp-K}
Choosing the behavior policy to be the $\varepsilon$-greedy modification of the target policy for $\varepsilon=0.1$, we generate datasets with varying sizes and evaluate the FPG's estimation error on two metrics: the cosine angle between the true policy gradient and the FPG estimator, and the relative estimation error in $\ell_2$-norm. The closer the cosine is to $1$ and the smaller the relative norm error is, the better the estimated policy gradient is. \textbf{Figure \ref{fig:FrozenLake_1}} shows that FPG gives good estimate even when the data is rather small. The FPG estimate converges to the true gradient with rather moderate variance. In comparison, importance sampling (IS) converges much slower and incurs substantially larger variance. 

\begin{figure}[!t]
 \centering
 \includegraphics[width=0.4\linewidth]{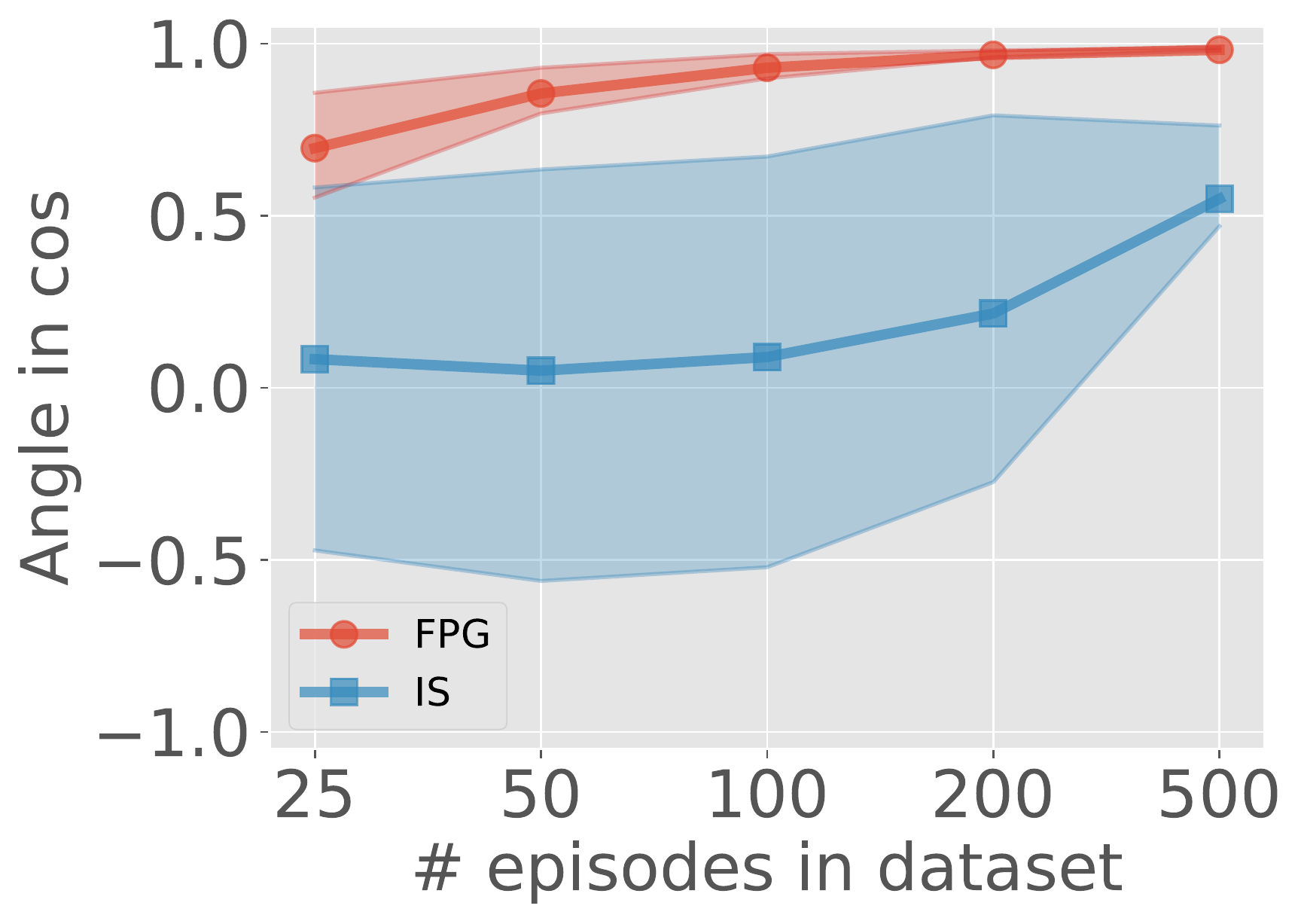}
 \includegraphics[width=0.4\linewidth]{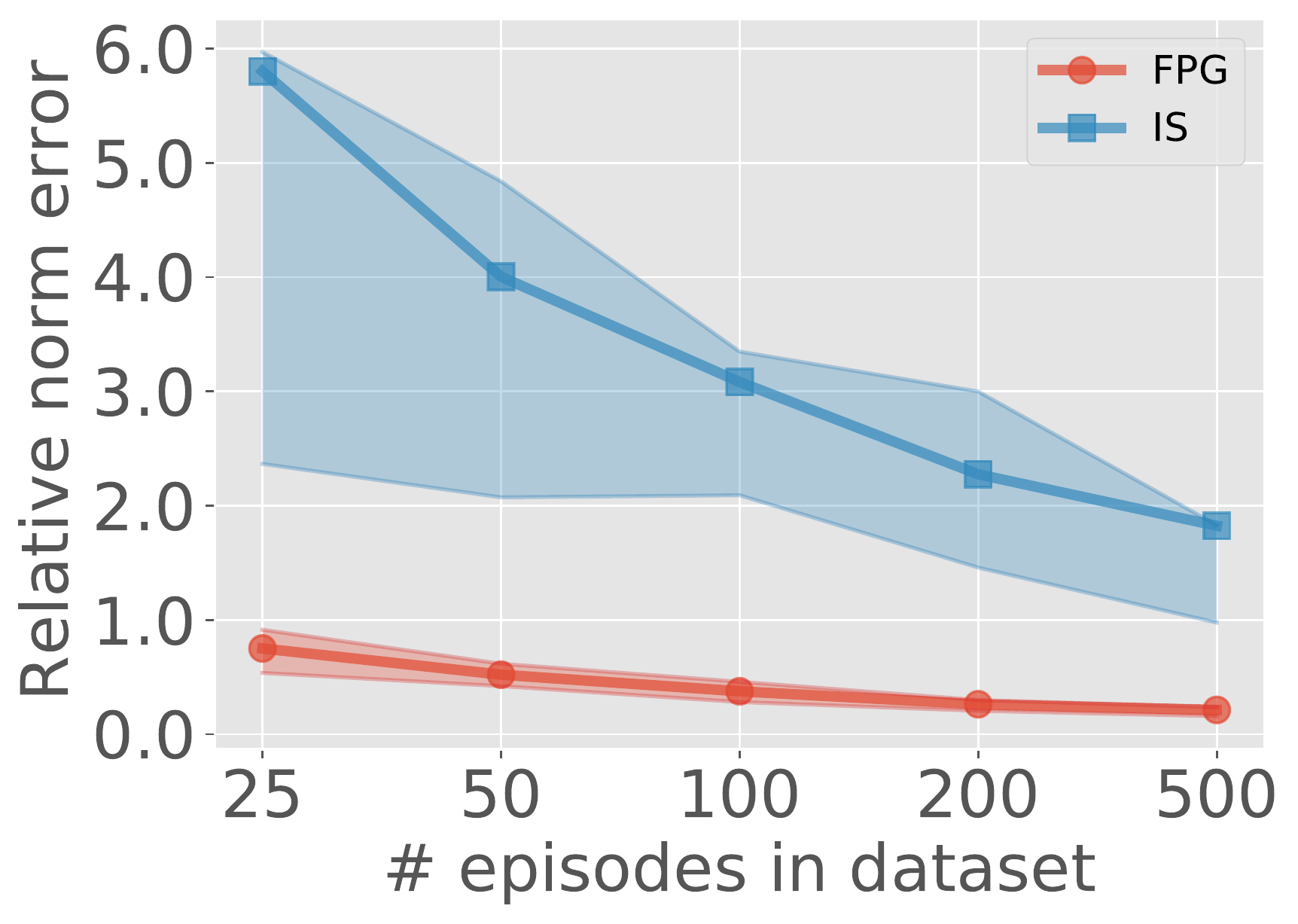}
\caption{\textbf{Sample efficiency of FPG on off-policy data.} The off-policy PG estimation accuracy is evaluated using two metrics: $\cos\angle (\hat{\grad_\theta v_{\theta}}, \grad_\theta v_{\theta})$ and the relative error norm $\frac{\| \hat{\grad_\theta v_{\theta}} -  \grad_\theta v_{\theta}\| }{\|\grad_\theta v_{\theta}\|}$. }
\label{fig:FrozenLake_1}
\end{figure}

\paragraph{The effect of distributional mismatch} 
\label{sec:exp-mismatch}
Next we investigate the effect of distribution shift on off-policy PG estimation. We consider $5$ choices of behavior policies: the target policy, the $\varepsilon$-greedy policies of the target policy, with $\varepsilon=0.1$, $0.3$, $0.5$ and $0.7$. We generate a dataset containing $200$ episodes with each of these behavior policies, run FPG and IS, and evaluate their estimation errors. \textbf{Figure \ref{fig:FrozenLake_2}} shows that larger distribution mismatch leads to larger estimation error in both methods. However, when compared to IS, FPG is significantly more robust to off-policy distribution shift. The accuracy of FPG only degrades slightly with larger distribution mismatch, while IS suffers from exponentially blowing-up error and stops generating reasonable estimates. 

\begin{figure}[!t]
 \centering
  \includegraphics[width=0.4\linewidth]{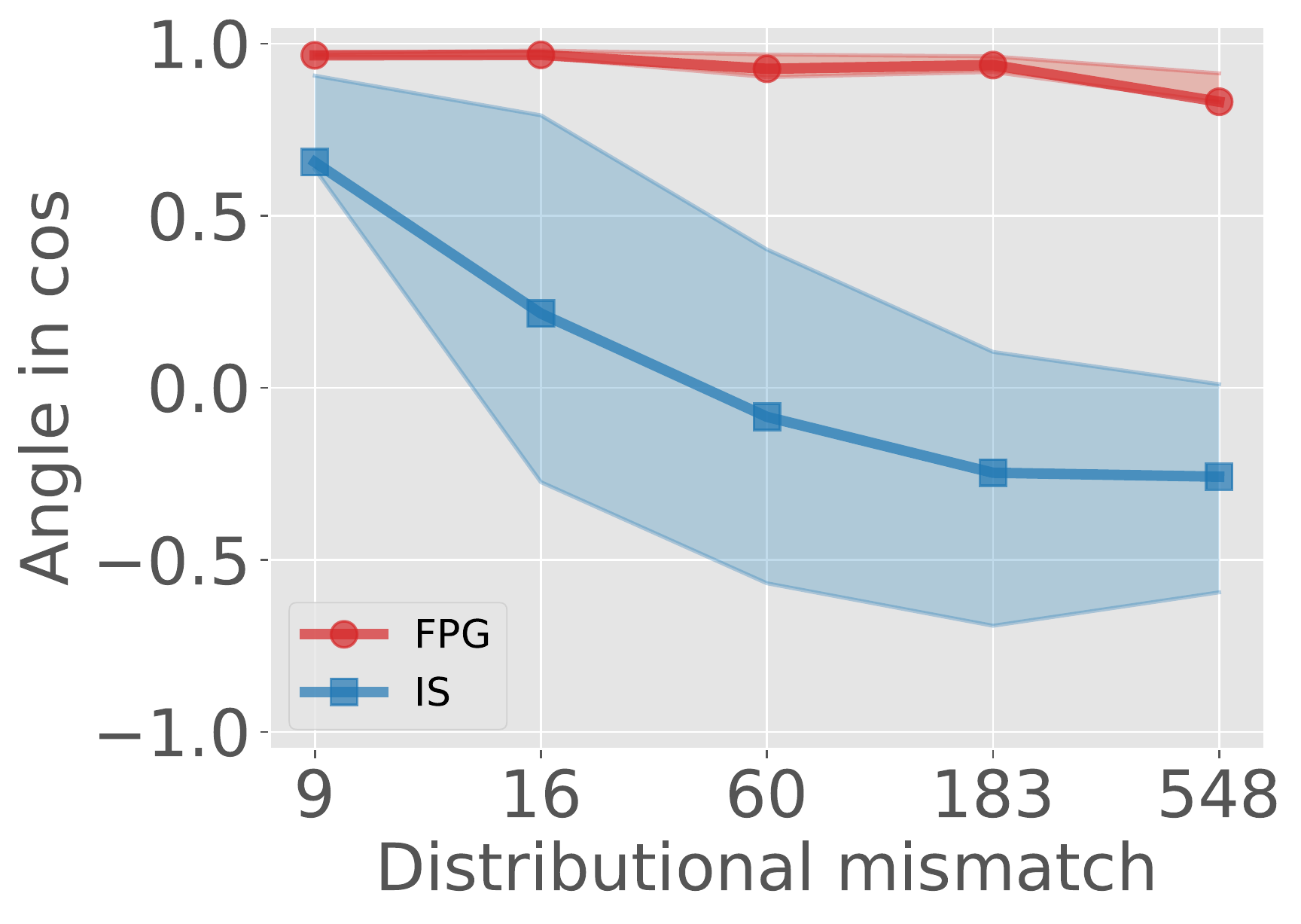}
  \includegraphics[width=0.4\linewidth]{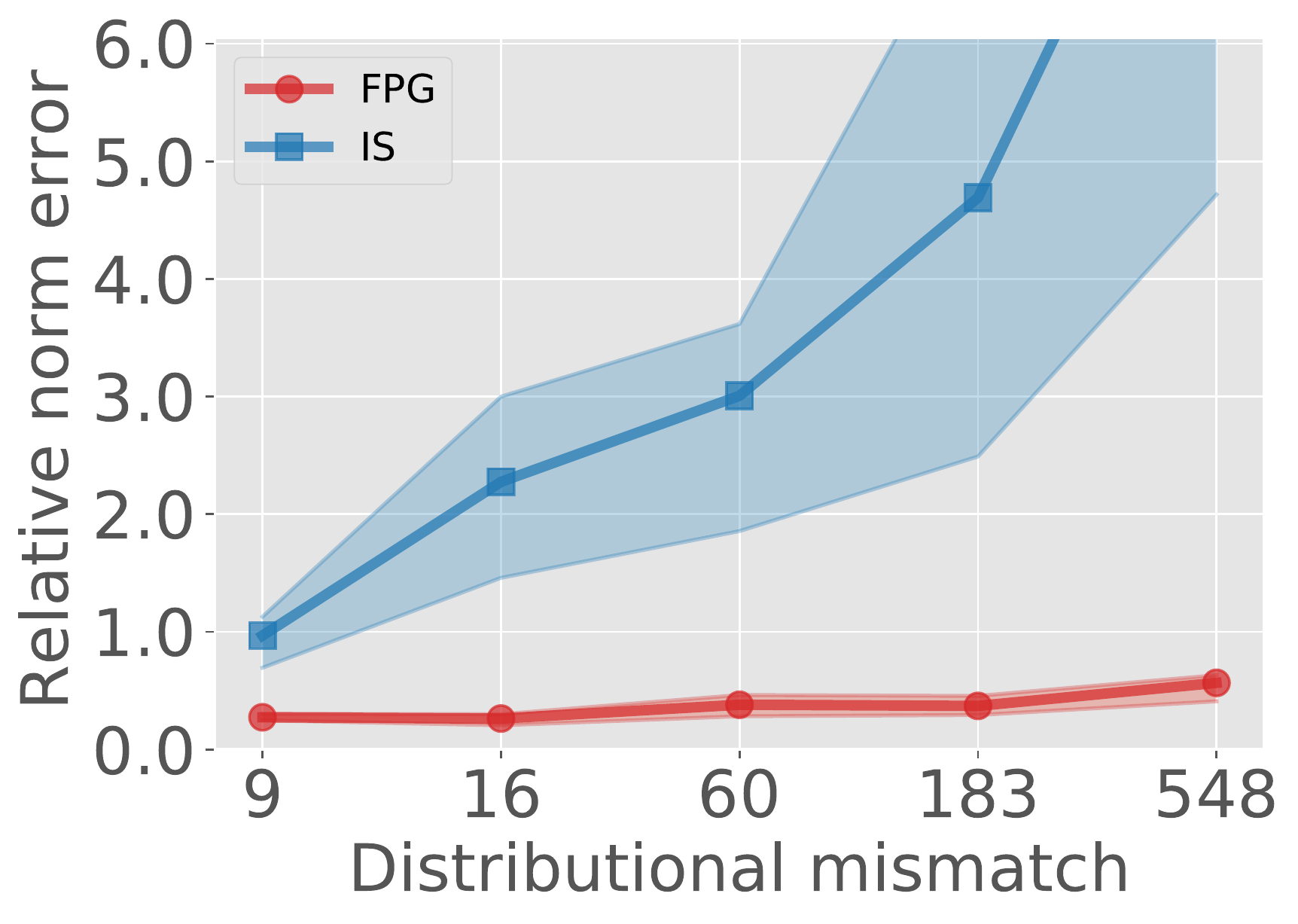}
\caption{\textbf{Tolerance to off-policy distribution shift.} The distributional mismatch is measured by $\hbox{cond}(\bar{\Sigma}^{\frac{1}{2}}\Sigma^{-1}\bar{\Sigma}^{\frac{1}{2}})$, where $\Sigma$ is the data covariance and $\bar{\Sigma}$ is the target policy's occupancy measure.}
\label{fig:FrozenLake_2}
\end{figure}

\paragraph{FPG for policy optimization}
\label{sec:exp-optimization}
We further showcase FPG's applicability to policy optimization. In particular, we test FPG as a gradient estimation module in policy gradient optimization methods. We conduct an experiment using FPG in on-policy REINFORCE and compare it with the vanilla REINFORCE and SVRPG \cite{papini2018stochastic}. All methods are configured to sample $100$ on-policy episodes per iteration. When implementing the FPG-REINFORCE, we take advantage of FPG's off-policy capability and use data from the recent $5$ iterations to improve the gradient estimation accuracy. \textbf{Figure \ref{fig:FrozenLake_3}} shows that such design indeed allows FPG-REINFORCE to converge significantly faster than the two baselines.

\begin{figure}[!t]
    \centering
    \includegraphics[width=0.5\linewidth]{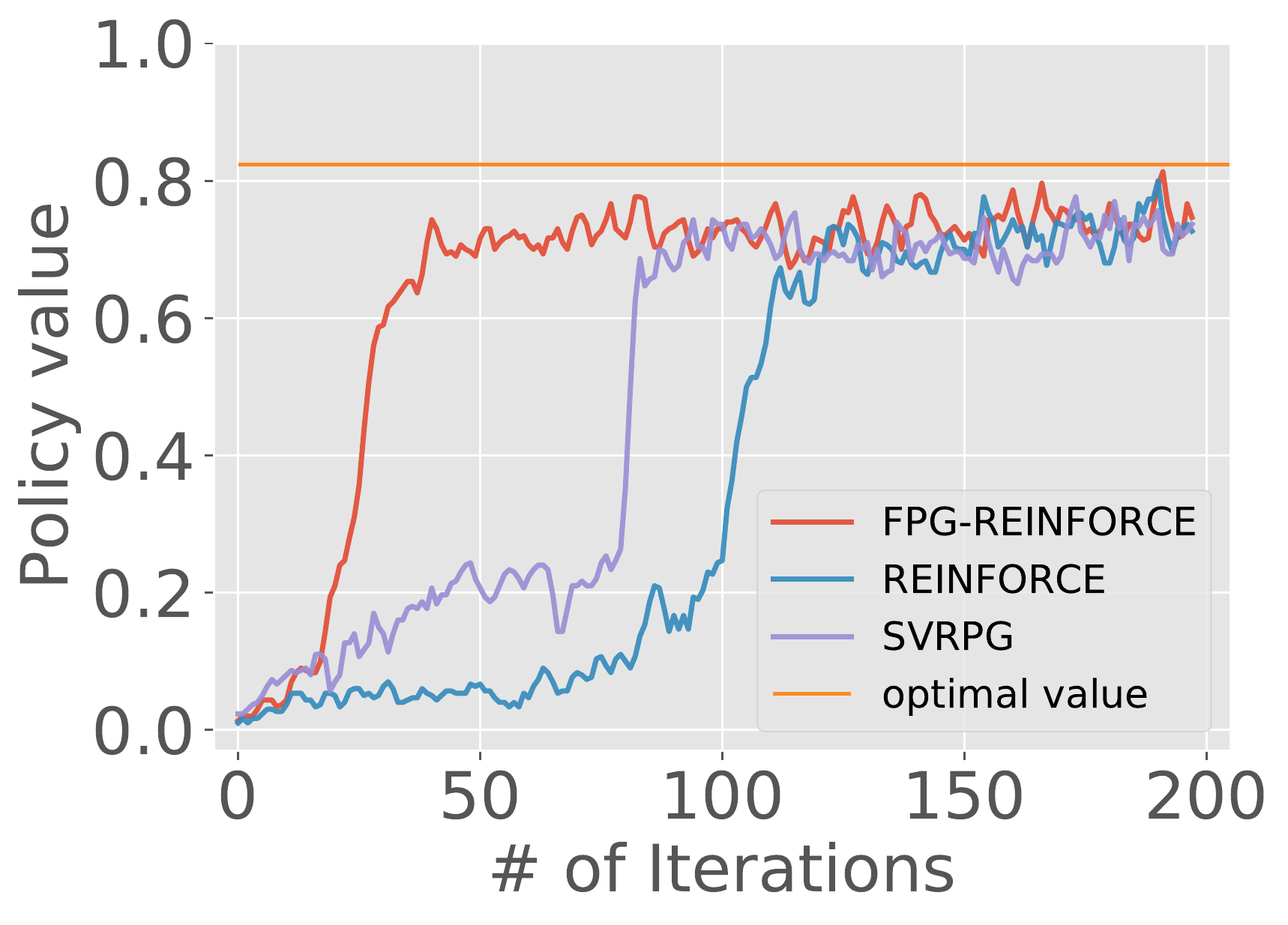}
    \caption{\textbf{FPG for policy optimization.} FPG is used as a module for PG estimates in REINFORCE, compared with other baselines.}
    \label{fig:FrozenLake_3}
\end{figure}

Next we test the use of FPG for offline policy optimization. Let the target policy be the optimal policy of the problem, and let the behavior policy be a $0.3$-greedy variant of the target policy. We generate a dataset consisting of $K=500$ episodes by simulating the behavior policy. For PG estimation, we only use the offline dataset and do not sample for fresh data. Thus, we replace the online policy gradient estimator in REINFORCE with an off-policy one using FPG, and for comparison we also test REINFORCE with an IS estimator. \textbf{Figure \ref{fig:FrozenLake_off}} shows that FPG-REINFORCE converges reasonably fast and approaches the optimal value. However, IS-REINFORCE appears to converge to a highly biased solution, due to that all PGs are estimated using the same small batch dataset and suffer from bias due to distribution shift.

\begin{figure}[!t]
 \centering
 \includegraphics[width=0.5\linewidth]{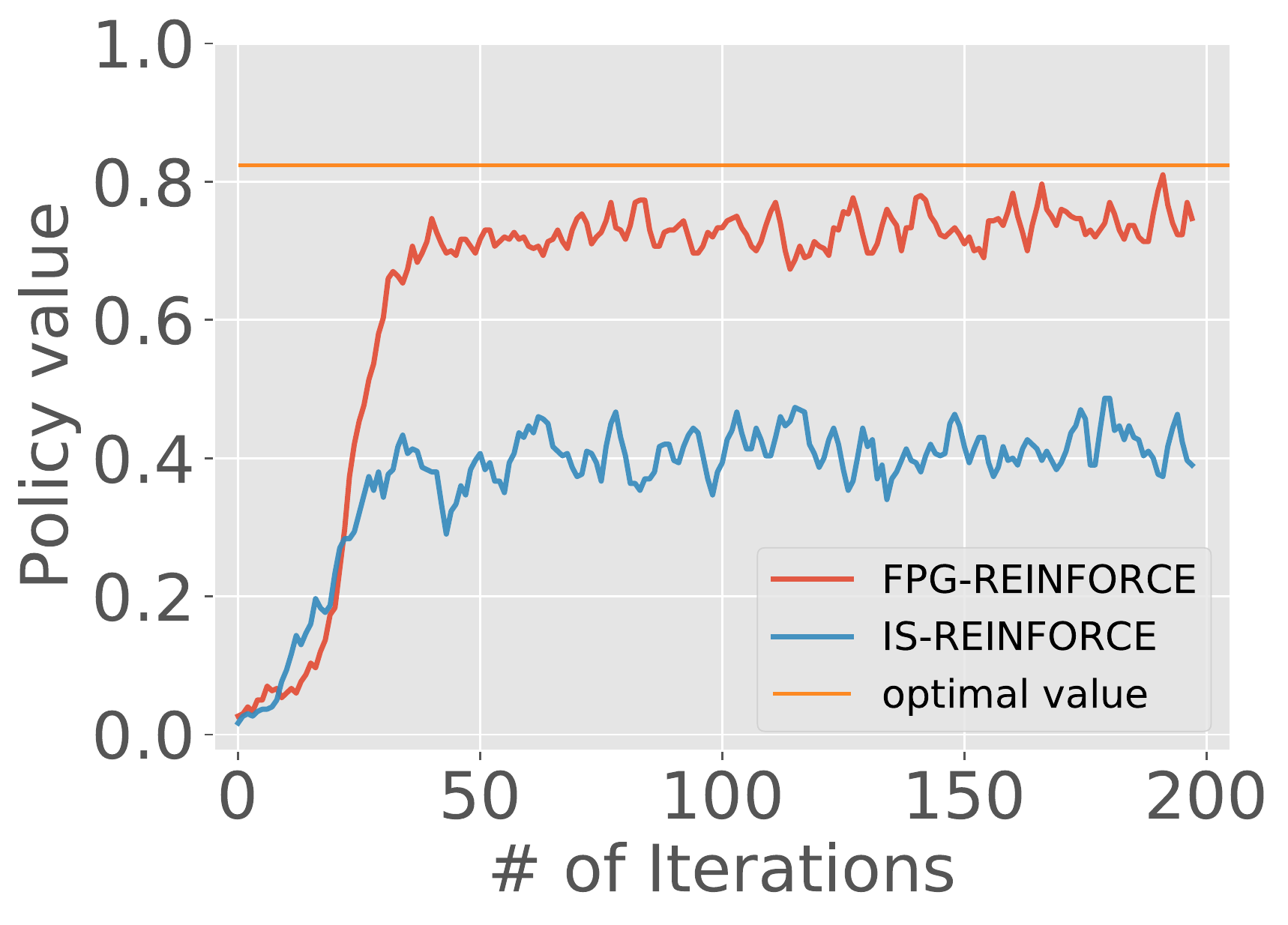}
\caption{\textbf{FPG for offline policy optimization.} FPG is compared with IS as the PG estimator module in offline REINFORCE. }
\label{fig:FrozenLake_off}
\end{figure}

\paragraph{FPG with deep neural network policy parameterization}
Further, we evaluate the efficiency of FPG when using a deep policy network for policy learning in the CliffWalking environment, where $H=100$. Specifically, the environment is modified by adding artificial randomness for stochastic transitions, that is, at each transition, a random action is taken with probability $0.1$. The policy is parameterized with a neural network with one ReLU hidden layer and a softmax layer. 

\def\grad{\nabla}
\def\hat{\widehat}
\begin{figure}[!t]
 \centering
 \includegraphics[width=0.4\linewidth]{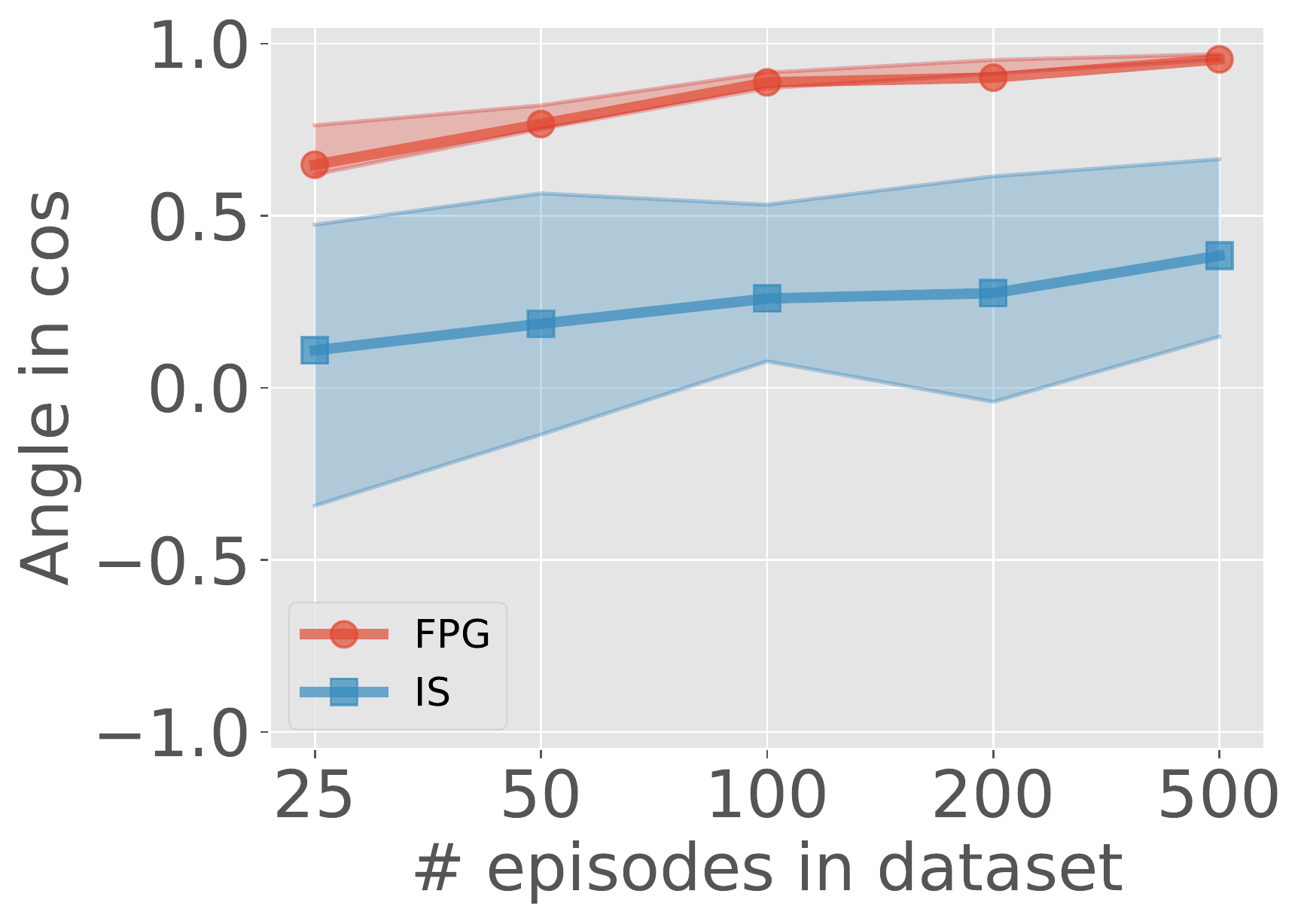}
  \includegraphics[width=0.4\linewidth]{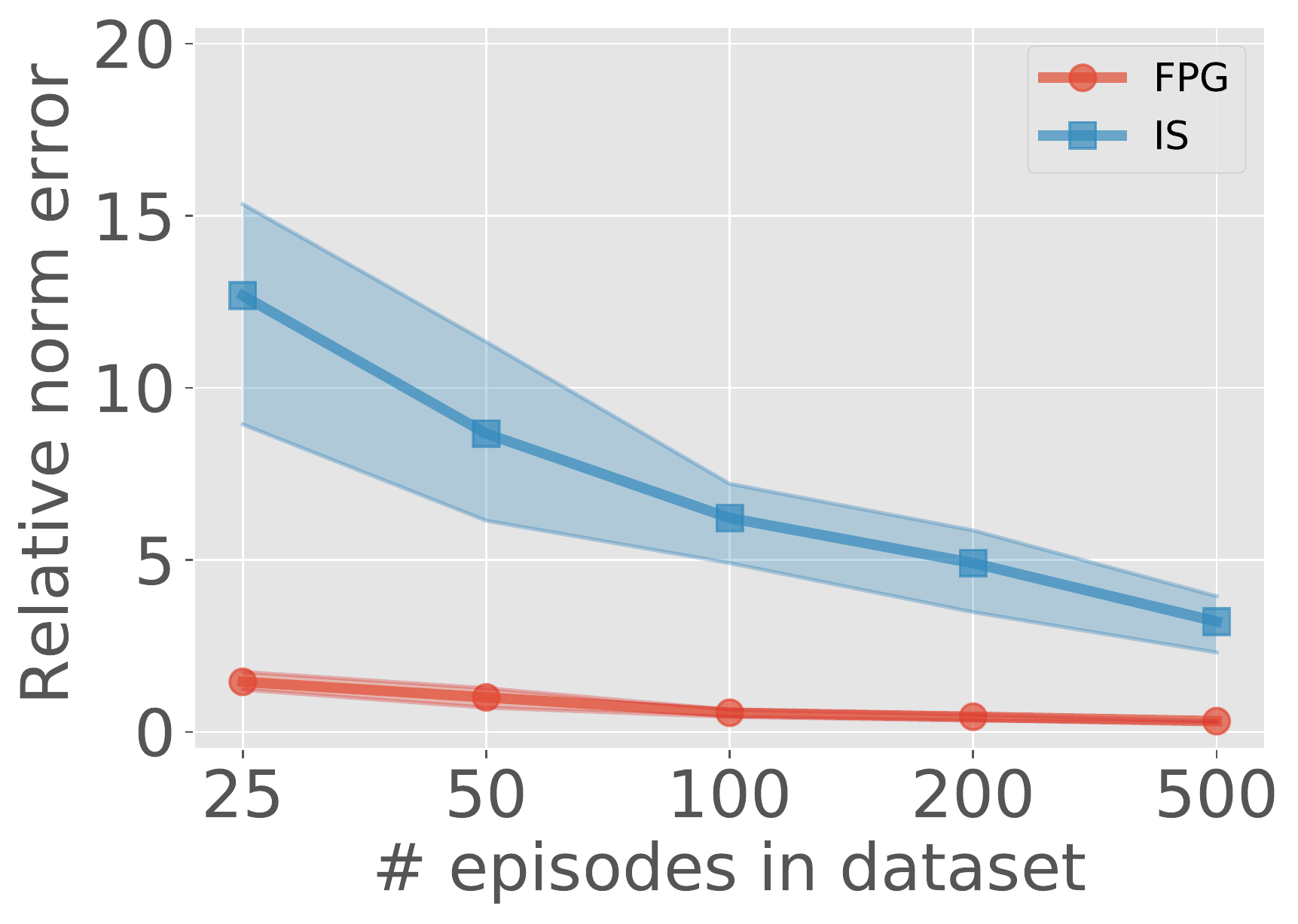}
\caption{\textbf{Sample efficiency of FPG with deep policy networks on off-policy data.} The off-policy PG estimation accuracy is evaluated using two metrics: $\cos\angle (\hat{\grad v_{\theta}}, \grad v_{\theta})$ and the relative error norm $\frac{\| \hat{\grad v_{\theta}} -  \grad v_{\theta}\| }{\|\grad v_{\theta}\|}$. }
\label{fig:CliffWalking_1}
\end{figure}

As before, we test the performance of FPG against the size of off-policy data and the degree of distribution shift, using the same cosine and relative norm error metrics. We test FPG varying the size of the dataset. \textbf{Figure \ref{fig:CliffWalking_1}} shows that FPG still gives accurate estimates with moderate variance, in contrast to IS's inaccuracy and high variance. Both methods become more accurate asymptotically as the dataset size increases.

\begin{figure}[!t]
 \centering
  \includegraphics[width=0.4\linewidth]{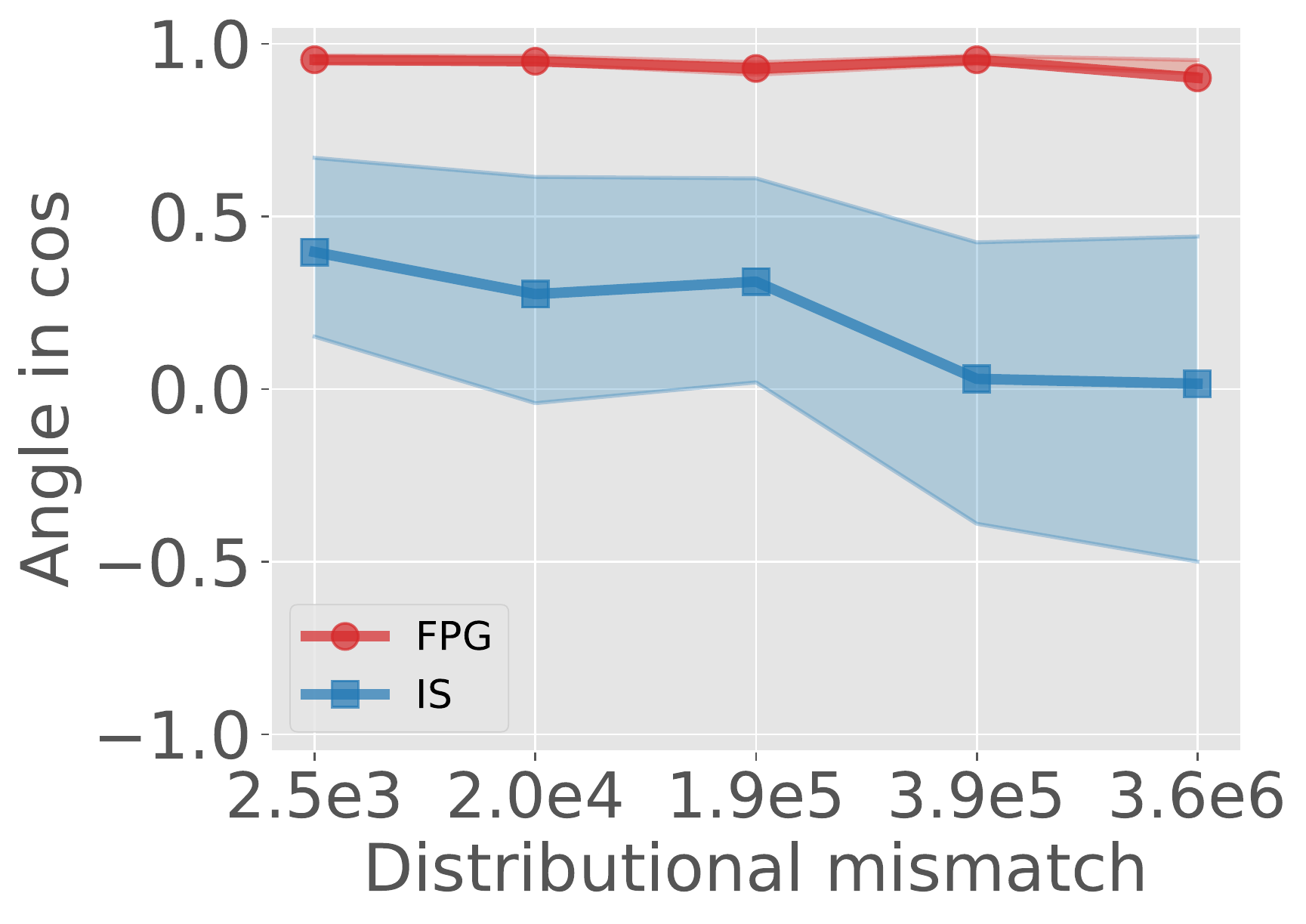}
  \includegraphics[width=0.4\linewidth]{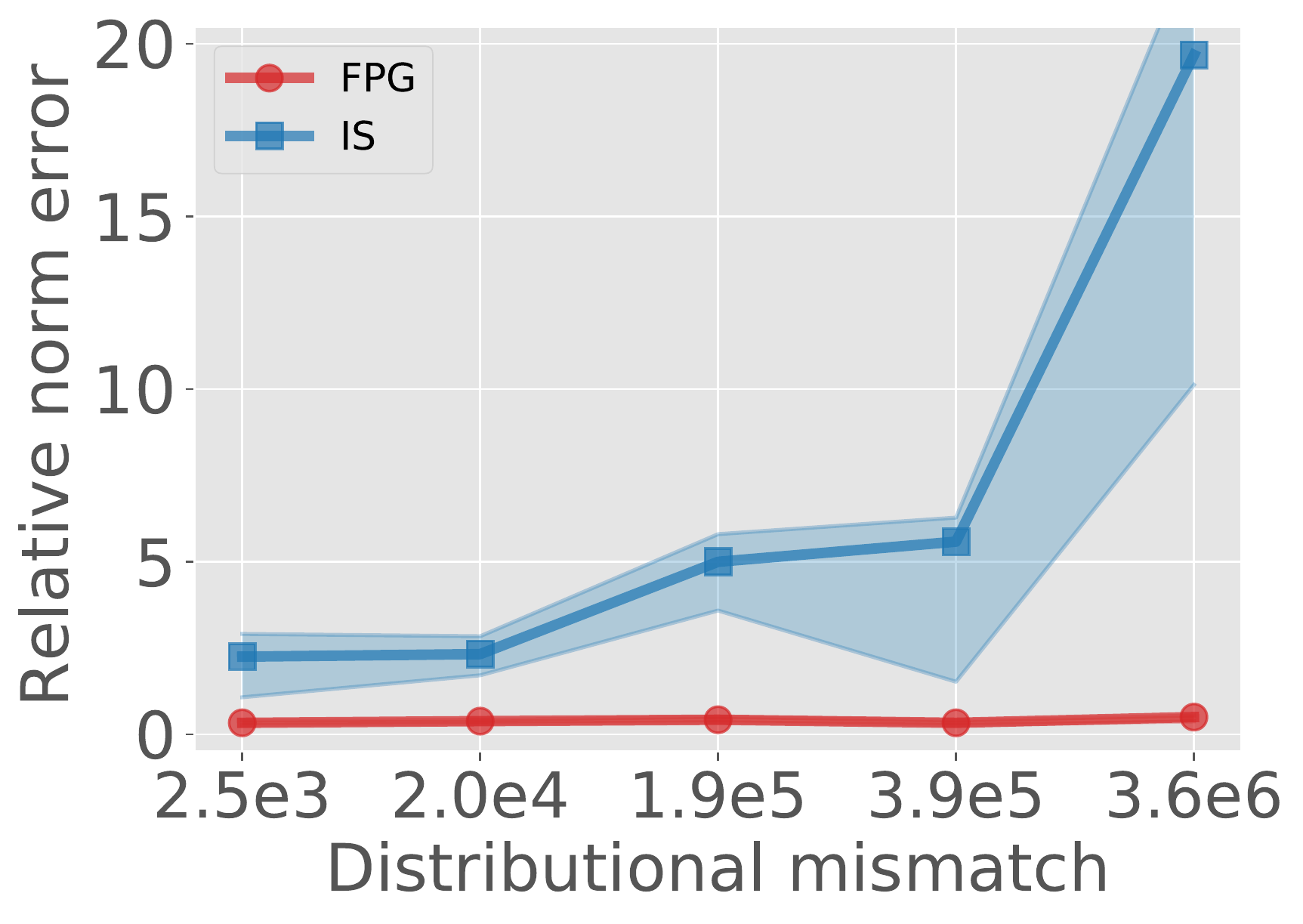}
\caption{\textbf{Tolerance to off-policy distribution shift with deep policy networks.} The distributional mismatch is measured by  $\hbox{cond}(\Sigma^{1/2}\overline{\Sigma}^{-1}\Sigma^{1/2})$, where $\overline{\Sigma}$ is the data covariance and $\Sigma$ is the target policy's occupancy measure.}
\label{fig:CliffWalking_2}
\end{figure}

We also test FPG on datasets with different amount of mismatch from the target policy. \textbf{Figure \ref{fig:CliffWalking_2}} shows that FPG's estimation error is much lower and less affected by enlarging distributional mismatch than IS's. The distribution mismatch is large in the CliffWalking experiments because some state-action pairs are almost never visited by the target policy and seldom visited by the behavior policy. Such state-action pairs cause $\Sigma^{1/2}\overline{\Sigma}^{-1}\Sigma^{1/2}$ to be nearly singular, but they are irrelevant to our estimation. The general trends in these CliffWalking experiments with deep neural network policy are consistent with our theoretical results and FrozenLake experiments.

\paragraph{Bootstrap inference for FPG}
Finally we apply bootstrap inference to construct confidence regions of FPG estimates by subsampling episodes and estimating the bootstrapped probability distribution. We plot contours of bootstrapped confidence regions via quantile KDE. \textbf{Figure \ref{fig:FrozenLake_4}} visualizes the bootstrapped confidence regions in 2D, compared with the confidence region for IS and the ground-truth confidence set. Across all experiments, we observe that the contours of bootstrapping FPG are much smaller and more accurate than the ones of bootstrapping IS. As the $K$ increases, the bootstrapped confidence regions become more concentrated, confirming our theoretical results.

\begin{figure}[!t]
 \centering
  \includegraphics[width=0.4\linewidth]{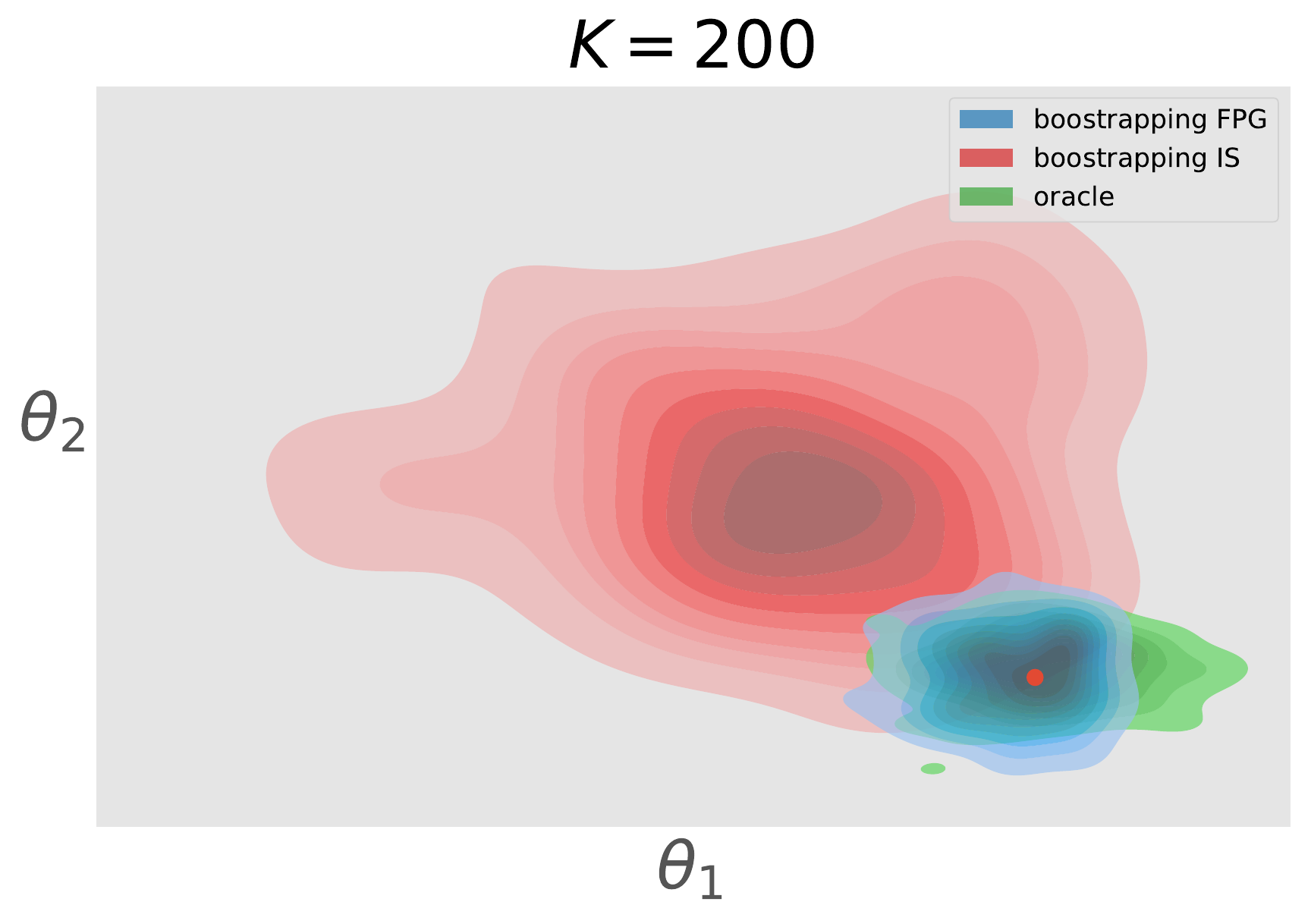}
  \includegraphics[width=0.4\linewidth]{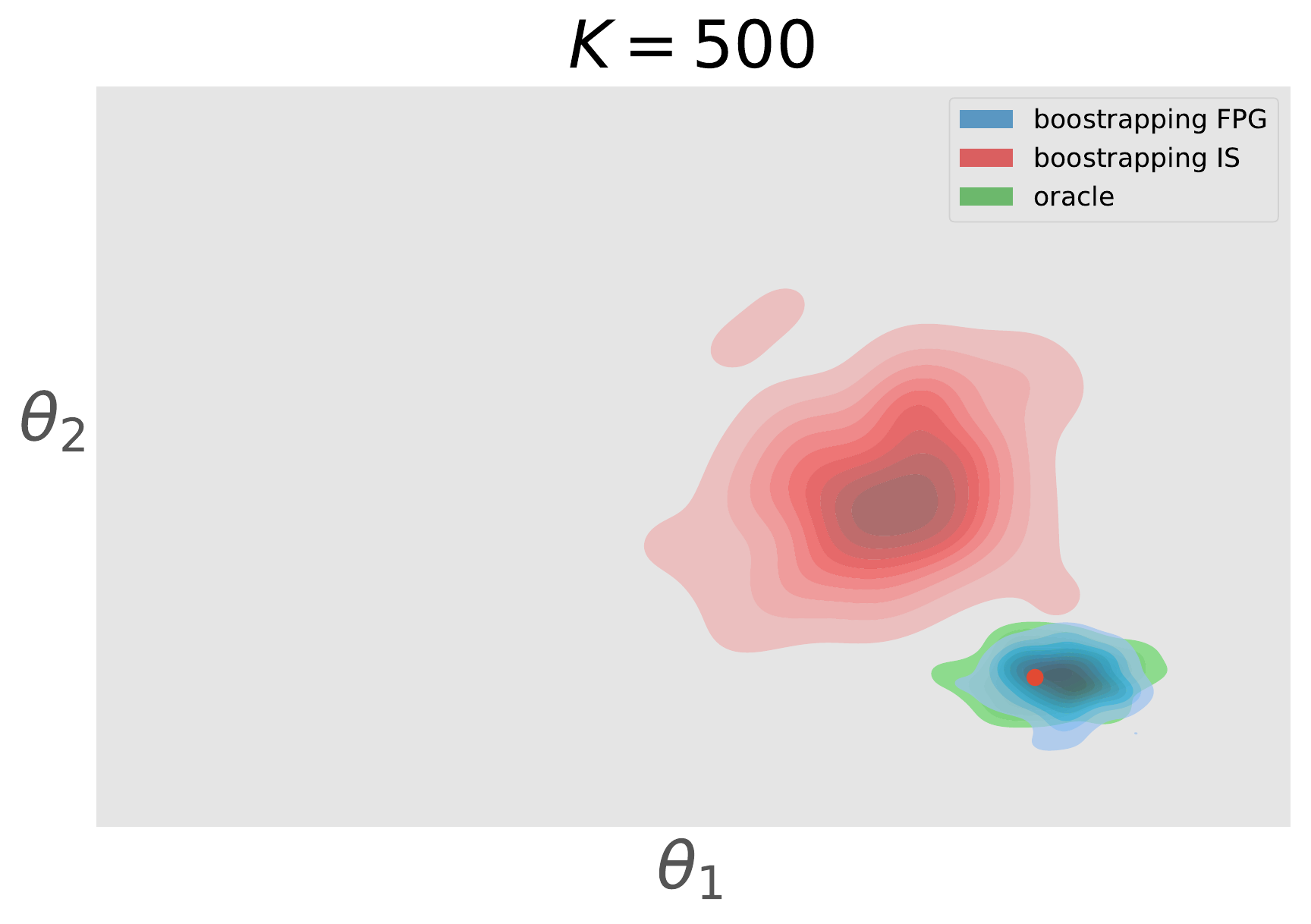}
\caption{\textbf{Bootstrapped confidence region for FPG estimator, with sample size $K=200,500$.} The red dot marks the true gradient. Blue and red areas are confidence sets obtained by bootstrapping FPG and IS respectively. The green oracle gives the empirical confidence region for FPG estimates. }
\label{fig:FrozenLake_4}
\end{figure}

\section{Conclusion}
We propose double Fitted Policy Gradient iteration (FPG) for off-policy PG estimation. FPG theoretically achieves near-optimal rate that matches the Cramar-Rao lower bound and empirically outperforms classic methods on a variety of tasks. Future work includes extension to non-linear function approximation and evaluation on more complex domains.

\bibliographystyle{plain}
\bibliography{example_paper}

\newpage
\appendix

\section{Technical Lemmas}
Let $U_h^\theta =\mathcal{P}_{\theta,h}\left(\nabla_\theta\log\Pi_{\theta,h+1}\right) Q_{h+1}^\theta,\ \widehat{U}_h^\theta=\widehat{\mathcal{P}}_{\theta,h} \left(\nabla_\theta\log\Pi_{\theta,h+1}\right)Q_{h+1}^\theta,\ \tilde{U}_h^\theta=\widehat{\mathcal{P}}_{\theta,h}\left(\nabla_\theta\log\Pi_{\theta,h+1}\right)\widehat{Q}_{h+1}^\theta$, 
\begin{lemma}
    \label{Q_decomp_base}
    We have
    \begin{align*}
        Q_h^\theta&=\sum_{h^\prime=h}^H\left(\prod_{h^{\prime\prime}=h}^{h^\prime-1}\mathcal{P}_{\theta, h^{\prime\prime}}\right)r_{h^\prime},\quad\widehat{Q_h^\theta}=\sum_{h^\prime=h}^H\left(\prod_{h^{\prime\prime}=h}^{h^\prime-1}\widehat{\mathcal{P}}_{\theta, h^{\prime\prime}}\right)\widehat{r}_{h^\prime},\\
        \nabla_\theta Q_h^\theta&=\sum_{h'=h}^H\left(\prod_{h^{\prime\prime}=h}^{h^\prime-1}\mathcal{P}_{\theta, h^{\prime\prime}}\right) U_{h^\prime}^{\theta},\quad\widehat{\nabla_\theta Q_h^\theta}=\sum_{h^\prime=h}^H\left(\prod_{h^{\prime\prime}=h}^{h^\prime-1}\widehat{\mathcal{P}}_{\theta, h^{\prime\prime}}\right)\tilde{U}_{h^\prime}^\theta.
    \end{align*}
\end{lemma}
\begin{proof}
By Bellman's equation, we have $Q_h^\theta=r_h+\mathcal{P}_{\theta,h} Q_{h+1}^\theta$. Therefore, by induction and use the fact that $Q_{H+1}^\theta=0$, we have proved the first equation. By the policy gradient Bellman's equation, we have
\begin{align*}
    \nabla_\theta Q_h^\theta(s,a)&=\mathbb{E}^{\pi_\theta}\left[\left(\nabla_\theta\log\pi_{\theta, h+1}(a^\prime\vert s^\prime)\right)Q_{h+1}^\theta(s^\prime,a^\prime)\vert s,a,h\right]+\mathbb{E}^{\pi_\theta}\left[\nabla_\theta Q_{h+1}^\theta(s^\prime, a^\prime)\vert s,a,h\right],
\end{align*}
i.e., $\nabla_\theta Q_h^\theta=U_h^\theta+\mathcal{P}_{\theta,h}\nabla_\theta Q_{h+1}^\theta$. By induction, we have proved the third equation. The expressions of $\widehat{Q}_h^\theta$ and $\widehat{\nabla_\theta Q_h^\theta}$ can be derived directly from their definitions and induction. 
\end{proof}
The decomposition leads to the following boundedness result: 
\begin{lemma}
\label{upbd}
We have $\vert Q_h^\theta(s,a)\vert\leq H-h+1,\ \Vert\nabla_\theta Q_h^\theta(s,a)\Vert_\infty\leq G(H-h)^2,\ \forall s\in\mathcal{S},a\in\mathcal{A},h\in[H].$
\end{lemma}
Now we consider the decomposition of $Q_h^\theta-\widehat{Q}_h^\theta$: 
\begin{lemma}
\label{Q_decomp}
We have 
\begin{align*}
    Q_h^\theta-\widehat{Q}_h^\theta=\sum_{h^\prime=h}^H\left(\prod_{h^{\prime\prime}=h}^{h^\prime-1}\widehat{\mathcal{P}}_{\theta,h^{\prime\prime}}\right)\left(Q_{h^\prime}^\theta-\widehat{r}_{h^\prime}- \widehat{\mathcal{P}}_{\theta,h^\prime}Q_{h^\prime+1}^\theta\right),\quad\forall h\in[H].
\end{align*}
\end{lemma}
\begin{proof}
Simply note that 
\begin{align*}
    Q_h^\theta-\widehat{Q}_h^\theta&=\sum_{h^\prime=h}^{H}\left(\prod_{h^{\prime\prime}=h}^{h^\prime-1}\mathcal{P}_{\theta,h^{\prime\prime}}\right)r_{h^\prime}-\sum_{h^\prime=h}^H\left(\prod_{h^{\prime\prime}=h}^{h^\prime-1}\widehat{\mathcal{P}}_{\theta, h^{\prime\prime}}\right)\widehat{r}_{h^\prime}\\
    &=\sum_{h^\prime=h}^H\left(\left(\prod_{h^{\prime\prime}=h}^{h^\prime-1}\mathcal{P}_{\theta,h^{\prime\prime}}\right)-\left(\prod_{h^{\prime\prime}=h}^{h^\prime-1}\widehat{\mathcal{P}}_{\theta, h^{\prime\prime}}\right)\right)r_{h^\prime}+\sum_{h^\prime=h}^H\left(\prod_{h^{\prime\prime}=h}^{h^\prime-1}\widehat{\mathcal{P}}_{\theta,h^{\prime\prime}}\right)\left(r_{h^\prime}-\widehat{r}_{h^\prime}\right)\\
    &=\sum_{h^\prime=h}^H\sum_{h^{\prime\prime}=h}^{h^\prime-1}\left(\prod_{h^{\prime\prime\prime}=h}^{h^{\prime\prime}-1}\widehat{\mathcal{P}}_{\theta,h^{\prime\prime\prime}}\right)\left(\mathcal{P}_{\theta, h^{\prime\prime}}-\widehat{\mathcal{P}}_{\theta,h^{\prime\prime}}\right)\left(\prod_{h^{\prime\prime\prime}=h^{\prime\prime}+1}^{h^\prime-1}\mathcal{P}_{\theta,h^{\prime\prime\prime}}\right)r_{h^\prime}+\sum_{h^\prime=h}^H\left(\prod_{h^{\prime\prime}=h}^{h^\prime-1}\widehat{\mathcal{P}}_{\theta,h^{\prime\prime}}\right)\left(r_{h^\prime}-\widehat{r}_{h^\prime}\right)\\
    &=\sum_{h^{\prime\prime}=h}^{H}\left(\prod_{h^{\prime\prime\prime}=h}^{h^{\prime\prime}-1}\widehat{\mathcal{P}}_{\theta,h^{\prime\prime\prime}}\right)\left(\mathcal{P}_{\theta, h^{\prime\prime}}-\widehat{\mathcal{P}}_{\theta,h^{\prime\prime}}\right)\sum_{h^\prime=h^{\prime\prime}+1}^H\left(\prod_{h^{\prime\prime\prime}=h^{\prime\prime}+1}^{h^\prime-1}\mathcal{P}_{\theta,h^{\prime\prime\prime}}\right)r_{h^\prime}+\sum_{h^\prime=h}^H\left(\prod_{h^{\prime\prime}=h}^{h^\prime-1}\widehat{\mathcal{P}}_{\theta,h^{\prime\prime}}\right)\left(r_{h^\prime}-\widehat{r}_{h^\prime}\right)\\
    &=\sum_{h^{\prime}=h}^{H}\left(\prod_{h^{\prime\prime}=h}^{h^{\prime}-1}\widehat{\mathcal{P}}_{\theta,h^{\prime\prime}}\right)\left(\mathcal{P}_{\theta, h^{\prime}}-\widehat{\mathcal{P}}_{\theta,h^{\prime}}\right)Q^\theta_{h^{\prime}+1}+\sum_{h^\prime=h}^H\left(\prod_{h^{\prime\prime}=h}^{h^\prime-1}\widehat{\mathcal{P}}_{\theta,h^{\prime\prime}}\right)\left(r_{h^\prime}-\widehat{r}_{h^\prime}\right)\\
    &=\sum_{h^\prime=h}^H\left(\prod_{h^{\prime\prime}=h}^{h^{\prime}-1}\widehat{\mathcal{P}}_{\theta,h^{\prime\prime}}\right)\left(Q_{h^\prime}^\theta-\widehat{r}_{h^\prime}-\widehat{\mathcal{P}}_{\theta,h^\prime}Q_{h^\prime+1}^\theta\right), 
\end{align*}
which is the desired result. 
\end{proof}
The following lemma provides an upper bound of matrix production, which will be used when bounding the higher order terms of the finite sample bound. 
\begin{lemma}
\label{decomp}
For any series of matrices $A_1,A_2,\ldots,A_n$ and $\Delta A_1,\Delta A_2,\ldots\Delta A_n$, we have
\begin{align*}
    \left\Vert\prod_{i=1}^n(A_i+\Delta A_i)-\prod_{i=1}^n A_i\right\Vert\leq \prod_{i=1}^n\left(\Vert A_i\Vert+\Vert\Delta A_i \Vert\right)-\prod_{i=1}^n\Vert A_i\Vert.
\end{align*}
\end{lemma}
\begin{proof}
We have
\begin{align*}
    \left\Vert\prod_{i=1}^n(A_i+\Delta A_i)-\prod_{i=1}^n A_i\right\Vert&= \left\Vert\sum_{\delta\in\{0,1\}^n\setminus\{(1,1,\ldots,1)\}}\prod_{i=1}^n A_i^{\delta_i}(\Delta A_i)^{1-\delta_i}\right\Vert\leq\sum_{\delta\in \{0,1\}^n\setminus\{(1,1,\ldots,1)\}}\prod_{i=1}^n\Vert A_i\Vert^{\delta_i}\Vert\Delta A_i\Vert^{1-\delta_i}\\
    &=\prod_{i=1}^n\left(\Vert A_i\Vert+\Vert\Delta A_i\Vert\right)-\prod_{i=1}^n\Vert A_i\Vert.
\end{align*}
\end{proof}
When $\mathcal{F}$ is the class of the linear functions, there exists matrix $M_{\theta,h}$ such that the transition probability satisfies
\begin{align*}
    \mathbb{E}^{\pi_\theta}\left[\phi(s^\prime,a^\prime)^\top\vert s,a,h\right] = \phi(s,a)^\top M_{\theta,h}. 
\end{align*}
The following lemma gives an upper bound on the 2-norm of $M_{\theta,h}$ and its derivatives. 
\begin{lemma}
\label{ineq}
We have $\left\Vert\Sigma_{\theta,h}^{\frac{1}{2}}M_{\theta,h}\Sigma_{\theta,h+1}^{-\frac{1}{2}}\right\Vert \leq 1$ and $\left\Vert\Sigma_{\theta,h}^{\frac{1}{2}}\left(\nabla_\theta^j M_{\theta,h}\right)\Sigma_{\theta,h+1}^{-\frac{1}{2}}\right\Vert\leq G,\ \forall j\in[m],h\in[H]$.
\end{lemma}
\begin{proof}
Note that for any $f:\mathcal{S}\times\mathcal{A}\rightarrow \mathbb{R},\ f(s,a):=\mu^\top\phi(s, a)$, we have
\begin{align*}
    \mathbb{E}^{\pi_\theta}\left[f^2(s_{h+1},a_{h+1})\vert s_1\sim\xi\right]&=\mathbb{E}^{\pi_\theta}\left[\mathbb{E}^{\pi_\theta}\left[f^2(s_{h+1},a_{h+1})\vert s_h,a_h\right]\vert s_1\sim\xi\right]\\
    &\geq\mathbb{E}^{\pi_\theta}[\mathbb{E}^{\pi_\theta}[f(s_{h+1},a_{h+1})\vert s_h,a_h]^2\vert  s_1\sim\xi]
\end{align*}
The LHS satisfies
\begin{align*}
    \mathbb{E}^{\pi_\theta}[f^2(s_{h+1},a_{h+1})\vert s_1\sim \xi]&=\mu^\top\Sigma_{\theta, h+1}\mu
\end{align*}
and the RHS satisfies
\begin{align*}
    \mathbb{E}^{\pi_\theta}[\mathbb{E}^{\pi_\theta}[f(s_{h+1},a_{h+1})\vert s_h,a_h]^2\vert s_1\sim\xi]=\mathbb{E}^{\pi_\theta}[\mu^\top M_{\theta,h}^\top\phi(s_h,a_h)\phi(s_h,a_h)^\top M_{\theta,h}\mu\vert s_1\sim\xi]=\mu^\top M_{\theta,h}^\top\Sigma_{\theta,h} M_{\theta,h}\mu
\end{align*}
Therefore, we have $\mu^\top\Sigma_{\theta,h+1}\mu\geq\mu^\top M_{\theta,h}^\top\Sigma_{\theta,h} M_{\theta,h} \mu,\ \forall\mu$, which implies $\Vert\Sigma_{\theta,h}^{\frac{1}{2}}M_{\theta,h} \Sigma_{\theta,h+1}^{-\frac{1}{2}}\Vert\leq 1$. Similarly, let $g:\mathcal{S}\times\mathcal{A}\rightarrow\mathbb{R},\ g(s, a):=\left(\nabla_\theta^j\log\pi_{\theta, h+1}(s, a)\right)\mu^\top\phi(s,a)$, we have
\begin{align*}
    \mathbb{E}^{\pi_\theta}[g^2(s_{h+1},a_{h+1})\vert s_1\sim\xi]=&\mathbb{E}^{\pi_\theta}[\mathbb{E}^{\pi_\theta}[g^2(s_{h+1}, a_{h+1})\vert s_h,a_h]\vert s_1\sim\xi]\\
    \geq&\mathbb{E}^{\pi_\theta}[\mathbb{E}^{\pi_\theta}[g(s_{h+1},a_{h+1})\vert s_h,a_h]^2\vert s_1\sim\xi]
\end{align*}
The LHS satisfies
\begin{align*}
    \mathbb{E}^{\pi_\theta}[g^2(s_{h+1},a_{h+1})\vert s_1\sim\xi]&=\mu^\top\mathbb{E}^{\pi_\theta}\left[\left(\nabla_\theta^j\log\pi_{\theta,h+1}(a\vert s)\right)^2\phi(s_{h+1}, a_{h+1})\phi(s_{h+1},a_{h+1})^\top\vert s_1\sim\xi\right]\mu\\
    &\leq G^2\mu^\top\mathbb{E}^{\pi_\theta}\left[\phi(s_{h+1},a_{h+1})\phi(s_{h+1},a_{h+1})^\top\vert s_1\sim\xi\right]\mu=G^2\mu^\top\Sigma_{\theta,h+1}\mu
\end{align*}
and the RHS satisfies 
\begin{align*}
    &\mathbb{E}^{\pi_\theta}[\mathbb{E}^{\pi_\theta}[g(s_{h+1},a_{h+1})\vert s_h, a_h]^2\vert s_1\sim\xi]\\
    =&\mathbb{E}^{\pi_\theta}[\mu^\top\left(\nabla_\theta^j M_{\theta,h}\right)^\top \phi(s_h,a_h)\phi(s_h,a_h)^\top\left(\nabla_\theta^j M_{\theta,h}\right)\mu\vert s_1\sim\xi]\\
    =&\mu^\top\left(\nabla_\theta^j M_{\theta,h}\right)^\top\Sigma_{\theta,h}\left(\nabla_\theta^j M_{\theta,h}\right)\mu.
\end{align*}
Therefore, we get $G^2\mu^\top\Sigma_{\theta,h+1}\mu\geq\mu^\top\left(\nabla_\theta^j M_{\theta,h}\right)^\top\Sigma_{\theta,h}\left(\nabla_\theta^j M_{\theta,h}\right)\mu,\ \forall\mu$, which implies $\left\Vert \Sigma_{\theta,h}^{\frac{1}{2}}\left(\nabla_\theta^j M_\theta\right)\Sigma_{\theta,h+1}^{-\frac{1}{2}}\right\Vert \leq G$.
\end{proof}

\begin{lemma}
\label{dsig1}
We have with probability at least $1-\delta$,
\begin{align*}
\left\Vert\Sigma_h^{-\frac{1}{2}}\left(\frac{1}{K}\sum_{k=1}^K\phi(s_h^{(k)}, a_h^{(k)})\phi(s_h^{(k)},a_h^{(k)})^\top\right)\Sigma^{-\frac{1}{2}}_h-I_d\right\Vert\leq \sqrt{\frac{2C_1d\log\frac{2dH}{\delta}}{K}} + \frac{2C_1d\log\frac{2dH}{\delta}}{3K}.
\end{align*}
\end{lemma}
\begin{proof}
Define
\begin{align*}
	X^{(k)}_h=\Sigma^{-\frac{1}{2}}_h\phi\left(s^{(k)}_h,a^{(k)}_h\right) \phi\left(s^{(k)}_h,a^{(k)}_h\right)^\top\Sigma_h^{-\frac{1}{2}}\in\mathbb{R}^{d\times d}.
\end{align*}
It's easy to see that $X^{(1)}_h,X^{(2)}_h,\ldots,X^{(K)}_h$ are independent and $\mathbb{E}\left[X^{(k)}_h\right]=I_d$. In the remaining part of the proof, we will apply the matrix Bernstein's inequality to analyze the concentration of $\frac{1}{K}\sum_{k=1}^K X^{(k)}_h$. We first consider the matrix-valued variance $\textrm{Var}\left(X^{(k)}_h\right)=\mathbb{E}\left[\left(X^{(k)}_h-I_d\right)^2\right]=\mathbb{E}\left[\left(X^{(k)}_h\right)^2\right]-I_d$. For any vector $\mu\in\mathbb{R}^d$,
\begin{align*} 
	\mu^\top\mathbb{E}\left[\left(X_h^{(k)}\right)^2\right]\mu=&\mathbb{E}\left[\left\Vert X_h^{(k)}\mu\right\Vert^2\right]=\mathbb{E}\left[\left\Vert\Sigma^{-\frac{1}{2}}_h\phi\left(s^{(k)}_h,a^{(k)}_h\right)\phi\left(s^{(k)}_h,a^{(k)}_h\right)^\top\Sigma_h^{-\frac{1}{2}}\mu\right\Vert^2\right]\\
	\leq&\mathbb{E}\left[\left\Vert\Sigma_h^{-\frac{1}{2}}\phi\left(s^{(k)}_h,a^{(k)}_h\right)\right\Vert^2\left\Vert\phi\left(s^{(k)}_h,a^{(k)}_h\right)^\top\Sigma_h^{-\frac{1}{2}}\mu\right\Vert^2\right]\leq C_1d\mathbb{E}\left[\left\Vert\phi\left(s^{(k)}_h,a^{(k)}_h\right)^\top\Sigma_h^{-\frac{1}{2}}\mu\right\Vert^2\right]\\
    =&C_1d\mu^\top\mathbb{E}\left[X_h^{(k)}\right]\mu = C_1d \Vert\mu\Vert^2,
\end{align*}
where we used the identity $\left\Vert\phi\left(s^{(k)}_h,a^{(k)}_h\right)^\top\Sigma_h^{-\frac{1}{2}}\mu\right\Vert^2=\mu^\top X_h^{(k)}\mu$ and $\mathbb{E}\left[X_h^{(k)}\right]=I_d$. We have
\begin{align*}
    \textrm{Var}(X_h^{(k)})\preceq\mathbb{E}\left[\left(X_h^{(k)}\right)^2\right]\preceq C_1dI_d. 
\end{align*}
Additionally,
\begin{align*}
	-I_d\preceq X_h^{(k)}-I_d=\Sigma_h^{-\frac{1}{2}}\phi\left(s^{(k)}_h,a^{(k)}_h\right)\phi\left(s^{(k)}_h,a^{(k)}_h\right)^\top\Sigma_h^{-\frac{1}{2}}-I_d \preceq C_1dI_d-I_d. 
\end{align*}
Therefore, $\Vert X_h^{(k)}-I_d\Vert\leq C_1d$. Since $X^{(1)}_h,X^{(2)}_h,\ldots,X^{(K)}_h$ are \textit{i.i.d.}, by the matrix-form Bernstein inequality, we have
\begin{align*}
	\mathbb{P}\left(\left\Vert\sum_{k=1}^K X_h^{(k)}-I_d\right\Vert\geq\varepsilon\right)\leq 2d\cdot\exp\left(-\frac{\varepsilon^2/2}{C_1dK+C_1d \varepsilon/3}\right),\quad \forall \varepsilon>0. 
\end{align*}
With probability at least $1 - \delta$,
\begin{align*}
    \left\Vert\frac{1}{K}\sum_{k=1}^K\left(X_h^{(k)}-I_d\right)\right\Vert\leq \sqrt{\frac{2C_1d\log\frac{2d}{\delta}}{K}}+\frac{2C_1d\log\frac{2d}{\delta}}{3K}, 
\end{align*}
Taking a union bound over $h\in[H]$, we derive the desired result.	
\end{proof}
Let $\Delta\Sigma_h^{-1}=\widehat{\Sigma}_h^{-1}-\Sigma_h$, 
\begin{lemma}
\label{dsig2}
If $\left\Vert\Sigma_h^{-\frac{1}{2}}\widehat{\Sigma}_h\Sigma_h^{-\frac{1}{2}}-I_d\right\Vert\leq\frac{1}{2}$, then 
$\left\Vert\Sigma_h^{\frac{1}{2}}\left(\Delta\Sigma_h^{-1}\right)\Sigma_h^{\frac{1}{2}}\right\Vert\leq 2\left\Vert\Sigma_h^{-\frac{1}{2}}\widehat{\Sigma}_h\Sigma_h^{-\frac{1}{2}}-I_d\right\Vert$. 
\end{lemma}
\begin{proof}
Note that
\begin{align}
    \label{2_1}
    \left\Vert\Sigma_h^{\frac{1}{2}}\left(\Delta\Sigma_h^{-1}\right)\Sigma_h^{\frac{1}{2}}\right\Vert=\left\Vert\Sigma_h^{\frac{1}{2}}\left(\widehat{\Sigma}_h^{-1}-\Sigma_h^{-1}\right)\Sigma_h^{\frac{1}{2}}\right\Vert\leq\left\Vert\Sigma_h^{\frac{1}{2}}\widehat{\Sigma}_h^{-1}\Sigma_h^{\frac{1}{2}}\right\Vert\left\Vert\Sigma_h^{-\frac{1}{2}}\widehat{\Sigma}_h\Sigma_h^{-\frac{1}{2}}-I_d\right\Vert.  
\end{align}
Because we have $\left\Vert\Sigma_h^{-\frac{1}{2}}\widehat{\Sigma}_h\Sigma_h^{-\frac{1}{2}}-I_d\right\Vert\leq\frac{1}{2}$, we get $\sigma_{\textrm{min}}\left(\Sigma_h^{-\frac{1}{2}}\widehat{\Sigma}_h\Sigma_h^{-\frac{1}{2}}\right)\geq\frac{1}{2}$, which implies $\left\Vert\Sigma_h^{\frac{1}{2}}\widehat{\Sigma}_h^{-1}\Sigma_h^{\frac{1}{2}}\right\Vert\leq 2$. Combining this result with \eqref{2_1} finishes the proof. 
\end{proof}
Let $\Delta Y_{\theta,h}=\frac{1}{K}\sum_{k=1}^K\phi\left(s_h^{(k)},a_h^{(k)}\right)\phi_{\theta,h+1}\left(s_{h+1}^{(k)}\right)^\top-\Sigma_h M_{\theta,h}$, 
\begin{lemma}
\label{dy}
With probability at least $1-\delta$, the following inequalities hold simultaneously:
\begin{align}
\label{3_1}
\left\Vert\Sigma_h^{-\frac{1}{2}}\left(\Delta Y_{\theta,h}\right)\Sigma_{h+1}^{-\frac{1}{2}}\right\Vert&\leq\left(\kappa_2\vee 1\right)\sqrt{\frac{2C_1d\log\frac{4dH}{\delta}}{K}} + \frac{4C_1d\log\frac{4dH}{\delta}}{3K},\\
\label{3_2}
\left\Vert\Sigma_h^{-\frac{1}{2}}\left(\nabla_\theta^j\left(\Delta Y_{\theta,h}\right)\right)\Sigma_{h+1}^{-\frac{1}{2}}\right\Vert&\leq\left(\kappa_3\vee 1\right)G\sqrt{\frac{2C_1d\log\frac{4mdH}{\delta}}{K}}+\frac{4C_1dG\log\frac{4mdH}{\delta}}{3K},\quad\forall j\in[m],
\end{align}
where $\kappa_2, \kappa_3$ are defined in Theorem \ref{thm2_var}. 
\end{lemma}
\begin{proof}
Take
\begin{align*}
	Y_{\theta,h}^{(k)}:=\Sigma_h^{-\frac{1}{2}}\phi\left(s^{(k)}_h,a^{(k)}_h\right)\phi_{\theta,h+1}\left(s^{(k)}_{h+1}\right)^\top\Sigma_{h+1}^{-\frac{1}{2}},\quad\forall k\in[K]. 
\end{align*}
Then, $\Sigma_h^{-\frac{1}{2}}\left(\Delta Y_{\theta,h}\right)\Sigma_{h+1}^{-\frac{1}{2}} =\frac{1}{K}\sum_{k=1}^K\left(Y_{\theta,h}^{(k)}-\Sigma_h^{\frac{1}{2}}M_{\theta,h} \Sigma^{-\frac{1}{2}}_{h+1}\right)$. Note that
\begin{align}
    \label{SMS0} 
    \begin{aligned} 
        \mathbb{E}\left[Y_{\theta,h}^{(k)}\right]=&\mathbb{E}\left[\Sigma_h^{-\frac{1}{2}}\phi\left(s^{(k)}_h,a^{(k)}_h\right)\phi_{\theta,h+1}\left(s^{(k)}_{h+1}\right)^\top\Sigma_{h+1}^{-\frac{1}{2}}\right]\\
        =&\mathbb{E}\left[\Sigma_h^{-\frac{1}{2}} \phi\left(s^{(k)}_h,a^{(k)}_h\right)\mathbb{E}^{\pi_\theta}\left[\phi\left(s^\prime, a^\prime\right)^\top \vert s^{(k)}_h,a^{(k)}_h, h\right]\Sigma_{h+1}^{-\frac{1}{2}}\right]\\
        =&\mathbb{E}\left[\Sigma_h^{-\frac{1}{2}}\phi\left(s^{(k)}_h,a^{(k)}_h\right)\phi\left(s^{(k)}_h,a^{(k)}_h\right)^\top M_{\theta,h}\Sigma_{h+1}^{-\frac{1}{2}}\right]=\Sigma_h^{\frac{1}{2}}M_{\theta,h}\Sigma_{h+1}^{-\frac{1}{2}}, 
    \end{aligned}
\end{align}
To this end, $\Sigma_h^{-\frac{1}{2}}\left(\Delta Y_{\theta,h}\right)\Sigma_{h+1}^{-\frac{1}{2}}=\frac{1}{K}\sum_{k=1}^K\left( Y_{\theta,h}^{(k)}-\mathbb{E}\left[Y_{\theta,h}^{(k)}\right]\right)$. Since the trajectories are \textrm{i.i.d.}, we use the matrix-form Bernstein inequality to estimate $\left\Vert\Sigma_h^{-\frac{1}{2}}(\Delta Y_{\theta,h})\Sigma_{h+1}^{-\frac{1}{2}}\right\Vert$. For any $\mu\in \mathbb{R}^d$, we have
\begin{align*}
    \mu^\top\mathbb{E}\left[Y_{\theta,h}^{(k)}\left(Y_{\theta,h}^{(k)}\right)^\top\right]\mu=&\mathbb{E}\left[\left\Vert\left(Y_{\theta,h}^{(k)}\right)^\top\mu\right\Vert^2\right]=\mathbb{E}\left[\left\Vert\Sigma_{h+1}^{-\frac{1}{2}}\phi_{\theta,h+1}\left(s_{h+1}^{(k)}\right)\phi\left(s_h^{(k)},a_h^{(k)}\right)^\top\Sigma_h^{-\frac{1}{2}}\mu\right\Vert^2\right]\\
    \leq&\mathbb{E}\left[\left\Vert\Sigma_{h+1}^{-\frac{1}{2}} \phi_{\theta,h+1}\left(s_{h+1}^{(k)}\right)\right\Vert^2\left\Vert\phi\left(s_h^{(k)}, a_h^{(k)}\right)^\top\Sigma_h^{-\frac{1}{2}}\mu\right\Vert^2\right]. 
\end{align*}		
Parallel to the proof of Lemma \ref{dsig1}, it holds that $\left\Vert \Sigma_{h+1}^{-\frac{1}{2}}\phi_{\theta,h+1}\left(s_{h+1}^{(k)}\right)\right\Vert^2\leq C_1d$. Therefore,
\begin{align*}
    \mu^\top\mathbb{E}\left[Y_{\theta,h}^{(k)}\left(Y_{\theta,h}^{(k)}\right)^\top\right]\mu&\leq \mathbb{E}\left[C_1d\left\Vert\phi\left(s_h^{(k)},a_h^{(k)}\right)^\top\Sigma_h^{-\frac{1}{2}}\mu\right\Vert^2\right]\\
    &=C_1d\mu^\top\Sigma_h^{-\frac{1}{2}}\mathbb{E}\left[\phi\left(s_h^{(k)},a_h^{(k)}\right)\phi\left(s_h^{(k)},a_h^{(k)}\right)^\top\right]\Sigma_h^{-\frac{1}{2}}\mu\\
    &=C_1d\Vert\mu\Vert^2,
\end{align*}
where we have used the fact $\Sigma_h=\mathbb{E}\left[\phi\left(s_h^{(k)},a_h^{(k)}\right) \phi\left(s_h^{(k)},a_h^{(k)}\right)^\top\right]$. It follows that
\begin{align*} 
    \textrm{Var}_1\left(Y_{\theta,h}^{(k)}\right):=&\mathbb{E}\left[\left(Y_{\theta,h}^{(k)}-\mathbb{E}\left[Y_{\theta,h}^{(k)}\right]\right)\left(Y_{\theta,h}^{(k)}-\mathbb{E}\left[Y_{\theta,h}^{(k)}\right]\right)^\top\right]\preceq\mathbb{E}\left[Y_{\theta,h}^{(k)}\left(Y_{\theta,h}^{(k)}\right)^\top\right]\preceq C_1dI_d. 
\end{align*} 
Analogously,
\begin{align*} 
	\textrm{Var}_2\left(Y_{\theta,h}^{(k)}\right):=&\mathbb{E}\left[\left(Y_{\theta,h}^{(k)}-\mathbb{E}\left[Y_{\theta,h}^{(k)}\right]\right)^\top\left(Y_{\theta,h}^{(k)}-\mathbb{E}\left[Y_{\theta,h}^{(k)}\right]\right)\right]\preceq\mathbb{E}\left[\left(Y_{\theta,h}^{(k)}\right)^\top Y_{\theta,h}^{(k)}\right]\\
	\preceq& C_1d\Sigma_{h+1}^{-\frac{1}{2}}\mathbb{E}\left[\phi_{\theta,h+1}\left(s_{h+1}^{(k)}\right)\phi_{\theta,h+1}\left(s_{h+1}^{(k)}\right)^\top\right]\Sigma_{h+1}^{-\frac{1}{2}}. 
\end{align*}
Therefore, $\max\left\{\left\Vert\textrm{Var}_1\left(Y_{\theta,h}^{(k)}\right)\right\Vert, \left\Vert\textrm{Var}_2\left(Y_{\theta,h}^{(k)}\right)\right\Vert\right\}\leq C_1d\left(\kappa_2^2\vee 1\right)$. It also holds that $\Vert Y_{\theta,h}^{(k)}\Vert\leq C_1d$. Hence,
\begin{align*}
    \left\Vert Y_{\theta,h}^{(k)}-\Sigma_h^{\frac{1}{2}}M_{\theta,h}\Sigma_{h+1}^{-\frac{1}{2}} \right\Vert\leq 2C_1d. 
\end{align*}
Applying Matrix Bernstein's inequality, we derive for any $\varepsilon>0$,
\begin{align*}
    \mathbb{P}\left(\left\Vert\sum_{k=1}^K\left(Y_{\theta,h}^{(k)}-\Sigma_h^{\frac{1}{2}}M_{\theta,h}\Sigma_{h+1}^{-\frac{1}{2}}\right)\right\Vert>\varepsilon\right)\leq 2d\exp\left(-\frac{\varepsilon^2/2}{C_1dK\left(\kappa_2^2\vee 1\right)+2C_1d\varepsilon/3}\right), 
\end{align*}
which implies \eqref{3_1} holds with probability $1-\frac{\delta}{2}$. For \eqref{3_2}, notice that for any $j\in[m]$, we have $\Sigma_h^{-\frac{1}{2}}(\nabla_\theta^j(\Delta Y_{\theta,h}))\Sigma_{h+1}^{-\frac{1}{2}}=\frac{1}{K}\sum_{k=1}^K\left(\nabla_\theta^j Y_{\theta,h}^{(k)}-\mathbb{E}\left[\nabla_\theta^j Y_{\theta,h}^{(k)}\right]\right)$, and $\nabla_\theta^j Y_{\theta,h}^{(k)}=\Sigma_h^{-\frac{1}{2}}\phi\left(s_h^{(k)},a_h^{(k)}\right) \left(\nabla_\theta^j\phi_{\theta,h+1}\left(s_{h+1}^{(k)}\right)\right)^\top\Sigma_{h+1}^{-\frac{1}{2}}$. For any $\mu\in \mathbb{R}^d$, we have
\begin{align*}
    \mu^\top\mathbb{E}\left[\left(\nabla_\theta^j Y_{\theta,h}^{(k)}\right) \left(\nabla_\theta^j Y_{\theta,h}^{(k)}\right)^\top\right]\mu=&\mathbb{E}\left[\left\Vert\Sigma_{h+1}^{-\frac{1}{2}}\left(\nabla_\theta^j\phi_{\theta,h+1}\left(s_{h+1}^{(k)}\right)\right)\phi\left(s_h^{(k)},a_h^{(k)}\right)^\top\Sigma_h^{-\frac{1}{2}}\mu\right\Vert^2\right]\\
    \leq&\mathbb{E}\left[\left\Vert\Sigma_{h+1}^{-\frac{1}{2}}\nabla_\theta^j\phi_{\theta,h+1}\left(s_{h+1}^{(k)}\right)\right\Vert^2\left\Vert\phi\left(s_h^{(k)},a_h^{(k)}\right)^\top\Sigma_h^{-\frac{1}{2}}\mu\right\Vert^2\right]. 
\end{align*}
Since we have 
\begin{align*}
    &\left(\nabla_\theta^j\phi_{\theta,h+1}\left(s_{h+1}^{(k)}\right)\right)^\top\Sigma_{h+1}^{-1}\nabla_\theta^j\phi_{\theta,h+1}\left(s_{h+1}^{(k)}\right)\\
    =&\int_{\mathcal{A}\times\mathcal{A}}\pi_{\theta,h+1}\left(a\left\vert s_{h+1}^{(k)}\right.\right)\pi_{\theta,h+1}\left(a^\prime\left\vert s_{h+1}^{(k)}\right.\right)\left(\nabla_\theta^j\log\pi_{\theta,h+1}\left(a\left\vert s_{h+1}^{(k)}\right.\right)\right)\left(\nabla_\theta^j\log\pi_{\theta,h+1}\left(a^\prime\left\vert s_{h+1}^{(k)}\right.\right)\right)\\
    &\cdot\phi\left(s_{h+1}^{(k)},a\right)^\top\Sigma_{h+1}^{-1}\phi\left(s_{h+1}^{(k)}, a^\prime\right)\mathrm{d}a\mathrm{d}a^\prime\\
    \leq&G^2\int_{\mathcal{A}\times\mathcal{A}}\pi_{\theta,h+1}\left(a\vert s_{h+1}^{(k)}\right)\pi_{\theta,h+1}\left(a^\prime\vert s_{h+1}^{(k)}\right)\left\Vert \Sigma_{h+1}^{-\frac{1}{2}}\phi\left(s_{h+1}^{(k)},a\right)\right\Vert\left\Vert\Sigma_{h+1}^{-\frac{1}{2}}\phi\left(s_{h+1}^{(k)},a^\prime\right)\right\Vert\mathrm{d}a\mathrm{d}a^\prime\leq G^2C_1d,
\end{align*}
which implies 
\begin{align*}
    \mu^\top\mathbb{E}\left[\left(\nabla_\theta^j Y_{\theta,h}^{(k)}\right)\left(\nabla_\theta^j Y_{\theta,h}^{(k)}\right)^\top \right]\mu\leq G^2C_1d\mathbb{E}\left[\left\Vert\phi\left(s_h^{(k)},a_h^{(k)}\right)^\top\Sigma_h^{-\frac{1}{2}}\mu\right\Vert^2\right]=G^2C_1d\Vert\mu\Vert^2. 
\end{align*}
Therefore, 
\begin{align*}
    \textrm{Var}_1\left(\nabla_\theta^j Y_{\theta,h}^{(k)}\right):=&\mathbb{E}\left[\left(\nabla_\theta^j Y_{\theta,h}^{(k)}-\mathbb{E}\left[\nabla_\theta^j Y_{\theta,h}^{(k)}\right]\right)\left(\nabla_\theta^j Y_{\theta,h}^{(k)}-\mathbb{E}\left[\nabla_\theta^j Y_{\theta,h}^{(k)}\right]\right)^\top\right]\preceq\mathbb{E}\left[\left(\nabla_\theta^j Y_{\theta,h}^{(k)}\right)\left(\nabla_\theta^j Y_{\theta,h}^{(k)}\right)^\top\right]\preceq G^2C_1dI_d. 
\end{align*}
Meanwhile, we have
\begin{align*} 
	\textrm{Var}_2\left(\nabla_\theta^j Y_{\theta,h}^{(k)}\right)\preceq\mathbb{E}\left[\left(\nabla_\theta^j Y_{\theta,h}^{(k)}\right)^\top\nabla_\theta^j Y_{\theta,h}^{(k)}\right]\preceq C_1d\Sigma_h^{-\frac{1}{2}}\mathbb{E}\left[\left(\nabla_\theta^j \phi_{\theta,h+1}\left(s_{h+1}^{(k)}\right)\right)\left(\nabla_\theta^j \phi_{\theta,h+1}\left(s_{h+1}^{(k)}\right)\right)^\top\right]\Sigma_h^{-\frac{1}{2}}. 
\end{align*}
In conclusion, we get
\begin{align*}
    \max\left\{\left\Vert\textrm{Var}_1\left(\nabla_\theta^j Y_{\theta,h}^{(k)}\right)\right\Vert,\left\Vert\textrm{Var}_2\left(\nabla_\theta^j Y_{\theta,h}^{(k)}\right)\right\Vert\right\}\leq G^2C_1d\left(\kappa_3^2\vee 1\right),
\end{align*}
Note that $\left\Vert\nabla_\theta^j Y_{\theta,h}^{(k)}\right\Vert\leq C_1dG$, we know $\left\Vert\nabla_\theta^j Y_{\theta,h}^{(k)}-\mathbb{E}\left[\nabla_\theta^j Y_{\theta,h}^{(k)}\right]\right\Vert\leq 2C_1dG$. By Matrix Bernstein's inequality, we get for any $\varepsilon>0$,
\begin{align*}
    \mathbb{P}\left(\left\Vert\sum_{k=1}^K\left(\nabla_\theta^j Y_{k,h}^\theta-\Sigma_h^{\frac{1}{2}}\left(\nabla_\theta^j M_{\theta,h}\right) \Sigma_{h+1}^{-\frac{1}{2}}\right)\right\Vert\geq\varepsilon\right)\leq 2d \exp\left(-\frac{\varepsilon^2/2}{G^2C_1dK\left(\kappa_3^2\vee 1\right)+2C_1dG \varepsilon/3}\right), 
\end{align*}
taking a union bound over all $j\in[m]$ and $h\in[H]$ proves that \eqref{3_2} holds with probability $1-\frac{\delta}{2}$. Using a union bound argument again, we know with probability $1-\delta$, \eqref{3_1} and \eqref{3_2} hold simultaneously, which has finished the proof.
\end{proof}
\begin{lemma}
\label{eps}
For $h\in[H]$, with probability at least $1-\delta$, the following inequalities hold simultaneously:
\begin{align}
    \label{w1}
    \left\Vert\Sigma_h^{-\frac{1}{2}}\frac{1}{K}\sum_{k=1}^K\phi\left(s_h^{(k)},a_h^{(k)}\right)\varepsilon_{h,k}^\theta\right\Vert&\leq\sqrt{d}(H-h+1)\left(\sqrt{\frac{2\log\frac{8dH}{\delta}}{K}}+\frac{2\sqrt{C_1d}\log\frac{8dH}{\delta}}{K}+\frac{2C_1d\left(\log\frac{8dH}{\delta}\right)^{\frac{3}{2}}}{3K^{\frac{3}{2}}}\right)\\
    \label{w2}
    \left\Vert\Sigma_h^{-\frac{1}{2}}\frac{1}{K}\sum_{k=1}^K\phi\left(s_h^{(k)},a_h^{(k)}\right)\nabla_\theta^j\varepsilon_{h,k}^\theta\right\Vert&\leq 2\sqrt{d}G(H-h)^2\left(\sqrt{\frac{2\log\frac{8mdH}{\delta}}{K}}+\frac{2\sqrt{C_1d}\log\frac{8mdH}{\delta}}{K}+\frac{2C_1d\left(\log\frac{8mdH}{\delta}\right)^{\frac{3}{2}}}{3K^{\frac{3}{2}}}\right),\forall j\in[m].
\end{align}
\end{lemma}
\begin{proof}
Let $X_{\theta,h}^{(k)}:=\Sigma_h^{-\frac{1}{2}}\phi\left(s_h^{(k)},a_h^{(k)}\right)\varepsilon_{h,k}^\theta\in\mathbb{R}^d$ and let $\mathcal{F}_{h,k}$ be $\sigma$-algebra generated by the history up to step $h$ at episode $k$, we have $\mathbb{E}\left[\left.X_{\theta,h}^{(k)}\right\vert\mathcal{F}_{h,k}\right]=0$. We apply matrix-form Freedman's inequality to analyze the concentration property. Consider conditional variances $\textrm{Var}_1\left[\left. X_{\theta,h}^{(k)}\right\vert\mathcal{F}_{h,k}\right]:=\mathbb{E}\left[\left.X_{\theta,h}^{(k)} \left(X_{\theta,h}^{(k)}\right)^\top\right\vert\mathcal{F}_{h,k}\right]\in\mathbb{R}^{d\times d}$ and $\textrm{Var}_2\left[\left.X_{\theta,h}^{(k)}\right\vert\mathcal{F}_{h,k}\right]:=\mathbb{E} \left[\left(X_{\theta,h}^{(k)}\right)^\top X_{\theta,h}^{(k)}\vert\mathcal{F}_{h,k}\right]\in\mathbb{R}$. It holds that
\begin{align*} 
    \left\Vert\textrm{Var}_1\left[\left.X_{\theta,h}^{(k)}\right\vert\mathcal{F}_{h,k}\right]\right\Vert=&\left\Vert\mathbb{E}\left[\left.X_{\theta,h}^{(k)}\left(X_{\theta,h}^{(k)}\right)^\top\right\vert\mathcal{F}_{h,k}\right]\right\Vert\leq\mathbb{E}\left[\left.\left\Vert X_{\theta,h}^{(k)}\left(X_{\theta,h}^{(k)}\right)^\top\right\Vert\right\vert\mathcal{F}_{h,k}\right]\\
    =&\mathbb{E}\left[\left.\left\Vert X_{\theta,h}^{(k)}\right\Vert^2\right\vert\mathcal{F}_{h,k}\right]= \textrm{Var}_2\left[\left.X_{\theta,h}^{(k)}\right\vert\mathcal{F}_{h,k}\right] 
\end{align*}
and
\begin{align*} 
    \textrm{Var}_2\left[\left.X_{\theta,h}^{(k)}\right\vert\mathcal{F}_{h,k}\right]=&\mathbb{E}\left[\left.\left\Vert X_{\theta,h}^{(k)}\right\Vert^2\right\vert\mathcal{F}_{h,k}\right]=\phi\left(s_h^{(k)},a_h^{(k)}\right)^\top\Sigma_h^{-1}\phi\left(s_h^{(k)},a_h^{(k)}\right)\textrm{Var}\left[\left.\varepsilon_{h,k}^\theta\right\vert s_h^{(k)},a_h^{(k)},h\right]\\
    \leq&(H-h+1)^2\phi\left(s_h^{(k)},a_h^{(k)}\right)^\top\Sigma_h^{-1}\phi\left(s_h^{(k)},a_h^{(k)}\right), 
\end{align*}
where we have used $\textrm{Var}\left[\left.\varepsilon_{h,k}^\theta\right\vert s_h^{(k)},a_h^{(k)},h \right]\leq(H-h+1)^2$. Note that
\begin{align*} 
    \sum_{k=1}^K\phi\left(s_h^{(k)},a_h^{(k)}\right)^\top\Sigma_h^{-1}\phi\left(s_h^{(k)},a_h^{(k)}\right)=&Kd+K\textrm{Tr}\left(\Sigma_h^{-\frac{1}{2}}\left(\frac{1}{K}\sum_{k=1}^K\phi\left(s_h^{(k)},a_h^{(k)}\right)\phi\left(s_h^{(k)},a_h^{(k)}\right)^\top\right)\Sigma_h^{-\frac{1}{2}}-I_d\right)\\
    \leq&Kd+Kd\left\Vert\Sigma_h^{-\frac{1}{2}}\left(\frac{1}{K}\sum_{k=1}^K\phi\left(s_h^{(k)},a_h^{(k)}\right)\phi\left(s_h^{(k)},a_h^{(k)}\right)^\top\right)\Sigma_h^{-\frac{1}{2}}-I_d\right\Vert. 
\end{align*}
We take
\begin{align} 
    \label{sigma2} 
    \sigma^2:=Kd(H-h+1)^2\left(1+\sqrt{\frac{2C_1d\log\frac{8dH}{\delta}}{K}}+\frac{2C_1d\log\frac{8dH}{\delta}}{3K}\right). 
\end{align}
According to Lemma \ref{dsig1}, it holds that
\begin{align}
    \label{Var2} 
    \mathbb{P}\left(\left\Vert\sum_{k=1}^K\textrm{Var}_1\left[\left.X_{\theta,h}^{(k)}\right\vert\mathcal{F}_{h,k}\right]\right\Vert\leq\sum_{k=1}^K\textrm{Var}_2\left[\left.X_{\theta,h}^{(k)}\right\vert\mathcal{F}_{h,k}\right]\leq\sigma^2\right)\geq 1-\frac{\delta}{4}. 
\end{align}
Additionally, we have $\left\Vert X_{\theta,h}^{(k)}\right\Vert\leq(H-h+1)\sqrt{C_1d}$. The Freedman's inequality therefore implies that for any $\varepsilon>0$,
{\small
\begin{align} 
    \label{Freedman2} 
    \begin{aligned}
        &\mathbb{P}\left(\left\vert\sum_{k=1}^K X_{\theta,h}^{(k)}\right\vert\geq \varepsilon,\ \max\left\{\left\Vert\sum_{k=1}^K\textrm{Var}_1\left[\left.X_{\theta,h}^{(k)}\right\vert\mathcal{F}_{h,k}\right]\right\Vert,\sum_{k=1}^K\textrm{Var}_2\left[\left.X_{\theta,h}^{(k)}\right\vert\mathcal{F}_{h,k}\right]\right\}\leq\sigma^2\right)\leq 2d\exp\left(-\frac{\varepsilon^2/2}{\sigma^2+(H-h+1)\sqrt{C_1d}\varepsilon/3} \right),
    \end{aligned} 
\end{align}}
where $\sigma^2$ is defined in \eqref{sigma2}. We take
\begin{align*}
    \varepsilon:=\sigma\sqrt{2\log\frac{8d}{\delta}}+\frac{2(H-h+1)\sqrt{C_1d}}{3}\log\frac{8d}{\delta}.
\end{align*}
Then we get
\begin{align*}
    \mathbb{P}\left(\left\vert\sum_{k=1}^K X_{\theta,h}^{(k)}\right\vert\geq\varepsilon,\ \max\left\{\left\Vert\sum_{k=1}^K\textrm{Var}_1\left[\left.X_{\theta,h}^{(k)}\right\vert\mathcal{F}_{h,k}\right]\right\Vert,\sum_{k=1}^K\textrm{Var}_2\left[\left.X_{\theta,h}^{(k)}\right\vert\mathcal{F}_{h,k}\right]\right\}\leq\sigma^2\right)\leq\frac{\delta}{4},
\end{align*}
which implies
\begin{align*}
    \mathbb{P}\left(\left\vert\sum_{k=1}^K X_{\theta,h}^{(k)}\right\vert\geq \varepsilon\right)\leq&\mathbb{P}\left(\left\vert\sum_{k=1}^K X_{\theta,h}^{(k)}\right\vert\geq \varepsilon,\ \max\left\{\left\Vert\sum_{k=1}^K\textrm{Var}_1\left[\left.X_{\theta,h}^{(k)}\right\vert\mathcal{F}_{h,k}\right]\right\Vert,\sum_{k=1}^K\textrm{Var}_2\left[\left.X_{\theta,h}^{(k)}\right\vert\mathcal{F}_{h,k}\right]\right\}\leq\sigma^2\right)\\
    &+\mathbb{P}\left(\max\left\{\left\Vert\sum_{k=1}^K\textrm{Var}_1\left[\left.X_{\theta,h}^{(k)}\right\vert\mathcal{F}_{h,k}\right]\right\Vert,\sum_{k=1}^K\textrm{Var}_2\left[\left.X_{\theta,h}^{(k)}\right\vert\mathcal{F}_{h,k}\right]\right\}>\sigma^2\right)\leq\frac{\delta}{2}.
\end{align*}
which, combined with a union bound over $h\in[H]$, has proved \eqref{w1}. We use Freedman's inequality again to prove \eqref{w2}. For a fixed $j\in[m]$, we have
\begin{align*} 
    \left\Vert\textrm{Var}_1\left[\left.\nabla_\theta^j X_{\theta,h}^{(k)}\right\vert\mathcal{F}_{h,k}\right]\right\Vert=&\left\Vert\mathbb{E}\left[\left.\left(\nabla_\theta^j X_{\theta,h}^{(k)}\right)\left(\nabla_\theta^j X_{\theta,h}^{(k)}\right)^\top\right\vert\mathcal{F}_{h,k}\right]\right\Vert\leq\mathbb{E}\left[\left.\left\Vert\left(\nabla_\theta^j X_{\theta,h}^{(k)}\right)\left(\nabla_\theta^j X_{\theta,h}^{(k)}\right)^\top\right\Vert\right\vert\mathcal{F}_{h,k}\right]\\
    =&\mathbb{E}\left[\left.\left\Vert\nabla_\theta^j X_{\theta,h}^{(k)}\right\Vert^2\right\vert\mathcal{F}_{h,k}\right]= \textrm{Var}_2\left[\left.\nabla_\theta^j X_{\theta,h}^{(k)}\right\vert\mathcal{F}_{h,k}\right] 
\end{align*}
and
\begin{align*} 
    \textrm{Var}_2\left[\left.\nabla_\theta^j X_{\theta,h}^{(k)}\right\vert\mathcal{F}_{h,k}\right]=&\mathbb{E}\left[\left.\left\Vert\nabla_\theta^j X_{\theta,h}^{(k)}\right\Vert^2\right\vert\mathcal{F}_{h,k}\right]=\phi\left(s_h^{(k)},a_h^{(k)}\right)^\top\Sigma_h^{-1}\phi\left(s_h^{(k)},a_h^{(k)}\right)\textrm{Var}\left[\left.\nabla_\theta^j\varepsilon_{h,k}^\theta\right\vert s_h^{(k)},a_h^{(k)}\right]\\
    \leq&4G^2(H-h)^2\phi\left(s_h^{(k)},a_h^{(k)}\right)^\top\Sigma_h^{-1}\phi\left(s_h^{(k)},a_h^{(k)}\right),  
\end{align*}
where we have used $\textrm{Var}\left[\left.\nabla_\theta^j\varepsilon_{h,k}^\theta\right\vert s_h^{(k)}, a_h^{(k)},h \right]\leq 4G^2(H-h)^2$. Furthermore, notice that $\left\Vert\nabla_\theta^j X_{\theta,h}^{(k)}\right\Vert\leq 2G(H-h)\sqrt{C_1d}$, the remaining steps will be exactly the same as those in the proof of \eqref{w1}, combined with a union bound over $j\in[m]$. In this way, we have proved \eqref{w2}. Taking a union bound again finishes the proof. 
\end{proof}

\section{Proofs of Main Theorems}
Define $\widehat{\nu}^\theta_h:=\left(\prod_{h^\prime=1}^{h-1}\widehat{M}_{\theta,h^\prime}^\top\right)\nu_1^\theta$.  We may prove the following decomposition of $\nabla_\theta v_\theta-\widehat{\nabla_\theta v_\theta}$:
\begin{lemma}
\label{error_decomp}
We have $\nabla_\theta v_\theta-\widehat{\nabla_\theta v_\theta}=E_1+E_2+E_3$, where 
\begin{align*}
    E_1=&\sum_{h=1}^H\nabla_\theta\left[\left(\nu^\theta_h\right)^\top\Sigma^{-1}_h\frac{1}{K}\sum_{k=1}^K\phi\left(s_h^{(k)},a_h^{(k)}\right)\varepsilon_{h,k}^\theta\right]\\
    E_2=&\sum_{h=1}^H\nabla_\theta\left[\left(\left(\widehat{\nu}^\theta_h\right)^\top\widehat{\Sigma}_h^{-1}-\left(\nu^\theta_h\right)^\top\Sigma_h^{-1}\right)\frac{1}{K}\sum_{k=1}^K\phi\left(s_h^{(k)},a_h^{(k)}\right)\varepsilon_{h,k}^\theta\right]\\
    E_3=&\frac{\lambda}{K}\sum_{h=1}^H\nabla_\theta\left[\left(\widehat{\nu}^\theta_h\right)^\top\widehat{\Sigma}_h^{-1}w_h^\theta\right].
\end{align*}
\end{lemma}
The proof of Lemma \ref{error_decomp} is deferred to appendix \ref{missing_proof}. Based on this observation, here we show the proofs of our main theorems. 

\subsection{Proof of Theorem \ref{thm2_var}}
\label{pfthm2_var}
\begin{proof}
We use Lemma \ref{error_decomp} to decompose $\langle\nabla_\theta v_\theta - \widehat{\nabla_\theta v_\theta}, t\rangle=\langle E_1, t\rangle+\langle E_2,t\rangle+\langle E_3,t\rangle$. To bound each term individually, we introduce the following lemmas, whose proofs are deferred to appendix \ref{missing_proof}. 
\begin{lemma}
\label{e1_finite_product}
For any $t\in\mathbb{R}^m$, with probability $1-\delta$, we have
\begin{align*}
    \vert\langle E_1, t\rangle\vert\leq\sqrt{\frac{2t^\top\Lambda_\theta t\log(2/\delta)}{K}}+\frac{2\log(2/\delta)\sqrt{C_1md}\Vert t\Vert B}{3K}. 
\end{align*}
where $B=\sum_{h=1}^H(H-h+1)\max_{j\in[m]}\sqrt{\left(\nabla^j_\theta\nu^\theta_h\right)^\top\Sigma_h^{-1}\nabla^j_\theta\nu^\theta_h}+2G\sum_{h=1}^H(H-h)^2\sqrt{\left(\nu^\theta_h\right)^\top\Sigma_h^{-1}\nu^\theta_h}$.
\end{lemma}
\begin{lemma}
\label{e2}
Let $E_2^j$ be the $j$th entry of $E_2$, suppose $K\geq 36\kappa_1(4+\kappa_2+\kappa_3)^2C_1dH^2\log\frac{24dmH}{\delta}$ and\\
$\lambda\leq C_1d\min_{h\in[H]}\sigma_{\min}(\Sigma_h)\log\frac{24dmH}{\delta}$, then with probability $1-\delta$, 
\begin{align*}
    \vert E_2^j\vert\leq 240\sqrt{\kappa_1}(2+\kappa_2+\kappa_3)\sqrt{C_1}dH^3\left(\left\Vert\Sigma_{\theta,1}^{-\frac{1}{2}}\nabla_\theta^j\nu^\theta_1\right\Vert+HG\left\Vert\Sigma_{\theta,1}^{-\frac{1}{2}}\nu^\theta_1\right\Vert\right)\max_{h\in[H]}\left\Vert\Sigma_{\theta,h}^{\frac{1}{2}}\Sigma_h^{-\frac{1}{2}}\right\Vert\frac{\log\frac{24dmH}{\delta}}{K},\quad\forall j\in[m].
\end{align*}
\end{lemma}
\begin{lemma}
\label{e3}
Let $E_3^j$ be the $j$th entry of $E_3$, suppose $K\geq 36\kappa_1(4+\kappa_2+\kappa_3)^2\log\frac{24dmH}{\delta}C_1dH^3$ and\\
$\lambda\leq C_1dH\min_{h\in[H]}\sigma_{\min}(\Sigma_h)\log\frac{24dmH}{\delta}$, with probability $1-\delta$, 
\begin{align*}
    \vert E_3^j\vert\leq&6C_1dH^2\max_{h\in[H]}\left\Vert\Sigma_{\theta,h}^{\frac{1}{2}}\Sigma_h^{-\frac{1}{2}}\right\Vert\left(\left\Vert\Sigma_{\theta,1}^{-\frac{1}{2}}\nabla_\theta^j\nu^{\theta}_1\right\Vert+HG\left\Vert\Sigma_{\theta,1}^{-\frac{1}{2}}\nu^\theta_1\right\Vert\right)\frac{\log\frac{24dmH}{\delta}}{K},\quad\forall j\in[m]. 
\end{align*}
\end{lemma}
Let $B_1^j=\sum_{h=1}^H(H-h+1)\sqrt{\left(\nabla^j_\theta\nu^\theta_h\right)^\top\Sigma_h^{-1}\nabla^j_\theta\nu^\theta_h},\ B_2=\sum_{h=1}^H(H-h)^2G\sqrt{\left(\nu^\theta_h\right)^\top\Sigma_h^{-1}\nu^\theta_h}$, then we have the relation $B = \max_{j\in[m]}B_1^j + 2B_2$. For any $j\in[m]$, note that
{\small
\begin{align*}
    B_1^j=&\sum_{h=1}^H (H-h+1)\left\Vert \Sigma_h^{-\frac{1}{2}}\nabla_\theta^j\nu^\theta_h\right\Vert\leq \sum_{h=1}^H H\left\Vert\Sigma_{\theta,h}^{\frac{1}{2}}\Sigma_h^{-\frac{1}{2}}\right\Vert\left\Vert\Sigma_{\theta,h}^{-\frac{1}{2}}\nabla_\theta^j\left(\left(\prod_{h^\prime=1}^{h-1}M_{\theta,h^\prime}\right)^\top\nu^\theta_1\right)\right\Vert\\
    =&\sum_{h=1}^H H\left\Vert\Sigma_{\theta,h}^{\frac{1}{2}}\Sigma_h^{-\frac{1}{2}}\right\Vert\\
    \cdot&\left(\left(\prod_{h^\prime=1}^{h-1}\left\Vert\Sigma_{\theta,h^\prime}^{\frac{1}{2}}M_{\theta,h^\prime}\Sigma_{\theta,h^\prime+1}^{-\frac{1}{2}}\right\Vert\right)\left\Vert\Sigma_{\theta,1}^{-\frac{1}{2}}\nabla_\theta^j\nu^{\theta}_1\right\Vert+\sum_{h^\prime=1}^{h-1}\left(\prod_{h^{\prime\prime}\neq h^\prime}\left\Vert\Sigma_{\theta,h^{\prime\prime}}^{\frac{1}{2}}M_{\theta,h^{\prime\prime}}\Sigma_{\theta,h^{\prime\prime}+1}^{-\frac{1}{2}}\right\Vert\right)\left\Vert\Sigma_{\theta,h^\prime}^{\frac{1}{2}}\left(\nabla_\theta^j M_{\theta,h^\prime}\right)\Sigma_{\theta,h^\prime+1}^{-\frac{1}{2}}\right\Vert\left\Vert\Sigma_{\theta,1}^{-\frac{1}{2}}\nu^\theta_1\right\Vert\right)\\
    \leq&\sum_{h=1}^H H\left\Vert\Sigma_{\theta,h}^{\frac{1}{2}}\Sigma_h^{-\frac{1}{2}}\right\Vert\left(\left\Vert\Sigma_{\theta,1}^{-\frac{1}{2}}\nabla_\theta^j\nu^\theta_1\right\Vert+HG\left\Vert\Sigma_{\theta,1}^{-\frac{1}{2}}\nu^\theta_1\right\Vert\right)\leq H^2\max_{h\in[H]}\left\Vert\Sigma_{\theta,h}^{\frac{1}{2}}\Sigma_h^{-\frac{1}{2}}\right\Vert\left(\left\Vert\Sigma_{\theta,1}^{-\frac{1}{2}}\nabla_\theta^j\nu^\theta_1\right\Vert+HG\left\Vert\Sigma_{\theta,1}^{-\frac{1}{2}}\nu^\theta_1\right\Vert\right),
\end{align*}}
where we use the result of Lemma \ref{ineq}. Similarly, 
\begin{align*}
    B_2&=\sum_{h=1}^H(H-h)^2G\left\Vert\Sigma_h^{-\frac{1}{2}}\nu^\theta_h\right\Vert\leq \sum_{h=1}^H H^2G\left\Vert\Sigma_{\theta,h}^{\frac{1}{2}}\Sigma_h^{-\frac{1}{2}}\right\Vert\left(\prod_{h^\prime=1}^{h-1}\left\Vert\Sigma_{\theta,h^\prime}^{\frac{1}{2}}M_{\theta,h^\prime}\Sigma_{\theta,h^\prime+1}^{-\frac{1}{2}}\right\Vert\right)\left\Vert\Sigma_{\theta,1}^{-\frac{1}{2}}\nu^\theta_1\right\Vert\\
    &\leq\sum_{h=1}^H H^2G\left\Vert\Sigma_{\theta,h}^{\frac{1}{2}}\Sigma_h^{-\frac{1}{2}}\right\Vert\left\Vert\Sigma_{\theta,1}^{-\frac{1}{2}}\nu^\theta_1\right\Vert\leq H^3G\max_{h\in[H]}\left\Vert\Sigma_{\theta,h}^{\frac{1}{2}}\Sigma_h^{-\frac{1}{2}}\right\Vert\left\Vert\Sigma_{\theta,1}^{-\frac{1}{2}}\nu^\theta_1\right\Vert.
\end{align*}
We conclude that when $K\geq 36C_1dH^2\kappa_1(4+\kappa_2+\kappa_3)^2\log\frac{24dmH}{\delta}$, we have
\begin{align*}
    \left(\max_{j\in[m]}B_1^j+2B_2\right)\frac{2\log\frac{2}{\delta}\sqrt{C_1md}\Vert t\Vert}{K}\leq H^2\max_{h\in[H]}\left\Vert\Sigma_{\theta,h}^{\frac{1}{2}}\Sigma_h^{-\frac{1}{2}}\right\Vert\left(\max_{j\in[m]}\left\Vert\Sigma_{\theta,1}^{-\frac{1}{2}}\nabla_\theta^j\nu^\theta_1\right\Vert+2HG\left\Vert\Sigma_{\theta,1}^{-\frac{1}{2}}\nu^\theta_1\right\Vert\right)\frac{2\log\frac{2}{\delta}\sqrt{C_1md}\Vert t\Vert}{K},
\end{align*}
and therefore, with probability $1-3\delta$, we have
\begin{align*}
    &\vert\langle E_1,t\rangle\vert+\vert\langle E_2,t\rangle\vert+\vert \langle E_3, t\rangle\vert\leq\sqrt{\frac{2t^\top\Lambda_\theta t\log(2/\delta)}{K}}\\
    &+\sqrt{\kappa_1}(5+\kappa_2+\kappa_3)\left(\max_{j\in[m]}\left\Vert\Sigma_{\theta,1}^{-\frac{1}{2}}\nabla_\theta^j\nu^\theta_1\right\Vert+HG\left\Vert\Sigma_{\theta,1}^{-\frac{1}{2}}\nu^\theta_1\right\Vert\right)\max_{h\in[H]}\left\Vert\Sigma_{\theta,h}^{\frac{1}{2}}\Sigma_h^{-\frac{1}{2}}\right\Vert\frac{240C_1dH^3\log\frac{24dmH}{\delta}}{K}\\
    \leq&\sqrt{\frac{2t^\top\Lambda_\theta t\log(2/\delta)}{K}}+\kappa_1(5+\kappa_2+\kappa_3)\left(\max_{j\in[m]}\left\Vert \Sigma_{\theta,1}^{-\frac{1}{2}}\nabla_\theta^j\nu^\theta_1\right\Vert+HG\left\Vert\Sigma_{\theta,1}^{-\frac{1}{2}}\nu^\theta_1\right\Vert\right)\frac{240C_1dH^3\sqrt{m}\Vert t\Vert\log\frac{24dmH}{\delta}}{K}.
\end{align*}
replacing $\delta$ by $\frac{\delta}{3}$, we have finished the proof. 
\end{proof}

\subsection{Proof of Theorem \ref{thm2}}
\label{pfthm2}
\begin{proof}
According to the result of Theorem \ref{thm2_var}, we know
\begin{align*}
    &\vert\langle t,\widehat{\nabla_\theta v_\theta}-\nabla_\theta v_\theta\rangle\vert\leq\sqrt{\frac{2t^\top\Lambda_\theta t}{K}\cdot \log\frac{8}{\delta}}+\frac{C_\theta\Vert t\Vert\log\frac{72mdH}{\delta}}{K},
\end{align*}
Pick $t=e_j, j\in[m]$, we have
\begin{align*}
    t^\top \Lambda_\theta t=&\mathbb{E}\left[\sum_{h=1}^H\left(\nabla_\theta^j\left(\varepsilon^\theta_{h,1}\phi\left(s_h^{(1)},a_h^{(1)}\right)^\top\Sigma_h^{-1}\nu_h^\theta\right)\right)^2\right]\\
    \leq&2\mathbb{E}\left[\sum_{h=1}^H\left(\nabla_\theta^j\varepsilon^\theta_{h,1}\phi\left(s_h^{(1)},a_h^{(1)}\right)^\top\Sigma_h^{-1}\nu_h^\theta\right)^2\right]+2\mathbb{E}\left[\sum_{h=1}^H\left(\varepsilon^\theta_{h,1}\phi\left(s_h^{(1)},a_h^{(1)}\right)^\top\Sigma_h^{-1}\nabla_\theta^j\nu_h^\theta\right)^2\right]\\
    \leq&2\mathbb{E}\left[\sum_{h=1}^HG^2(H-h)^4\left(\phi\left(s_h^{(1)},a_h^{(1)}\right)^\top\Sigma_h^{-1}\nu_h^\theta\right)^2\right]+2\mathbb{E}\left[\sum_{h=1}^H(H-h+1)^2\left(\phi\left(s_h^{(1)},a_h^{(1)}\right)^\top\Sigma_h^{-1}\nabla_\theta^j\nu_h^\theta\right)^2\right]\\
     \leq&2\sum_{h=1}^H\left((H-h)^4G^2\left\Vert\Sigma_h^{-\frac{1}{2}}\nu_h^\theta\right\Vert^2+(H-h+1)^2\left\Vert\Sigma_h^{-\frac{1}{2}}\nabla_\theta^j\nu_h^\theta\right\Vert^2\right).
\end{align*}
On the other hand, we have
\begin{align*}
    t^\top \Lambda_\theta t=& \mathbb{E}\left[\sum_{h=1}^H\left(\nabla_\theta^j\left(\varepsilon^\theta_{h,1}\phi\left(s_{h}^{(1)},a_{h}^{(1)}\right)^\top\Sigma_h^{-1}\nu_h^\theta\right)\right)^2\right]\\
    \leq&2\mathbb{E}\left[\sum_{h=1}^H\left(\nabla_\theta^j\varepsilon^\theta_{h,1}\phi\left(s_h^{(1)},a_h^{(1)}\right)^\top\Sigma_h^{-1}\nu_h^\theta\right)^2\right]+2\mathbb{E}\left[\sum_{h=1}^H\left(\varepsilon^\theta_{h,1}\phi\left(s_h^{(1)},a_h^{(1)}\right)^\top\Sigma_h^{-1}\nabla_\theta^j\nu_h^\theta\right)^2\right]\\
    \leq&2C_1d\mathbb{E}\left[\sum_{h=1}^H\left(\nabla_\theta^j \varepsilon^\theta_{h,1}\right)^2\right]\max_{h\in[H]}\left\Vert\Sigma_h^{-\frac{1}{2}}\nu_h^\theta\right\Vert^2+2C_1d\mathbb{E}\left[\sum_{h=1}^H\left(\varepsilon^\theta_{h,1}\right)^2\right]\max_{h\in[H]}\left\Vert\Sigma_h^{-\frac{1}{2}}\nabla_\theta^j\nu_h^\theta\right\Vert^2.
\end{align*}
Define
\begin{align*}
    \zeta_h&=\nabla_\theta^j Q_h^\theta\left(s_h^{(1)},a_h^{(1)}\right)-\int_{\mathcal{A}}\pi_{\theta,h+1}\left(a\left\vert s_{h+1}^{(1)}\right.\right)\left(\nabla_\theta^j Q_{h+1}^\theta\left(s_{h+1}^{(1)},a\right)+Q_{h+1}^\theta\left(s_{h+1}^{(1)}, a\right)\nabla_\theta^j\log\pi_{\theta,h+1}\left(a\left\vert s_{h+1}^{(1)}\right.\right)\right)\mathrm{d}a\\
    \eta_h&=\int_{\mathcal{A}}\pi_{\theta,h+1}\left(a\left\vert s_{h+1}^{(1)}\right.\right)\left(\nabla_\theta^j Q_{h+1}^\theta\left(s_{h+1}^{(1)},a\right)+Q_{h+1}^\theta\left(s_{h+1}^{(1)},a\right)\nabla_\theta^j\log\pi_{\theta,h+1}\left(a\left\vert s_{h+1}^{(1)}\right.\right)\right)\mathrm{d}a\\
    &-\nabla_\theta^j Q_{h+1}^\theta\left(s_{h+1}^{(1)},a_{h+1}^{(1)}\right)-Q_{h+1}^\theta\left(s_{h+1}^{(1)},a_{h+1}^{(1)}\right)\nabla_\theta^j\log\pi_{\theta,h+1}\left(a_{h+1}^{(1)}\left\vert s_{h+1}^{(1)}\right.\right).
\end{align*}
Note that the sequence $\zeta_1,\eta_1,\zeta_2,\eta_2\ldots, \zeta_H,\eta_H$ forms a martingale difference sequence, therefore, we have
\begin{align*}
    \mathbb{E}\left[\sum_{h=1}^H\left(\nabla_\theta^j\varepsilon^\theta_{h,1}\right)^2\right]=&\mathbb{E}\left[\sum_{h=1}^H\zeta_h^2\right]\leq\mathbb{E}\left[\sum_{h=1}^H(\zeta_h^2+\eta_h^2)\right]=\mathbb{E}\left[\left(\sum_{h=1}^H\left(\zeta_h+\eta_h\right)\right)^2\right]\\
    =&\mathbb{E}\left[\left(\nabla_\theta Q_1^\theta\left(s_1^{(1)},a_1^{(1)}\right)-\sum_{h=1}^HQ_{h+1}^\theta\left(s_{h+1}^{(1)},a^{(1)}_{h+1}\right)\nabla_\theta \log\pi_{\theta,h+1}\left(a_{h+1}^{(1)}\left\vert s_{h+1}^{(1)}\right.\right)\right)^2\right]\\
    \leq& 4H^4G^2.
\end{align*}
Similarly, 
\begin{align*}
    \mathbb{E}\left[\sum_{h=1}^H\left(\varepsilon^\theta_{h,1}\right)^2\right]=&\mathbb{E}\left[\sum_{h=1}^H\left(Q_h^\theta\left(s_h^{(1)},a_h^{(1)}\right)-r_h^{(1)}-\int_{\mathcal{A}}\pi_{\theta,h+1}\left(a\left\vert s^{(1)}_{h+1}\right.\right)Q_{h+1}^\theta\left(s^{(1)}_{h+1},a\right)\mathrm{d}a\right)^2\right]\\
    &+\mathbb{E}\left[\sum_{h=1}^H\left(\int_{\mathcal{A}}\pi_{\theta,h+1}\left(a\left\vert s_{h+1}^{(1)}\right.\right)Q_{h+1}^\theta\left(s_{h+1}^{(1)},a\right)\mathrm{d}a-Q_{h+1}^\theta\left(s_{h+1}^{(1)},a_{h+1}^{(1)}\right)\right)^2\right]\\
    =&\mathbb{E}\left[\left(Q_1^\theta\left(s_1^{(1)},a_1^{(1)}\right)-\sum_{h=1}^Hr_h^{(1)}\right)^2\right]\leq H^2.
\end{align*}
Therefore, 
\begin{align*}
    t^\top \Lambda_\theta t \leq&8C_1dH^4G^2\max_{h\in[H]}\left\Vert\Sigma_h^{-\frac{1}{2}}\nu_h^\theta\right\Vert^2+2C_1dH^2\max_{h\in[H]}\left\Vert\Sigma_h^{-\frac{1}{2}}\nabla_\theta^j\nu_h^\theta\right\Vert^2.
\end{align*}
Therefore, taking a union bound over $j$, we get
\begin{align*}
\left\vert\widehat{\nabla_\theta^j v_\theta}-\nabla_\theta^j v_\theta\right\vert\leq 4b_\theta\sqrt{\frac{\min\left\{C_1d,H\right\}\log\frac{8m}{\delta}}{K}}+\frac{2C_\theta\sqrt{m}\log\frac{72mdH}{\delta}}{K},\quad\forall j\in[m],
\end{align*}
where 
\begin{align*}
b_\theta=H^2G\max_{h\in[H]}\left\Vert\Sigma_h^{-\frac{1}{2}}\nu_h^\theta\right\Vert+H\max_{h\in[H]}\left\Vert\Sigma_h^{-\frac{1}{2}}\nabla_\theta^j\nu_h^\theta\right\Vert.
\end{align*}
When we in addition have $\phi(s^\prime,a^\prime)^\top\Sigma_h^{-1}\phi(s,a)\geq 0,\forall h\in[H],(s,a),(s^\prime,a^\prime)\in\mathcal{S}\times\mathcal{A}$, we have for any $(s,a)\in\mathcal{S}\times\mathcal{A}$, 
\begin{align*}
    \left\vert\left(\nabla^j_\theta\nu^\theta_{h}\right)^\top\Sigma^{-1}_h\phi(s,a)\right\vert =& \left\vert\mathbb{E}^{\pi_\theta}\left[\phi(s_h,a_h)\Sigma^{-1}_h\phi(s,a)\sum_{h^\prime=1}^h\nabla_\theta^j\log\pi_{\theta,h^\prime}(a_{h^\prime}\vert s_{h^\prime})\right]\right\vert\\
    \leq&\mathbb{E}^{\pi_\theta}\left[\phi(s_h, a_h)\Sigma^{-1}_h\phi(s, a)\sum_{h^\prime=1}^h\left\vert\nabla_\theta^j\log\pi_{\theta,h^\prime}(a_{h^\prime}\vert s_{h^\prime})\right\vert\right]\\
    \leq&Gh\mathbb{E}^{\pi_\theta}\left[\phi(s_h,a_h)\Sigma^{-1}_h\phi(s,a)\right]\\
    =&Gh\left(\nu^\theta_{h}\right)^\top\Sigma^{-1}_h\phi(s,a),
\end{align*}
which implies
\begin{align*}
    \left(\phi\left(s_h^{(1)},a_h^{(1)}\right)^\top\Sigma_h^{-1}\nabla_\theta^j\nu_{h}^\theta\right)^2\leq G^2h^2\left(\phi\left(s_h^{(1)},a_h^{(1)}\right)^\top\Sigma_h^{-1}\nu_h^\theta\right)^2.
\end{align*}
Therefore, we get 
\begin{align*}
    t^\top\Lambda_\theta t=&\mathbb{E}\left[\sum_{h=1}^H\left(\nabla_\theta^j\left(\varepsilon^\theta_{h,1}\phi\left(s_h^{(1)},a_h^{(1)}\right)^\top\Sigma_h^{-1}\nu_h^\theta\right)\right)^2\right]\\
    \leq&2\mathbb{E}\left[\sum_{h=1}^H\left(\nabla_\theta^j\varepsilon^\theta_{h,1}\phi\left(s_h^{(1)},a_h^{(1)}\right)^\top\Sigma_h^{-1}\nu_h^\theta\right)^2\right]+2\mathbb{E}\left[\sum_{h=1}^H\left(\varepsilon^\theta_{h,1}\phi\left(s_h^{(1)},a_h^{(1)}\right)^\top\Sigma_h^{-1}\nabla_\theta^j\nu_h^\theta\right)^2\right]\\
    \leq&2\mathbb{E}\left[\sum_{h=1}^HG^2(H-h)^4\left(\phi\left(s_h^{(1)},a_h^{(1)}\right)^\top\Sigma_h^{-1}\nu_h^\theta\right)^2\right]+2\mathbb{E}\left[\sum_{h=1}^HG^2h^2(H-h+1)^2\left(\phi\left(s_h^{(1)},a_h^{(1)}\right)^\top\Sigma_h^{-1}\nu_h^\theta\right)^2\right]\\
     \leq&2H^2G^2\sum_{h=1}^H(H-h+1)^2\left\Vert\Sigma_h^{-\frac{1}{2}}\nu_h^\theta\right\Vert^2,
\end{align*}
and
\begin{align*}
    t^\top\Lambda_\theta t\leq&2\mathbb{E}\left[\sum_{h=1}^H\left(\nabla_\theta^j \varepsilon^\theta_{h,1}\phi\left(s_h^{(1)},a_h^{(1)}\right)^\top\Sigma_h^{-1}\nu_h^\theta\right)^2\right]+2\mathbb{E}\left[\sum_{h=1}^H\left(\varepsilon^\theta_{h,1}\phi\left(s_h^{(1)},a_h^{(1)}\right)^\top\Sigma_h^{-1}\nabla_\theta^j\nu_h^\theta\right)^2\right]\\
    \leq&2C_1d\mathbb{E}\left[\sum_{h=1}^H\left(\nabla_\theta^j \varepsilon^\theta_{h,1}\right)^2\right]\max_{h\in[H]}\left\Vert\Sigma_h^{-\frac{1}{2}}\nu_h^\theta\right\Vert^2+2C_1dG^2H^2\mathbb{E}\left[\sum_{h=1}^H\left(\varepsilon^\theta_{h,1}\right)^2\right]\max_{h\in[H]}\left\Vert\Sigma_h^{-\frac{1}{2}}\nu_h^\theta\right\Vert^2
\end{align*}
Repeating the steps that we bound $\mathbb{E}\left[\sum_{h=1}^H \left(\varepsilon^\theta_{h,1}\right)^2\right]$ and $\mathbb{E}\left[\sum_{h=1}^H \left(\nabla_\theta^j\varepsilon^\theta_{h,1}\right)^2\right]$, we get
\begin{align*}
    \left\vert\widehat{\nabla_\theta^j v_\theta}-\nabla_\theta^j v_\theta\right\vert\leq 4H^2G\sqrt{\frac{\min\{C_1d,H\}\log\frac{8m}{\delta}}{K}}\max_{h\in[H]}\left\Vert\Sigma_h^{-\frac{1}{2}}\nu_h^\theta\right\Vert+\frac{2C_\theta\log\frac{72mdH}{\delta}}{K},\quad\forall j\in[m].
\end{align*}
\end{proof}

\subsection{Proof of Theorem \ref{thm_tabular}}
\begin{proof}
When the MDP is tabular and $\phi$ is the one-hot vector, we have
\begin{align*}
    \nu^\theta_{h,s,a} = \mu_{\theta,h}(s,a), \quad\Sigma_h=\textrm{diag}(\bar{\mu}_h(s,a))
\end{align*}
which implies
\begin{align*}
    t^\top \Lambda_\theta t=& \mathbb{E}\left[\sum_{h=1}^H\left(\nabla_\theta^j\left(\varepsilon^\theta_{h,1}\phi\left(s_{h}^{(1)},a_{h}^{(1)}\right)^\top\Sigma_h^{-1}\nu_h^\theta\right)\right)^2\right]\\
    \leq&2\mathbb{E}\left[\sum_{h=1}^H\left(\nabla_\theta^j\varepsilon^\theta_{h,1}\phi\left(s_h^{(1)},a_h^{(1)}\right)^\top\Sigma_h^{-1}\nu_h^\theta\right)^2\right]+2\mathbb{E}\left[\sum_{h=1}^H\left(\varepsilon^\theta_{h,1}\phi\left(s_h^{(1)},a_h^{(1)}\right)^\top\Sigma_h^{-1}\nabla_\theta^j\nu_h^\theta\right)^2\right]\\
    =&2\mathbb{E}\left[\sum_{h=1}^H\left(\nabla_\theta^j \varepsilon^\theta_{h,1}\right)^2\left(\frac{\mu_{\theta,h}(s,a)}{\bar{\mu}_h(s,a)}\right)^2\right]+2G^2H^2\mathbb{E}\left[\sum_{h=1}^H\left(\varepsilon^\theta_{h,1}\right)^2\left(\frac{\mu_{\theta,h}(s,a)}{\bar{\mu}_h(s,a)}\right)^2\right]\\
    =&2\mathbb{E}^{\pi_\theta}\left[\sum_{h=1}^H\left(\nabla_\theta^j \varepsilon^\theta_{h,1}\right)^2\frac{\mu_{\theta,h}(s,a)}{\bar{\mu}_h(s,a)}\right]+2G^2H^2\mathbb{E}^{\pi_\theta}\left[\sum_{h=1}^H\left(\varepsilon^\theta_{h,1}\right)^2\frac{\mu_{\theta,h}(s,a)}{\bar{\mu}_h(s,a)}\right]\\
    \leq&2\mathbb{E}^{\pi_\theta}\left[\sum_{h=1}^H\left(\nabla_\theta^j \varepsilon^\theta_{h,1}\right)^2\right]\max_{h\in[H],s\in\mathcal{S},a\in\mathcal{A}}\frac{\mu_{\theta,h}(s,a)}{\bar{\mu}_h(s,a)}+2G^2H^2\mathbb{E}^{\pi_\theta}\left[\sum_{h=1}^H\left(\varepsilon^\theta_{h,1}\right)^2\right]\max_{h\in[H],s\in\mathcal{S},a\in\mathcal{A}}\frac{\mu_{\theta,h}(s,a)}{\bar{\mu}_h(s,a)}.
\end{align*}
Following the same argument above, we can derive
\begin{align*}
    t^\top\Lambda_\theta t\leq&4H^4G^2\max_{h\in[H],s\in\mathcal{S},a\in\mathcal{A}}\frac{\mu^\theta_h(s,a)}{\bar{\mu}_h(s,a)},
\end{align*}
i.e., 
\begin{align*}
    \left\vert\widehat{\nabla_\theta^j v_\theta}-\nabla_\theta^j v_\theta\right\vert\leq 4H^2G\sqrt{\frac{\log\frac{8m}{\delta}}{K}}\max_{h\in[H],s\in\mathcal{S},a\in\mathcal{A}}\frac{\mu^\theta_h(s,a)}{\bar{\mu}_h(s,a)}+\frac{2C_\theta\log\frac{72mdH}{\delta}}{K},\quad\forall j\in[m].
\end{align*}
On the other hand, the result of Theorem \ref{thm2} implies
\begin{align*}
    \left\vert\widehat{\nabla_\theta^j v_\theta}-\nabla_\theta^j v_\theta\right\vert\leq 4H^2G\sqrt{\frac{\min\{C_1d,H\}\log\frac{8m}{\delta}}{K}}\max_{h\in[H]}\left\Vert\Sigma_h^{-\frac{1}{2}}\nu_h^\theta\right\Vert+\frac{2C_\theta\log\frac{72mdH}{\delta}}{K},\quad\forall j\in[m].
\end{align*}
Using the relation $\left\Vert\Sigma_h^{-\frac{1}{2}}\nu_h^\theta\right\Vert=\sqrt{\sum_{(s,a)\in\mathcal{S}\times\mathcal{A}}\frac{\left(\mu^\theta_h(s,a)\right)^2}{\bar{\mu}_h(s,a)}}=\sqrt{\mathbb{E}^{\pi_\theta}\frac{\mu^\theta_h(s_h,a_h)}{\bar{\mu}(s_h,a_h)}}$, and taking minimum over the above two inequalities, we have finished the proof. 
\end{proof}

\subsection{Proof of Theorem \ref{thm1}}
\label{pfthm1}
\begin{proof}
We use the same decomposition as in Theorem \ref{thm2_var}. Define a martingale difference sequence $\{e_k^\theta\}_{k=1}^K$ by
\begin{align*}
    e_k^\theta&=\frac{1}{\sqrt{K}}\sum_{h=1}^H\nabla_\theta\left(\left(\nu^\theta_h\right)^\top\Sigma_h^{-1}\phi(s_h^{(k)},a_h^{(k)})\varepsilon_{h,k}^{\theta}\right)\\
    &=\frac{1}{\sqrt{K}}\sum_{h=1}^H\left(\nabla_\theta\nu^\theta_h\right)^\top\Sigma_h^{-1}\phi(s_h^{(k)},a_h^{(k)})\varepsilon_{h,k}^\theta+\frac{1}{\sqrt{K}}\sum_{h=1}^H\left(\nu^\theta_h\right)^\top\Sigma_h^{-1}\phi(s_h^{(k)},a_h^{(k)})\nabla_\theta\varepsilon_{h,k}^\theta,
\end{align*}
we have 
\begin{align*}
    \Vert e_k^\theta\Vert_\infty\leq\frac{1}{\sqrt{K}}\sum_{h=1}^H\max_{j\in[m]}\Vert\Sigma_h^{-\frac{1}{2}}\nabla_\theta^j\nu^\theta_h\Vert\sqrt{C_1d}(H-h+1)+\frac{2}{\sqrt{K}}\sum_{h=1}^H\Vert\Sigma_h^{-\frac{1}{2}}\nu^\theta_h\Vert\sqrt{C_1d}(H-h)^2G\rightarrow 0, 
\end{align*}
where we use the result of Lemma \ref{upbd}. Furthermore, 
\begin{align*}
    \mathbb{E}\left[e_k^\theta\left(e_k^\theta\right)^\top\right]_{ij}=&\frac{1}{K}\mathbb{E}\left[\sum_{h=1}^H\left[\nabla_{\theta_1}^i\left(\left(\nu^{\theta_1}_h\right)^\top\Sigma_h^{-1}\phi(s_h^{(k)},a_h^{(k)})\varepsilon_{h,k}^{\theta_1}\right)\right]\left[\nabla_{\theta_2}^j\left(\left(\nu^{\theta_2}_h\right)^\top\Sigma_h^{-1}\phi(s_h^{(k)},a_h^{(k)})\varepsilon_{h,k}^{\theta_2}\right)\right]^\top\right]\Bigg\vert_{\theta_1=\theta_2=\theta}=\frac{[\Lambda_\theta]_{ij}}{K}.
\end{align*}
Therefore,by WLLN, we have
\begin{align*}
    \sum_{k=1}^K\left[e_k^\theta\left(e_k^\theta\right)^\top\right]_{ij}\rightarrow_p\sum_{k=1}^K\mathbb{E}\left[e_k^\theta\left(e_k^\theta\right)^\top\right]_{ij}=\left[\Lambda_\theta\right]_{ij},
\end{align*}
To finish the rest of the proof, we introduce the following lemmas, 
\begin{lemma}[Martingale CLT, Corollary 2.8 in (McLeish et al., 1974)] \label{CLT}
Let $\left\{X_{mn},n=1,\ldots,k_m\right\}$ be a martingale difference array (row-wise) on the probability triple $(\Omega, \mathcal{F}, P)$.Suppose $X_{mn}$ satisfy the following two conditions:
\begin{align*}
    \max _{1\leq n\leq k_m}\left\vert X_{mn}\right\vert\stackrel{p}{\rightarrow}0,\textrm{ and } \sum_{n=1}^{k_m}X_{mn}^2\stackrel{p}{\rightarrow}\sigma^2
\end{align*}
for $k_m\rightarrow\infty$. Then $\sum_{n=1}^{k_m}X_{mn}\stackrel{d}{\rightarrow}\mathcal{N}\left(0,\sigma^2\right)$.
\end{lemma}
\begin{lemma}[Cramér–Wold Theorem] 
\label{CW_thm}
Let $X_n=(X_n^1,X_n^2,\ldots,X_n^k)^\top$ be a $k$-dimensional random vector series and $X=(X^1,X^2,\ldots,X^k)^\top$ be a random vector of same dimension. Then $X_n$ converges in distribution to $X$ if and only if for any constant vector $t=(t_1,t_2,\ldots,t_k)^\top$, $t^\top X_n$ converges to $t^\top X$ in distribution.
\end{lemma}
Lemma \ref{CLT} implies $\sum_{k=1}^K t^\top e_k\rightarrow_d\mathcal{N}(0,t^\top\Lambda_\theta t)$ for any $t$, and Lemma \ref{CW_thm} implies
\begin{align*}
    \sqrt{K}E_1=\sum_{k=1}^K e_k\rightarrow_p\mathcal{N}(0,\Lambda_\theta). 
\end{align*}
Furthermore, notice that the results of Lemma \ref{e2} and Lemma \ref{e3} imply $\sqrt{K}E_2\rightarrow_p 0, \sqrt{K}E_3\rightarrow_p 0$. Combining the above results, we have finished the proof. 
\end{proof}

\subsection{Proof of Theorem \ref{thm4}}
\label{pfthm4}
\begin{proof}
Our proof is similar to that of \cite{hao2021bootstrapping}. We first derive the influence function of policy gradient estimator for sake of completeness. We denote each of the $K$ sampled trajectories as
$$
\boldsymbol{\tau}:=\left(s_{1}, a_{1}, r_{1}, s_{2}, a_{2}, r_{2}, \ldots, s_{H}, a_{H}, r_{H}, s_{H+1}\right)
$$
We denote $\bar{\pi}(a \mid s)$ as the behavior policy. The distribution of trajectory is then given by
$$
\mathcal{P}(d \boldsymbol{\tau})= \bar{\xi}\left(d s_{1}, d a_{1}\right) p_1\left(d s_{2} \mid s_{1}, a_{1}\right) \bar{\pi}_2\left(d a_{2} \mid s_{2}\right) \ldots \bar{\pi}_H\left(d a_{H} \mid s_{H}\right) p_H\left(d s_{H+1} \mid s_{H}, a_{H}\right)
$$
Define $p_{\eta} = p + \eta\Delta p$ as a new transition probability function and $\mathcal{P}_{\eta}:=\mathcal{P}+\eta \Delta\mathcal{P}$ where $\Delta \mathcal{P}$ satisfies
$$(\Delta \mathcal{P}_h) \mathcal{F} \subseteq \mathcal{F},\forall h\in[H].$$ 
Define $g_{\eta,h}\left(s^{\prime} \mid s, a\right):=\frac{\partial}{\partial \eta} \log p_{\eta,h}\left( s^{\prime} \mid s, a\right)$ and the score function as 
$$
g_{\eta}(\boldsymbol{\tau}):=\frac{\partial}{\partial \eta} \log \mathcal{P}_{\eta}(d \boldsymbol{\tau})=\sum_{h=1}^{H} g_{\eta,h}\left(s_{h+1} \mid s_{h}, a_{h}\right).
$$
Without loss of generality, we assume $p_{\eta}$ is continuously derivative with respect to $\eta.$ This guarantees that we can change the order of taking derivatives with respect to $\eta$ and $\theta.$ When the subscript $\eta$ vanishes, it means $\eta = 0$ and the underlying transition probability is $p(s^{\prime}|s,a),$ i.e. $p_0(s^{\prime} |s,a) = p(s^{\prime} |s,a).$ Then we denote $g_h(s^{\prime}|s,a) := \left.\frac{\partial}{\partial \eta}\log p_{\eta,h}(s^{\prime}|s,a)\right|_{\eta = 0},$ and $g(\boldsymbol{\tau}) = \sum_{h=1}^H g_h(s_{h+1}|s_h,a_h).$ We define the policy value under new transition kernel is
\begin{equation*}
    v_{\theta,\eta} := \mathbb{E}^{\pi_{\theta}} \left[\left.\sum_{h=1}^H r_h(s_h,a_h) \right| s_1 \sim \xi, \mathcal{P}_{\eta}\right]
\end{equation*}
Then, our objective function is
$$
\psi_{\eta} := \nabla_{\theta} v_{\theta,\eta} =\mathbb{E}^{\pi_{\theta}}\left[\left.\sum_{h=1}^{H} \nabla_{\theta} \log \pi_{\theta,h}\left(a_{h} \mid s_{h}\right) \cdot\left(\sum_{h^{\prime}=h}^{H} r_h\left(s_{h^{\prime}}, a_{h^{\prime}}\right)\right) \right| s_{1} \sim \xi, \mathcal{P}_{\eta}\right].
$$
We are going to compute the influence function with respect to the above objective function. We denote this influence function as $\mathcal{I}(\boldsymbol{\tau}).$ By definition, it satisfies that
\begin{equation*}
    \left.\frac{\partial}{\partial \eta} \psi_{\eta}\right|_{\eta = 0} = \mathbb{E}\left[g(\boldsymbol{\tau}) \mathcal{I}(\boldsymbol{\tau})\right].
\end{equation*}
By exchanging the order of derivatives, we find that
\begin{equation*}
    \left.\frac{\partial}{\partial \eta} \psi_{\eta}\right|_{\eta = 0} = \nabla_{\theta} \left[\left.\frac{\partial}{\partial \eta} v_{\theta,\eta}\right|_{\eta = 0}\right].
\end{equation*}
Therefore, we calculate the derivatives.
\begin{align*}
    \frac{\partial}{\partial \eta} v_{\theta, \eta}
    &= \frac{\partial}{\partial \eta}\left[\sum_{h=1}^{H} \int_{(\mathcal{S} \times \mathcal{A})^{h}} r_h\left(s_{h}, a_{h}\right) \xi(s_1) \prod_{j=1}^{h-1} p_{\eta,j}\left(s_{j+1} \mid s_{j}, a_{j}\right) \prod_{j=1}^{h} \pi_{\theta,j}\left(a_{j} \mid s_{j}\right)d\boldsymbol{\tau}_h\right] \\
    &=\sum_{h=1}^{H} \int_{(\mathcal{S} \times \mathcal{A})^{h}} r_h\left(s_{h}, a_{h}\right)\left(\sum_{j=1}^{h-1} g_{\eta,j}\left(s_{j+1} \mid s_{j}, a_{j}\right)\right) \xi(s_1) \prod_{j=1}^{h-1} p_{\eta,j}\left(s_{j+1} \mid s_{j}, a_{j}\right) \prod_{j=1}^{h} \pi_{\theta,j}\left(a_{j} \mid s_{j}\right) d\boldsymbol{\tau}_h\\
    &= \int_{(\mathcal{S} \times \mathcal{A})^{H}} \sum_{h=1}^{H} r_h\left(s_{h}, a_{h}\right)\left(\sum_{j=1}^{h-1} g_{\eta,j}\left(s_{j+1} \mid s_{j}, a_{j}\right)\right) \left[\xi(s_1) \prod_{j=1}^{H} p_{\eta,j}\left(s_{j+1} \mid s_{j}, a_{j}\right) \prod_{j=1}^{H} \pi_{\theta,j}\left(a_{j} \mid s_{j}\right)\right] d\boldsymbol{\tau}.
\end{align*}
We denote $Q_{h,\eta}^{\theta}$ and  $\nabla_{\theta}Q_{h,\eta}^{\theta}$ as the state-action function and its gradient with underlying transition probability being $p_{\eta}.$ For sake of simplicity, we define the state value function as
\begin{equation*}
    V_h^{\theta}(s) := \mathbb{E}^{\pi_{\theta}} \left[\left.\sum_{h^{\prime} = h}^H r_h(s_h,a_h) \right| s_h = s, \mathcal{P}\right].
\end{equation*}
We denote $V_{h,\eta}^{\theta}(s)$ as the same function except for transition probability substituted by $p_{\eta}.$ Therefore,
\begin{align*}
    \frac{\partial}{\partial \eta} v_{\theta, \eta}
    &= \mathbb{E}^{\pi_{\theta}} \left[\left.\sum_{h=1}^{H} r_h\left(s_{h}, a_{h}\right)\left(\sum_{j=1}^{h-1} g_{\eta,j}\left(s_{j+1} \mid s_{j}, a_{j}\right)\right) \right| s_1 \sim \xi, \mathcal{P}_{\eta} \right] \\
    &= \mathbb{E}^{\pi_{\theta}} \left[\left.\sum_{j=1}^{H} g_{\eta,j}\left(s_{j+1} \mid s_{j}, a_{j}\right) \sum_{h=j+1}^H r_h\left(s_{h}, a_{h}\right) \right| s_1 \sim \xi, \mathcal{P}_{\eta} \right] \\
    &= \mathbb{E}^{\pi_{\theta}} \left[\left.\sum_{j=1}^{H} g_{\eta,j}\left(s_{j+1} \mid s_{j}, a_{j}\right) \cdot \mathbb{E}^{\pi_{\theta}} \left[\left.\sum_{h=j+1}^H r_h\left(s_{h}, a_{h}\right) \right| s_{j+1}\right]\right| s_1 \sim \xi, \mathcal{P}_{\eta} \right] \\
    &= \mathbb{E}^{\pi_{\theta}} \left[\left.\sum_{j=1}^{H} \mathbb{E}\left[\left. g_{\eta,j}\left(s_{j+1} \mid s_{j}, a_{j}\right) V_{j+1,\eta}^{\theta}(s_{j+1}) \right| s_j,a_j \right]\right| s_1 \sim \xi, \mathcal{P}_{\eta} \right].
\end{align*}
Therefore,
\begin{equation}\label{influence_function1}
    \left.\frac{\partial}{\partial \eta} v_{\theta, \eta}\right|_{\eta=0} = \mathbb{E}^{\pi_{\theta}} \left[\left.\sum_{h=1}^{H} \mathbb{E}\left[\left. g_h\left(s_{h+1} \mid s_{h}, a_{h}\right) V_{h+1}^{\theta}(s_{h+1}) \right| s_h,a_h \right]\right| s_1 \sim \xi, \mathcal{P}_{\eta} \right].
\end{equation}
Notice that $\Sigma_h = \mathbb{E}\left[\phi(s_h^{(1)},a_h^{(1)})\phi(s_h^{(1)},a_h^{(1)})^{\top}\right]$ and denote $w_h(s,a) := \phi^{\top}(s,a)\Sigma_h^{-1} \nu_h^{\theta} = \phi^{\top}(s,a)\Sigma_h^{-1} \mathbb{E}^{\pi_{\theta}} \left[\phi(s_h,a_h) \mid s_1 \sim \xi\right].$ We leverage the following fact to rewrite \eqref{influence_function1}: for any $f(s,a) = w_f^{\top} \phi(s,a) \in \mathcal{F}$ where $w_f \in \mathbb{R}^d,$ we have
\begin{align*}
\mathbb{E}^{\pi_{\theta}}\left[f(s_h,a_h)\right]
&= \mathbb{E}^{\pi_{\theta}} \left[ w_f^{\top} \phi(s_h,a_h)\right] \\
&= \mathbb{E}^{\pi_{\theta}} \left[ w_f^{\top} \mathbb{E}\left[\phi(s_h^{(1)},a_h^{(1)})\phi^{\top}(s_h^{(1)},a_h^{(1)})\right] \Sigma_h^{-1} \phi(s_h,a_h) \right] \\
&= \mathbb{E} \left[w_f^{\top}\phi(s_h^{(1)},a_h^{(1)})\phi^{\top}(s_h^{(1)},a_h^{(1)})\Sigma_h^{-1}\mathbb{E}^{\pi_{\theta}} \left[\phi(s_h,a_h)\right]\right]\\
&= \mathbb{E}\left[ f(s_h^{(1)},a_h^{(1)}) w_h(s_h^{(1)},a_h^{(1)})\right]
\end{align*}
Since 
\begin{equation*}
    \mathbb{E} \left[ g_h\left(s^{\prime} \mid s, a\right) V_{h+1}^{\theta}(s^{\prime}) | s, a \right] = \left.\frac{\partial}{\partial \eta} \left(Q_{h,\eta}^{\theta}(s,a) - r_{\eta}(s,a)\right)\right|_{\eta = 0} \in \mathcal{F},
\end{equation*}
we have
\begin{align*}
    \left.\frac{\partial}{\partial\eta} v_{\theta,\eta}\right|_{\eta = 0}
    &= \mathbb{E}\left[\sum_{h=1}^{H} w_{h}(s_h^{(1)},a_h^{(1)}) \mathbb{E}\left[g_h\left(s^{\prime} \mid s_h^{(1)},a_h^{(1)}\right) \cdot  V_{h+1}^{\theta}\left(s^{\prime}\right) \mid s_h^{(1)},a_h^{(1)}\right]\right] \\
    &=\mathbb{E}\left[\mathbb{E}_{s^{\prime} \sim p(\cdot \mid s_h^{(1)},a_h^{(1)})}\left[\sum_{h=1}^{H} w_{h}(s_h^{(1)},a_h^{(1)}) g_h\left(s^{\prime} \mid s_h^{(1)},a_h^{(1)}\right) \cdot  V_{h+1}^{\theta}\left(s^{\prime}\right)\right]\right] \\
    &=\mathbb{E}\left[\mathbb{E}_{s^{\prime} \sim p(\cdot \mid s_h^{(1)},a_h^{(1)})}\left[\sum_{h=1}^{H} w_{h}(s_h^{(1)},a_h^{(1)}) g_h\left(s^{\prime} \mid s_h^{(1)},a_h^{(1)}\right)\left( V_{h+1}^{\theta}\left(s^{\prime}\right)-\mathbb{E}\left[V_{h+1}^{\theta}\left(s^{\prime}\right) \mid s_h^{(1)},a_h^{(1)}\right]\right)\right]\right]\\
    &=\mathbb{E}\left[\sum_{h=1}^{H} w_{h}(s_h^{(1)},a_h^{(1)}) g_h\left(s_{h+1}^{(1)} \mid s_h^{(1)},a_h^{(1)}\right)\left( V_{h+1}^{\theta}\left(s_{h+1}^{(1)}\right)-\mathbb{E}\left[V_{h+1}^{\theta}\left(s_{h+1}^{(1)}\right) \mid s_h^{(1)},a_h^{(1)}\right]\right)\right]\\
    &= \mathbb{E} \left[g\left(\boldsymbol{\tau}\right)\sum_{h=1}^H w_{h}(s_h^{(1)}, a_h^{(1)})\left( V_{h+1}^{\theta}\left(s_{h+1}^{(1)}\right)-\mathbb{E}\left[V_{h+1}^{\theta}\left(s_{h+1}^{(1)}\right)\mid s_h^{(1)},a_h^{(1)}\right]\right)\right].
\end{align*}
Taking gradient in both sides and we have
\begin{equation*}
    \nabla_{\theta}\left(\left.\frac{\partial}{\partial\eta} v_{\theta,\eta}\right|_{\eta = 0}\right) = \mathbb{E}\left\{g\left(\boldsymbol{\tau}\right) \cdot \nabla_{\theta} \left[\sum_{h=1}^H w_{h}(s_h^{(1)}, a_h^{(1)}) \left( V_{h+1}^{\theta}\left(s_{h+1}^{(1)}\right)-\mathbb{E}\left[V_{h+1}^{\theta}\left(s_{h+1}^{(1)}\right) \mid s_h^{(1)}, a_h^{(1)}\right]\right)\right]\right\}.
\end{equation*}
The implies that the influence function we want is
\begin{equation*}
    \mathcal{I}(\boldsymbol{\tau}) = \nabla_{\theta} \left[\sum_{h=1}^H w_{h}(s_h^{(1)}, a_h^{(1)}) \left( V_{h+1}^{\theta}\left(s_{h+1}^{(1)}\right)-\mathbb{E}\left[V_{h+1}^{\theta}\left(s_{h+1}^{(1)}\right) \mid s_h^{(1)}, a_h^{(1)}\right]\right)\right].
\end{equation*}
Insert the expression of $w_h(s,a)$ and exploit $\varepsilon_{h,k}^{\theta}=Q_{h}^{\theta}(s_h^{(k)}, a_h^{(k)})-r_h^{(k)}-\int_{\mathcal{A}} \pi_{\theta,h+1}\left(a^{\prime} \mid s_{h+1}^{(k)}\right) Q_{h+1}^{\theta}\left(s_{h+1}^{(k)}, a^{\prime}\right) \mathrm{d} a^{\prime},$ we can rewrite the influence function as
\begin{equation*}
    \mathcal{I}(\boldsymbol{\tau})=-\nabla_\theta\left[\sum_{h=1}^H\phi(s_h^{(1)},a_h^{(1)})^\top\Sigma_h^{-1}\varepsilon_{h,1}^\theta\nu_h^\theta\right]
\end{equation*}
Therefore, since the cross terms vanish by taking conditional expectation, we have
\begin{align*}
    &\mathbb{E}\left[\mathcal{I}(\boldsymbol{\tau})^{\top} \mathcal{I}(\boldsymbol{\tau})\right]=\mathbb{E}\Bigg[\sum_{h=1}^H\left(\nabla_\theta\left(\varepsilon^\theta_{h,1}\phi(s_h^{(1)},a_h^{(1)})^\top\Sigma^{-1}_h\nu_{h}^\theta\right)\right)^\top\nabla_\theta\left(\varepsilon^\theta_{h,1}\phi(s_h^{(1)},a_h^{(1)})^\top\Sigma^{-1}_h\nu_{h}^\theta \right)\Bigg]=\Lambda_{\theta}.
\end{align*}
For any vector $t \in \mathbb{R}^m,$ when it comes to $\left\langle t,\psi_{\eta}\right\rangle,$ by linearity we have
\begin{equation*}
    \left.\frac{\partial}{\partial \eta} \left\langle t,\psi_{\eta}\right\rangle\right|_{\eta=0}=\mathbb{E}[g(\boldsymbol{\tau}) \left\langle t,\mathcal{I}(\boldsymbol{\tau})\right\rangle].
\end{equation*}
Then the influence function of $\left\langle t,\nabla_{\theta}v_{\theta}\right\rangle$ is $\left\langle t,\mathcal{I}(\boldsymbol{\tau})\right\rangle.$ The Cramer-Rao lower bound for $\left\langle t,\nabla_{\theta}v_{\theta}\right\rangle$ is
\begin{equation*}
    \mathbb{E}\left[\left\langle t,\mathcal{I}(\boldsymbol{\tau})\right\rangle^2\right] = t^{\top} \mathbb{E}\left[\mathcal{I}(\boldsymbol{\tau})^{\top}\mathcal{I}(\boldsymbol{\tau})\right] t = t^{\top} \Lambda_{\theta} t.
\end{equation*}
By continuous mapping theorem, a trivial corollary of Theorem \ref{thm4} is that for any $t \in \mathbb{R}^m,$
\begin{equation*}
    \sqrt{K}\left(\left\langle t,\widehat{\nabla_{\theta} v_{\theta}}-\nabla_{\theta} v_{\theta}\right\rangle\right) \stackrel{d}{\rightarrow} \mathcal{N}\left(0, t^{\top} \Lambda_{\theta} t\right).
\end{equation*} 
This implies that the variance of any unbiased estimator for $\left\langle t, \nabla_{\theta} v_{\theta} \right\rangle \in \mathbb{R}$ is lower bounded by $\frac{1}{\sqrt{K}}t^{\top} \Lambda_{\theta} t.$
\end{proof}

\section{Missing Proofs}
\label{missing_proof}
\subsection{Proof of Proposition \ref{lin_rep}}
\begin{proof}
The differentiability of $w_h^\theta$ comes from the differentiability of $Q_h^\theta$. And simply taking derivatives w.r.t. $\theta$ on both sides of $Q_h^\theta=\phi^\top w_h^\theta$, we get the desired result. 
\end{proof}

\subsection{Proof of Proposition \ref{equiv_mb}}
\begin{proof}
To prove the equality, it suffices to prove that given the same input $\widehat{Q}^\theta_{h+1}, \nabla_\theta^j\widehat{Q}^\theta_{h+1}$, we have
{\small
\begin{align*}
    &\mathop{\arg\min}_{f\in\mathcal{F}}\left(\sum_{k=1}^K\left(f(s_h^{(k)},a_h^{(k)})-r_h^{(k)}-\int_{\mathcal{A}}\pi_{\theta,h+1}(a^\prime\vert s_{h+1}^{(k)})\widehat{Q}_{h+1}^{\theta}(s_{h+1}^{(k)}, a^\prime)\mathrm{d}a^\prime\right)^2+\lambda\rho(f)\right) = \widehat{r}_h+\widehat{\mathcal{P}}_{\theta,h}\widehat{Q}_{h+1}^{\theta}\\
    &\mathop{\arg\min}_{f\in\mathcal{F}}\left(\sum_{k=1}^K\left(f(s_h^{(k)},a_h^{(k)})-\int_{\mathcal{A}}\pi_{\theta,h+1}(a^\prime\vert s_{h+1}^{(k)})\left(\left(\nabla_\theta^j\log\pi_{\theta,h+1}(a^\prime\vert s_{h+1}^{(k)})\right)\widehat{Q}_{h+1}^{\theta}(s_{h+1}^{(k)}, a^\prime)+\widehat{\nabla_\theta^j Q_{h+1}^{\theta}}(s_{h+1}^{(k)}, a^\prime)\right)\mathrm{d}a^\prime\right)^2+\lambda\rho(f)\right) \\
    &= \widehat{\mathcal{P}}_{\theta,h}\left(\left(\nabla_\theta^j\log\Pi_\theta\right)\widehat{Q}_{h+1}^{\theta}+\widehat{\nabla_\theta^j Q^{\theta}_{h+1}}\right).
\end{align*}}
The second equation holds due to the definition of $\widehat{\mathcal{P}}_{\theta,h}$. For the first equation, note that when $\mathcal{F}$ is the class of linear functions and $\rho(\phi^\top w)=\Vert w\Vert^2$, the LHS has a closed form solution:
\begin{align*}
    &\mathop{\arg\min}_{f\in\mathcal{F}}\left(\sum_{k=1}^K\left(f(s_h^{(k)},a_h^{(k)})-r_h^{(k)}-\int_{\mathcal{A}}\pi_{\theta,h+1}(a^\prime\vert s_{h+1}^{(k)})\widehat{Q}_{h+1}^{\theta}(s_{h+1}^{(k)}, a^\prime)\mathrm{d}a^\prime\right)^2+\lambda\rho(f)\right) \\
    =& \phi^\top\widehat{\Sigma}_h^{-1}\frac{1}{K}\sum_{k=1}^K\left(r_h^{(k)} + \int_{\mathcal{A}}\pi_{\theta,h+1}(a^\prime\vert s_{h+1}^{(k)})\widehat{Q}_{h+1}^{\theta}(s_{h+1}^{(k)}, a^\prime)\mathrm{d}a^\prime\right)\\
    =&\phi^\top\widehat{\Sigma}_h^{-1}\frac{1}{K}\sum_{k=1}^K r_h^{(k)} + \phi^\top\widehat{\Sigma}_h^{-1}\frac{1}{K}\sum_{k=1}^K \int_{\mathcal{A}}\pi_\theta(a^\prime\vert s_{h+1}^{(k)})\widehat{Q}_{h+1}^{\theta}(s_{h+1}^{(k)}, a^\prime)\mathrm{d}a^\prime\\
    =&\widehat{r}_h+\widehat{\mathcal{P}}_{\theta,h}\widehat{Q}_{h+1}^{\theta}. 
\end{align*}
Therefore, we have finished the proof. 
\end{proof}

\subsection{Proof of Proposition \ref{union_bd}}
\begin{proof}
The result of Theorem \ref{thm2} implies for any fixed $\theta$, when we choose $\lambda \leq\log\frac{8dmH}{\delta}C_1d\min_{h\in[H]}\sigma_{\textrm{min}}(\Sigma_h)$, and $K\geq 36\kappa_1(4+\kappa_2+\kappa_3)^2\log\frac{8dmH}{\delta}C_1dH^2$ sufficiently large such that
\begin{align*}
    &4H^2G\sqrt{\min\{C_1d,H\}}\sqrt{1+\chi^2_{\mathcal{F}}(\mu^\theta,\bar{\mu})}\sqrt{\frac{2\log\frac{24m}{\delta}}{K}}\\
    \geq&480C_1dm^{0.5}H^{3.5}\kappa_1(5+\kappa_2+\kappa_3)(\max_{j\in[m]}\Vert\Sigma_{\theta,1}^{-\frac{1}{2}}\nabla_\theta^j\nu^\theta_1\Vert+HG\Vert\Sigma_{\theta,1}^{-\frac{1}{2}}\nu^\theta_1\Vert)\frac{\log\frac{72mdH}{\delta}}{K},
\end{align*}
then we have
\begin{align*}
    \Vert\widehat{\nabla_\theta v_\theta}-\nabla_\theta v_\theta\Vert \leq 8&H^2G\sqrt{m\min\{C_1d,H\}}\sqrt{1+\chi^2_{\mathcal{F}}(\mu^\theta,\bar{\mu})}\sqrt{\frac{2\log\frac{24m}{\delta}}{K}}.
\end{align*}
Note that when the diameter of $\Theta$ is bounded by $D$, for any $\varepsilon > 0$, it's always possible to find an $\varepsilon$-net $\mathcal{N}_\varepsilon$ such that $\vert \mathcal{N}_\varepsilon\vert\leq\left(\frac{mD}{\varepsilon}\right)^m$. Taking a union bound over $\mathcal{N}_\varepsilon$, we get with probability $1-\delta$, 
\begin{align*}
     \Vert\widehat{\nabla_\theta v_\theta}-\nabla_\theta v_\theta\Vert\leq 16H^2Gm\sqrt{\min\{C_1d,H\}}\sqrt{1+\chi^2_{\mathcal{F}}(\mu^\theta,\bar{\mu})}\sqrt{\frac{\log\frac{24mD}{\delta\varepsilon}}{K}}, \quad\forall\theta\in\mathcal{N}_\varepsilon.
\end{align*}
Therefore, for any $\theta\in\Theta$, pick $\theta^\prime\in\mathcal{N}_\varepsilon$ such that $\Vert\theta-\theta^\prime\Vert\leq \varepsilon$, we have
\begin{align*}
    \Vert\widehat{\nabla_\theta v_\theta}-\nabla_\theta v_\theta\Vert&\leq 2L\varepsilon + \Vert\widehat{\nabla_\theta v_{\theta^\prime}}-\nabla_\theta v_{\theta^\prime}\Vert\\
    &\leq 2L\varepsilon+16H^2Gm\sqrt{\min\{C_1d,H\}}\sqrt{1+\chi^2_{\mathcal{F}}(\mu^{\theta^\prime},\bar{\mu})}\sqrt{\frac{\log\frac{24mD}{\delta\varepsilon}}{K}}. 
\end{align*}
Because $\chi^2_{\mathcal{F}}(\mu^{\theta^\prime},\bar{\mu})$ is $L^\prime$-Lipschitz in $\theta$, we have
\begin{align*}
    \Vert\widehat{\nabla_\theta v_\theta}-\nabla_\theta v_\theta\Vert\leq 2L\varepsilon+16H^2Gm\sqrt{\min\{C_1d,H\}}\sqrt{1+2L^\prime\varepsilon + \chi^2_{\mathcal{F}}(\mu^{\theta},\bar{\mu})}\sqrt{\frac{\log\frac{24mD}{\delta\varepsilon}}{K}}. 
\end{align*}
In particular, pick
\begin{align*}
    \varepsilon = \min\left\{\frac{1}{L^\prime},\frac{16H^2Gm}{L}\sqrt{\frac{\min\{C_1d,H\}}{K}}\right\},
\end{align*}
we get
\begin{align*}
    \Vert\widehat{\nabla_\theta v_\theta}-\nabla_\theta v_\theta\Vert\leq 64H^2Gm\sqrt{\min\{C_1d,H\}}\sqrt{1+\chi^2_{\mathcal{F}}(\mu^\theta,\bar{\mu})}\sqrt{\frac{\log\frac{24DKLL^\prime}{\delta HG}}{K}},\quad\forall\theta\in\Theta. 
\end{align*}
\end{proof}

\subsection{Proof of Lemma \ref{error_decomp}}
\begin{proof}
Note that
\begin{align*}
    \nabla_\theta Q_1^\theta-\widehat{\nabla_\theta Q_1^\theta} &= \sum_{h=1}^H\left(\prod_{h^\prime=1}^{h-1}\mathcal{P}_{\theta,h^\prime}\right)U_h^\theta-\sum_{h=1}^H\left(\prod_{h^\prime=1}^{h-1}\widehat{\mathcal{P}}_{\theta,h^\prime}\right)\tilde{U}_h^\theta\\
    &=\sum_{h=1}^H\left(\prod_{h^\prime=1}^{h-1}\mathcal{P}_{\theta,h^\prime}\right)U_h^\theta-\sum_{h=1}^H\left(\prod_{h^\prime=1}^{h-1}\widehat{\mathcal{P}}_{\theta,h^\prime}\right)\widehat{U}_h^\theta+\sum_{h=1}^H\left(\prod_{h^\prime=1}^{h-1}\widehat{\mathcal{P}}_{\theta,h^\prime}\right)\left(\widehat{U}_h^\theta-\tilde{U}_h^\theta\right)\\
    &=\sum_{h=1}^H\left(\prod_{h^\prime=1}^{h-1}\widehat{\mathcal{P}}_{\theta,h^\prime}\right)\left( \nabla_\theta Q_h^\theta-\widehat{U}_h^\theta-\widehat{\mathcal{P}}_{\theta,h}\nabla_\theta Q_{h+1}^\theta\right)+\sum_{h=1}^H\left(\prod_{h^\prime=1}^{h-1}\widehat{\mathcal{P}}_{\theta,h^\prime}\right)\left(\widehat{U}_h^\theta -\tilde{U}_h^\theta\right).
\end{align*}
For the first term, we have
\begin{align*}
    &\sum_{h=1}^H\left(\prod_{h^\prime=1}^{h-1}\widehat{\mathcal{P}}_{\theta,h^\prime}\right)\left(\nabla_\theta Q_h^\theta-\widehat{U}_h^\theta-\widehat{\mathcal{P}}_{\theta,h}\nabla_\theta Q_{h+1}^\theta\right)\\
    =&\sum_{h=1}^H\left(\prod_{h^\prime=1}^{h-1}\widehat{\mathcal{P}}_{\theta,h^\prime}\right)\phi^\top\widehat{\Sigma}_h^{-1}\frac{1}{K}\sum_{k=1}^K\phi\left(s_h^{(k)},a_h^{(k)}\right)\\
    &\cdot\left(\nabla_\theta Q_h^\theta\left(s_h^{(k)},a_h^{(k)}\right)-\int_{\mathcal{A}}\left(\left(\nabla_\theta\pi_{\theta,h+1}\left(a^\prime\left\vert s_{h+1}^{(k)}\right.\right)\right)Q_{h+1}^{\theta}\left(s_{h+1}^{(k)},a^\prime\right)+\pi_{\theta,h+1}\left(a^\prime\left\vert s_{h+1}^{(k)}\right.\right)\nabla_\theta Q_{h+1}^\theta\left(s_{h+1}^{(k)},a^\prime\right)\right)\mathrm{d}a^\prime\right)\\
    &+\frac{\lambda}{K}\sum_{h=1}^H\left(\prod_{h^\prime=1}^{h-1}\widehat{\mathcal{P}}_{\theta,h^\prime}\right)\phi^\top\widehat{\Sigma}^{-1}_h\nabla_\theta w_h^\theta\\
    =&\sum_{h=1}^H\phi^\top\left(\prod_{h^\prime=1}^{h-1}\widehat{M}_{\theta,h^\prime}\right)\widehat{\Sigma}_h^{-1}\frac{1}{K}\sum_{k=1}^K\phi\left(s_h^{(k)},a_h^{(k)}\right)\nabla_\theta\varepsilon_{h,k}^\theta+\frac{\lambda}{K}\sum_{h=1}^H\phi^\top\left(\prod_{h^\prime=1}^{h-1}\widehat{M}_{\theta,h^\prime}\right)\widehat{\Sigma}^{-1}_h\nabla_\theta w_h^\theta.
\end{align*}
Using the definition of $\widehat{\nu}_h^\theta$, we get
\begin{align}
    \label{p1}
    \begin{aligned}
        &\int_{\mathcal{S}\times\mathcal{A}}\xi(s)\pi_{\theta,1}(a\vert s)\left(\sum_{h=1}^H\left(\prod_{h^\prime=1}^{h-1}\widehat{\mathcal{P}}_{\theta,h^\prime}\right)\left(\nabla_\theta Q_h^\theta-\widehat{U_h^\theta}-\widehat{\mathcal{P}}_{\theta,h}\nabla_\theta Q_{h+1}^\theta\right)\right)(s,a)\mathrm{d}s\mathrm{d}a\\
        =&\sum_{h=1}^H\left(\widehat{\nu}_h^\theta\right)^\top\widehat{\Sigma}^{-1}_h\frac{1}{K}\sum_{k=1}^K\phi\left(s_h^{(k)},a_h^{(k)}\right)\nabla_\theta\varepsilon_{h,k}^\theta+\frac{\lambda}{K}\sum_{h=1}^H\left(\widehat{\nu}_h^\theta\right)^\top\widehat{\Sigma}_h^{-1}\nabla_\theta w^\theta_h.
    \end{aligned}
\end{align}
For the second term,  by Lemma \ref{Q_decomp}, we have
\begin{align*}
    \sum_{h=1}^H\left(\prod_{h^\prime=1}^{h-1}\widehat{\mathcal{P}}_{\theta,h^\prime}\right)\left( \widehat{U}_h^\theta-\tilde{U}_h^\theta\right)&=\sum_{h=1}^H\left(\prod_{h^\prime=1}^{h}\widehat{\mathcal{P}}_{\theta,h^\prime}\right)\left(\nabla_\theta\log\Pi_{\theta,h+1}\right)\left(Q_{h+1}^\theta-\widehat{Q}_{h+1}^\theta\right)\\
    &=\sum_{h=1}^H\left(\prod_{h^\prime=1}^{h}\widehat{\mathcal{P}}_{\theta,h^\prime}\right)\left(\nabla_\theta\log\Pi_{\theta,h+1}\right)\sum_{h^\prime=h+1}^H\left(\prod_{h^{\prime\prime}=h+1}^{h^\prime-1}\widehat{\mathcal{P}}_{\theta,h^{\prime\prime}}\right)\left(Q_{h^\prime}^\theta-\widehat{r}_{h^\prime}- \widehat{\mathcal{P}}_{\theta,h^\prime}Q_{h^\prime+1}^\theta\right)\\
    &=\sum_{h=1}^H\sum_{h^\prime=1}^{h-1}\left(\prod_{h=1}^{h^\prime}\widehat{\mathcal{P}}_{\theta,h}\right)\left(\nabla_\theta\log\Pi_{\theta,h^\prime+1}\right)\left(\prod_{h^{\prime\prime}=h^\prime+1}^{h-1}\widehat{\mathcal{P}}_{\theta,h^{\prime\prime}}\right)\left(Q_{h}^\theta-\widehat{r}_{h}- \widehat{\mathcal{P}}_{\theta,h}Q_{h+1}^\theta\right).
\end{align*}
Meanwhile, again by Lemma \ref{Q_decomp}, we have
\begin{align*}
    \left(\nabla_\theta\log\Pi_{\theta,1}\right)(Q_1^\theta-\widehat{Q}_1^\theta)=\left(\nabla_\theta\log\Pi_{\theta,1}\right)\sum_{h=1}^{H}\left(\prod_{h^\prime=1}^{h-1}\widehat{\mathcal{P}}_{\theta,h^\prime}\right)\left(Q_{h}^\theta-\widehat{r}_{h}-\widehat{\mathcal{P}}_{\theta,h}Q_{h+1}^\theta\right),
\end{align*}
which implies
\begin{align*}
    &\sum_{h=1}^H\left(\prod_{h^\prime=1}^{h-1}\widehat{\mathcal{P}}_{\theta,h^\prime}\right)\left(\widehat{U}_h^\theta-\tilde{U}_h^\theta\right)+\left(\nabla_\theta\log\Pi_{\theta,1}\right)(Q_1^\theta-\widehat{Q}_1^\theta)\\
    =&\sum_{h=1}^H\left(\left(\nabla_\theta\log\Pi_{\theta,1}\right)\left(\prod_{h^\prime=1}^{h-1}\widehat{\mathcal{P}}_{\theta,h^\prime}\right)+\sum_{h^\prime=1}^{h-1}\left(\prod_{h=1}^{h^\prime}\widehat{\mathcal{P}}_{\theta,h}\right)\left(\nabla_\theta\log\Pi_{\theta,h^\prime+1}\right)\left(\prod_{h^{\prime\prime}=h^\prime+1}^{h-1}\widehat{\mathcal{P}}_{\theta,h^{\prime\prime}}\right)\right)\left(Q_{h}^\theta-\widehat{r}_{h}-\widehat{\mathcal{P}}_{\theta,h}Q_{h+1}^\theta\right)\\
    =&\sum_{h=1}^H\sum_{h^\prime=0}^{h-1}\left(\prod_{h=1}^{h^\prime}\widehat{\mathcal{P}}_{\theta,h}\right)\left(\nabla_\theta\log\Pi_{\theta,h^\prime+1}\right)\left(\prod_{h^{\prime\prime}=h^\prime+1}^{h-1}\widehat{\mathcal{P}}_{\theta,h^{\prime\prime}}\right)\left(Q_{h}^\theta-\widehat{r}_{h}-\widehat{\mathcal{P}}_{\theta,h}Q_{h+1}^\theta\right)\\
    =&\sum_{h=1}^H\sum_{h^\prime=0}^{h-1}\left(\prod_{h=1}^{h^\prime}\widehat{\mathcal{P}}_{\theta,h}\right)\left(\nabla_\theta\log\Pi_{\theta,h^\prime+1}\right)\left(\prod_{h^{\prime\prime}=h^\prime+1}^{h-1}\widehat{\mathcal{P}}_{\theta,h^{\prime\prime}}\right)\phi^\top\widehat{\Sigma}_h^{-1}\\
    &\cdot\frac{1}{K}\sum_{k=1}^K\phi\left(s_h^{(k)},a_h^{(k)}\right)\left(Q_h^\theta\left(s_h^{(k)},a_h^{(k)}\right)-r_h^{(k)}-\int_{\mathcal{A}}\pi_{\theta,h+1}\left(a^\prime\left\vert s_{h+1}^{(k)}\right.\right)Q_{h+1}^\theta\left(s_{h+1}^{(k)},a^\prime\right)\mathrm{d}a^\prime\right)\\
    &+\frac{\lambda}{K}\sum_{h=1}^H\sum_{h^\prime=0}^{h-1}\left(\prod_{h^{\prime\prime}=1}^{h^\prime}\widehat{\mathcal{P}}_{\theta,h^{\prime\prime}}\right)\left(\nabla_\theta\log\Pi_{\theta,h^\prime+1}\right)\left(\prod_{h^{\prime\prime}=h^\prime+1}^{h-1}\widehat{\mathcal{P}}_{\theta,h^{\prime\prime}}\right)\phi^\top\widehat{\Sigma}^{-1}_hw_h^\theta.
\end{align*}
For each $j\in[m]$, notice the relation
\begin{align*}
    \left(\nabla_\theta^j\widehat{\nu}_h^\theta\right)^\top&=\left(\nabla_\theta^j\nu_1^\theta\right)^\top\left(\prod_{h^\prime=1}^{h-1}\widehat{M}_{\theta,h^\prime}\right)+\sum_{h^\prime=1}^{h-1}\left(\nu_1^{\theta}\right)^\top\left(\prod_{h^{\prime\prime}=1}^{h^\prime-1}\widehat{M}_{\theta,h^{\prime\prime}}\right)\left(\widehat{\nabla_\theta^j M_{\theta,h^\prime}}\right)\left(\prod_{h^{\prime\prime}=h^\prime+1}^{h-1}\widehat{M}_{\theta,h^{\prime\prime}}\right)\\
    &=\int_{\mathcal{S}\times\mathcal{A}}\xi(s)\pi_{\theta,1}(a\vert s)\left(\sum_{h^\prime=0}^{h-1}\left(\prod_{h^{\prime\prime}=1}^{h^\prime}\widehat{\mathcal{P}}_{\theta,h^{\prime\prime}}\right)\left(\nabla_\theta\log\Pi_{\theta,h^\prime+1}\right)\left(\prod_{h^{\prime\prime}=h^\prime+1}^{h-1}\widehat{\mathcal{P}}_{\theta,h^{\prime\prime}}\right)\phi^\top\right)(s,a)\mathrm{d}s\mathrm{d}a.
\end{align*} 
Therefore, we have
\begin{align}
    \label{p2}
    \begin{aligned}
        &\left[\int\xi(s)\pi_{\theta,1}(a\vert s)\left(\sum_{h=1}^H\left(\prod_{h^\prime=1}^{h-1}\widehat{\mathcal{P}}_{\theta,h^\prime}\right)\left(\widehat{U}_h^\theta-\tilde{U}_h^\theta\right)+\left(\nabla_\theta\log\Pi_{\theta,1}\right)(Q_1^\theta-\widehat{Q_1^\theta})\right)(s,a)\mathrm{d}s\mathrm{d}a\right]_j\\
        =&\sum_{h=1}^H\left(\nabla_\theta^j\widehat{\nu}_h^\theta\right)^\top\widehat{\Sigma}_h^{-1}\frac{1}{K}\sum_{k=1}^K\phi\left(s_h^{(k)},a_h^{(k)}\right)\left(Q_h^\theta\left(s_h^{(k)},a_h^{(k)}\right)-r_h^{(k)}-\int_{\mathcal{A}}\pi_{\theta,h+1}\left(a^\prime\left\vert s_{h+1}^{(k)}\right.\right)Q_{h+1}^\theta\left(s_{h+1}^{(k)},a^\prime\right)\mathrm{d}a^\prime\right)\\
        &+\frac{\lambda}{K}\sum_{h=1}^H\left(\nabla_\theta^j\widehat{\nu}_h^\theta\right)^\top\widehat{\Sigma}_h^{-1}w_h^\theta\\
        =&\sum_{h=1}^H\left(\nabla_\theta^j\widehat{\nu}_h^\theta\right)^\top\widehat{\Sigma}_h^{-1}\frac{1}{K}\sum_{k=1}^K\phi\left(s_h^{(k)},a_h^{(k)}\right)\varepsilon_{h,k}^\theta+\frac{\lambda}{K}\sum_{h=1}^H\left(\nabla_\theta^j\widehat{\nu}_h^\theta\right)^\top\widehat{\Sigma}_h^{-1}w_h^\theta.
    \end{aligned}
\end{align}
Combing the results of \eqref{p1} and \eqref{p2}, we get for each $j\in[m]$, 
\begin{align*}
    &\nabla_\theta^j v_\theta-\widehat{\nabla_\theta^j v_\theta}=\int_{\mathcal{S}\times\mathcal{A}}\xi(s)\pi_{\theta,1}(a\vert s)\left(\nabla_\theta^j Q_1^\theta-\widehat{\nabla_\theta^j Q_1^\theta}+\left(\nabla_\theta^j\log\Pi_{\theta,1}\right)(Q_1^\theta-\widehat{Q_1^\theta})\right)(s,a)\textrm{d}s\textrm{d}a\\
    =&\sum_{h=1}^H\Bigg(\left(\widehat{\nu}_h^\theta\right)^\top\widehat{\Sigma}_h^{-1}\frac{1}{K}\sum_{k=1}^K\phi\left(s_h^{(k)},a_h^{(k)}\right)\nabla_\theta^j\varepsilon_{h,k}^\theta+\frac{\lambda}{K}\left(\widehat{\nu}_h^\theta\right)^\top\widehat{\Sigma}_h^{-1}\nabla_\theta^j w^\theta_h+\left(\nabla_\theta^j\nu_h^\theta\right)^\top\widehat{\Sigma}_h^{-1}\frac{1}{K}\sum_{k=1}^K\phi\left(s_h^{(k)},a_h^{(k)}\right)\varepsilon_{h,k}^\theta\\
    &+\frac{\lambda}{K}\left(\nabla_\theta^j\widehat{\nu}_h^\theta\right)^\top\widehat{\Sigma}^{-1}_hw_h^\theta\Bigg)\\
    =&\sum_{h=1}^H\nabla_\theta^j\left(\left(\widehat{\nu}_h^\theta\right)^\top\widehat{\Sigma}_h^{-1}\frac{1}{K}\sum_{k=1}^K\phi\left(s_h^{(k)},a_h^{(k)}\right)\varepsilon_{h,k}^\theta+\frac{\lambda}{K}\left(\widehat{\nu}_h^\theta\right)^\top\widehat{\Sigma}_h^{-1}w^{\theta}_h\right)\\
    =&\sum_{h=1}^H\nabla_\theta^j\left(\left(\widehat{\nu}_h^\theta\right)^\top\widehat{\Sigma}_h^{-1}\frac{1}{K}\sum_{k=1}^K\phi\left(s_h^{(k)},a_h^{(k)}\right)\varepsilon_{h,k}^\theta+\frac{\lambda}{K}\left(\widehat{\nu}_h^\theta\right)^\top\widehat{\Sigma}_h^{-1}w^\theta_h+\left(\left(\widehat{\nu}_h^\theta\right)^\top\widehat{\Sigma}_h^{-1}-\left(\nu_h^\theta\right)^\top\Sigma^{-1}_h\right)\frac{1}{K}\sum_{k=1}^K\phi\left(s_h^{(k)},a_h^{(k)}\right)\varepsilon_{h,k}^\theta\right).
\end{align*}
Rewriting the above decomposition in a vector form, we get
\begin{align*}
    \nabla_\theta v_\theta-\widehat{\nabla_\theta v_\theta}=&\sum_{h=1}^H\nabla_\theta\Bigg(\left(\widehat{\nu}_h^\theta\right)^\top\widehat{\Sigma}^{-1}_h\frac{1}{K}\sum_{k=1}^K\phi\left(s_h^{(k)},a_h^{(k)}\right)\varepsilon_{h,k}^\theta\\
    &+\frac{\lambda}{K}\left(\widehat{\nu}_h^\theta\right)^\top\widehat{\Sigma}_h^{-1}w^\theta_h+\left(\left(\widehat{\nu}_h^\theta\right)^\top\widehat{\Sigma}_h^{-1}-\left(\nu_h^\theta\right)^\top\Sigma_h^{-1}\right)\frac{1}{K}\sum_{k=1}^K\phi\left(s_h^{(k)},a_h^{(k)}\right)\varepsilon_{h,k}^\theta\Bigg),
\end{align*}
which is the desired result. 
\end{proof}

\subsection{Proof of Lemma \ref{e1_finite_product}}
\begin{proof}
Note that, 
\begin{align*}
    \langle E_1,t\rangle=&\sum_{h=1}^H\left\langle\nabla_\theta\left(\left(\nu^\theta_h\right)^\top\Sigma_h^{-1}\frac{1}{K}\sum_{k=1}^K\phi\left(s_h^{(k)},a_h^{(k)}\right)\varepsilon_{h,k}^\theta\right), t\right\rangle\\
    =&\sum_{h=1}^H\left\langle\left(\nabla_\theta\nu^\theta_h\right)^\top\Sigma_h^{-1}\frac{1}{K}\sum_{k=1}^K\phi\left(s_h^{(k)},a_h^{(k)}\right)\varepsilon_{h,k}^\theta, t\right\rangle+\sum_{h=1}^H\left\langle\left(\nu^\theta_h\right)^\top\Sigma_h^{-1}\frac{1}{K}\sum_{k=1}^K\phi\left(s_h^{(k)},a_h^{(k)}\right)\nabla_\theta\varepsilon_{h,k}^\theta,t\right\rangle.
\end{align*}
Let $e_k = \sum_{h=1}^H\left\langle\nabla_\theta\left(\left(\nu^\theta_h\right)^\top\Sigma_h^{-1}\phi\left(s_h^{(k)},a_h^{(k)}\right)\varepsilon_{h,k}^\theta\right),t\right\rangle$, we have 
\begin{align*}
    \vert e_k\vert&\leq\sqrt{C_1dm}\Vert t\Vert\sum_{h=1}^H(H-h+1)\max_{j\in[m]}\sqrt{\left(\nabla^j_\theta\nu^\theta_h\right)^\top\Sigma_h^{-1}\nabla^j_\theta\nu^\theta_h}+2G\sqrt{C_1dm}\Vert t\Vert\sum_{h=1}^H(H-h)^2\sqrt{\left(\nu^\theta_h\right)^\top\Sigma_h^{-1}\nu^\theta_h}=B\sqrt{C_1dm}\Vert t\Vert.
\end{align*}
We have
\begin{align*}
    &\sum_{k=1}^K\textrm{Var}[e_k]=\sum_{k=1}^K\mathbb{E}\left[\left\langle\sum_{h=1}^H\nabla_\theta\left(\left(\nu^\theta_h\right)^\top\Sigma_h^{-1}\phi\left(s_h^{(k)},a_h^{(k)}\right)\varepsilon_{h,k}^\theta\right),t\right\rangle^2\right]=Kt^\top\Lambda_\theta t.
\end{align*}
We pick $\sigma^2=Kt^\top\Lambda_\theta t$, the Bernstein’s inequality implies that for any $\varepsilon\in\mathbb{R}$, 
\begin{align*}
    \mathbb{P}\left(\left\vert\sum_{k=1}^K e_k\right\vert\geq\varepsilon\right)\leq 2\exp\left(-\frac{\varepsilon^2/2}{\sigma^2+\sqrt{C_1dm}\Vert t\Vert B\varepsilon/3}\right).
\end{align*}
Therefore, if we pick $\varepsilon=\sigma\sqrt{2\log(2/\delta)}+2\log(2/\delta)\sqrt{C_1dm}\Vert t\Vert B/3$, we get
\begin{align*}
    \mathbb{P}\left(\left\vert\sum_{k=1}^K e_k\right\vert\geq\varepsilon\right)\leq\delta,
\end{align*}
i.e., we have with probability $1-\delta$, 
\begin{align*}
    \left\vert\frac{1}{K}\sum_{k=1}^K e_k\right\vert\leq\sqrt{\frac{2t^\top\Lambda_\theta t\log(2/\delta)}{K}}+\frac{2\log(2/\delta)\sqrt{C_1dm}\Vert t\Vert B}{3K}
\end{align*}
\end{proof}

\subsection{Proof of Lemma \ref{e2}}
\begin{proof}
For an arbitrarily given $\theta_0$, let $\Sigma_{\theta_0,h}=\mathbb{E}^{\pi_{\theta_0}}[\phi(s_h,a_h)\phi(s_h,a_h)^\top]$,  we have
\begin{align*}
    &\left(\left(\widehat{\nu}^\theta_h\right)^\top\widehat{\Sigma}_h^{-1}-\left(\nu^\theta_h\right)^\top\Sigma_h^{-1}\right)\frac{1}{K}\sum_{k=1}^K\phi\left(s_h^{(k)},a_h^{(k)}\right)\varepsilon_{h,k}^\theta\\
    =&\left(\nu^\theta_1\right)^\top\left(\left(\prod_{h^\prime=1}^{h-1}\widehat{M}_{\theta,h^\prime}\right)\widehat{\Sigma}_h^{-1}-\left(\prod_{h^\prime=1}^{h-1}M_{\theta,h^\prime}\right)\Sigma_h^{-1}\right)\frac{1}{K}\sum_{k=1}^K\phi\left(s_h^{(k)},a_h^{(k)}\right)\varepsilon_{h,k}^\theta\\
    =&\left(\Sigma_{\theta_0,1}^{-\frac{1}{2}}\nu^\theta_1\right)^\top\left(\left(\prod_{h^\prime=1}^{h-1}\Sigma_{\theta_0,h^\prime}^{\frac{1}{2}}\widehat{M}_{\theta,h^\prime}\Sigma_{\theta_0,h^\prime+1}^{-\frac{1}{2}}\right)\Sigma_{\theta_0,h}^{\frac{1}{2}}\Sigma_h^{-\frac{1}{2}}\Sigma_h^{\frac{1}{2}}\widehat{\Sigma}_h^{-1}\Sigma_h^{\frac{1}{2}}-\left(\prod_{h^\prime=1}^{h-1}\Sigma_{\theta_0,h^\prime}^{\frac{1}{2}}M_{\theta,h^\prime}\Sigma_{\theta_0,h^\prime+1}^{-\frac{1}{2}}\right)\Sigma_{\theta_0,h}^{\frac{1}{2}}\Sigma^{-\frac{1}{2}}_h\right)\\
    &\cdot\Sigma_h^{-\frac{1}{2}}\frac{1}{K}\sum_{k=1}^K\phi\left(s_h^{(k)},a_h^{(k)}\right)\varepsilon_{h,k}^\theta.
\end{align*}
Taking derivatives on both sides, and let $\theta_0=\theta$, we get
\begin{align*}
    &\nabla_\theta^j E_2=\nabla_\theta^j\left(\sum_{h=1}^H\left(\left(\widehat{\nu}^\theta_h\right)^\top\widehat{\Sigma}_h^{-1}-\left(\nu^\theta_h\right)^\top\Sigma_h^{-1}\right)\frac{1}{K}\sum_{k=1}^K\phi\left(s_h^{(k)},a_h^{(k)}\right)\varepsilon_{h,k}^\theta\right)=E_{21}^j+E_{22}^j+E_{23}^j,
\end{align*}
where 
\begin{align*}
    E_{21}^j=&\sum_{h=1}^H\left(\Sigma_{\theta,1}^{-\frac{1}{2}}\nu^\theta_1\right)^\top\left(\left(\prod_{h^\prime=1}^{h-1}\Sigma_{\theta,h^\prime}^{\frac{1}{2}}\widehat{M}_{\theta,h^\prime}\Sigma_{\theta,h^\prime+1}^{-\frac{1}{2}}\right)\Sigma_{\theta,h}^{\frac{1}{2}}\Sigma_h^{-\frac{1}{2}}\Sigma_h^{\frac{1}{2}}\widehat{\Sigma}_h^{-1}\Sigma_h^{\frac{1}{2}}-\left(\prod_{h^\prime=1}^{h-1}\Sigma_{\theta,h^\prime}^{\frac{1}{2}}M_{\theta,h^\prime}\Sigma_{\theta,h^\prime+1}^{-\frac{1}{2}}\right)\Sigma_{\theta,h}^{\frac{1}{2}}\Sigma_h^{-\frac{1}{2}}\right)\\
    &\cdot\Sigma_h^{-\frac{1}{2}}\frac{1}{K}\sum_{k=1}^K\phi\left(s_h^{(k)},a_h^{(k)}\right)\nabla_\theta^j\varepsilon_{h,k}^\theta\\
    E_{22}^j=&\sum_{h=1}^H\left(\Sigma_{\theta,1}^{-\frac{1}{2}}\nabla_\theta^j\nu^\theta_1\right)^\top\left(\left(\prod_{h^\prime=1}^{h-1}\Sigma_{\theta,h^\prime}^{\frac{1}{2}}\widehat{M}_{\theta,h^\prime}\Sigma_{\theta,h^\prime+1}^{-\frac{1}{2}}\right)\Sigma_{\theta,h}^{\frac{1}{2}}\Sigma_h^{-\frac{1}{2}}\Sigma_h^{\frac{1}{2}}\widehat{\Sigma}_h^{-1}\Sigma_h^{\frac{1}{2}}-\left(\prod_{h^\prime=1}^{h-1}\Sigma_{\theta,h^\prime}^{\frac{1}{2}}M_{\theta,h^\prime}\Sigma_{\theta,h^\prime+1}^{-\frac{1}{2}}\right)\Sigma_{\theta,h}^{\frac{1}{2}}\Sigma_h^{-\frac{1}{2}}\right)\\
    &\cdot\Sigma_h^{-\frac{1}{2}}\frac{1}{K}\sum_{k=1}^K\phi\left(s_h^{(k)},a_h^{(k)}\right)\varepsilon_{h,k}^\theta\\
    E_{23}^j=&\sum_{h=1}^H\left(\Sigma_{\theta,1}^{-\frac{1}{2}}\nu^\theta_1\right)^\top\\
    &\cdot\left(\left.\nabla_\theta^j\left(\prod_{h^\prime=1}^{h-1}\Sigma_{\theta_0,h^\prime}^{\frac{1}{2}}\widehat{M}_{\theta,h^\prime}\Sigma_{\theta_0,h^\prime+1}^{-\frac{1}{2}}\right)\right\vert_{\theta_0=\theta}\Sigma_{\theta,h}^{\frac{1}{2}}\Sigma_h^{-\frac{1}{2}}\Sigma_h^{\frac{1}{2}}\widehat{\Sigma}_h^{-1}\Sigma_h^{\frac{1}{2}}-\left.\nabla_\theta^j\left(\prod_{h^\prime=1}^{h-1}\Sigma_{\theta_0,h^\prime}^{\frac{1}{2}}M_{\theta,h^\prime}\Sigma_{\theta_0,h^\prime+1}^{-\frac{1}{2}}\right)\right\vert_{\theta_0=\theta}\Sigma_{\theta,h}^{\frac{1}{2}}\Sigma_h^{-\frac{1}{2}}\right)\\
    &\cdot\Sigma_h^{-\frac{1}{2}}\frac{1}{K}\sum_{k=1}^K\phi\left(s_h^{(k)},a_h^{(k)}\right)\varepsilon_{h,k}^\theta.
\end{align*}
Therefore, using the result of Lemma \ref{decomp}, we get
\begin{align*}
    \vert E_{21}^j\vert\leq&\sum_{h=1}^H\left\Vert\Sigma_{\theta,1}^{-\frac{1}{2}}\nu^\theta_1\right\Vert\left\Vert\Sigma_{\theta,h}^{\frac{1}{2}}\Sigma_h^{-\frac{1}{2}}\right\Vert\left(\left(\prod_{h^\prime=1}^{h-1}\left(1+\left\Vert\Sigma_{\theta,h^\prime}^{\frac{1}{2}}\left(\Delta M_{\theta,h^\prime}\right)\Sigma_{\theta,h^\prime+1}^{-\frac{1}{2}}\right\Vert\right)\right)\left(1+\left\Vert\Sigma_h^{\frac{1}{2}}\left(\Delta\Sigma_h^{-1}\right)\Sigma_h^{\frac{1}{2}}\right\Vert\right)-1\right)\\
    &\cdot\left\Vert\Sigma_h^{-\frac{1}{2}}\frac{1}{K}\sum_{k=1}^K\phi\left(s_h^{(k)},a_h^{(k)}\right)\nabla_\theta^j\varepsilon_{h,k}^\theta\right\Vert\\
    \vert E_{22}^j\vert\leq&\sum_{h=1}^H\left\Vert\Sigma_{\theta,1}^{-\frac{1}{2}}\nabla_\theta^j\nu^\theta_1\right\Vert\left\Vert\Sigma_{\theta,h}^{\frac{1}{2}}\Sigma_h^{-\frac{1}{2}}\right\Vert\left(\left(\prod_{h^\prime=1}^{h-1}\left(1+\left\Vert\Sigma_{\theta,h^\prime}^{\frac{1}{2}}\left(\Delta M_{\theta,h^\prime}\right)\Sigma_{\theta,h^\prime+1}^{-\frac{1}{2}}\right\Vert\right)\right)\left(1+\left\Vert\Sigma_h^{\frac{1}{2}}\left(\Delta\Sigma_h^{-1}\right)\Sigma_h^{\frac{1}{2}}\right\Vert\right)-1\right)\\
    &\cdot\left\Vert\Sigma_h^{-\frac{1}{2}}\frac{1}{K}\sum_{k=1}^K\phi\left(s_h^{(k)},a_h^{(k)}\right)\varepsilon_{h,k}^\theta\right\Vert\\
    \vert E_{23}^j\vert\leq&\sum_{h=1}^H\sum_{h^\prime=1}^{h-1}G\left\Vert\Sigma_{\theta,1}^{-\frac{1}{2}}\nu^\theta_1\right\Vert\left\Vert\Sigma_{\theta,h}^{\frac{1}{2}}\Sigma_h^{-\frac{1}{2}}\right\Vert\\
    &\cdot\left(\left(\prod_{h^{\prime\prime}\neq h^\prime}\left(1+\left\Vert\Sigma_{\theta,h^{\prime\prime}}^{\frac{1}{2}}\left(\Delta M_{\theta,h^{\prime\prime}}\right)\Sigma_{\theta,h^{\prime\prime}}^{-\frac{1}{2}}\right\Vert\right)\right)\left(1+\left\Vert\Sigma_{\theta,h^\prime}^{\frac{1}{2}}\left(\frac{\nabla_\theta^j\left(\Delta M_{\theta,h^\prime}\right)}{G}\right)\Sigma_{\theta,h^\prime}^{-\frac{1}{2}}\right\Vert\right)\left(1+\left\Vert\Sigma_h^{\frac{1}{2}}\left(\Delta\Sigma_h^{-1}\right)\Sigma_h^{\frac{1}{2}}\right\Vert\right)-1\right)\\
    &\cdot\left\Vert\Sigma_h^{-\frac{1}{2}}\frac{1}{K}\sum_{k=1}^K\phi\left(s_h^{(k)},a_h^{(k)}\right)\varepsilon_{h,k}^\theta\right\Vert,
\end{align*}
where $\Delta\Sigma^{-1}_h=\widehat{\Sigma}^{-1}_h-\Sigma_h$ and we use the fact $\left\Vert\Sigma_{\theta,h}^{\frac{1}{2}}M_{\theta,h}\Sigma_{\theta,h+1}^{-\frac{1}{2}}\right\Vert\leq 1$ and $\left\Vert\Sigma_{\theta,h}^{\frac{1}{2}}\left(\nabla_\theta^j M_{\theta,h}\right)\Sigma_{\theta,h+1}^{-\frac{1}{2}}\right\Vert\leq G$ from Lemma \ref{ineq}. Furthermore, we have
\begin{align*}
    \left\Vert\Sigma_{\theta,h}^{\frac{1}{2}}\left(\Delta M_{\theta,h}\right)\Sigma_{\theta,h+1}^{-\frac{1}{2}}\right\Vert &\leq\left\Vert\Sigma_{\theta,h}^{\frac{1}{2}}\Sigma_h^{-\frac{1}{2}}\right\Vert\left\Vert\Sigma_{h+1}^{\frac{1}{2}}\Sigma_{\theta,h+1}^{-\frac{1}{2}}\right\Vert\left\Vert\Sigma_h^{\frac{1}{2}}\left(\Delta M_{\theta,h}\right)\Sigma_{h+1}^{-\frac{1}{2}}\right\Vert=\sqrt{\kappa_1}\left\Vert\Sigma_h^{\frac{1}{2}}\left(\Delta M_{\theta,h}\right)\Sigma_{h+1}^{-\frac{1}{2}}\right\Vert\\
    &=\sqrt{\kappa_1}\left\Vert\Sigma_h^{\frac{1}{2}}\left(\widehat{\Sigma}_h^{-1}\frac{1}{K}\sum_{k=1}^K\phi\left(s_h^{(k)},a_h^{(k)}\right)\int_{\mathcal{A}}\phi\left(s_{h+1}^{(k)},a^\prime\right)\pi_{\theta,h+1}\left(a^\prime\left\vert s_{h+1}^{(k)}\right.\right)\mathrm{d}a^\prime- M_{\theta,h}\right)\Sigma_{h+1}^{-\frac{1}{2}}\right\Vert\\
    &\leq\sqrt{\kappa_1}\left(\left(1+\left\Vert\Sigma_h^{\frac{1}{2}}\left(\Delta\Sigma_h^{-1}\right)\Sigma_h^{\frac{1}{2}}\right\Vert\right)\left(1+\left\Vert\Sigma_h^{-\frac{1}{2}}\left(\Delta Y_{\theta,h}\right)\Sigma_{h+1}^{-\frac{1}{2}}\right\Vert\right)-1\right),
\end{align*}
where $\Delta Y_{\theta,h}=\frac{1}{K}\sum_{k=1}^K\phi\left(s_h^{(k)},a_h^{(k)}\right)\phi_{\theta,h+1}\left(s_{h+1}^{(k)}\right)^\top-\Sigma_h M_{\theta,h}$ and the last inequality uses Lemma \ref{decomp} again. Similarly, we have
\begin{align*}
    \left\Vert\Sigma_{\theta,h}^{\frac{1}{2}}\left(\frac{\nabla_\theta^j\left(\Delta M_{\theta,h}\right)}{G}\right)\Sigma_{\theta,h+1}^{-\frac{1}{2}}\right\Vert\leq\sqrt{\kappa_1}\left(\left(1+\left\Vert\Sigma_h^{\frac{1}{2}}\left(\Delta\Sigma_h\right)\Sigma_h^{\frac{1}{2}}\right\Vert\right)\left(1+\left\Vert\Sigma_h^{-\frac{1}{2}}\left(\frac{\nabla_\theta^j\left(\Delta Y_{\theta,h}\right)}{G}\right)\Sigma_{h+1}^{-\frac{1}{2}}\right\Vert\right)-1\right).
\end{align*}
Now, define $\alpha=6\sqrt{\kappa_1}(4+\kappa_2+\kappa_3)\sqrt{\frac{C_1d\log\frac{8dmH}{\delta}}{K}}$ and pick
\begin{align*}
    K\geq 36\kappa_1(4+\kappa_2+\kappa_3)^2C_1dH^2\log\frac{8dmH}{\delta},\quad\lambda\leq C_1d\min_{h\in[H]}\sigma_{\textrm{min}}(\Sigma_h)\cdot\log\frac{8dmH}{\delta},
\end{align*}
we get $\alpha\leq\frac{1}{H}$. Using the results of Lemma \ref{dsig1}, Lemma \ref{dsig2}, Lemma \ref{dy}, we get with probability $1-2\delta$, 
\begin{align}
    \label{sig}
    \begin{aligned}
        \left\Vert\Sigma_h^{\frac{1}{2}}\left(\Delta\Sigma_h^{-1}\right)\Sigma_h^{\frac{1}{2}}\right\Vert\leq &2\left\Vert\Sigma_h^{-\frac{1}{2}}\widehat{\Sigma}_h\Sigma_h^{-\frac{1}{2}}-I_d\right\Vert\leq 2\sqrt{\frac{2C_1d\log\frac{2dH}{\delta}}{K}}+\frac{4C_1d\log\frac{2dH}{\delta}}{3K}+\frac{2\lambda\Vert\Sigma^{-1}_h\Vert}{K}\\
        \leq &4\sqrt{\frac{C_1d\log\frac{8dmH}{\delta}}{K}}\leq\alpha\leq 1, 
    \end{aligned}
\end{align}
and $\forall j\in[m]$, 
\begin{align*}
    \left\Vert\Sigma_h^{\frac{1}{2}}\left(\Delta Y_{\theta,h}\right)\Sigma_{h+1}^{-\frac{1}{2}}\right\Vert\leq 2(\kappa_2+1)\sqrt{\frac{C_1d\log\frac{8dmH}{\delta}}{K}}\leq 1,\quad\left\Vert\Sigma_h^{\frac{1}{2}}\left(\frac{\nabla_\theta^j\left(\Delta Y_{\theta,h}\right)}{G}\right)\Sigma_{h+1}^{-\frac{1}{2}}\right\Vert\leq 2(\kappa_3+1)\sqrt{\frac{C_1d\log\frac{8dmH}{\delta}}{K}}\leq 1,
\end{align*}
which implies
\begin{align}
    \label{dm}
    \left\Vert\Sigma_{\theta,h}^{\frac{1}{2}}\left(\Delta M_{\theta,h}\right)\Sigma_{\theta,h+1}^{-\frac{1}{2}}\right\Vert\leq 2\sqrt{\kappa_1}\left(\left\Vert\Sigma_h^{\frac{1}{2}}\left(\Delta\Sigma_h^{-1}\right)\Sigma_h^{\frac{1}{2}}\right\Vert+\left\Vert\Sigma_h^{\frac{1}{2}}\left(\Delta Y_{\theta,h}\right)\Sigma_{h+1}^{-\frac{1}{2}}\right\Vert\right)\leq\alpha,
\end{align}
where we use the fact $(1+x_1)(1+x_2)-1\leq 2(x_1+x_2)$ whenever $x_1,x_2\in[0,1]$. Similarly, we get
\begin{align}
    \label{ddm}
    \left\Vert\Sigma_{\theta,h}^{\frac{1}{2}}\left(\frac{\nabla_\theta^j\left(\Delta M_{\theta,h}\right)}{G}\right)\Sigma_{\theta,h+1}^{-\frac{1}{2}}\right\Vert\leq\alpha,\quad\forall j\in[m].
\end{align}
Meanwhile, by Lemma \ref{eps}, we get with probability $1-\delta$,
\begin{align}
    \label{ep1}
    \left\Vert\Sigma_h^{-\frac{1}{2}}\frac{1}{K}\sum_{k=1}^K\phi\left(s_h^{(k)},a_h^{(k)}\right)\varepsilon_{h,k}^\theta\right\Vert&\leq 4\sqrt{d}(H-h+1)\sqrt{\frac{\log\frac{8dmH}{\delta}}{K}}\\
    \label{ep2}
    \left\Vert\Sigma^{-\frac{1}{2}}_h\frac{1}{K}\sum_{k=1}^K\phi\left(s_h^{(k)},a_h^{(k)}\right)\nabla_\theta^j\varepsilon_{h,k}^\theta\right\Vert&\leq 8\sqrt{d}G(H-h)^2\sqrt{\frac{\log\frac{8dmH}{\delta}}{K}}.
\end{align}
Combining the results of \eqref{sig}, \eqref{dm}, \eqref{ddm}, \eqref{ep1},\eqref{ep2} and use a union bound, we have with probability $1-3\delta$, 
\begin{align*}
    \vert E_{21}^j\vert\leq&\sum_{h=1}^H16h\alpha\sqrt{d}G(H-h)^2\left\Vert\Sigma_{\theta,1}^{-\frac{1}{2}}\nu^\theta_1\right\Vert\left\Vert\Sigma_{\theta,h}^{\frac{1}{2}}\Sigma_h^{-\frac{1}{2}}\right\Vert \sqrt{\frac{\log\frac{8dmH}{\delta}}{K}}\\
    \leq&16\alpha\sqrt{d}H^4G\left\Vert\Sigma_{\theta,1}^{-\frac{1}{2}}\nu^\theta_1\right\Vert\max_{h\in[H]}\left\Vert\Sigma_{\theta,h}^{\frac{1}{2}}\Sigma_h^{-\frac{1}{2}}\right\Vert\sqrt{\frac{\log\frac{8dmH}{\delta}}{K}},\\
    \vert E_{22}^j\vert\leq&\sum_{h=1}^H8h\alpha\sqrt{d}(H-h+1)\left\Vert\Sigma_{\theta,1}^{-\frac{1}{2}}\nabla_\theta^j\nu^\theta_1\right\Vert\left\Vert\Sigma_{\theta,h}^{\frac{1}{2}}\Sigma_h^{-\frac{1}{2}}\right\Vert\sqrt{\frac{\log\frac{8dmH}{\delta}}{K}}\\
    \leq &8\alpha\sqrt{d}H^3\left\Vert\Sigma_{\theta,1}^{-\frac{1}{2}}\nabla_\theta^j\nu^\theta_1\right\Vert\max_{h\in[H]}\left\Vert\Sigma_{\theta,h}^{\frac{1}{2}}\Sigma_h^{-\frac{1}{2}}\right\Vert\sqrt{\frac{\log\frac{8dmH}{\delta}}{K}}\\
    \vert E_{23}^j\vert\leq&\sum_{h=1}^H16Gh\alpha\sqrt{d}(H-h+1)(h-1)\left\Vert\Sigma_{\theta,1}^{-\frac{1}{2}}\nu^\theta_1\right\Vert\left\Vert\Sigma_{\theta,h}^{\frac{1}{2}}\Sigma_h^{-\frac{1}{2}}\right\Vert\sqrt{\frac{\log\frac{8dmH}{\delta}}{K}}\\
    \leq&16\alpha\sqrt{d}H^4G\left\Vert\Sigma_{\theta,1}^{-\frac{1}{2}}\nu^\theta_1\right\Vert\max_{h\in[H]}\left\Vert\Sigma_{\theta,h}^{\frac{1}{2}}\Sigma_h^{-\frac{1}{2}}\right\Vert\sqrt{\frac{\log\frac{8dmH}{\delta}}{K}},
\end{align*}
where we use the fact $(1+\alpha)^h-1\leq 2h\alpha$ whenever $\alpha h\leq 1$.  Summing up the above terms and using the definition of $\alpha$, we get
\begin{align*}
    \vert E_2^j\vert\leq 240\sqrt{\kappa_1}(4+\kappa_2+\kappa_3)\sqrt{C_1}dH^3\left(\left\Vert\Sigma_{\theta,1}^{-\frac{1}{2}}\nabla_\theta^j\nu^\theta_1\right\Vert+HG\left\Vert\Sigma_{\theta,1}^{-\frac{1}{2}}\nu^\theta_1\right\Vert\right)\max_{h\in[H]}\left\Vert\Sigma_{\theta,h}^{\frac{1}{2}}\Sigma_h^{-\frac{1}{2}}\right\Vert\frac{\log\frac{8dmH}{\delta}}{K},\quad\forall j\in[m]
\end{align*}
Replacing $\delta$ by $\frac{\delta}{3}$, we have finished the proof. 
\end{proof}

\subsection{Proof of Lemma \ref{e3}}
\begin{proof}
Similar to the decomposition in the proof of Lemma \ref{e2}, we have
\begin{align*}
    \left\vert E_3^j\right\vert=&\frac{\lambda}{K}\left\vert\sum_{h=1}^H\nabla_\theta^j\left(\left(\widehat{\nu}^\theta_h\right)^\top\widehat{\Sigma}_h^{-1}w_h^\theta\right)\right\vert\\
    \leq&\frac{\lambda}{K}\sum_{h=1}^H\left\Vert\Sigma_{\theta,1}^{-\frac{1}{2}}\nu^\theta_1\right\Vert\left\Vert\Sigma_{\theta,h}^{\frac{1}{2}}\Sigma_h^{-\frac{1}{2}}\right\Vert\left(\prod_{h^\prime=1}^{h-1}\left(1+\left\Vert\Sigma_{\theta,h^\prime}^{\frac{1}{2}}\left(\Delta M_{\theta,h^\prime}\right)\Sigma_{\theta,h^\prime+1}^{-\frac{1}{2}}\right\Vert\right)\right)\left(1+\left\Vert\Sigma_h^{\frac{1}{2}}\left(\Delta\Sigma_h^{-1}\right)\Sigma_h^{\frac{1}{2}}\right\Vert\right)\left\Vert\Sigma_h^{-1}\right\Vert\left\Vert\Sigma_h^{\frac{1}{2}}\nabla_\theta^j w_h^\theta\right\Vert\\
    +&\frac{\lambda}{K}\sum_{h=1}^H\left\Vert\Sigma_{\theta,1}^{-\frac{1}{2}}\nabla_\theta^j\nu^\theta_1\right\Vert\left\Vert\Sigma_{\theta,h}^{\frac{1}{2}}\Sigma_h^{-\frac{1}{2}}\right\Vert\left(\prod_{h^\prime=1}^{h-1}\left(1+\left\Vert\Sigma_{\theta,h^\prime}^{\frac{1}{2}}\left(\Delta M_{\theta,h^\prime}\right)\Sigma_{\theta,h^\prime+1}^{-\frac{1}{2}}\right\Vert\right)\right)\left(1+\left\Vert\Sigma_h^{\frac{1}{2}}\left(\Delta\Sigma_h^{-1}\right)\Sigma_h^{\frac{1}{2}}\right\Vert\right)\left\Vert\Sigma_h^{-1}\right\Vert\left\Vert\Sigma_h^{\frac{1}{2}}w_h^\theta\right\Vert\\
    +&\frac{\lambda}{K}\sum_{h=1}^H\sum_{h^\prime=1}^{h-1}G\left\Vert\Sigma_{\theta,1}^{-\frac{1}{2}}\nu^\theta_1\right\Vert\left\Vert\Sigma_{\theta,h}^{\frac{1}{2}}\Sigma_h^{-\frac{1}{2}}\right\Vert\\
    &\cdot\left(\prod_{h^{\prime\prime}\neq h^\prime}\left(1+\left\Vert\Sigma_{\theta,h^{\prime\prime}}^{\frac{1}{2}}\left(\Delta M_{\theta,h^{\prime\prime}}\right)\Sigma_{\theta,h^{\prime\prime}}^{-\frac{1}{2}}\right\Vert\right)\right)\left(1+\left\Vert\Sigma_{\theta,h^\prime}^{\frac{1}{2}}\left(\frac{\nabla_\theta^j\left(\Delta M_{\theta,h^\prime}\right)}{G}\right)\Sigma_{\theta,h^\prime}^{-\frac{1}{2}}\right\Vert\right)\left(1+\left\Vert\Sigma_h^{\frac{1}{2}}\left(\Delta\Sigma_h^{-1}\right)\Sigma_h^{\frac{1}{2}}\right\Vert\right)\\
    &\cdot\left\Vert\Sigma_h^{-1}\right\Vert\left\Vert\Sigma_h^{\frac{1}{2}}w_h^\theta\right\Vert\\
    \leq&\frac{\lambda}{K}\sum_{h=1}^H\left(1+\alpha\right)^h\left\Vert\Sigma_h^{-1}\right\Vert\left\Vert\Sigma_{\theta,h}^{\frac{1}{2}}\Sigma_h^{-\frac{1}{2}}\right\Vert\left(\left\Vert\Sigma_{\theta,1}^{-\frac{1}{2}}\nu^\theta_1\right\Vert\left\Vert\Sigma_h^{\frac{1}{2}}\nabla_\theta^j w_h^\theta\right\Vert+\left\Vert\Sigma_{\theta,1}^{-\frac{1}{2}}\nabla_\theta^j\nu^\theta_1\right\Vert\left\Vert\Sigma_h^{\frac{1}{2}}w_h^\theta\right\Vert+G(h-1)\left\Vert\Sigma_{\theta,1}^{-\frac{1}{2}}\nu^\theta_1\right\Vert\left\Vert\Sigma_h^{\frac{1}{2}}w_h^\theta\right\Vert\right),
\end{align*}
where $\alpha$ is defined in the same way as that in the proof of Lemma \ref{e2}. Similarly, we have $\alpha\leq\frac{1}{H}$ with probability $1-3\delta$ and we have
\begin{align*}
    \left\Vert\Sigma_h^{\frac{1}{2}}\nabla_\theta^j w_h^\theta\right\Vert^2&=\mathbb{E}\left[\left(\nabla_\theta^j Q^\theta_h\left(s_h^{(1)},a_h^{(1)}\right)\right)^2\right]\leq G^2(H-h)^4\\
    \left\Vert\Sigma_h^{\frac{1}{2}}w_h^\theta\right\Vert^2&=\mathbb{E}\left[\left(Q^\theta_h\left(s_h^{(1)},a_h^{(1)}\right)\right)^2\right]\leq(H-h+1)^2.
\end{align*}
We conclude
\begin{align*}
    \vert E_3^j\vert\leq&3\frac{\lambda}{K}\sum_{h=1}^H\left\Vert\Sigma_h^{-1}\right\Vert\left\Vert\Sigma_{\theta,h}^{\frac{1}{2}}\Sigma_h^{-\frac{1}{2}}\right\Vert\left(\left\Vert\Sigma_{\theta,1}^{-\frac{1}{2}}\nu^\theta_1\right\Vert G(H-h)^2+\left\Vert\Sigma_{\theta,1}^{-\frac{1}{2}}\nabla_\theta^j\nu^\theta_1\right\Vert(H-h+1)+G(h-1)\left\Vert\Sigma_{\theta,1}^{-\frac{1}{2}}\nu^\theta_1\right\Vert(H-h+1)\right)\\
    \leq&6\frac{\lambda}{K}H^2\max_{h\in[H]}\left\Vert\Sigma_h^{-1}\right\Vert\left\Vert\Sigma_{\theta,h}^{\frac{1}{2}}\Sigma_h^{-\frac{1}{2}}\right\Vert\left(\left\Vert\Sigma_{\theta,1}^{-\frac{1}{2}}\nu^\theta_1\right\Vert GH+\left\Vert\Sigma_{\theta,1}^{-\frac{1}{2}}\nabla_\theta^j\nu^\theta_1\right\Vert\right)\\
    \leq&6\frac{\log\frac{8dmH}{\delta}C_1dH^2}{K}\max_{h\in[H]}\left\Vert\Sigma_{\theta,h}^{\frac{1}{2}}\Sigma_h^{-\frac{1}{2}}\right\Vert\left(\left\Vert\Sigma_{\theta,1}^{-\frac{1}{2}}\nu^\theta_1\right\Vert GH+\left\Vert\Sigma_{\theta,1}^{-\frac{1}{2}}\nabla_\theta^j\nu^\theta_1\right\Vert\right).
\end{align*}
Replacing $\delta$ by $\frac{\delta}{3}$, we have finished the proof. 
\end{proof}

\section{Extension to Time-homogeneous Discounted MDP}
\label{homo}
\subsection{Approach}
Our method can be easily extended to the case of time-homogeneous discounted MDP. Similar to the time-inhomogeneous case, under the setting of the time-homogeneous discounted MDP, an instance of MDP is defined by $(\mathcal{S},\mathcal{A},p,r, \xi,\gamma)$ where $\mathcal{S}$ and $\mathcal{A}$ are the state and action spaces, $\gamma\in(\frac{1}{2},1)$ is the discount factor, $p:\mathcal{S}\times\mathcal{A}\times\mathcal{S}\rightarrow\mathbb{R}_+$ is the transition probability, $r:\mathcal{S}\times\mathcal{A}\rightarrow[0,1]$ is the reward function and $\xi:\mathcal{S}\rightarrow\mathbb{R}_+$ is the initial state distribution. Similarly, the policy $\pi:\mathcal{S}\times\mathcal{A}\rightarrow\mathbb{R}_+$ is a distribution over the action space conditioned on an arbitrary given state $s$. We define the value function and $Q$ function by
\begin{align*}
v^\theta=\mathbb{E}\left[\sum_{h=1}^\infty\gamma^{h-1}r(s_h,a_h)\right],\quad Q^\theta(s,a)=\mathbb{E}\left[\left.\sum_{h=1}^\infty\gamma^{h-1}r(s_h,a_h)\right\vert s_1=s,a_1=a\right]
\end{align*}
Note that here the reward and Q function no longer contain the subscript $h$. We still consider the class of linear functions $\mathcal{F}$ with state-action feature $\phi$, and denote $\mathcal{P}_\theta$ as the transition operator where $\theta\in\mathbb{R}^m$ is the parameter of the policy. 
\begin{assumption}
\label{fclass_homo}
For any $f\in\mathcal{F}$, we have $\mathcal{P}_\theta f\in\mathcal{F}$, and we suppose $r\in\mathcal{F}$.
\end{assumption}
In addition, we assume that the constant function belongs to $\mathcal{F}$, i.e., there exists some $w_0$ such that $\phi(s,a)^\top w_0 = 1, \forall s\in\mathcal{S}, a\in\mathcal{A}$. Define the covariance matrix $\Sigma$ and its empirical version $\widehat{\Sigma}$ by
\begin{align*}
    \Sigma:=\mathbb{E}\left[\frac{1}{H}\sum_{h=1}^H\phi\left(s^{(1)}_h,a^{(1)}_h\right)\phi\left(s^{(1)}_h,a^{(1)}_h\right)^\top\right], \quad\widehat{\Sigma}:=\frac{1}{HK}\left(\lambda I_d+\sum_{k=1}^K\sum_{h=1}^H\phi\left(s_h^{(k)}, a_h^{(k)}\right)\phi\left(s_h^{(k)},a_h^{(k)}\right)^\top\right),
\end{align*}
where $I_d\in\mathbb{R}^{d\times d}$ is the identity matrix. 
\begin{assumption}[Boundedness Conditions]\label{Boundedness_Conditions_homo}
	Assume $\Sigma$ is invertible. There exist absolute constants $C_1, G$ such that for any $(s,a)\in \mathcal{S}\times\mathcal{A},j\in[m]$, we have
	\begin{align*}
		\phi(s,a)^\top\Sigma^{-1}\phi(s,a)\leq C_1 d,\quad\left\vert\nabla_\theta^j\log\pi_\theta(a\vert s)\right\vert\leq G.
	\end{align*}
\end{assumption}
Define $\widehat{w}_r\in\mathbb{R}^{d},\widehat{M_\theta}\in\mathbb{R}^{d\times d},\widehat{\nabla_\theta M_\theta}\in\mathbb{R}^{d\times d},j\in[m]$ by
\begin{align*}
	\widehat{w}_r&:=\widehat{\Sigma}^{-1}\frac{1}{HK}\sum_{k=1}^K\sum_{h=1}^H\phi\left(s_h^{(k)},a_h^{(k)}\right)r_h^{(k)},\\
	\widehat{M_\theta}&:=\widehat{\Sigma}^{-1}\frac{1}{KH}\sum_{k=1}^K\sum_{h=1}^H\phi\left(s_h^{(k)}, a_h^{(k)}\right)\int_{\mathcal{A}}\pi_\theta\left(a^\prime\left\vert s_{h+1}^{(k)}\right.\right)\phi\left(s_{h+1}^{(k)}, a^\prime\right)^\top\mathrm{d}a^\prime,\\
	\widehat{\nabla_\theta^j M_\theta}&:=\widehat{\Sigma}^{-1}\frac{1}{KH}\sum_{k=1}^K\sum_{h=1}^H\phi\left(s_h^{(k)},a_h^{(k)}\right)\int_{\mathcal{A}}\phi\left(s_{h+1}^{(k)}, a^\prime\right)^\top\nabla_\theta\pi_\theta\left(a^\prime\left\vert s_{h+1}^{(k)}\right.\right)\mathrm{d}a^\prime,\quad j\in[m].
\end{align*}
In this way, one can compute 
\begin{align*}
\widehat{Q}^\theta(\cdot,\cdot)=\phi(\cdot,\cdot)^\top\widehat{w}^\theta,\quad\widehat{\nabla_\theta^j Q^\theta}(\cdot,\cdot)=\phi(\cdot,\cdot)^\top\widehat{\nabla_\theta^j w^\theta},
\end{align*}
where 
\begin{align*}
\widehat{w}^\theta=\left(I_d-\gamma\widehat{M}_\theta\right)^{-1}\widehat{w}_r,\quad\widehat{\nabla_\theta^j w^\theta}=\left(I_d-\gamma\widehat{M}_\theta\right)^{-1}\widehat{\nabla_\theta^j M_\theta}\widehat{w}^\theta.
\end{align*}
Then the estimator is derived from
\begin{align*}
    &\widehat{\nabla_\theta v_\theta}=\int_{\mathcal{S}\times\mathcal{A}}\xi(s)\pi_\theta(a\vert s)\left(\widehat{\nabla_\theta Q^\theta}(s,a)+\left(\nabla_\theta\log\pi_\theta(a\vert s)\right)\widehat{Q}^\theta(s,a)\right)\mathrm{d}s\mathrm{d}a.
\end{align*}

\subsection{Results}
Define $\nu^\theta_h:=\mathbb{E}^{\pi_\theta}\left[\phi(s_h,a_h)\vert s_1\sim\xi\right], \nu^\theta=\sum_{h=1}^\infty\gamma^{h-1}\nu^\theta_h$ and $\Sigma_\theta:=\mathbb{E}^{\pi_\theta}\left[\left.\phi(s,a)\phi(s,a)^\top\right\vert s\sim\xi_\theta, a\sim\pi_\theta(\cdot\vert s)\right]$ where $\xi_\theta$ is the stationary distribution under $\pi_\theta$. Define
\begin{align*}
\phi_\theta(s)=\int_{\mathcal{A}}\pi_\theta(a^\prime\vert s)\phi(s,a^\prime)\mathrm{d}a^\prime,\quad\varepsilon_{h,k}^\theta=Q^\theta\left(s_h^{(k)},a_h^{(k)}\right)-r_h^{(k)}-\gamma\int_{\mathcal{A}}\pi_\theta\left(a^\prime\left\vert s_{h+1}^{(k)}\right.\right)Q^\theta\left(s_{h+1}^{(k)},a^\prime\right)\mathrm{d}a^\prime,     
\end{align*}
and 
\begin{align*}
   &\Lambda_\theta=\mathbb{E}\left[\frac{1}{H}\sum_{h=1}^H\left(\nabla_\theta\left(\varepsilon^\theta_{h,1}\phi\left(s_h^{(1)},a_h^{(1)}\right)^\top\Sigma^{-1}\nu^\theta\right)\right)^\top\nabla_\theta\left(\varepsilon^\theta_{h,1}\phi\left(s_h^{(1)},a_h^{(1)}\right)^\top\Sigma^{-1}\nu^\theta\right)\right].
\end{align*}
We first give the finite sample guarantee. 
\begin{theorem}[Finite Sample Guarantee] 
\label{thm2_var_homo}
For any $t\in\mathbb{R}^m$, when $K\geq 36\kappa_1(4+\kappa_2+\kappa_3)^2C_1d(1-\gamma)^{-2}\log\frac{16dmH}{\delta}$ and $\lambda\leq\log\frac{8dmH}{\delta}C_1d\sigma_{\min}(\Sigma)$, with probability $1-\delta$, we have,
\begin{align*}
    &\vert\langle t, \widehat{\nabla_\theta v_\theta}-\nabla_\theta v_\theta\rangle \vert\leq  \sqrt{\frac{2t^\top\Lambda_\theta t}{HK}\cdot \log\frac{8}{\delta}}+\frac{C_\theta\Vert t\Vert\log\frac{32mdH}{\delta}}{HK},
\end{align*}
where $C_\theta=240C_1dm^{0.5}(1-\gamma)^{-3}\kappa_1(5+\kappa_2+\kappa_3)\left(\max_{j\in[m]}\left\Vert\Sigma_\theta^{-\frac{1}{2}}\nabla_\theta^j\nu^\theta_1\right\Vert+\frac{G}{1-\gamma}\left\Vert\Sigma_\theta^{-\frac{1}{2}}\nu^\theta_1\right\Vert\right)$ and
\begin{align*}
    \kappa_1&=\frac{\sigma_{\max}\left(\Sigma^{-\frac{1}{2}}\Sigma_\theta\Sigma^{-\frac{1}{2}}\right)}{\sigma_{\min}\left(\Sigma^{-\frac{1}{2}}\Sigma_\theta\Sigma^{-\frac{1}{2}}\right)\wedge 1},\quad\kappa_2=\left\Vert\Sigma^{-\frac{1}{2}}\mathbb{E}\left[\phi_\theta\left(s_{h+1}^{(1)}\right)\phi_\theta\left(s_{h+1}^{(1)}\right)^\top\right]\Sigma^{-\frac{1}{2}}\right\Vert^{\frac{1}{2}},\\
    \kappa_3&=\frac{1}{G}\max_{j\in[m]}\left\Vert\Sigma^{-\frac{1}{2}}\mathbb{E}\left[\left(\nabla_\theta^j\phi_\theta\left(s_{h+1}^{(1)}\right)\right)\left(\nabla_\theta^j\phi_\theta\left(s_{h+1}^{(1)}\right)\right)^\top\right]\Sigma^{-\frac{1}{2}}\right\Vert^{\frac{1}{2}}.
\end{align*}
\end{theorem}
\begin{theorem}[Finite Sample Guarantee - Reward Free]
\label{thm2_homo}
Let the conditions in Theorem \ref{thm2_var_homo} hold, with probability $1-\delta$, we have for any reward function $r$,
\begin{align*}
&\left\vert\widehat{\nabla_\theta^j v_\theta}-\nabla_\theta^j v_\theta\right\vert\leq 4b_\theta\sqrt{\frac{\log\frac{8m}{\delta}}{HK}}+\frac{2C_\theta\log\frac{32mdH}{\delta}}{HK}, \quad\forall j\in[m],
\end{align*}
where $b_\theta=\frac{G}{(1-\gamma)^2}\left\Vert\Sigma^{-\frac{1}{2}}\nu^\theta\right\Vert+\frac{1}{1-\gamma}\left\Vert\Sigma^{-\frac{1}{2}}\nabla_\theta^j\nu^\theta\right\Vert$ and $C_\theta$ is the same as that in Theorem \ref{thm2_var_homo}. If we in addition have $\phi(s^\prime,a^\prime)^\top\Sigma^{-1}\phi(s,a)\geq 0,\forall (s,a),(s^\prime,a^\prime)\in\mathcal{S}\times\mathcal{A}$, we have
\begin{align*}
\left\vert\widehat{\nabla_\theta^j v_\theta}-\nabla_\theta^j v_\theta\right\vert\leq \frac{16G}{(1-\gamma)^2}\sqrt{\frac{\log\frac{8m}{\delta}}{HK}}\left\Vert\Sigma^{-\frac{1}{2}}\nu^\theta\right\Vert\log(C_1d)+\frac{2C_\theta\sqrt{m}\log\frac{32mdH}{\delta}}{HK},\quad\forall j\in[m].
\end{align*}
\end{theorem}
The complete proofs of Theorem \ref{thm2_var_homo} and Theorem \ref{thm2_homo} are deferred to Appendix \ref{pfthm2_var_homo} and \ref{pfthm2_homo}. Next we show that FPG is an asymptotically normal and efficient estimator. 
\begin{theorem}[Asymptotic Normality]
\label{thm1_homo}
The FPG estimator is asymptotically normal:
\begin{align*}
    \sqrt{HK}\left(\widehat{\nabla_\theta v_\theta}-\nabla_\theta v_\theta\right)\stackrel{d}{\rightarrow}\mathcal{N}(0,\Lambda_\theta).
\end{align*}
\end{theorem}
The proof of Theorem \ref{thm1_homo} is deferred to Appendix \ref{pfthm1_homo}. An obvious corollary of Theorem \ref{thm1_homo} is that for any vector $t \in \mathbb{R}^m,$ 
\begin{equation*}
    \sqrt{HK}\left\langle t,  \widehat{\nabla_{\theta} v_{\theta}} -\nabla_{\theta} v_{\theta} \right\rangle \stackrel{d}{\rightarrow} \mathcal{N}\left(0, t^{\top} \Lambda_{\theta} t\right).
\end{equation*}
The following theorem states the Cramer Rao bound for FPG estimation. 
\begin{theorem}
\label{thm4_homo}
Let Assumption \ref{fclass_homo} hold. For any vector $t\in\mathbb{R}^m$, the variance of any unbiased estimator for $t^{\top}\nabla_{\theta}v_{\theta}  \in \mathbb{R}$ is lower bounded by $\frac{1}{\sqrt{HK}}t^{\top} \Lambda_{\theta} t.$
\end{theorem}
The proof of Theorem \ref{thm4_homo} is deferred to Appendix \ref{pfthm4_homo}. 

\subsection{Additional Notations}
Define
\begin{align*}
\widehat{r}(\cdot,\cdot)&:= \phi(\cdot,\cdot)^\top \widehat{w}_r,\\
U^\theta &:=\gamma\mathcal{P}_\theta\left(\nabla_\theta\log\Pi_\theta\right) Q^\theta,\\
\widehat{U}^\theta&:=\gamma\widehat{\mathcal{P}}_\theta\left(\nabla_\theta\log\Pi_\theta\right)Q^\theta,\\
\tilde{U}^\theta&:=\gamma\widehat{\mathcal{P}}_\theta\left(\nabla_\theta\log\Pi_\theta\right)\widehat{Q}^\theta\\
\Delta Y_\theta&:=\frac{1}{HK}\sum_{k=1}^K\sum_{h=1}^H\phi\left(s_h^{(k)},a_h^{(k)}\right)\phi_\theta\left(s_{h+1}^{(k)}\right)^\top-\Sigma M_\theta\\
\Delta\Sigma^{-1}&:=\widehat{\Sigma}^{-1}-\Sigma.
\end{align*}
When $\mathcal{F}$ is the class of the linear functions, there exists matrix $M_\theta$ such that the transition probability satisfies
\begin{align*}
    \mathbb{E}^{\pi_\theta}\left[\phi(s^\prime,a^\prime)^\top\vert s,a\right] = \phi(s,a)^\top M_\theta. 
\end{align*}

\subsection{Technical Lemmas}
\begin{lemma}
\label{Q_decomp_base_homo}
We have
\begin{align*}
    Q^\theta=\sum_{h=1}^\infty\gamma^{h-1}\left(\mathcal{P}_\theta\right)^{h-1}r,\quad\nabla_\theta Q^\theta=\sum_{h=1}^\infty\gamma^{h-1}\left(\mathcal{P}_\theta\right)^{h-1}U^\theta.
\end{align*}
\end{lemma}
\begin{proof}
By Bellman's equation, we have $Q^\theta=r+\gamma\mathcal{P}_\theta Q^\theta$, which implies
\begin{align*}
Q^\theta = r+\gamma\mathcal{P}_\theta Q^\theta = r+\gamma\mathcal{P}_\theta r + \gamma^2\left(\mathcal{P}_\theta\right)^2 Q^\theta = \ldots = \sum_{h=1}^\infty\gamma^{h-1}\left(\mathcal{P}_\theta\right)^{h-1}r,
\end{align*}
which proves the first equation. Differentiating on both sides of the Bellman's equation w.r.t. $\theta$, we have
\begin{align*}
    \nabla_\theta Q^\theta(s,a)=\gamma\mathbb{E}^{\pi_\theta}\left[\left(\nabla_\theta\log\pi_\theta(a^\prime\vert s^\prime)\right)Q^\theta(s^\prime,a^\prime)\vert s,a\right]+\gamma\mathbb{E}^{\pi_\theta}\left[\nabla_\theta Q^\theta(s^\prime,a^\prime)\vert s,a\right],
\end{align*}
i.e., $\nabla_\theta Q^\theta=U^\theta+\gamma\mathcal{P}_\theta\nabla_\theta Q^\theta$. By induction, we have proved the second equation. 
\end{proof}

The decomposition leads to the following boundedness result: 
\begin{lemma}
\label{upbd_homo}
We have $\vert Q^\theta(s,a)\vert\leq \frac{1}{1-\gamma},\ \left\Vert\nabla_\theta Q^\theta(s,a)\right\Vert_\infty\leq\frac{G}{(1-\gamma)^2},\ \forall s\in\mathcal{S},a\in\mathcal{A}.$
\end{lemma}

\begin{lemma}
\label{decomp_homo}
For any series of matrices $A_1,A_2,\ldots,A_n$ and $\Delta A_1,\Delta A_2,\ldots\Delta A_n$, we have
\begin{align*}
    \left\Vert\prod_{i=1}^n(A_i+\Delta A_i)-\prod_{i=1}^n A_i\right\Vert\leq \prod_{i=1}^n\left(\Vert A_i\Vert+\Vert\Delta A_i \Vert\right)-\prod_{i=1}^n\Vert A_i\Vert.
\end{align*}
\end{lemma}
\begin{proof}
We have
\begin{align*}
    \left\Vert\prod_{i=1}^n(A_i+\Delta A_i)-\prod_{i=1}^n A_i\right\Vert&= \left\Vert\sum_{\delta\in\{0,1\}^n\setminus\{(1,1,\ldots,1)\}}\prod_{i=1}^n A_i^{\delta_i}(\Delta A_i)^{1-\delta_i}\right\Vert\leq\sum_{\delta\in \{0,1\}^n\setminus\{(1,1,\ldots,1)\}}\prod_{i=1}^n\Vert A_i\Vert^{\delta_i}\Vert\Delta A_i\Vert^{1-\delta_i}\\
    &=\prod_{i=1}^n\left(\Vert A_i\Vert+\Vert\Delta A_i\Vert\right)-\prod_{i=1}^n\Vert A_i\Vert.
\end{align*}
\end{proof}
The following lemma gives an upper bound on the 2-norm of $M_\theta$ and its derivatives. 
\begin{lemma}
\label{ineq_homo}
We have $\left\Vert\Sigma_\theta^{\frac{1}{2}}M_\theta\Sigma_\theta^{-\frac{1}{2}}\right\Vert \leq 1$ and $\left\Vert\Sigma_\theta^{\frac{1}{2}}\left(\nabla_\theta^j M_\theta\right)\Sigma_\theta^{-\frac{1}{2}}\right\Vert\leq G,\ \forall j\in[m]$.
\end{lemma}
\begin{proof}
Note that for any $f:\mathcal{S}\times\mathcal{A}\rightarrow \mathbb{R},\ f(s,a):=\mu^\top\phi(s, a)$, and any fixed $h\in\mathbb{N}_+$, we have
\begin{align*}
    \mathbb{E}^{\pi_\theta}\left[f^2(s_{h+1},a_{h+1})\vert s_1\sim\xi_\theta\right]&=\mathbb{E}^{\pi_\theta}\left[\mathbb{E}^{\pi_\theta}\left[f^2(s_{h+1},a_{h+1})\vert s_h,a_h\right]\vert s_1\sim\xi_\theta\right]\\
    &\geq\mathbb{E}^{\pi_\theta}[\mathbb{E}^{\pi_\theta}[f(s_{h+1},a_{h+1})\vert s_h,a_h]^2\vert  s_1\sim\xi_\theta].
\end{align*}
The LHS satisfies
\begin{align*}
    \mathbb{E}^{\pi_\theta}[f^2(s_{h+1},a_{h+1})\vert s_1\sim\xi_\theta]&=\mu^\top\Sigma_\theta\mu,
\end{align*}
and the RHS satisfies
\begin{align*}
    \mathbb{E}^{\pi_\theta}[\mathbb{E}^{\pi_\theta}[f(s_{h+1},a_{h+1})\vert s_h,a_h]^2\vert s_1\sim\xi_\theta]=\mathbb{E}^{\pi_\theta}[\mu^\top M_\theta^\top\phi(s_h,a_h)\phi(s_h,a_h)^\top M_\theta\mu\vert s_1\sim\xi_\theta]=\mu^\top M_\theta^\top\Sigma_\theta M_\theta\mu.
\end{align*}
Therefore, we have $\mu^\top\Sigma_\theta\mu\geq\mu^\top M_\theta^\top\Sigma_\theta M_\theta \mu,\ \forall\mu$, which implies $\Vert\Sigma_\theta^{\frac{1}{2}}M_\theta \Sigma_\theta^{-\frac{1}{2}}\Vert\leq 1$. Similarly, let $g:\mathcal{S}\times\mathcal{A}\rightarrow\mathbb{R},\ g(s, a):=\left(\nabla_\theta^j\log\pi_\theta(s,a)\right)\mu^\top\phi(s,a)$, we have
\begin{align*}
    \mathbb{E}^{\pi_\theta}[g^2(s_{h+1},a_{h+1})\vert s_1\sim\xi_\theta]=&\mathbb{E}^{\pi_\theta}[\mathbb{E}^{\pi_\theta}[g^2(s_{h+1},a_{h+1})\vert s_h,a_h]\vert s_1\sim\xi_\theta]\\
    \geq&\mathbb{E}^{\pi_\theta}[\mathbb{E}^{\pi_\theta}[g(s_{h+1},a_{h+1})\vert s_h,a_h]^2\vert s_1\sim\xi_\theta].
\end{align*}
The LHS satisfies
\begin{align*}
    \mathbb{E}^{\pi_\theta}[g^2(s_{h+1},a_{h+1})\vert s_1\sim\xi_\theta]&=\mu^\top\mathbb{E}^{\pi_\theta}\left[\left(\nabla_\theta^j\log\pi_\theta(a\vert s)\right)^2\phi(s_{h+1},a_{h+1})\phi(s_{h+1},a_{h+1})^\top\vert s_1\sim\xi_\theta\right]\mu\\
    &\leq G^2\mu^\top\mathbb{E}^{\pi_\theta}\left[\phi(s_{h+1},a_{h+1})\phi(s_{h+1},a_{h+1})^\top\vert s_1\sim\xi_\theta\right]\mu=G^2\mu^\top\Sigma_\theta\mu,
\end{align*}
and the RHS satisfies 
\begin{align*}
    &\mathbb{E}^{\pi_\theta}[\mathbb{E}^{\pi_\theta}[g(s_{h+1},a_{h+1})\vert s_h, a_h]^2\vert s_1\sim\xi_\theta]\\
    =&\mathbb{E}^{\pi_\theta}[\mu^\top\left(\nabla_\theta^j M_\theta\right)^\top\phi(s_h,a_h)\phi(s_h,a_h)^\top\left(\nabla_\theta^j M_\theta\right)\mu\vert s_1\sim\xi_\theta]\\
    =&\mu^\top\left(\nabla_\theta^j M_\theta\right)^\top\Sigma_\theta\left(\nabla_\theta^j M_\theta\right)\mu.
\end{align*}
Therefore, we get $G^2\mu^\top\Sigma_\theta\mu\geq\mu^\top\left(\nabla_\theta^j M_\theta\right)^\top\Sigma_\theta\left(\nabla_\theta^j M_\theta\right)\mu,\ \forall\mu$, which implies $\left\Vert\Sigma_\theta^{\frac{1}{2}}\left(\nabla_\theta^j M_\theta\right)\Sigma_\theta^{-\frac{1}{2}}\right\Vert\leq G$.
\end{proof}

\subsection{Probabilistic Events}
We define the following probabilistic events:
{\small
\begin{align*}
    \mathcal{E}_{\Sigma} &:= \left\{\left\Vert\Sigma^{-\frac{1}{2}}\left(\frac{1}{HK}\sum_{k=1}^K\sum_{h=1}^H\phi\left(s_h^{(k)}, a_h^{(k)}\right)\phi\left(s_h^{(k)},a_h^{(k)}\right)^\top\right)\Sigma^{-\frac{1}{2}}-I_d\right\Vert\leq \sqrt{\frac{2C_1d\log\frac{8dH}{\delta}}{K}} + \frac{2C_1d\log\frac{8dH}{\delta}}{3K}\right\},\\
    \mathcal{E}_{Y,0}&:=\left\{\left\Vert\Sigma^{-\frac{1}{2}}\left(\Delta Y_\theta\right)\Sigma^{-\frac{1}{2}}\right\Vert\leq\left(\kappa_2\vee 1\right)\sqrt{\frac{2C_1d\log\frac{16dH}{\delta}}{K}} + \frac{4C_1d\log\frac{16dH}{\delta}}{3K}\right\},\\
    \mathcal{E}_{Y,j}&:=\left\{\left\Vert\Sigma^{-\frac{1}{2}}\left(\nabla_\theta^j\left(\Delta Y_\theta\right)\right)\Sigma^{-\frac{1}{2}}\right\Vert\leq\left(\kappa_3\vee 1\right)G\sqrt{\frac{2C_1d\log\frac{16mdH}{\delta}}{K}}+\frac{4C_1dG\log\frac{16mdH}{\delta}}{3K}\right\}, j\in[m],\\
    \mathcal{E}_{Y}&:=\bigcap_{j=0}^m\mathcal{E}_{Y,j}\\
    \mathcal{E}_{\varepsilon,0}&:=\left\{\left\Vert\Sigma^{-\frac{1}{2}}\frac{1}{KH}\sum_{k=1}^K\sum_{h=1}^H\phi\left(s_h^{(k)},a_h^{(k)}\right)\varepsilon_{h,k}^\theta\right\Vert\leq\frac{\sqrt{d}}{1-\gamma}\left(\sqrt{\frac{2\log\frac{32dH}{\delta}}{KH}}+\frac{2\sqrt{C_1d}\log\frac{32dH}{\delta}}{K\sqrt{H}}+\frac{2C_1d\left(\log\frac{32dH}{\delta}\right)^{\frac{3}{2}}}{3K^{\frac{3}{2}}\sqrt{H}}\right)\right\},\\
    \mathcal{E}_{\varepsilon,j}&:=\left\{\left\Vert\Sigma^{-\frac{1}{2}}\frac{1}{KH}\sum_{k=1}^K\sum_{h=1}^H\phi\left(s_h^{(k)},a_h^{(k)}\right)\nabla_\theta^j\varepsilon_{h,k}^\theta\right\Vert\leq \frac{2\sqrt{d}G}{(1-\gamma)^2}\left(\sqrt{\frac{2\log\frac{32mdH}{\delta}}{KH}}+\frac{2\sqrt{C_1d}\log\frac{32mdH}{\delta}}{K\sqrt{H}}+\frac{2C_1d\left(\log\frac{32mdH}{\delta}\right)^{\frac{3}{2}}}{3K^{\frac{3}{2}}\sqrt{H}}\right)\right\}, j\in[m],\\
    \mathcal{E}_{\varepsilon} &:= \bigcap_{j=0}^m\mathcal{E}_{\varepsilon,j},\\
    \mathcal{E} &:= \mathcal{E}_\Sigma\cap\mathcal{E}_Y\cap\mathcal{E}_\varepsilon.
\end{align*}}
We have the following guarantees on the above high probability events: 
\begin{lemma}
\label{dsig1_homo}
$\mathbb{P}\left(\mathcal{E}_\Sigma\right) \geq 1-\frac{\delta}{4}$.
\end{lemma}
\begin{proof}
Define
\begin{align*}
	X^{(k)}=\frac{1}{H}\sum_{h=1}^H\Sigma^{-\frac{1}{2}}\phi\left(s^{(k)}_h,a^{(k)}_h\right) \phi\left(s^{(k)}_h,a^{(k)}_h\right)^\top\Sigma^{-\frac{1}{2}}\in\mathbb{R}^{d\times d}.
\end{align*}
It's easy to see that $X^{(1)},X^{(2)},\ldots,X^{(K)}$ are independent and $\mathbb{E}\left[X^{(k)}\right]=I_d$. In the remaining part of the proof, we will apply the matrix Bernstein's inequality to analyze the concentration of $\frac{1}{K}\sum_{k=1}^K X^{(k)}$. We first consider the matrix-valued variance $\textrm{Var}\left(X^{(k)}\right)=\mathbb{E}\left[\left(X^{(k)}-I_d\right)^2\right]=\mathbb{E}\left[\left(X^{(k)}\right)^2\right]-I_d$. Let
\begin{align*}
\Phi^{(k)}:=\left[\phi\left(s^{(k)}_1,a^{(k)}_1\right),\phi\left(s^{(k)}_2,a^{(k)}_2\right),\ldots,\phi\left(s^{(k)}_H,a^{(k)}_H\right)\right]\in\mathbb{R}^{d\times H},
\end{align*}
Then $X^{(k)}=\frac{1}{H}\Sigma^{-\frac{1}{2}}\Phi^{(k)}\left(\Phi^{(k)}\right)^\top\Sigma^{-\frac{1}{2}}$. For any vector $\mu\in\mathbb{R}^d$,
\begin{align*} 
	\mu^\top\mathbb{E}\left[\left(X^{(k)}\right)^2\right]\mu=&\mathbb{E}\left[\left\Vert X^{(k)}\mu\right\Vert^2\right]=\frac{1}{H^2}\mathbb{E}\left[\left\Vert\Sigma^{-\frac{1}{2}}\Phi^{(k)}\left(\Phi^{(k)}\right)^\top\Sigma^{-\frac{1}{2}}\mu\right\Vert^2\right]\\
	\leq&\frac{1}{H^2}\mathbb{E}\left[\left\Vert\Sigma^{-\frac{1}{2}}\Phi^{(k)}\right\Vert^2\left\Vert\left(\Phi^{(k)}\right)^\top\Sigma^{-\frac{1}{2}}\mu\right\Vert^2\right]\leq \frac{C_1d}{H}\mathbb{E}\left[\left\Vert\left(\Phi^{(k)}\right)^\top\Sigma^{-\frac{1}{2}}\mu\right\Vert^2\right]\\
    =&C_1d\mu^\top\mathbb{E}\left[X^{(k)}\right]\mu = C_1d \Vert\mu\Vert^2,
\end{align*}
where we used the identity $\left\Vert\left(\Phi^{(k)}\right)^\top\Sigma^{-\frac{1}{2}}\mu\right\Vert^2=\mu^\top X^{(k)}\mu$ and $\mathbb{E}\left[X^{(k)}\right]=I_d$. We have
\begin{align*}
    \textrm{Var}(X^{(k)})\preceq\mathbb{E}\left[\left(X^{(k)}\right)^2\right]\preceq C_1dI_d. 
\end{align*}
Additionally,
\begin{align*}
	-I_d\preceq X^{(k)}-I_d=\frac{1}{H}\sum_{h=1}^H\Sigma^{-\frac{1}{2}}\phi\left(s^{(k)}_h,a^{(k)}_h\right)\phi\left(s^{(k)}_h,a^{(k)}_h\right)^\top\Sigma^{-\frac{1}{2}}-I_d \preceq C_1dI_d-I_d. 
\end{align*}
Therefore, $\Vert X^{(k)}-I_d\Vert\leq C_1d$. Since $X^{(1)},X^{(2)},\ldots,X^{(K)}$ are \textit{i.i.d.}, by the matrix-form Bernstein inequality, we have
\begin{align*}
	\mathbb{P}\left(\left\Vert\sum_{k=1}^K X^{(k)}-I_d\right\Vert\geq\varepsilon\right)\leq 2d\cdot\exp\left(-\frac{\varepsilon^2/2}{C_1dK+C_1d \varepsilon/3}\right),\quad \forall \varepsilon>0, 
\end{align*}
i.e., with probability at least $1-\frac{\delta}{4}$,
\begin{align*}
    \left\Vert\frac{1}{K}\sum_{k=1}^K\left(X^{(k)}-I_d\right)\right\Vert\leq \sqrt{\frac{2C_1d\log\frac{8d}{\delta}}{K}}+\frac{2C_1d\log\frac{8d}{\delta}}{3K}, 
\end{align*}
which has finished the proof. 
\end{proof}
 
\begin{lemma}
\label{dy_homo}
$\mathbb{P}\left(\mathcal{E}_{Y}\right)\geq 1-\frac{\delta}{4}$. 
\end{lemma}
\begin{proof}
Take
\begin{align*}
	Y_\theta^{(k)}:=\frac{1}{H}\sum_{h=1}^H\Sigma^{-\frac{1}{2}}\phi\left(s^{(k)}_h,a^{(k)}_h\right)\phi_\theta\left(s^{(k)}_{h+1}\right)^\top\Sigma^{-\frac{1}{2}},\quad\forall k\in[K]. 
\end{align*}
Then, $\Sigma^{-\frac{1}{2}}\left(\Delta Y_\theta\right)\Sigma^{-\frac{1}{2}}=\frac{1}{K}\sum_{k=1}^K\left(Y_\theta^{(k)}-\Sigma^{\frac{1}{2}}M_\theta\Sigma^{-\frac{1}{2}}\right)$. Note that
\begin{align}
    \label{SMS0_homo} 
    \begin{aligned} 
        \mathbb{E}\left[Y_\theta^{(k)}\right]=&\frac{1}{H}\sum_{h=1}^H\mathbb{E}\left[\Sigma^{-\frac{1}{2}}\phi\left(s^{(k)}_h,a^{(k)}_h\right)\phi_\theta\left(s^{(k)}_{h+1}\right)^\top\Sigma^{-\frac{1}{2}}\right]\\
        =&\frac{1}{H}\sum_{h=1}^H\mathbb{E}\left[\Sigma^{-\frac{1}{2}}\phi\left(s^{(k)}_h,a^{(k)}_h\right)\mathbb{E}^{\pi_\theta}\left[\phi\left(s^\prime,a^\prime\right)^\top\vert s^{(k)}_h,a^{(k)}_h\right]\Sigma^{-\frac{1}{2}}\right]\\
        =&\frac{1}{H}\sum_{h=1}^H\mathbb{E}\left[\Sigma^{-\frac{1}{2}}\phi\left(s^{(k)}_h,a^{(k)}_h\right)\phi\left(s^{(k)}_h,a^{(k)}_h\right)^\top M_\theta\Sigma^{-\frac{1}{2}}\right]=\Sigma^{\frac{1}{2}}M_\theta\Sigma^{-\frac{1}{2}}, 
    \end{aligned}
\end{align}
To this end, $\Sigma^{-\frac{1}{2}}\left(\Delta Y_\theta\right)\Sigma^{-\frac{1}{2}}=\frac{1}{K}\sum_{k=1}^K\left(Y_\theta^{(k)}-\mathbb{E}\left[Y_\theta^{(k)}\right]\right)$. Since the trajectories are \textrm{i.i.d.}, we use the matrix-form Bernstein inequality to estimate $\left\Vert\Sigma^{-\frac{1}{2}}(\Delta Y_\theta)\Sigma^{-\frac{1}{2}}\right\Vert$. Let
\begin{align*}
\Phi^{(k)}&:=\left[\phi\left(s^{(k)}_1,a^{(k)}_1\right),\phi\left(s^{(k)}_2,a^{(k)}_2\right),\ldots,\phi\left(s^{(k)}_H,a^{(k)}_H\right)\right]\in\mathbb{R}^{d\times H},\\
\Phi^{(k)}_\theta&:=\left[\phi_\theta\left(s^{(k)}_2\right),\phi_\theta\left(s^{(k)}_3\right),\ldots,\phi_\theta\left(s^{(k)}_{H+1}\right)\right]\in\mathbb{R}^{d\times H},
\end{align*}
We have $Y^{(k)}_\theta=\frac{1}{H}\Sigma^{-\frac{1}{2}}\Phi^{(k)}\left(\Phi^{(k)}_\theta\right)^\top\Sigma^{-\frac{1}{2}}$. For any $\mu\in\mathbb{R}^d$, we have
\begin{align*}
    \mu^\top\mathbb{E}\left[Y_\theta^{(k)}\left(Y_\theta^{(k)}\right)^\top\right]\mu=&\mathbb{E}\left[\left\Vert\left(Y_\theta^{(k)}\right)^\top\mu\right\Vert^2\right]=\frac{1}{H^2}\mathbb{E}\left[\left\Vert\Sigma^{-\frac{1}{2}}\Phi_\theta^{(k)}\left(\Phi^{(k)}\right)^\top\Sigma^{-\frac{1}{2}}\mu\right\Vert^2\right]\\
    \leq&\frac{1}{H^2}\mathbb{E}\left[\left\Vert\Sigma^{-\frac{1}{2}}\Phi_\theta^{(k)}\right\Vert^2\left\Vert\left(\Phi^{(k)}\right)^\top\Sigma^{-\frac{1}{2}}\mu\right\Vert^2\right]\\
    \leq&\frac{C_1d}{H}\mathbb{E}\left[\left\Vert\left(\Phi^{(k)}\right)^\top\Sigma^{-\frac{1}{2}}\mu\right\Vert^2\right]\\
    =&\frac{C_1d}{H}\mu^\top\Sigma^{-\frac{1}{2}}\mathbb{E}\left[\Phi^{(k)}\left(\Phi^{(k)}\right)^\top\right]\Sigma^{-\frac{1}{2}}\mu\\
    =&C_1d\Vert\mu\Vert^2,
\end{align*}
where we have used the fact $\Sigma=\frac{1}{H}\mathbb{E}\left[\Phi^{(k)}\left(\Phi^{(k)}\right)^\top\right]$. It follows that
\begin{align*} 
    \textrm{Var}_1\left(Y_\theta^{(k)}\right):=&\mathbb{E}\left[\left(Y_\theta^{(k)}-\mathbb{E}\left[Y_\theta^{(k)}\right]\right)\left(Y_\theta^{(k)}-\mathbb{E}\left[Y_\theta^{(k)}\right]\right)^\top\right]\preceq\mathbb{E}\left[Y_\theta^{(k)}\left(Y_\theta^{(k)}\right)^\top\right]\preceq C_1dI_d. 
\end{align*} 
Analogously,
\begin{align*} 
	\textrm{Var}_2\left(Y_\theta^{(k)}\right):=&\mathbb{E}\left[\left(Y_\theta^{(k)}-\mathbb{E}\left[Y_\theta^{(k)}\right]\right)^\top\left(Y_\theta^{(k)}-\mathbb{E}\left[Y_\theta^{(k)}\right]\right)\right]\preceq\mathbb{E}\left[\left(Y_\theta^{(k)}\right)^\top Y_\theta^{(k)}\right]\\
	\preceq& \frac{C_1d}{H}\Sigma^{-\frac{1}{2}}\mathbb{E}\left[\Phi^{(k)}_\theta\left(\Phi^{(k)}_\theta\right)^\top\right]\Sigma^{-\frac{1}{2}}. 
\end{align*}
Therefore, $\max\left\{\left\Vert\textrm{Var}_1\left(Y_\theta^{(k)}\right)\right\Vert, \left\Vert\textrm{Var}_2\left(Y_\theta^{(k)}\right)\right\Vert\right\}\leq C_1d\left(\kappa_2^2\vee 1\right)$. It also holds that $\Vert Y_\theta^{(k)}\Vert\leq C_1d$. Hence,
\begin{align*}
    \left\Vert Y_\theta^{(k)}-\Sigma^{\frac{1}{2}}M_\theta\Sigma^{-\frac{1}{2}}\right\Vert\leq 2C_1d. 
\end{align*}
Applying Matrix Bernstein's inequality, we derive for any $\varepsilon>0$,
\begin{align*}
    \mathbb{P}\left(\left\Vert\sum_{k=1}^K\left(Y_\theta^{(k)}-\Sigma^{\frac{1}{2}}M_\theta\Sigma^{-\frac{1}{2}}\right)\right\Vert>\varepsilon\right)\leq 2d\exp\left(-\frac{\varepsilon^2/2}{C_1dK\left(\kappa_2^2\vee 1\right)+2C_1d\varepsilon/3}\right), 
\end{align*}
which implies $\mathcal{E}_{Y,0}$ holds with probability $1-\frac{\delta}{8}$. For $\mathcal{E}_{Y,j},j\in[m]$, notice that for any $j\in[m]$, we have $\Sigma^{-\frac{1}{2}}(\nabla_\theta^j(\Delta Y_\theta))\Sigma^{-\frac{1}{2}}=\frac{1}{K}\sum_{k=1}^K\left(\nabla_\theta^j Y_\theta^{(k)}-\mathbb{E}\left[\nabla_\theta^j Y_\theta^{(k)}\right]\right)$, and $\nabla_\theta^j Y_\theta^{(k)}=\frac{1}{H}\Sigma^{-\frac{1}{2}}\Phi^{(k)}\left(\nabla_\theta^j\Phi_\theta^{(k)}\right)^\top\Sigma^{-\frac{1}{2}}$. For any $\mu\in\mathbb{R}^d$, we have
\begin{align*}
    \mu^\top\mathbb{E}\left[\left(\nabla_\theta^j Y_\theta^{(k)}\right)\left(\nabla_\theta^j Y_\theta^{(k)}\right)^\top\right]\mu=&\frac{1}{H^2}\mathbb{E}\left[\left\Vert\Sigma^{-\frac{1}{2}}\left(\nabla_\theta^j\Phi_\theta^{(k)}\right)\left(\Phi^{(k)}\right)^\top\Sigma^{-\frac{1}{2}}\mu\right\Vert^2\right]\\
    \leq&\frac{1}{H^2}\mathbb{E}\left[\left\Vert\Sigma^{-\frac{1}{2}}\nabla_\theta^j\Phi_\theta^{(k)}\right\Vert^2\left\Vert\left(\Phi^{(k)}\right)^\top\Sigma^{-\frac{1}{2}}\mu\right\Vert^2\right]. 
\end{align*}
Since we have 
\begin{align*}
    &\left(\nabla_\theta^j\phi_\theta\left(s_{h+1}^{(k)}\right)\right)^\top\Sigma^{-1}\nabla_\theta^j\phi_\theta\left(s_{h+1}^{(k)}\right)\\
    =&\int_{\mathcal{A}\times\mathcal{A}}\pi_\theta\left(a\left\vert s_{h+1}^{(k)}\right.\right)\pi_\theta\left(a^\prime\left\vert s_{h+1}^{(k)}\right.\right)\left(\nabla_\theta^j\log\pi_\theta\left(a\left\vert s_{h+1}^{(k)}\right.\right)\right)\left(\nabla_\theta^j\log\pi_\theta\left(a^\prime\left\vert s_{h+1}^{(k)}\right.\right)\right)\\
    &\cdot\phi\left(s_{h+1}^{(k)},a\right)^\top\Sigma^{-1}\phi\left(s_{h+1}^{(k)},a^\prime\right)\mathrm{d}a\mathrm{d}a^\prime\\
    \leq&G^2\int_{\mathcal{A}\times\mathcal{A}}\pi_\theta\left(a\vert s_{h+1}^{(k)}\right)\pi_\theta\left(a^\prime\vert s_{h+1}^{(k)}\right)\left\Vert\Sigma^{-\frac{1}{2}}\phi\left(s_{h+1}^{(k)},a\right)\right\Vert\left\Vert\Sigma^{-\frac{1}{2}}\phi\left(s_{h+1}^{(k)},a^\prime\right)\right\Vert\mathrm{d}a\mathrm{d}a^\prime\leq G^2C_1d,
\end{align*}
which implies 
\begin{align*}
    \mu^\top\mathbb{E}\left[\left(\nabla_\theta^j Y_\theta^{(k)}\right)\left(\nabla_\theta^j Y_\theta^{(k)}\right)^\top\right]\mu\leq \frac{G^2C_1d}{H}\mathbb{E}\left[\left\Vert\left(\Phi^{(k)}\right)^\top\Sigma^{-\frac{1}{2}}\mu\right\Vert^2\right]=G^2C_1d\Vert\mu\Vert^2. 
\end{align*}
Therefore, 
\begin{align*}
    \textrm{Var}_1\left(\nabla_\theta^j Y_\theta^{(k)}\right):=&\mathbb{E}\left[\left(\nabla_\theta^j Y_\theta^{(k)}-\mathbb{E}\left[\nabla_\theta^j Y_\theta^{(k)}\right]\right)\left(\nabla_\theta^j Y_\theta^{(k)}-\mathbb{E}\left[\nabla_\theta^j Y_\theta^{(k)}\right]\right)^\top\right]\preceq\mathbb{E}\left[\left(\nabla_\theta^j Y_\theta^{(k)}\right)\left(\nabla_\theta^j Y_\theta^{(k)}\right)^\top\right]\preceq G^2C_1dI_d.
\end{align*}
Meanwhile, we have
\begin{align*} 
	\textrm{Var}_2\left(\nabla_\theta^j Y_\theta^{(k)}\right)\preceq\mathbb{E}\left[\left(\nabla_\theta^j Y_\theta^{(k)}\right)^\top\nabla_\theta^j Y_\theta^{(k)}\right]\preceq \frac{C_1d}{H}\Sigma^{-\frac{1}{2}}\mathbb{E}\left[\left(\nabla_\theta^j \Phi_\theta^{(k)}\right)\left(\nabla_\theta^j\Phi_\theta^{(k)}\right)^\top\right]\Sigma^{-\frac{1}{2}}. 
\end{align*}
In conclusion, we get
\begin{align*}
    \max\left\{\left\Vert\textrm{Var}_1\left(\nabla_\theta^j Y_\theta^{(k)}\right)\right\Vert,\left\Vert\textrm{Var}_2\left(\nabla_\theta^j Y_\theta^{(k)}\right)\right\Vert\right\}\leq G^2C_1d\left(\kappa_3^2\vee 1\right),
\end{align*}
Note that $\left\Vert\nabla_\theta^j Y_\theta^{(k)}\right\Vert\leq C_1dG$, we know $\left\Vert\nabla_\theta^j Y_\theta^{(k)}-\mathbb{E}\left[\nabla_\theta^j Y_\theta^{(k)}\right]\right\Vert\leq 2C_1dG$. By Matrix Bernstein's inequality, we get for any $\varepsilon>0$,
\begin{align*}
    \mathbb{P}\left(\left\Vert\sum_{k=1}^K\left(\nabla_\theta^j Y_k^\theta-\Sigma^{\frac{1}{2}}\left(\nabla_\theta^j M_\theta\right) \Sigma^{-\frac{1}{2}}\right)\right\Vert\geq\varepsilon\right)\leq 2d\exp\left(-\frac{\varepsilon^2/2}{G^2C_1dK\left(\kappa_3^2\vee 1\right)+2C_1dG \varepsilon/3}\right), 
\end{align*}
taking a union bound over all $j\in[m]$ proves that $\bigcap_{j=1}^m\mathcal{E}_{Y,j}$ holds with probability $1-\frac{\delta}{8}$. Using a union bound argument again, we know with probability $1-\frac{\delta}{4}$, $\mathcal{E}_Y$ holds, which has finished the proof.
\end{proof}

\begin{lemma}
\label{eps_homo}
$\mathbb{P}(\mathcal{E}_{\varepsilon}) \geq 1-\frac{\delta}{4}$. 
\end{lemma}
\begin{proof}
Let $X_{\theta,h}^{(k)}:=\Sigma^{-\frac{1}{2}}\phi\left(s_h^{(k)},a_h^{(k)}\right)\varepsilon_{h,k}^\theta\in\mathbb{R}^d$ and let $\mathcal{F}_{h,k}$ be $\sigma$-algebra generated by the history up to step $h$ at episode $k$, we have $\mathbb{E}\left[\left.X_{\theta,h}^{(k)}\right\vert\mathcal{F}_{h,k}\right]=0$. We apply matrix-form Freedman's inequality to analyze the concentration property. Consider conditional variances $\textrm{Var}_1\left[\left.X_{\theta,h}^{(k)}\right\vert\mathcal{F}_{h,k}\right]:=\mathbb{E}\left[\left.X_{\theta,h}^{(k)} \left(X_{\theta,h}^{(k)}\right)^\top\right\vert\mathcal{F}_{h,k}\right]\in\mathbb{R}^{d\times d}$ and $\textrm{Var}_2\left[\left.X_{\theta,h}^{(k)}\right\vert\mathcal{F}_{h,k}\right]:=\mathbb{E} \left[\left(X_{\theta,h}^{(k)}\right)^\top X_{\theta,h}^{(k)}\vert\mathcal{F}_{h,k}\right]\in\mathbb{R}$. It holds that
\begin{align*} 
    \left\Vert\textrm{Var}_1\left[\left.X_{\theta,h}^{(k)}\right\vert\mathcal{F}_{h,k}\right]\right\Vert=&\left\Vert\mathbb{E}\left[\left.X_{\theta,h}^{(k)}\left(X_{\theta,h}^{(k)}\right)^\top\right\vert\mathcal{F}_{h,k}\right]\right\Vert\leq\mathbb{E}\left[\left.\left\Vert X_{\theta,h}^{(k)}\left(X_{\theta,h}^{(k)}\right)^\top\right\Vert\right\vert\mathcal{F}_{h,k}\right]\\
    =&\mathbb{E}\left[\left.\left\Vert X_{\theta,h}^{(k)}\right\Vert^2\right\vert\mathcal{F}_{h,k}\right]= \textrm{Var}_2\left[\left.X_{\theta,h}^{(k)}\right\vert\mathcal{F}_{h,k}\right] 
\end{align*}
and
\begin{align*} 
    \textrm{Var}_2\left[\left.X_{\theta,h}^{(k)}\right\vert\mathcal{F}_{h,k}\right]=&\mathbb{E}\left[\left.\left\Vert X_{\theta,h}^{(k)}\right\Vert^2\right\vert\mathcal{F}_{h,k}\right]=\phi\left(s_h^{(k)},a_h^{(k)}\right)^\top\Sigma^{-1}\phi\left(s_h^{(k)},a_h^{(k)}\right)\textrm{Var}\left[\left.\varepsilon_{h,k}^\theta\right\vert s_h^{(k)},a_h^{(k)}\right]\\
    \leq&\frac{1}{(1-\gamma)^2}\phi\left(s_h^{(k)},a_h^{(k)}\right)^\top\Sigma^{-1}\phi\left(s_h^{(k)},a_h^{(k)}\right), 
\end{align*}
where we have used $\textrm{Var}\left[\left.\varepsilon_{h,k}^\theta\right\vert s_h^{(k)},a_h^{(k)} \right]\leq\frac{1}{(1-\gamma)^2}$. Note that
\begin{align*} 
    \sum_{k=1}^K\sum_{h=1}^H\phi\left(s_h^{(k)},a_h^{(k)}\right)^\top\Sigma^{-1}\phi\left(s_h^{(k)},a_h^{(k)}\right)=&KHd+KH\textrm{Tr}\left(\Sigma^{-\frac{1}{2}}\left(\frac{1}{HK}\sum_{k=1}^K\sum_{h=1}^H\phi\left(s_h^{(k)},a_h^{(k)}\right)\phi\left(s_h^{(k)},a_h^{(k)}\right)^\top\right)\Sigma^{-\frac{1}{2}}-I_d\right)\\
    \leq&KHd+KHd\left\Vert\Sigma^{-\frac{1}{2}}\left(\frac{1}{HK}\sum_{k=1}^K\sum_{h=1}^H\phi\left(s_h^{(k)},a_h^{(k)}\right)\phi\left(s_h^{(k)},a_h^{(k)}\right)^\top\right)\Sigma^{-\frac{1}{2}}-I_d\right\Vert. 
\end{align*}
We take
\begin{align} 
    \label{sigma2_homo} 
    \sigma^2:=\frac{KHd}{(1-\gamma)^2}\left(1+\sqrt{\frac{2C_1d\log\frac{32dH}{\delta}}{K}}+\frac{2C_1d\log\frac{32dH}{\delta}}{3K}\right). 
\end{align}
The result of Lemma \ref{dsig1_homo} implies that
\begin{align}
    \label{Var2_homo} 
    \mathbb{P}\left(\left\Vert\sum_{k=1}^K\sum_{h=1}^H\textrm{Var}_1\left[\left.X_{\theta,h}^{(k)}\right\vert\mathcal{F}_{h,k}\right]\right\Vert\leq\sum_{k=1}^K\sum_{h=1}^H\textrm{Var}_2\left[\left.X_{\theta,h}^{(k)}\right\vert\mathcal{F}_{h,k}\right]\leq\sigma^2\right)\geq 1-\frac{\delta}{16}. 
\end{align}
Additionally, we have $\left\Vert X_{\theta,h}^{(k)}\right\Vert\leq\frac{\sqrt{C_1d}}{1-\gamma}$. The Freedman's inequality therefore implies that for any $\varepsilon>0$,
\begin{align} 
    \label{Freedman2_homo} 
    \begin{aligned}
        &\mathbb{P}\left(\left\vert\sum_{k=1}^K\sum_{h=1}^H X_{\theta,h}^{(k)}\right\vert\geq \varepsilon,\ \max\left\{\left\Vert\sum_{k=1}^K\sum_{h=1}^H\textrm{Var}_1\left[\left.X_{\theta,h}^{(k)}\right\vert\mathcal{F}_{h,k}\right]\right\Vert,\sum_{k=1}^K\sum_{h=1}^H\textrm{Var}_2\left[\left.X_{\theta,h}^{(k)}\right\vert\mathcal{F}_{h,k}\right]\right\}\leq\sigma^2\right)\\
        \leq &2d\exp\left(-\frac{\varepsilon^2/2}{\sigma^2+\sqrt{C_1d}\varepsilon/(3(1-\gamma))} \right),
    \end{aligned} 
\end{align}
where $\sigma^2$ is defined in \eqref{sigma2_homo}. We take
\begin{align*}
    \varepsilon:=\sigma\sqrt{2\log\frac{32d}{\delta}}+\frac{2\sqrt{C_1d}}{3(1-\gamma)}\log\frac{32d}{\delta}.
\end{align*}
Then we get
\begin{align*}
    \mathbb{P}\left(\left\vert\sum_{k=1}^K\sum_{h=1}^H X_{\theta,h}^{(k)}\right\vert\geq\varepsilon,\ \max\left\{\left\Vert\sum_{k=1}^K\sum_{h=1}^H\textrm{Var}_1\left[\left.X_{\theta,h}^{(k)}\right\vert\mathcal{F}_{h,k}\right]\right\Vert,\sum_{k=1}^K\sum_{h=1}^H\textrm{Var}_2\left[\left.X_{\theta,h}^{(k)}\right\vert\mathcal{F}_{h,k}\right]\right\}\leq\sigma^2\right)\leq\frac{\delta}{16},
\end{align*}
which implies
\begin{align*}
    \mathbb{P}\left(\left\vert\sum_{k=1}^K\sum_{h=1}^H X_{\theta,h}^{(k)}\right\vert\geq \varepsilon\right)\leq&\mathbb{P}\left(\left\vert\sum_{k=1}^K\sum_{h=1}^H X_{\theta,h}^{(k)}\right\vert\geq \varepsilon,\ \max\left\{\left\Vert\sum_{k=1}^K\sum_{h=1}^H\textrm{Var}_1\left[\left.X_{\theta,h}^{(k)}\right\vert\mathcal{F}_{h,k}\right]\right\Vert,\sum_{k=1}^K\sum_{h=1}^H\textrm{Var}_2\left[\left.X_{\theta,h}^{(k)}\right\vert\mathcal{F}_{h,k}\right]\right\}\leq\sigma^2\right)\\
    &+\mathbb{P}\left(\max\left\{\left\Vert\sum_{k=1}^K\sum_{h=1}^H\textrm{Var}_1\left[\left.X_{\theta,h}^{(k)}\right\vert\mathcal{F}_{h,k}\right]\right\Vert,\sum_{k=1}^K\sum_{h=1}^H\textrm{Var}_2\left[\left.X_{\theta,h}^{(k)}\right\vert\mathcal{F}_{h,k}\right]\right\}>\sigma^2\right)\leq\frac{\delta}{8}.
\end{align*}
which has proved $\mathbb{P}\left(\mathcal{E}_{\varepsilon,0}\right)\geq 1-\frac{\delta}{8}$. For any fixed $j\in[m]$, we use Freedman's inequality again to prove $\mathbb{P}\left(\mathcal{E}_{\varepsilon,j}\right)\geq 1-\frac{\delta}{8m}$. We have
\begin{align*} 
    \left\Vert\textrm{Var}_1\left[\left.\nabla_\theta^j X_{\theta,h}^{(k)}\right\vert\mathcal{F}_{h,k}\right]\right\Vert=&\left\Vert\mathbb{E}\left[\left.\left(\nabla_\theta^j X_{\theta,h}^{(k)}\right)\left(\nabla_\theta^j X_{\theta,h}^{(k)}\right)^\top\right\vert\mathcal{F}_{h,k}\right]\right\Vert\leq\mathbb{E}\left[\left.\left\Vert\left(\nabla_\theta^j X_{\theta,h}^{(k)}\right)\left(\nabla_\theta^j X_{\theta,h}^{(k)}\right)^\top\right\Vert\right\vert\mathcal{F}_{h,k}\right]\\
    =&\mathbb{E}\left[\left.\left\Vert\nabla_\theta^j X_{\theta,h}^{(k)}\right\Vert^2\right\vert\mathcal{F}_{h,k}\right]= \textrm{Var}_2\left[\left.\nabla_\theta^j X_{\theta,h}^{(k)}\right\vert\mathcal{F}_{h,k}\right],
\end{align*}
and
\begin{align*} 
    \textrm{Var}_2\left[\left.\nabla_\theta^j X_{\theta,h}^{(k)}\right\vert\mathcal{F}_{h,k}\right]=&\mathbb{E}\left[\left.\left\Vert\nabla_\theta^j X_{\theta,h}^{(k)}\right\Vert^2\right\vert\mathcal{F}_{h,k}\right]=\phi\left(s_h^{(k)},a_h^{(k)}\right)^\top\Sigma^{-1}\phi\left(s_h^{(k)},a_h^{(k)}\right)\textrm{Var}\left[\left.\nabla_\theta^j\varepsilon_{h,k}^\theta\right\vert s_h^{(k)},a_h^{(k)}\right]\\
    \leq&\frac{4G^2}{(1-\gamma)^2}\phi\left(s_h^{(k)},a_h^{(k)}\right)^\top\Sigma^{-1}\phi\left(s_h^{(k)},a_h^{(k)}\right),  
\end{align*}
where we have used $\textrm{Var}\left[\left.\nabla_\theta^j\varepsilon_{h,k}^\theta\right\vert s_h^{(k)},a_h^{(k)}\right]\leq \frac{4G^2}{(1-\gamma)^2}$. Furthermore, notice that $\left\Vert\nabla_\theta^j X_{\theta,h}^{(k)}\right\Vert\leq\frac{2G\sqrt{C_1d}}{1-\gamma}$, the remaining steps will be exactly the same as those in the proof of the case $\mathcal{E}_{\varepsilon,0}$. Taking a union bound over $j\in[m]$ and $\mathcal{E}_{\varepsilon,0}$, we have proved $\mathbb{P}\left(\bigcap_{j=0}^m\mathcal{E}_{\varepsilon, j}\right)\geq 1-\frac{\delta}{4}$, which has finished the proof. 
\end{proof}
Combining the results of Lemma \ref{dsig1_homo}, Lemma \ref{dy_homo}, Lemma \ref{eps_homo} and take a union bound, we conclude 
\begin{align*}
\mathbb{P}\left(\mathcal{E}\right) \geq 1-\frac{3}{4}\delta. 
\end{align*}

Next, we prove some immediate results when the event $\mathcal{E}$ holds. 
\begin{lemma}
\label{dsig2_homo}
When $\mathcal{E}_\Sigma$ holds and 
\begin{align*}
    K\geq C_1d\log\frac{8dmH}{\delta},\quad\lambda\leq C_1d\sigma_{\textrm{min}}(\Sigma)\cdot\log\frac{8dmH}{\delta},
\end{align*}
we have
\begin{align*}
        \left\Vert\Sigma^{\frac{1}{2}}\left(\Delta\Sigma^{-1}\right)\Sigma^{\frac{1}{2}}\right\Vert\leq 4\sqrt{\frac{C_1d\log\frac{8dmH}{\delta}}{K}}. 
\end{align*}
\end{lemma}
\begin{proof}
Note that
\begin{align}
    \label{2_1_homo}
    \left\Vert\Sigma^{\frac{1}{2}}\left(\Delta\Sigma^{-1}\right)\Sigma^{\frac{1}{2}}\right\Vert=\left\Vert\Sigma^{\frac{1}{2}}\left(\widehat{\Sigma}^{-1}-\Sigma^{-1}\right)\Sigma^{\frac{1}{2}}\right\Vert\leq\left\Vert\Sigma^{\frac{1}{2}}\widehat{\Sigma}^{-1}\Sigma^{\frac{1}{2}}\right\Vert\left\Vert\Sigma^{-\frac{1}{2}}\widehat{\Sigma}\Sigma^{-\frac{1}{2}}-I_d\right\Vert.  
\end{align}
When $\mathcal{E}_\Sigma$ holds, with the condition
\begin{align*}
    K\geq C_1d\log\frac{8dmH}{\delta},\quad\lambda\leq C_1d\sigma_{\textrm{min}}(\Sigma)\cdot\log\frac{8dmH}{\delta},
\end{align*}
we have
\begin{align*}
\left\Vert\Sigma^{-\frac{1}{2}}\widehat{\Sigma}\Sigma^{-\frac{1}{2}}-I_d\right\Vert\leq \sqrt{\frac{2C_1d\log\frac{8dH}{\delta}}{K}}+\frac{2C_1d\log\frac{8dH}{\delta}}{3K}+\frac{\lambda\Vert\Sigma^{-1}\Vert}{K}\leq 2\sqrt{\frac{C_1d\log\frac{8dmH}{\delta}}{K}}\leq \frac{1}{2},
\end{align*}
which further implies $\sigma_{\textrm{min}}\left(\Sigma^{-\frac{1}{2}}\widehat{\Sigma}\Sigma^{-\frac{1}{2}}\right)\geq\frac{1}{2}$, and $\left\Vert\Sigma^{\frac{1}{2}}\widehat{\Sigma}^{-1}\Sigma^{\frac{1}{2}}\right\Vert\leq 2$. Combining this result with \eqref{2_1_homo}, we get 
\begin{align*}
    \left\Vert\Sigma^{\frac{1}{2}}\left(\Delta\Sigma^{-1}\right)\Sigma^{\frac{1}{2}}\right\Vert\leq 4\sqrt{\frac{C_1d\log\frac{8dmH}{\delta}}{K}}.
\end{align*}
which has finished the proof. 
\end{proof}

\begin{lemma}
\label{dm2_homo}
When $\mathcal{E}_\Sigma$ and $\mathcal{E}_Y$ hold, and
\begin{align*}
    K\geq 36\kappa_1(4+\kappa_2+\kappa_3)^2\frac{C_1d}{(1-\gamma)^2}\log\frac{16dmH}{\delta},\quad\lambda\leq C_1d\sigma_{\textrm{min}}(\Sigma)\cdot\log\frac{8dmH}{\delta},
\end{align*}
we have
\begin{align*}
    \left\Vert\Sigma_\theta^{\frac{1}{2}}\left(\Delta M_\theta\right)\Sigma_\theta^{-\frac{1}{2}}\right\Vert\leq 6\sqrt{\kappa_1}(4+\kappa_2+\kappa_3)\sqrt{\frac{C_1d\log\frac{16dmH}{\delta}}{K}},
\end{align*}
and
\begin{align*}
    \left\Vert\Sigma_\theta^{\frac{1}{2}}\left(\frac{\nabla_\theta^j\left(\Delta M_\theta\right)}{G}\right)\Sigma_\theta^{-\frac{1}{2}}\right\Vert\leq 6\sqrt{\kappa_1}(4+\kappa_2+\kappa_3)\sqrt{\frac{C_1d\log\frac{16dmH}{\delta}}{K}},\quad\forall j\in[m].
\end{align*}
\end{lemma}
\begin{proof}
We have
\begin{align*}
    \left\Vert\Sigma_\theta^{\frac{1}{2}}\left(\Delta M_\theta\right)\Sigma_\theta^{-\frac{1}{2}}\right\Vert&\leq\left\Vert\Sigma_\theta^{\frac{1}{2}}\Sigma^{-\frac{1}{2}}\right\Vert\left\Vert\Sigma^{\frac{1}{2}}\Sigma_\theta^{-\frac{1}{2}}\right\Vert\left\Vert\Sigma^{\frac{1}{2}}\left(\Delta M_\theta\right)\Sigma^{-\frac{1}{2}}\right\Vert=\sqrt{\kappa_1}\left\Vert\Sigma^{\frac{1}{2}}\left(\Delta M_\theta\right)\Sigma^{-\frac{1}{2}}\right\Vert\\
    &=\sqrt{\kappa_1}\left\Vert\Sigma^{\frac{1}{2}}\left(\widehat{\Sigma}^{-1}\frac{1}{HK}\sum_{k=1}^K\sum_{h=1}^H\phi\left(s_h^{(k)},a_h^{(k)}\right)\int_{\mathcal{A}}\phi\left(s_{h+1}^{(k)},a^\prime\right)\pi_\theta\left(a^\prime\left\vert s_{h+1}^{(k)}\right.\right)\mathrm{d}a^\prime- M_\theta\right)\Sigma^{-\frac{1}{2}}\right\Vert\\
    &\leq\sqrt{\kappa_1}\left(\left(1+\left\Vert\Sigma^{\frac{1}{2}}\left(\Delta\Sigma^{-1}\right)\Sigma^{\frac{1}{2}}\right\Vert\right)\left(1+\left\Vert\Sigma^{-\frac{1}{2}}\left(\Delta Y_\theta\right)\Sigma^{-\frac{1}{2}}\right\Vert\right)-1\right),
\end{align*}
where $\Delta Y_\theta=\frac{1}{HK}\sum_{k=1}^K\sum_{h=1}^H\phi\left(s_h^{(k)},a_h^{(k)}\right)\phi_\theta\left(s_{h+1}^{(k)}\right)^\top-\Sigma M_\theta$ and the last inequality uses Lemma \ref{decomp_homo}. Similarly, we have
\begin{align*}
    \left\Vert\Sigma_\theta^{\frac{1}{2}}\left(\frac{\nabla_\theta^j\left(\Delta M_\theta\right)}{G}\right)\Sigma_\theta^{-\frac{1}{2}}\right\Vert\leq\sqrt{\kappa_1}\left(\left(1+\left\Vert\Sigma^{\frac{1}{2}}\left(\Delta\Sigma\right)\Sigma^{\frac{1}{2}}\right\Vert\right)\left(1+\left\Vert\Sigma^{-\frac{1}{2}}\left(\frac{\nabla_\theta^j\left(\Delta Y_\theta\right)}{G}\right)\Sigma^{-\frac{1}{2}}\right\Vert\right)-1\right).
\end{align*}
Using the result of Lemma \ref{dsig2_homo} the event $\mathcal{E}_Y$, we get 
\begin{align*}
    \left\Vert\Sigma^{\frac{1}{2}}\left(\Delta\Sigma^{-1}\right)\Sigma^{\frac{1}{2}}\right\Vert\leq 4\sqrt{\frac{C_1d\log\frac{8dmH}{\delta}}{K}}\leq 1, 
\end{align*}
and $\forall j\in[m]$, 
\begin{align*}
    \left\Vert\Sigma^{\frac{1}{2}}\left(\Delta Y_\theta\right)\Sigma^{-\frac{1}{2}}\right\Vert\leq 2(\kappa_2+1)\sqrt{\frac{C_1d\log\frac{16dmH}{\delta}}{K}}\leq 1,\quad\left\Vert\Sigma^{\frac{1}{2}}\left(\frac{\nabla_\theta^j\left(\Delta Y_\theta\right)}{G}\right)\Sigma^{-\frac{1}{2}}\right\Vert\leq 2(\kappa_3+1)\sqrt{\frac{C_1d\log\frac{16dmH}{\delta}}{K}}\leq 1,
\end{align*}
which implies
\begin{align*}
    \left\Vert\Sigma_\theta^{\frac{1}{2}}\left(\Delta M_\theta\right)\Sigma_\theta^{-\frac{1}{2}}\right\Vert\leq 2\sqrt{\kappa_1}\left(\left\Vert\Sigma^{\frac{1}{2}}\left(\Delta\Sigma^{-1}\right)\Sigma^{\frac{1}{2}}\right\Vert+\left\Vert\Sigma^{\frac{1}{2}}\left(\Delta Y_\theta\right)\Sigma^{-\frac{1}{2}}\right\Vert\right)\leq 6\sqrt{\kappa_1}(4+\kappa_2+\kappa_3)\sqrt{\frac{C_1d\log\frac{16dmH}{\delta}}{K}},
\end{align*}
where we use the fact $(1+x_1)(1+x_2)-1\leq 2(x_1+x_2)$ whenever $x_1,x_2\in[0,1]$. Similarly, we get
\begin{align*}
    \left\Vert\Sigma_\theta^{\frac{1}{2}}\left(\frac{\nabla_\theta^j\left(\Delta M_\theta\right)}{G}\right)\Sigma_\theta^{-\frac{1}{2}}\right\Vert\leq 6\sqrt{\kappa_1}(4+\kappa_2+\kappa_3)\sqrt{\frac{C_1d\log\frac{16dmH}{\delta}}{K}},\quad\forall j\in[m],
\end{align*}
which has finished the proof. 
\end{proof}
\begin{lemma}
\label{Q_hat_decomp_homo}
When $\mathcal{E}_\Sigma$ and $\mathcal{E}_Y$ hold, and 
\begin{align*}
    K\geq 36\kappa_1(4+\kappa_2+\kappa_3)^2\frac{C_1d}{(1-\gamma)^2}\log\frac{16dmH}{\delta},\quad\lambda\leq C_1d\sigma_{\textrm{min}}(\Sigma)\cdot\log\frac{8dmH}{\delta},
\end{align*}
we have
\begin{align*}
    \widehat{Q}^\theta=\sum_{h=1}^\infty\gamma^{h-1}\left(\widehat{\mathcal{P}}_\theta\right)^{h-1}\widehat{r},\quad\widehat{\nabla_\theta Q^\theta}=\sum_{h=1}^\infty\gamma^{h-1}\left(\widehat{\mathcal{P}}_\theta\right)^{h-1}\tilde{U}^\theta.
\end{align*}
\begin{proof}
Firstly, note that given the conditions, the result of Lemma \ref{dm2_homo} implies
\begin{align*}
    \left\Vert\Sigma_\theta^{\frac{1}{2}}\left(\Delta M_\theta\right)\Sigma_\theta^{-\frac{1}{2}}\right\Vert\leq 6\sqrt{\kappa_1}(4+\kappa_2+\kappa_3)\sqrt{\frac{C_1d\log\frac{16dmH}{\delta}}{K}}\leq 1-\gamma.
\end{align*}
Therefore, 
\begin{align*}
\Vert\gamma\Sigma_\theta^{\frac{1}{2}}\widehat{M}_\theta\Sigma_\theta^{-\frac{1}{2}}\Vert\leq \gamma\left(\left\Vert\Sigma_\theta^{\frac{1}{2}}\widehat{M}_\theta\Sigma_\theta^{-\frac{1}{2}}\right\Vert + \left\Vert\Sigma_\theta^{\frac{1}{2}}\left(\Delta\widehat{M}_\theta\right)\Sigma_\theta^{-\frac{1}{2}}\right\Vert\right)\leq \gamma(2-\gamma) < 1. 
\end{align*}
where we use the result of Lemma \ref{ineq_homo} to get $\left\Vert\Sigma_\theta^{\frac{1}{2}}\widehat{M}_\theta\Sigma_\theta^{-\frac{1}{2}}\right\Vert\leq 1$. Therefore, we have 
\begin{align*}
\left(I_d - \gamma\widehat{M}_\theta\right)^{-1} = \Sigma_\theta^{-\frac{1}{2}}\left(I_d - \gamma\Sigma_\theta^{\frac{1}{2}}\widehat{M}_\theta\Sigma_\theta^{-\frac{1}{2}}\right)^{-1}\Sigma_\theta^{\frac{1}{2}}=\Sigma_\theta^{-\frac{1}{2}}\sum_{h=1}^\infty\gamma^{h-1}\left(\Sigma_\theta^{\frac{1}{2}}\widehat{M}_\theta\Sigma_\theta^{-\frac{1}{2}}\right)^{h-1}\Sigma_\theta^{\frac{1}{2}} = \sum_{h=1}^\infty\gamma^{h-1}\widehat{M}_\theta^{h-1}. 
\end{align*}
Based on this result, we prove the main result by definition:
\begin{align*}
\widehat{Q}^\theta(\cdot,\cdot) &= \phi(\cdot,\cdot)^\top\left(I_d-\gamma\widehat{M}_\theta\right)^{-1}\widehat{w}_r=\phi(\cdot,\cdot)^\top\sum_{h=1}^\infty\gamma^{h-1}\widehat{M}_\theta^{h-1}\widehat{w}_r=\left(\sum_{h=1}^\infty\gamma^{h-1}\left(\widehat{\mathcal{P}}_\theta\right)^{h-1}\widehat{r}\right)(\cdot,\cdot),\\
\widehat{\nabla_\theta Q^\theta}(\cdot,\cdot) &= \phi(\cdot,\cdot)^\top\left(I_d-\gamma\widehat{M}_\theta\right)^{-1}\widehat{\nabla_\theta^j M_\theta}\widehat{w}^\theta=\phi(\cdot,\cdot)^\top\sum_{h=1}^\infty\gamma^{h-1}\widehat{M}_\theta^{h-1}\widehat{\nabla_\theta^j M_\theta}\widehat{w}^\theta=\left(\sum_{h=1}^\infty\gamma^{h-1}\left(\widehat{\mathcal{P}}_\theta\right)^{h-1}\tilde{U}^\theta\right)(\cdot,\cdot),
\end{align*}
which has finished the proof. 
\end{proof}
\end{lemma}
Now we consider the decomposition of $Q^\theta-\widehat{Q}^\theta$: 
\begin{lemma}
\label{Q_decomp_homo}
Under the same condition of Lemma \ref{Q_hat_decomp_homo}, we have 
\begin{align*}
    Q^\theta-\widehat{Q}^\theta=\sum_{h=1}^\infty\gamma^{h-1}\left(\widehat{\mathcal{P}}_\theta\right)^{h-1}\left(Q^\theta-\widehat{r}-\gamma\widehat{\mathcal{P}}_\theta Q^\theta\right).
\end{align*}
\end{lemma}
\begin{proof}
Simply note that 
\begin{align*}
    Q^\theta-\widehat{Q}^\theta&=\sum_{h=1}^\infty\gamma^{h-1}\left(\mathcal{P}_\theta\right)^{h-1}r-\sum_{h=1}^\infty\gamma^{h-1}\left(\widehat{\mathcal{P}}_\theta\right)^{h-1}\widehat{r}\\
    &=\sum_{h=1}^\infty\gamma^{h-1}\left(\left(\mathcal{P}_\theta\right)^{h-1}-\left(\widehat{\mathcal{P}}_\theta\right)^{h-1}\right)r+\sum_{h^\prime=1}^\infty\gamma^{h-1}\left(\widehat{\mathcal{P}}_\theta\right)^{h-1}\left(r-\widehat{r}\right)\\
    &=\sum_{h=1}^\infty\gamma^{h-1}\sum_{h^{\prime}=1}^{h-1}\left(\widehat{\mathcal{P}}_\theta\right)^{h^\prime-1}\left(\mathcal{P}_\theta-\widehat{\mathcal{P}}_\theta\right)\left(\mathcal{P}_\theta\right)^{h-h^\prime-1}r+\sum_{h=1}^\infty\gamma^{h-1}\left(\widehat{\mathcal{P}}_\theta\right)^{h-1}\left(r-\widehat{r}\right)\\
    &=\sum_{h^\prime=1}^\infty\left(\widehat{\mathcal{P}}_\theta\right)^{h^\prime-1}\left(\mathcal{P}_\theta-\widehat{\mathcal{P}}_\theta\right)\sum_{h=h^\prime+1}^\infty\gamma^{h-1}\left(\mathcal{P}_\theta\right)^{h-h^\prime-1}r+\sum_{h=1}^\infty\gamma^{h-1}\left(\widehat{\mathcal{P}}_\theta\right)^{h-1}\left(r-\widehat{r}\right)\\
    &=\sum_{h=1}^\infty\gamma^h\left(\widehat{\mathcal{P}}_\theta\right)^{h-1}\left(\mathcal{P}_\theta-\widehat{\mathcal{P}}_\theta\right)Q^\theta+\sum_{h=1}^\infty\gamma^{h-1}\left(\widehat{\mathcal{P}}_\theta\right)^{h-1}\left(r-\widehat{r}\right)\\
    &=\sum_{h=1}^\infty\gamma^{h-1}\left(\widehat{\mathcal{P}}_\theta\right)^{h-1}\left(Q^\theta-\widehat{r}-\gamma\widehat{\mathcal{P}}_\theta Q^\theta\right), 
\end{align*}
which is the desired result. 
\end{proof}

\subsection{Proofs of Main Theorems}
Define $\widehat{\nu}^\theta_h:=\left(\widehat{M}_\theta^\top\right)^{h-1}\nu_1^\theta$ and $\widehat{\nu}^\theta = \sum_{h=1}^\infty \gamma^{h-1}\widehat{\nu}_h^\theta$. We may prove the following decomposition of $\nabla_\theta v_\theta-\widehat{\nabla_\theta v_\theta}$:
\begin{lemma}
\label{error_decomp_homo}
Given the same condition of Lemma \ref{Q_hat_decomp_homo}, we have $\nabla_\theta v_\theta-\widehat{\nabla_\theta v_\theta}=E_1+E_2+E_3$, where 
\begin{align*}
    E_1=&\nabla_\theta\left[\left(\nu^\theta\right)^\top\Sigma^{-1}\frac{1}{KH}\sum_{k=1}^K\sum_{h=1}^H\phi\left(s_h^{(k)},a_h^{(k)}\right)\varepsilon_{h,k}^\theta\right]\\
    E_2=&\nabla_\theta\left[\left(\left(\widehat{\nu}^\theta\right)^\top\widehat{\Sigma}^{-1}-\left(\nu^\theta\right)^\top\Sigma^{-1}\right)\frac{1}{KH}\sum_{k=1}^K\sum_{h=1}^H\phi\left(s_h^{(k)},a_h^{(k)}\right)\varepsilon_{h,k}^\theta\right]\\
    E_3=&\frac{\lambda}{KH}\sum_{h=1}^T\nabla_\theta\left[\left(\widehat{\nu}^\theta\right)^\top\widehat{\Sigma}^{-1}w^\theta\right].
\end{align*}
\end{lemma}
The proof of Lemma \ref{error_decomp_homo} is deferred to appendix \ref{missing_proof_homo}. Based on this observation, here we show the proofs of our main theorems. 

\subsubsection{Proof of Theorem \ref{thm2_var_homo}}
\label{pfthm2_var_homo}
\begin{proof}
We use Lemma \ref{error_decomp_homo} to decompose $\langle\nabla_\theta v_\theta - \widehat{\nabla_\theta v_\theta}, t\rangle=\langle E_1, t\rangle+\langle E_2,t\rangle+\langle E_3,t\rangle$. To bound each term individually, we introduce the following lemmas, whose proofs are deferred to appendix \ref{missing_proof_homo}. 
\begin{lemma}
\label{e1_finite_product_homo}
For any $t\in\mathbb{R}^m$, with probability $1-\frac{\delta}{4}$, we have
\begin{align*}
    \vert\langle E_1, t\rangle\vert\leq\sqrt{\frac{2t^\top\Lambda_\theta t\log(8/\delta)}{HK}}+\frac{2\log(8/\delta)\sqrt{C_1md}\Vert t\Vert B}{3HK}. 
\end{align*}
where $B=\frac{1}{1-\gamma}\max_{j\in[m]}\sqrt{\left(\nabla^j_\theta\nu^\theta\right)^\top\Sigma^{-1}\nabla^j_\theta\nu^\theta}+\frac{2G}{(1-\gamma)^2}\sqrt{\left(\nu^\theta\right)^\top\Sigma^{-1}\nu^\theta}$.
\end{lemma}
\begin{lemma}
\label{e2_homo}
Let $E_2^j$ be the $j$th entry of $E_2$, suppose $\mathcal{E}$ holds and $K\geq 36\kappa_1(4+\kappa_2+\kappa_3)^2C_1d(1-\gamma)^{-2}\log\frac{16dmH}{\delta}$ and $\lambda\leq C_1d\sigma_{\min}(\Sigma)\log\frac{8dmH}{\delta}$, then we have 
\begin{align*}
    \vert E_2^j\vert\leq \frac{240\sqrt{\kappa_1}(2+\kappa_2+\kappa_3)\sqrt{C_1}d}{(1-\gamma)^3}\left(\left\Vert\Sigma_\theta^{-\frac{1}{2}}\nabla_\theta^j\nu^\theta_1\right\Vert+\frac{G}{1-\gamma}\left\Vert\Sigma_\theta^{-\frac{1}{2}}\nu^\theta_1\right\Vert\right)\left\Vert\Sigma_\theta^{\frac{1}{2}}\Sigma^{-\frac{1}{2}}\right\Vert\frac{\log\frac{32dmH}{\delta}}{KH},\quad\forall j\in[m].
\end{align*}
\end{lemma}
\begin{lemma}
\label{e3_homo}
Let $E_3^j$ be the $j$th entry of $E_3$, suppose $\mathcal{E}$ holds and $K\geq 36\kappa_1(4+\kappa_2+\kappa_3)^2C_1d(1-\gamma)^{-2}\log\frac{16dmH}{\delta}$ and $\lambda\leq C_1d\sigma_{\min}(\Sigma)\log\frac{8dmH}{\delta}$, we have 
\begin{align*}
    \vert E_3^j\vert\leq&\frac{6C_1d}{(1-\gamma)^2}\left\Vert\Sigma_\theta^{\frac{1}{2}}\Sigma^{-\frac{1}{2}}\right\Vert\left(\left\Vert\Sigma_\theta^{-\frac{1}{2}}\nabla_\theta^j\nu^\theta_1\right\Vert+\frac{G}{1-\gamma}\left\Vert\Sigma_\theta^{-\frac{1}{2}}\nu^\theta_1\right\Vert\right)\frac{\log\frac{8dmH}{\delta}}{KH},\quad\forall j\in[m]. 
\end{align*}
\end{lemma}
Let $B_1^j=\frac{1}{1-\gamma}\sqrt{\left(\nabla^j_\theta\nu^\theta\right)^\top\Sigma^{-1}\nabla^j_\theta\nu^\theta},\ B_2=\frac{G}{(1-\gamma)^2}\sqrt{\left(\nu^\theta\right)^\top\Sigma^{-1}\nu^\theta}$, then we have the relation $B=\max_{j\in[m]}B_1^j+2B_2$. For any $j\in[m]$, note that
\begin{align*}
    B_1^j=&\frac{1}{1-\gamma}\left\Vert\Sigma^{-\frac{1}{2}}\nabla_\theta^j\nu^\theta\right\Vert\leq \sum_{h=1}^\infty \frac{\gamma^{h-1}}{1-\gamma}\left\Vert\Sigma_\theta^{\frac{1}{2}}\Sigma^{-\frac{1}{2}}\right\Vert\left\Vert\Sigma_\theta^{-\frac{1}{2}}\nabla_\theta^j\left(\left(M_\theta^{h-1}\right)^\top\nu^\theta_1\right)\right\Vert\\
    =&\sum_{h=1}^\infty\frac{\gamma^{h-1}}{1-\gamma}\left\Vert\Sigma_\theta^{\frac{1}{2}}\Sigma^{-\frac{1}{2}}\right\Vert\\
    \cdot&\left(\left\Vert\Sigma_\theta^{\frac{1}{2}}M_\theta\Sigma_\theta^{-\frac{1}{2}}\right\Vert^{h-1}\left\Vert\Sigma_\theta^{-\frac{1}{2}}\nabla_\theta^j\nu^\theta_1\right\Vert+(h-1)\left\Vert\Sigma_\theta^{\frac{1}{2}}M_\theta\Sigma_\theta^{-\frac{1}{2}}\right\Vert^{h-2}\left\Vert\Sigma_\theta^{\frac{1}{2}}\left(\nabla_\theta^j M_\theta\right)\Sigma_\theta^{-\frac{1}{2}}\right\Vert\left\Vert\Sigma_\theta^{-\frac{1}{2}}\nu^\theta_1\right\Vert\right)\\
    \leq&\frac{1}{(1-\gamma)^2}\left\Vert\Sigma_\theta^{\frac{1}{2}}\Sigma^{-\frac{1}{2}}\right\Vert\left(\left\Vert\Sigma_\theta^{-\frac{1}{2}}\nabla_\theta^j\nu^\theta_1\right\Vert+\frac{G}{1-\gamma}\left\Vert\Sigma_\theta^{-\frac{1}{2}}\nu^\theta_1\right\Vert\right),
\end{align*}
where we use the result of Lemma \ref{ineq_homo}. Similarly, 
\begin{align*}
    B_2&\leq\sum_{h=1}^\infty\frac{G\gamma^{h-1}}{(1-\gamma)^2}\left\Vert\Sigma^{-\frac{1}{2}}\nu^\theta_h\right\Vert\leq\sum_{h=1}^\infty\frac{G\gamma^{h-1}}{(1-\gamma)^2}\left\Vert\Sigma_\theta^{\frac{1}{2}}\Sigma^{-\frac{1}{2}}\right\Vert\left\Vert\Sigma_\theta^{\frac{1}{2}}M_\theta\Sigma_\theta^{-\frac{1}{2}}\right\Vert^{h-1}\left\Vert\Sigma_\theta^{-\frac{1}{2}}\nu^\theta_1\right\Vert\\
    &\leq\sum_{h=1}^\infty\frac{G\gamma^{h-1}}{(1-\gamma)^2}\left\Vert\Sigma_\theta^{\frac{1}{2}}\Sigma^{-\frac{1}{2}}\right\Vert\left\Vert\Sigma_\theta^{-\frac{1}{2}}\nu^\theta_1\right\Vert\leq \frac{G}{(1-\gamma)^3}\left\Vert\Sigma_\theta^{\frac{1}{2}}\Sigma^{-\frac{1}{2}}\right\Vert\left\Vert\Sigma_\theta^{-\frac{1}{2}}\nu^\theta_1\right\Vert.
\end{align*}
We conclude that when $K\geq 36C_1d(1-\gamma)^{-2}\kappa_1(4+\kappa_2+\kappa_3)^2\log\frac{16dmH}{\delta}$, we have
\begin{align*}
    \left(\max_{j\in[m]}B_1^j+2B_2\right)\frac{2\log\frac{2}{\delta}\sqrt{C_1md}\Vert t\Vert}{KH}\leq \frac{1}{(1-\gamma)^2}\left\Vert\Sigma_\theta^{\frac{1}{2}}\Sigma^{-\frac{1}{2}}\right\Vert\left(\max_{j\in[m]}\left\Vert\Sigma_\theta^{-\frac{1}{2}}\nabla_\theta^j\nu^\theta_1\right\Vert+2HG\left\Vert\Sigma_\theta^{-\frac{1}{2}}\nu^\theta_1\right\Vert\right)\frac{2\log\frac{2}{\delta}\sqrt{C_1md}\Vert t\Vert}{KH},
\end{align*}
Now, take a union bound on $\mathcal{E}$ and the event in Lemma \ref{e1_finite_product_homo}, we get with probability $1-\delta$, we have
\begin{align*}
    &\vert\langle E_1,t\rangle\vert+\vert\langle E_2,t\rangle\vert+\vert \langle E_3, t\rangle\vert\leq\sqrt{\frac{2t^\top\Lambda_\theta t\log(8/\delta)}{HK}}\\
    &+\sqrt{\kappa_1}(5+\kappa_2+\kappa_3)\left(\max_{j\in[m]}\left\Vert\Sigma_\theta^{-\frac{1}{2}}\nabla_\theta^j\nu^\theta_1\right\Vert+\frac{G}{1-\gamma}\left\Vert\Sigma_\theta^{-\frac{1}{2}}\nu^\theta_1\right\Vert\right)\left\Vert\Sigma_\theta^{\frac{1}{2}}\Sigma^{-\frac{1}{2}}\right\Vert\frac{240C_1d\log\frac{32dmH}{\delta}}{(1-\gamma)^3KH}\\
    \leq&\sqrt{\frac{2t^\top\Lambda_\theta t\log(8/\delta)}{KH}}+\kappa_1(5+\kappa_2+\kappa_3)\left(\max_{j\in[m]}\left\Vert\Sigma_\theta^{-\frac{1}{2}}\nabla_\theta^j\nu^\theta_1\right\Vert+\frac{G}{1-\gamma}\left\Vert\Sigma_\theta^{-\frac{1}{2}}\nu^\theta_1\right\Vert\right)\frac{240C_1d\sqrt{m}\Vert t\Vert\log\frac{32dmH}{\delta}}{(1-\gamma)^3HK}.
\end{align*}
we have finished the proof. 
\end{proof}

\subsubsection{Proof of Theorem \ref{thm2_homo}}
\label{pfthm2_homo}
\begin{proof}
According to the result of Theorem \ref{thm2_var_homo}, we know
\begin{align*}
    &\vert\langle t,\widehat{\nabla_\theta v_\theta}-\nabla_\theta v_\theta\rangle\vert\leq\sqrt{\frac{2t^\top\Lambda_\theta t}{HK}\cdot\log\frac{8}{\delta}}+\frac{C_\theta\Vert t\Vert\log\frac{32mdH}{\delta}}{HK},
\end{align*}
Pick $t=e_j, j\in[m]$, we have
\begin{align*}
    t^\top \Lambda_\theta t=&\mathbb{E}\left[\frac{1}{H}\sum_{h=1}^H\left(\nabla_\theta^j\left(\varepsilon^\theta_{h,1}\phi\left(s_h^{(1)},a_h^{(1)}\right)^\top\Sigma^{-1}\nu^\theta\right)\right)^2\right]\\
    \leq&2\mathbb{E}\left[\frac{1}{H}\sum_{h=1}^H\left(\nabla_\theta^j\varepsilon^\theta_{h,1}\phi\left(s_h^{(1)},a_h^{(1)}\right)^\top\Sigma^{-1}\nu^\theta\right)^2\right]+2\mathbb{E}\left[\frac{1}{H}\sum_{h=1}^H\left(\varepsilon^\theta_{h,1}\phi\left(s_h^{(1)},a_h^{(1)}\right)^\top\Sigma^{-1}\nabla_\theta^j\nu^\theta\right)^2\right]\\
    \leq&2\mathbb{E}\left[\sum_{h=1}^H\frac{G^2}{H(1-\gamma)^4}\left(\phi\left(s_h^{(1)},a_h^{(1)}\right)^\top\Sigma^{-1}\nu^\theta\right)^2\right]+2\mathbb{E}\left[\sum_{h=1}^H\frac{1}{H(1-\gamma)^2}\left(\phi\left(s_h^{(1)},a_h^{(1)}\right)^\top\Sigma^{-1}\nabla_\theta^j\nu^\theta\right)^2\right]\\
     \leq&2\left(\frac{G^2}{(1-\gamma)^4}\left\Vert\Sigma^{-\frac{1}{2}}\nu^\theta\right\Vert^2+\frac{1}{(1-\gamma)^2}\left\Vert\Sigma^{-\frac{1}{2}}\nabla_\theta^j\nu^\theta\right\Vert^2\right).
\end{align*}
Therefore, take a union bound over $j\in[m]$, we get
\begin{align*}
\left\vert\widehat{\nabla_\theta^j v_\theta}-\nabla_\theta^j v_\theta\right\vert\leq 4b_\theta\sqrt{\frac{\log\frac{8m}{\delta}}{HK}}+\frac{2C_\theta\sqrt{m}\log\frac{32mdH}{\delta}}{HK},\quad\forall j\in[m],
\end{align*}
where 
\begin{align*}
b_\theta=\frac{G}{(1-\gamma)^2}\left\Vert\Sigma^{-\frac{1}{2}}\nu^\theta\right\Vert+\frac{1}{1-\gamma}\left\Vert\Sigma^{-\frac{1}{2}}\nabla_\theta^j\nu^\theta\right\Vert.
\end{align*}
When $\phi(s^\prime,a^\prime)^\top\Sigma^{-1}\phi(s,a)\geq 0,\ \forall(s,a),(s^\prime,a^\prime)\in\mathcal{S}\times\mathcal{A}$, for any $h\in\mathbb{N}_+$ and $(s,a)\in\mathcal{S}\times\mathcal{A}$, we have
\begin{align*}
    \left\vert\left(\nabla^j_\theta\nu^\theta_{h}\right)^\top\Sigma^{-1}\phi(s,a)\right\vert=& \left\vert\mathbb{E}^{\pi_\theta}\left[\phi(s_h,a_h)\Sigma^{-1}\phi(s,a)\sum_{h^\prime=1}^h\nabla_\theta^j\log\pi_{\theta,h^\prime}(a_{h^\prime}\vert s_{h^\prime})\right]\right\vert\\
    \leq&\mathbb{E}^{\pi_\theta}\left[\phi(s_h, a_h)\Sigma^{-1}\phi(s, a)\sum_{h^\prime=1}^h\left\vert\nabla_\theta^j\log\pi_{\theta,h^\prime}(a_{h^\prime}\vert s_{h^\prime})\right\vert\right]\\
    \leq&Gh\mathbb{E}^{\pi_\theta}\left[\phi(s_h,a_h)\Sigma^{-1}\phi(s,a)\right]\\
    =&Gh\left(\nu^\theta_{h}\right)^\top\Sigma^{-1}\phi(s,a).
\end{align*}
Meanwhile, for any positive integer $\tilde{H}$, we have
\begin{align*}
\phi\left(s_h^{(1)},a_h^{(1)}\right)^\top\Sigma^{-1}\nabla_\theta^j\nu^\theta &= \sum_{h=1}^\infty\gamma^{h-1}\phi\left(s_h^{(1)},a_h^{(1)}\right)^\top\Sigma^{-1}\nabla_\theta^j\nu_h^\theta \\
&= \sum_{h=1}^{\tilde{H}}\gamma^{h-1}\phi\left(s_h^{(1)},a_h^{(1)}\right)^\top\Sigma^{-1}\nabla_\theta^j\nu_h^\theta + \sum_{h=\tilde{H}+1}^\infty\gamma^{h-1}\phi\left(s_h^{(1)},a_h^{(1)}\right)^\top\Sigma^{-1}\nabla_\theta^j\nu_h^\theta.
\end{align*}
For the second term, we have
\begin{align*}
\sum_{h=\tilde{H}+1}^\infty\gamma^{h-1}\phi\left(s_h^{(1)},a_h^{(1)}\right)^\top\Sigma^{-1}\nabla_\theta^j\nu_h^\theta&\leq \sum_{h=\tilde{H}+1}^\infty Gh\gamma^{h-1}\phi\left(s_h^{(1)},a_h^{(1)}\right)^\top\Sigma^{-1}\nu_h^\theta\\
&\leq GC_1d\sum_{h=\tilde{H}+1}^\infty h\gamma^{h-1}\\
&=GC_1d\left(\tilde{H} + \frac{1}{1-\gamma}\right)\gamma^{\tilde{H}}.
\end{align*}
For the first term, we have
\begin{align*}
\sum_{h=1}^{\tilde{H}}\gamma^{h-1}\phi\left(s_h^{(1)},a_h^{(1)}\right)^\top\Sigma^{-1}\nabla_\theta^j\nu_h^\theta &\leq G\tilde{H}\sum_{h=1}^{\tilde{H}}\gamma^{h-1}\phi\left(s_h^{(1)},a_h^{(1)}\right)^\top\Sigma^{-1}\nu_h^\theta\leq G\tilde{H}\sum_{h=1}^\infty\gamma^{h-1}\phi\left(s_h^{(1)},a_h^{(1)}\right)^\top\Sigma^{-1}\nu_h^\theta\\
&=G\tilde{H}\phi\left(s_h^{(1)},a_h^{(1)}\right)^\top\Sigma^{-1}\nu^\theta.
\end{align*}
Therefore,
\begin{align*}
t^\top \Lambda_\theta t=&2\mathbb{E}\left[\sum_{h=1}^H\frac{G^2}{H(1-\gamma)^4}\left(\phi\left(s_h^{(1)},a_h^{(1)}\right)^\top\Sigma^{-1}\nu^\theta\right)^2\right]+2\mathbb{E}\left[\sum_{h=1}^H\frac{1}{H(1-\gamma)^2}\left(\phi\left(s_h^{(1)},a_h^{(1)}\right)^\top\Sigma^{-1}\nabla_\theta^j\nu^\theta\right)^2\right]\\
\leq&4\mathbb{E}\left[\sum_{h=1}^H\frac{G^2}{H(1-\gamma)^2}\left(\tilde{H}^2+\frac{1}{(1-\gamma)^2}\right)\left(\phi\left(s_h^{(1)},a_h^{(1)}\right)^\top\Sigma^{-1}\nu^\theta\right)^2\right]+\frac{4G^2C_1^2d^2}{(1-\gamma)^2}\left(\tilde{H}+\frac{1}{1-\gamma}\right)^2\gamma^{2\tilde{H}}\\
\leq&\frac{4G^2}{(1-\gamma)^2}\left(\tilde{H}+\frac{1}{1-\gamma}\right)^2\left(\left\Vert\Sigma^{-\frac{1}{2}}\nu^\theta\right\Vert^2 + C_1^2d^2\gamma^{2\tilde{H}}\right).
\end{align*}
In particular, pick 
\begin{align*}
\tilde{H} = \frac{2\log(C_1d)}{1-\gamma}\geq\frac{\log\frac{C_1d}{\left\Vert\Sigma^{-\frac{1}{2}}\nu^\theta\right\Vert}}{\log\frac{1}{\gamma}},
\end{align*}
where we use the fact $\log\frac{1}{\gamma}\geq \frac{1-\gamma}{2\gamma}$ whenever $\gamma\in(\frac{1}{2},1)$, and
\begin{align*}
\left(\nu^\theta\right)^\top\Sigma^{-1}\nu^\theta = \sup_{w\in\mathbb{R}^d}\frac{\left(w^\top \nu^\theta\right)^2}{w^\top\Sigma w} \geq \frac{1}{(1-\gamma)^2} \geq 1,
\end{align*}
where the inequality is due to the fact that $\mathcal{F}$ includes the constant functions. We have
\begin{align*}
t^\top \Lambda_\theta t\leq \frac{8G^2}{(1-\gamma)^2}\left(\tilde{H}+\frac{1}{(1-\gamma)}\right)^2\left\Vert\Sigma^{-\frac{1}{2}}\nu^\theta\right\Vert^2\leq \frac{128G^2\left(\log(C_1d)\right)^2}{(1-\gamma)^4}\left\Vert\Sigma^{-\frac{1}{2}}\nu^\theta\right\Vert^2.
\end{align*}
Taking a union bound w.r.t. $m$, we get
\begin{align*}
\left\vert\widehat{\nabla_\theta^j v_\theta}-\nabla_\theta^j v_\theta\right\vert\leq\frac{16G\log(C_1d)}{(1-\gamma)^2}\left\Vert\Sigma^{-\frac{1}{2}}\nu^\theta\right\Vert\sqrt{\frac{\log\frac{8m}{\delta}}{HK}}+\frac{2C_\theta\sqrt{m}\log\frac{32mdH}{\delta}}{HK}.
\end{align*}
\end{proof}

\subsubsection{Proof of Theorem \ref{thm1_homo}}
\label{pfthm1_homo}
\begin{proof}
We use the same decomposition as in Theorem \ref{thm2_var_homo}. Define a martingale difference sequence $\{e_k^\theta\}_{k=1}^K$ by
\begin{align*}
    e_{h,k}^\theta&=\frac{1}{\sqrt{HK}}\nabla_\theta\left(\left(\nu^\theta\right)^\top\Sigma^{-1}\phi(s_h^{(k)},a_h^{(k)})\varepsilon_{h,k}^{\theta}\right)\\
    &=\frac{1}{\sqrt{HK}}\left(\nabla_\theta\nu^\theta\right)^\top\Sigma^{-1}\phi(s_h^{(k)},a_h^{(k)})\varepsilon_{h,k}^\theta+\frac{1}{\sqrt{HK}}\sum_{h=1}^H\left(\nu^\theta\right)^\top\Sigma^{-1}\phi(s_h^{(k)},a_h^{(k)})\nabla_\theta\varepsilon_{h,k}^\theta,
\end{align*}
we have 
\begin{align*}
    \Vert e_{h,k}^\theta\Vert_\infty\leq\frac{1}{\sqrt{HK}(1-\gamma)}\max_{j\in[m]}\Vert\Sigma^{-\frac{1}{2}}\nabla_\theta^j\nu^\theta\Vert\sqrt{C_1d}+\frac{2}{\sqrt{HK}(1-\gamma)^2}\Vert\Sigma^{-\frac{1}{2}}\nu^\theta\Vert\sqrt{C_1d}G\rightarrow 0, 
\end{align*}
where we use the result of Lemma \ref{upbd_homo}. Furthermore, 
\begin{align*}
    \sum_{h=1}^H\mathbb{E}\left[e_{h,k}^\theta\left(e_{h,k}^\theta\right)^\top\right]_{ij}=&\frac{1}{HK}\mathbb{E}\left[\sum_{h=1}^H\left[\nabla_{\theta_1}^i\left(\left(\nu^{\theta_1}\right)^\top\Sigma^{-1}\phi(s_h^{(k)},a_h^{(k)})\varepsilon_{h,k}^{\theta_1}\right)\right]\left[\nabla_{\theta_2}^j\left(\left(\nu^{\theta_2}\right)^\top\Sigma^{-1}\phi(s_h^{(k)},a_h^{(k)})\varepsilon_{h,k}^{\theta_2}\right)\right]^\top\right]\Bigg\vert_{\theta_1=\theta_2=\theta}\\
    =&\frac{[\Lambda_\theta]_{ij}}{K}.
\end{align*}
Therefore,by WLLN, we have
\begin{align*}
    \sum_{k=1}^K\sum_{h=1}^H\left[e_{h,k}^\theta\left(e_{h,k}^\theta\right)^\top\right]_{ij}\rightarrow_p\sum_{k=1}^K\sum_{h=1}^H\mathbb{E}\left[e_{h,k}^\theta\left(e_{h,k}^\theta\right)^\top\right]_{ij}=\left[\Lambda_\theta\right]_{ij},
\end{align*}
To finish the rest of the proof, we introduce the following lemmas, 
\begin{lemma}[Martingale CLT, Corollary 2.8 in (McLeish et al., 1974)] \label{CLT_homo}
Let $\left\{X_{mn},n=1,\ldots,k_m\right\}$ be a martingale difference array (row-wise) on the probability triple $(\Omega, \mathcal{F}, P)$.Suppose $X_{mn}$ satisfy the following two conditions:
\begin{align*}
    \max _{1\leq n\leq k_m}\left\vert X_{mn}\right\vert\stackrel{p}{\rightarrow}0,\textrm{ and } \sum_{n=1}^{k_m}X_{mn}^2\stackrel{p}{\rightarrow}\sigma^2
\end{align*}
for $k_m\rightarrow\infty$. Then $\sum_{n=1}^{k_m}X_{mn}\stackrel{d}{\rightarrow}\mathcal{N}\left(0,\sigma^2\right)$.
\end{lemma}
\begin{lemma}[Cramér–Wold Theorem] 
\label{CW_thm_homo}
Let $X_n=(X_n^1,X_n^2,\ldots,X_n^k)^\top$ be a $k$-dimensional random vector series and $X=(X^1,X^2,\ldots,X^k)^\top$ be a random vector of same dimension. Then $X_n$ converges in distribution to $X$ if and only if for any constant vector $t=(t_1,t_2,\ldots,t_k)^\top$, $t^\top X_n$ converges to $t^\top X$ in distribution.
\end{lemma}
Lemma \ref{CLT_homo} implies $\sum_{k=1}^K\sum_{h=1}^H t^\top e_{h,k}\rightarrow_d\mathcal{N}(0,t^\top\Lambda_\theta t)$ for any $t$, and Lemma \ref{CW_thm_homo} implies
\begin{align*}
    \sqrt{HK}E_1=\sum_{k=1}^K\sum_{h=1}^H e_{h,k}\rightarrow_p\mathcal{N}(0,\Lambda_\theta). 
\end{align*}
Furthermore, notice that the results of Lemma \ref{e2_homo} and Lemma \ref{e3_homo} imply $\sqrt{HK}E_2\rightarrow_p 0, \sqrt{HK}E_3\rightarrow_p 0$. Combining the above results, we have finished the proof. 
\end{proof}

\subsubsection{Proof of Theorem \ref{thm4_homo}}
\label{pfthm4_homo}
\begin{proof}
We first derive the influence function of policy gradient estimator for sake of completeness. We denote each of the $K$ sampled trajectories as
$$
\boldsymbol{\tau}:=\left(s_{1}, a_{1}, r_{1}, s_{2}, a_{2}, r_{2}, \ldots, s_{H}, a_{H}, r_{H}, s_{H+1}\right)
$$
We denote $\bar{\pi}(a \mid s)$ as the behavior policy. The distribution of trajectory is then given by
$$
\mathcal{P}(d \boldsymbol{\tau})= \bar{\xi}\left(d s_{1}, d a_{1}\right) p\left(d s_{2} \mid s_{1}, a_{1}\right) \bar{\pi}\left(d a_{2} \mid s_{2}\right) \ldots \bar{\pi}\left(d a_{H} \mid s_{H}\right) p\left(d s_{H+1} \mid s_{H}, a_{H}\right)
$$
Define $p_{\eta} = p + \eta\Delta p$ as a new transition probability function and $\mathcal{P}_{\eta}:=\mathcal{P}+\eta \Delta\mathcal{P}$ where $\Delta \mathcal{P}$ satisfies
$$(\Delta \mathcal{P}) \mathcal{F} \subseteq \mathcal{F}.$$ 
Define $g_\eta\left(s^{\prime} \mid s, a\right):=\frac{\partial}{\partial \eta} \log p_\eta\left( s^{\prime} \mid s, a\right)$ and the score function as 
$$
g_{\eta}(\boldsymbol{\tau}):=\frac{\partial}{\partial \eta} \log \mathcal{P}_{\eta}(d \boldsymbol{\tau})=\sum_{h=1}^{H} g_{\eta}\left(s_{h+1} \mid s_{h}, a_{h}\right).
$$
Without loss of generality, we assume $p_{\eta}$ is continuously derivative with respect to $\eta.$ This guarantees that we can change the order of taking derivatives with respect to $\eta$ and $\theta.$ When the subscript $\eta$ vanishes, it means $\eta = 0$ and the underlying transition probability is $p(s^{\prime}|s,a),$ i.e. $p_0(s^{\prime} |s,a) = p(s^{\prime} |s,a).$ Then we denote $g(s^{\prime}|s,a) := \left.\frac{\partial}{\partial \eta}\log p_\eta(s^{\prime}|s,a)\right|_{\eta = 0},$ and $g(\boldsymbol{\tau}) = \sum_{h=1}^H g(s_{h+1}|s_h,a_h).$ We define the policy value under new transition kernel is
\begin{equation*}
    v_{\theta,\eta} := \mathbb{E}^{\pi_{\theta}} \left[\left.\sum_{h=1}^\infty\gamma^{h-1} r(s_h,a_h) \right| s_1 \sim \xi, \mathcal{P}_{\eta}\right]
\end{equation*}
Then, our objective function is
$$
\psi_{\eta} := \nabla_{\theta} v_{\theta,\eta} =\mathbb{E}^{\pi_{\theta}}\left[\left.\sum_{h=1}^\infty \nabla_{\theta} \log \pi_{\theta}\left(a_{h} \mid s_{h}\right) \cdot\left(\sum_{h^{\prime}=h}^{\infty} \gamma^{h^\prime-1}r\left(s_{h^{\prime}}, a_{h^{\prime}}\right)\right) \right| s_{1} \sim \xi, \mathcal{P}_{\eta}\right].
$$
We are going to compute the influence function with respect to the above objective function. We denote this influence function as $\mathcal{I}(\boldsymbol{\tau}).$ By definition, it satisfies that
\begin{equation*}
    \left.\frac{\partial}{\partial \eta} \psi_{\eta}\right|_{\eta = 0} = \mathbb{E}\left[g(\boldsymbol{\tau}) \mathcal{I}(\boldsymbol{\tau})\right].
\end{equation*}
By exchanging the order of derivatives, we find that
\begin{equation*}
    \left.\frac{\partial}{\partial \eta} \psi_{\eta}\right|_{\eta = 0} = \nabla_{\theta} \left[\left.\frac{\partial}{\partial \eta} v_{\theta,\eta}\right|_{\eta = 0}\right].
\end{equation*}
Therefore, we calculate the derivatives.
\begin{align*}
    \frac{\partial}{\partial\eta}v_{\theta,\eta}&=\frac{\partial}{\partial\eta}\left[\sum_{h=1}^\infty\gamma^{h-1}\int_{(\mathcal{S}\times\mathcal{A})^h}r\left(s_h,a_h\right)\xi(s_1)\prod_{j=1}^{h-1}p_{\eta}\left(s_{j+1}\mid s_j,a_j\right)\prod_{j=1}^h\pi_\theta\left(a_j\mid s_j\right)d\boldsymbol{\tau}_h\right]\\
    &=\sum_{h=1}^{\infty}\gamma^{h-1} \int_{(\mathcal{S} \times \mathcal{A})^{h}} r\left(s_{h}, a_{h}\right)\left(\sum_{j=1}^{h-1} g_{\eta}\left(s_{j+1} \mid s_{j}, a_{j}\right)\right) \xi(s_1) \prod_{j=1}^{h-1} p_{\eta}\left(s_{j+1} \mid s_{j}, a_{j}\right) \prod_{j=1}^{h} \pi_{\theta}\left(a_{j} \mid s_{j}\right) d\boldsymbol{\tau}_h\\
    &= \int\sum_{h=1}^\infty\gamma^{h-1}r\left(s_{h}, a_{h}\right)\left(\sum_{j=1}^{h-1} g_\eta\left(s_{j+1}\mid s_j,a_j\right)\right) \left[\xi(s_1) \prod_{j=1}^h p_\eta\left(s_{j+1} \mid s_j,a_j\right)\prod_{j=1}^h\pi_\theta\left(a_{j} \mid s_{j}\right)\right] d\boldsymbol{\tau}.
\end{align*}
We denote $Q_{\eta}^{\theta}$ and  $\nabla_{\theta}Q_{\eta}^{\theta}$ as the state-action function and its gradient with underlying transition probability being $p_{\eta}.$ For sake of simplicity, we define the state value function as
\begin{equation*}
    V^{\theta}(s) := \mathbb{E}^{\pi_{\theta}} \left[\left.\sum_{h = 1}^\infty \gamma^{h-1}r(s_h,a_h) \right| s_1 = s, \mathcal{P}\right].
\end{equation*}
We denote $V_{\eta}^{\theta}(s)$ as the same function except for transition probability substituted by $p_{\eta}.$ Therefore,
\begin{align*}
    \frac{\partial}{\partial \eta} v_{\theta, \eta}
    &= \mathbb{E}^{\pi_{\theta}} \left[\left.\sum_{h=1}^{\infty}\gamma^{h-1}r\left(s_{h}, a_{h}\right)\left(\sum_{j=1}^{h-1} g_{\eta}\left(s_{j+1} \mid s_{j}, a_{j}\right)\right) \right| s_1 \sim \xi, \mathcal{P}_{\eta} \right] \\
    &= \mathbb{E}^{\pi_{\theta}} \left[\left.\sum_{j=1}^{\infty} g_{\eta}\left(s_{j+1} \mid s_{j}, a_{j}\right) \sum_{h=j+1}^\infty \gamma^{h-1}r\left(s_{h}, a_{h}\right) \right| s_1 \sim \xi, \mathcal{P}_{\eta} \right] \\
    &= \mathbb{E}^{\pi_{\theta}} \left[\left.\sum_{j=1}^{\infty} g_{\eta}\left(s_{j+1} \mid s_{j}, a_{j}\right) \cdot \mathbb{E}^{\pi_{\theta}} \left[\left.\sum_{h=j+1}^\infty\gamma^{h-1} r\left(s_{h}, a_{h}\right) \right| s_{j+1}\right]\right| s_1 \sim \xi, \mathcal{P}_{\eta} \right] \\
    &= \mathbb{E}^{\pi_{\theta}} \left[\left.\sum_{j=1}^{\infty} \mathbb{E}\left[\left. \gamma^{j}g_{\eta}\left(s_{j+1} \mid s_{j}, a_{j}\right) V_{\eta}^{\theta}(s_{j+1}) \right| s_j,a_j \right]\right| s_1 \sim \xi, \mathcal{P}_{\eta} \right].
\end{align*}
Therefore,
\begin{equation}\label{influence_function1_homo}
    \left.\frac{\partial}{\partial \eta} v_{\theta, \eta}\right|_{\eta=0} = \mathbb{E}^{\pi_{\theta}} \left[\left.\sum_{h=1}^{\infty} \mathbb{E}\left[\left. \gamma^h g\left(s_{h+1} \mid s_{h}, a_{h}\right) V^{\theta}(s_{h+1}) \right| s_h,a_h \right]\right| s_1 \sim \xi, \mathcal{P}_{\eta} \right].
\end{equation}
We notice that $\Sigma = \mathbb{E}\left[\frac{1}{H}\sum_{h=1}^H\phi\left(s_h^{(1)},a_h^{(1)}\right)\phi\left(s_h^{(1)},a_h^{(1)}\right)^{\top}\right]$. We denote $w_h(s,a) := \phi^{\top}(s,a)\Sigma^{-1} \nu_h^{\theta} = \phi^{\top}(s,a)\Sigma^{-1} \mathbb{E}^{\pi_{\theta}} \left[\phi(s_h,a_h) \mid s_1 \sim \xi\right].$ We leverage the following fact to rewrite \eqref{influence_function1_homo}: for any $f(s,a) = w_f^{\top} \phi(s,a) \in \mathcal{F}$ where $w_f \in \mathbb{R}^d,$ we have
\begin{align*}
\mathbb{E}^{\pi_{\theta}}\left[f(s_h,a_h)\right]
&= \mathbb{E}^{\pi_{\theta}} \left[ w_f^{\top} \phi(s_h,a_h)\right] \\
&= \mathbb{E}^{\pi_{\theta}} \left[ w_f^{\top} \mathbb{E}\left[\frac{1}{H}\sum_{h^\prime=1}^H\phi\left(s_{h^\prime}^{(1)},a_{h^\prime}^{(1)}\right)\phi^{\top}\left(s_{h^\prime}^{(1)},a_{h^\prime}^{(1)}\right)\right] \Sigma^{-1} \phi(s_h,a_h) \right] \\
&= \mathbb{E} \left[\frac{1}{H}\sum_{h^\prime=1}^Hw_f^{\top}\phi\left(s_{h^\prime}^{(1)},a_{h^\prime}^{(1)}\right)\phi^{\top}\left(s_{h^\prime}^{(1)},a_{h^\prime}^{(1)}\right)\Sigma^{-1}\mathbb{E}^{\pi_{\theta}} \left[\phi(s_h,a_h)\right]\right]\\
&= \mathbb{E}\left[\frac{1}{H}\sum_{h^\prime=1}^H f\left(s_{h^\prime}^{(1)},a_{h^\prime}^{(1)}\right) w_h\left(s_{h^\prime}^{(1)},a_{h^\prime}^{(1)}\right)\right]
\end{align*}
Since 
\begin{equation*}
    \mathbb{E} \left[ g\left(s^{\prime} \mid s, a\right) V^{\theta}(s^{\prime}) | s, a \right] = \left.\frac{\partial}{\partial \eta} \left(Q_{\eta}^{\theta}(s,a) - r_{\eta}(s,a)\right)\right|_{\eta = 0} \in \mathcal{F},
\end{equation*}
we have
\begin{align*}
    \left.\frac{\partial}{\partial\eta} v_{\theta,\eta}\right|_{\eta = 0}
    &= \mathbb{E}\left[\sum_{h=1}^{\infty}\gamma^h \frac{1}{H}\sum_{h^\prime=1}^H w_{h}\left(s_{h^\prime}^{(1)},a_{h^\prime}^{(1)}\right) \mathbb{E}\left[g\left(s^{\prime} \mid s_{h^\prime}^{(1)},a_{h^\prime}^{(1)}\right) \cdot  V^{\theta}\left(s^{\prime}\right) \mid s_{h^\prime}^{(1)},a_{h^\prime}^{(1)}\right]\right] \\
    &=\mathbb{E}\left[\frac{1}{H}\sum_{h^\prime=1}^H\mathbb{E}_{s^{\prime} \sim p(\cdot \mid s_{h^\prime}^{(1)},a_{h^\prime}^{(1)})}\left[\sum_{h=1}^{\infty}\gamma^h w_{h}\left(s_{h^\prime}^{(1)},a_{h^\prime}^{(1)}\right) g\left(s^{\prime} \mid s_{h^\prime}^{(1)},a_{h^\prime}^{(1)}\right) \cdot  V^{\theta}\left(s^{\prime}\right)\right]\right] \\
    &=\mathbb{E}\left[\frac{1}{H}\sum_{h^\prime=1}^H\mathbb{E}_{s^{\prime} \sim p(\cdot \mid s_{h^\prime}^{(1)},a_{h^\prime}^{(1)})}\left[\sum_{h=1}^{\infty}\gamma^h w_{h}\left(s_{h^\prime}^{(1)},a_{h^\prime}^{(1)}\right) g\left(s^{\prime} \mid s_{h^\prime}^{(1)},a_{h^\prime}^{(1)}\right)\left( V^{\theta}\left(s^{\prime}\right)-\mathbb{E}\left[V^{\theta}\left(s^{\prime}\right) \mid s_{h^\prime}^{(1)},a_{h^\prime}^{(1)}\right]\right)\right]\right]\\
    &=\mathbb{E}\left[\frac{1}{H}\sum_{h^\prime=1}^H\sum_{h=1}^{\infty}\gamma^h w_{h}\left(s_{h^\prime}^{(1)},a_{h^\prime}^{(1)}\right) g\left(s_{h^\prime+1}^{(1)} \mid s_{h^\prime}^{(1)},a_{h^\prime}^{(1)}\right)\left( V^{\theta}\left(s_{h^\prime+1}^{(1)}\right)-\mathbb{E}\left[V^{\theta}\left(s_{h^\prime+1}^{(1)}\right) \mid s_{h^\prime}^{(1)},a_{h^\prime}^{(1)}\right]\right)\right]\\
    &= \mathbb{E} \left[g\left(\boldsymbol{\tau}\right)\frac{1}{H}\sum_{h^\prime=1}^H\sum_{h=1}^\infty\gamma^h w_{h}\left(s_{h^\prime}^{(1)}, a_{h^\prime}^{(1)}\right)\left( V^{\theta}\left(s_{h^\prime+1}^{(1)}\right)-\mathbb{E}\left[V^{\theta}\left(s_{h^\prime+1}^{(1)}\right)\mid s_{h^\prime}^{(1)},a_{h^\prime}^{(1)}\right]\right)\right].
\end{align*}
Taking gradient in both sides and we have
\begin{equation*}
    \nabla_{\theta}\left(\left.\frac{\partial}{\partial\eta} v_{\theta,\eta}\right|_{\eta = 0}\right) = \mathbb{E}\left\{g\left(\boldsymbol{\tau}\right) \cdot \nabla_{\theta} \left[\frac{1}{H}\sum_{h^\prime=1}^H\sum_{h=1}^H \gamma^h w_{h}\left(s_{h^\prime}^{(1)}, a_{h^\prime}^{(1)}\right) \left( V^{\theta}\left(s_{h^\prime+1}^{(1)}\right)-\mathbb{E}\left[V^{\theta}\left(s_{h^\prime+1}^{(1)}\right) \mid s_{h^\prime}^{(1)}, a_{h^\prime}^{(1)}\right]\right)\right]\right\}.
\end{equation*}
The implies that the influence function we want is
\begin{equation*}
    \mathcal{I}(\boldsymbol{\tau}) = \nabla_{\theta} \left[\frac{1}{H}\sum_{h^\prime=1}^H\sum_{h=1}^\infty \gamma^h w_{h}\left(s_{h^\prime}^{(1)}, a_{h^\prime}^{(1)}\right) \left( V^{\theta}\left(s_{h^\prime+1}^{(1)}\right)-\mathbb{E}\left[V^{\theta}\left(s_{h^\prime+1}^{(1)}\right) \mid s_{h^\prime}^{(1)}, a_{h^\prime}^{(1)}\right]\right)\right].
\end{equation*}
Insert the expression of $w_h(s,a)$ and exploit $\varepsilon_{h,k}^{\theta}=Q^{\theta}(s_h^{(k)}, a_h^{(k)})-r_h^{(k)}-\gamma\int_{\mathcal{A}} \pi_{\theta}\left(a^{\prime} \mid s_{h+1}^{(k)}\right) Q^{\theta}\left(s_{h+1}^{(k)}, a^{\prime}\right) \mathrm{d} a^{\prime},$ we can rewrite the influence function as
\begin{equation*}
    \mathcal{I}(\boldsymbol{\tau})=-\nabla_\theta\left[\frac{1}{H}\sum_{h^\prime=1}^H\sum_{h=1}^\infty\gamma^{h-1}\phi\left(s_{h^\prime}^{(1)},a_{h^\prime}^{(1)}\right)^\top\Sigma^{-1}\varepsilon_{h^\prime,1}^\theta\nu_h^\theta\right]=-\nabla_\theta\left[\frac{1}{H}\sum_{h=1}^H\phi\left(s_{h}^{(1)},a_{h}^{(1)}\right)^\top\Sigma^{-1}\varepsilon_{h,1}^\theta\nu^\theta\right]
\end{equation*}
Therefore, since the cross terms vanish by taking conditional expectation, we have
\begin{align*}
    &\mathbb{E}\left[\mathcal{I}(\boldsymbol{\tau})^{\top} \mathcal{I}(\boldsymbol{\tau})\right]=\mathbb{E}\Bigg[\frac{1}{H^2}\sum_{h=1}^H\left(\nabla_\theta\left(\varepsilon^\theta_{h,1}\phi\left(s_h^{(1)},a_h^{(1)}\right)^\top\Sigma^{-1}\nu^\theta\right)\right)^\top\nabla_\theta\left(\varepsilon^\theta_{h,1}\phi\left(s_h^{(1)},a_h^{(1)}\right)^\top\Sigma^{-1}\nu^\theta \right)\Bigg]=\frac{1}{H}\Lambda_{\theta}.
\end{align*}
For any vector $t \in \mathbb{R}^m,$ when it comes to $\left\langle t,\psi_{\eta}\right\rangle,$ by linearity we have
\begin{equation*}
    \left.\frac{\partial}{\partial \eta} \left\langle t,\psi_{\eta}\right\rangle\right|_{\eta=0}=\mathbb{E}[g(\boldsymbol{\tau}) \left\langle t,\mathcal{I}(\boldsymbol{\tau})\right\rangle].
\end{equation*}
Then the influence function of $\left\langle t,\nabla_{\theta}v_{\theta}\right\rangle$ is $\left\langle t,\mathcal{I}(\boldsymbol{\tau})\right\rangle.$ The Cramer-Rao lower bound for $\left\langle t,\nabla_{\theta}v_{\theta}\right\rangle$ is
\begin{equation*}
    \mathbb{E}\left[\left\langle t,\mathcal{I}(\boldsymbol{\tau})\right\rangle^2\right] = t^{\top} \mathbb{E}\left[\mathcal{I}(\boldsymbol{\tau})^{\top}\mathcal{I}(\boldsymbol{\tau})\right] t = \frac{1}{H}t^{\top} \Lambda_{\theta} t.
\end{equation*}
By continuous mapping theorem, a trivial corollary of Theorem \ref{thm4_homo} is that for any $t \in \mathbb{R}^m,$
\begin{equation*}
    \sqrt{HK}\left(\left\langle t,\widehat{\nabla_{\theta} v_{\theta}}-\nabla_{\theta} v_{\theta}\right\rangle\right) \stackrel{d}{\rightarrow} \mathcal{N}\left(0, t^{\top} \Lambda_{\theta} t\right).
\end{equation*} 
This implies that the variance of any unbiased estimator for $\left\langle t, \nabla_{\theta} v_{\theta} \right\rangle \in \mathbb{R}$ is lower bounded by $\frac{1}{\sqrt{HK}}t^{\top} \Lambda_{\theta} t.$
\end{proof}

\subsection{Missing Proofs}
\label{missing_proof_homo}
\subsubsection{Proof of Lemma \ref{error_decomp_homo}}
\begin{proof}
Note that
\begin{align*}
    \nabla_\theta Q^\theta-\widehat{\nabla_\theta Q^\theta}&=\sum_{h=1}^\infty\gamma^{h-1}\left(\mathcal{P}_\theta\right)^{h-1}U^\theta-\sum_{h=1}^\infty\gamma^{h-1}\left(\widehat{\mathcal{P}}_\theta\right)^{h-1}\tilde{U}^\theta\\
    &=\sum_{h=1}^\infty\gamma^{h-1}\left(\mathcal{P}_\theta\right)^{h-1}U^\theta-\sum_{h=1}^\infty\gamma^{h-1}\left(\widehat{\mathcal{P}}_\theta\right)^{h-1}\widehat{U}^\theta+\sum_{h=1}^\infty\gamma^{h-1}\left(\widehat{\mathcal{P}}_\theta\right)^{h-1}\left(\widehat{U}^\theta-\tilde{U}^\theta\right)\\
    &=\sum_{h=1}^\infty\gamma^{h-1}\left(\widehat{\mathcal{P}}_\theta\right)^{h-1}\left(\nabla_\theta Q^\theta-\widehat{U}^\theta-\widehat{\mathcal{P}}_\theta\nabla_\theta Q^\theta\right)+\sum_{h=1}^\infty\gamma^{h-1}\left(\widehat{\mathcal{P}}_\theta\right)^{h-1}\left(\widehat{U}^\theta-\tilde{U}^\theta\right).
\end{align*}
For the first term, we have
\begin{align*}
    &\sum_{h=1}^\infty\gamma^{h-1}\left(\widehat{\mathcal{P}}_\theta\right)^{h-1}\left(\nabla_\theta Q^\theta-\widehat{U}^\theta-\gamma\widehat{\mathcal{P}}_\theta\nabla_\theta Q^\theta\right)\\
    =&\sum_{h=1}^\infty\gamma^{h-1}\left(\widehat{\mathcal{P}}_\theta\right)^{h-1}\phi^\top\widehat{\Sigma}^{-1}\frac{1}{KH}\sum_{k=1}^K\sum_{h^\prime=1}^H\phi\left(s_{h^\prime}^{(k)},a_{h^\prime}^{(k)}\right)\\
    &\cdot\left(\nabla_\theta Q^\theta\left(s_{h^\prime}^{(k)},a_{h^\prime}^{(k)}\right)-\gamma\int_{\mathcal{A}}\left(\left(\nabla_\theta\pi_\theta\left(a^\prime\left\vert s_{h^\prime+1}^{(k)}\right.\right)\right)Q^{\theta}\left(s_{h^\prime+1}^{(k)},a^\prime\right)+\pi_\theta\left(a^\prime\left\vert s_{h^\prime+1}^{(k)}\right.\right)\nabla_\theta Q^\theta\left(s_{h^\prime+1}^{(k)},a^\prime\right)\right)\mathrm{d}a^\prime\right)\\
    &+\frac{\lambda}{KH}\sum_{h=1}^\infty\gamma^{h-1}\left(\widehat{\mathcal{P}}_\theta\right)^{h-1}\phi^\top\widehat{\Sigma}^{-1}\nabla_\theta w^\theta\\
    =&\sum_{h=1}^\infty\gamma^{h-1}\phi^\top\left(\widehat{M}_\theta\right)^{h-1}\widehat{\Sigma}^{-1}\frac{1}{KH}\sum_{k=1}^K\sum_{h^\prime=1}^H\phi\left(s_{h^\prime}^{(k)},a_{h^\prime}^{(k)}\right)\nabla_\theta\varepsilon_{h^\prime,k}^\theta+\frac{\lambda}{KH}\sum_{h=1}^\infty\gamma^{h-1}\phi^\top\left(\widehat{M}_\theta\right)^{h-1}\widehat{\Sigma}^{-1}\nabla_\theta w^\theta.
\end{align*}
Using the definition of $\widehat{\nu}^\theta$, we get
\begin{align}
    \label{p1_homo}
    \begin{aligned}
        &\int_{\mathcal{S}\times\mathcal{A}}\xi(s)\pi_\theta(a\vert s)\left(\sum_{h=1}^\infty\gamma^{h-1}\left(\widehat{\mathcal{P}}_\theta\right)^{h-1}\left(\nabla_\theta Q^\theta-\widehat{U}^\theta-\gamma\widehat{\mathcal{P}}_\theta\nabla_\theta Q^\theta\right)\right)(s,a)\mathrm{d}s\mathrm{d}a\\
        =&\left(\widehat{\nu}^\theta\right)^\top\widehat{\Sigma}^{-1}\frac{1}{KH}\sum_{k=1}^K\sum_{h=1}^H\phi\left(s_h^{(k)},a_h^{(k)}\right)\nabla_\theta\varepsilon_{h,k}^\theta+\frac{\lambda}{KH}\left(\widehat{\nu}^\theta\right)^\top\widehat{\Sigma}^{-1}\nabla_\theta w^\theta.
    \end{aligned}
\end{align}
For the second term,  by Lemma \ref{Q_decomp_homo}, we have
\begin{align*}
    \sum_{h=1}^\infty\gamma^{h-1}\left(\widehat{\mathcal{P}}_\theta\right)^{h-1}\left(\widehat{U}^\theta-\tilde{U}^\theta\right)&=\sum_{h=1}^\infty\gamma^h\left(\widehat{\mathcal{P}}_\theta\right)^h\left(\nabla_\theta\log\Pi_\theta\right)\left(Q^\theta-\widehat{Q}^\theta\right)\\
    &=\sum_{h=1}^\infty\gamma^h\left(\widehat{\mathcal{P}}_\theta\right)^h\left(\nabla_\theta\log\Pi_\theta\right)\sum_{h^\prime=h}^\infty\gamma^{h^\prime-h-1}\left(\widehat{\mathcal{P}}_\theta\right)^{h^\prime-h-1}\left(Q^\theta-\widehat{r}-\gamma\widehat{\mathcal{P}}_\theta Q^\theta\right)\\
    &=\sum_{h=1}^\infty\sum_{h^\prime=1}^{h-1}\left(\gamma\widehat{\mathcal{P}}_\theta\right)^{h^\prime}\left(\nabla_\theta\log\Pi_\theta\right)\left(\gamma\widehat{\mathcal{P}}_\theta\right)^{h-h^\prime-1}\left(Q^\theta-\widehat{r}-\gamma\widehat{\mathcal{P}}_\theta Q^\theta\right).
\end{align*}
Meanwhile, again by Lemma \ref{Q_decomp_homo}, we have
\begin{align*}
    \left(\nabla_\theta\log\Pi_\theta\right)(Q^\theta-\widehat{Q}^\theta)=\left(\nabla_\theta\log\Pi_\theta\right)\sum_{h=1}^\infty\left(\gamma\widehat{\mathcal{P}}_\theta\right)^{h-1}\left(Q^\theta-\widehat{r}-\gamma\widehat{\mathcal{P}}_\theta Q^\theta\right),
\end{align*}
which implies
\begin{align*}
    &\sum_{h=1}^\infty\left(\gamma\widehat{\mathcal{P}}_\theta\right)^{h-1}\left(\widehat{U}^\theta-\tilde{U}^\theta\right)+\left(\nabla_\theta\log\Pi_\theta\right)(Q^\theta-\widehat{Q}^\theta)\\
    =&\sum_{h=1}^\infty\left(\left(\nabla_\theta\log\Pi_\theta\right)\left(\gamma\widehat{\mathcal{P}}_\theta\right)^{h-1}+\sum_{h^\prime=1}^{h-1}\left(\gamma\widehat{\mathcal{P}}_\theta\right)^{h^\prime}\left(\nabla_\theta\log\Pi_\theta\right)\left(\gamma\widehat{\mathcal{P}}_\theta\right)^{h-h^\prime-1}\right)\left(Q^\theta-\widehat{r}-\gamma\widehat{\mathcal{P}}_\theta Q^\theta\right)\\
    =&\sum_{h=1}^\infty\sum_{h^\prime=0}^{h-1}\left(\gamma\widehat{\mathcal{P}}_\theta\right)^{h^\prime}\left(\nabla_\theta\log\Pi_\theta\right)\left(\gamma\widehat{\mathcal{P}}_\theta\right)^{h-h^\prime-1}\left(Q^\theta-\widehat{r}-\gamma\widehat{\mathcal{P}}_\theta Q^\theta\right)\\
    =&\sum_{h=1}^\infty\sum_{h^\prime=0}^{h-1}\left(\gamma\widehat{\mathcal{P}}_\theta\right)^{h^\prime}\left(\nabla_\theta\log\Pi_\theta\right)\left(\gamma\widehat{\mathcal{P}}_\theta\right)^{h-h^\prime-1}\phi^\top\widehat{\Sigma}^{-1}\\
    &\cdot\frac{1}{KH}\sum_{k=1}^K\sum_{h^\prime=1}^H\phi\left(s_{h^\prime}^{(k)},a_{h^\prime}^{(k)}\right)\left(Q^\theta\left(s_{h^\prime}^{(k)},a_{h^\prime}^{(k)}\right)-r_{h^\prime}^{(k)}-\gamma\int_{\mathcal{A}}\pi_\theta\left(a^\prime\left\vert s_{h^\prime+1}^{(k)}\right.\right)Q^\theta\left(s_{h^\prime+1}^{(k)},a^\prime\right)\mathrm{d}a^\prime\right)\\
    &+\frac{\lambda}{KH}\sum_{h=1}^\infty\sum_{h^\prime=0}^{h-1}\left(\gamma\widehat{\mathcal{P}}_\theta\right)^{h^\prime}\left(\nabla_\theta\log\Pi_\theta\right)\left(\gamma\widehat{\mathcal{P}}_\theta\right)^{h-h^\prime-1}\phi^\top\widehat{\Sigma}^{-1}w^\theta.
\end{align*}
For each $j\in[m]$, notice the relation
\begin{align*}
    \left(\nabla_\theta^j\widehat{\nu}_h^\theta\right)^\top&=\left(\nabla_\theta^j\nu_1^\theta\right)^\top\left(\widehat{M}_\theta\right)^{h-1}+\sum_{h^\prime=1}^{h-1}\left(\nu_1^\theta\right)^\top\left(\widehat{M}_\theta\right)^{h^\prime-1}\left(\widehat{\nabla_\theta^j M_\theta}\right)\left(\widehat{M}_\theta\right)^{h-h^\prime-1}\\
    &=\int_{\mathcal{S}\times\mathcal{A}}\xi(s)\pi_\theta(a\vert s)\left(\sum_{h^\prime=0}^{h-1}\left(\widehat{\mathcal{P}}_\theta\right)^{h^\prime}\left(\nabla_\theta\log\Pi_\theta\right)\left(\widehat{\mathcal{P}}_\theta\right)^{h-h^\prime-1}\phi^\top\right)(s,a)\mathrm{d}s\mathrm{d}a.
\end{align*} 
Therefore, we have
\begin{align}
    \label{p2_homo}
    \begin{aligned}
        &\left[\int\xi(s)\pi_\theta(a\vert s)\left(\sum_{h=1}^\infty\left(\gamma\widehat{\mathcal{P}}_\theta\right)^{h-1}\left(\widehat{U}^\theta-\tilde{U}^\theta\right)+\left(\nabla_\theta\log\Pi_\theta\right)(Q^\theta-\widehat{Q}^\theta)\right)(s,a)\mathrm{d}s\mathrm{d}a\right]_j\\
        =&\sum_{h=1}^\infty\gamma^{h-1}\left(\nabla_\theta^j\widehat{\nu}_h^\theta\right)^\top\widehat{\Sigma}^{-1}\frac{1}{KH}\sum_{k=1}^K\sum_{h^\prime=1}^H\phi\left(s_{h^\prime}^{(k)},a_{h^\prime}^{(k)}\right)\left(Q^\theta\left(s_{h^\prime}^{(k)},a_{h^\prime}^{(k)}\right)-r_{h^\prime}^{(k)}-\int_{\mathcal{A}}\pi_\theta\left(a^\prime\left\vert s_{h^\prime+1}^{(k)}\right.\right)Q^\theta\left(s_{h^\prime+1}^{(k)},a^\prime\right)\mathrm{d}a^\prime\right)\\
        &+\frac{\lambda}{KH}\sum_{h=1}^\infty\gamma^{h-1}\left(\nabla_\theta^j\widehat{\nu}_h^\theta\right)^\top\widehat{\Sigma}^{-1}w^\theta\\
        =&\left(\nabla_\theta^j\widehat{\nu}^\theta\right)^\top\widehat{\Sigma}^{-1}\frac{1}{KH}\sum_{k=1}^K\sum_{h=1}^H\phi\left(s_h^{(k)},a_h^{(k)}\right)\varepsilon_{h,k}^\theta+\frac{\lambda}{KH}\left(\nabla_\theta^j\widehat{\nu}^\theta\right)^\top\widehat{\Sigma}^{-1}w^\theta.
    \end{aligned}
\end{align}
Combing the results of \eqref{p1_homo} and \eqref{p2_homo}, we get for each $j\in[m]$, 
\begin{align*}
    &\nabla_\theta^j v_\theta-\widehat{\nabla_\theta^j v_\theta}=\int_{\mathcal{S}\times\mathcal{A}}\xi(s)\pi_\theta(a\vert s)\left(\nabla_\theta^j Q^\theta-\widehat{\nabla_\theta^j Q^\theta}+\left(\nabla_\theta^j\log\Pi_\theta\right)(Q^\theta-\widehat{Q}^\theta)\right)(s,a)\textrm{d}s\textrm{d}a\\
    =&\left(\widehat{\nu}^\theta\right)^\top\widehat{\Sigma}^{-1}\frac{1}{KH}\sum_{k=1}^K\sum_{h=1}^H\phi\left(s_h^{(k)},a_h^{(k)}\right)\nabla_\theta^j\varepsilon_{h,k}^\theta+\frac{\lambda}{KH}\left(\widehat{\nu}^\theta\right)^\top\widehat{\Sigma}^{-1}\nabla_\theta^j w^\theta+\left(\nabla_\theta^j\nu^\theta\right)^\top\widehat{\Sigma}^{-1}\frac{1}{KH}\sum_{k=1}^K\sum_{h=1}^H\phi\left(s_h^{(k)},a_h^{(k)}\right)\varepsilon_{h,k}^\theta\\
    &+\frac{\lambda}{KH}\left(\nabla_\theta^j\widehat{\nu}^\theta\right)^\top\widehat{\Sigma}^{-1}w^\theta\\
    =&\nabla_\theta^j\left(\left(\widehat{\nu}^\theta\right)^\top\widehat{\Sigma}^{-1}\frac{1}{KH}\sum_{k=1}^K\sum_{h=1}^H\phi\left(s_h^{(k)},a_h^{(k)}\right)\varepsilon_{h,k}^\theta+\frac{\lambda}{KH}\left(\widehat{\nu}^\theta\right)^\top\widehat{\Sigma}^{-1}w^\theta\right)\\
    =&\nabla_\theta^j\left(\left(\widehat{\nu}^\theta\right)^\top\widehat{\Sigma}^{-1}\frac{1}{KH}\sum_{k=1}^K\sum_{h=1}^H\phi\left(s_h^{(k)},a_h^{(k)}\right)\varepsilon_{h,k}^\theta+\frac{\lambda}{KH}\left(\widehat{\nu}^\theta\right)^\top\widehat{\Sigma}^{-1}w^\theta+\left(\left(\widehat{\nu}^\theta\right)^\top\widehat{\Sigma}^{-1}-\left(\nu^\theta\right)^\top\Sigma^{-1}\right)\frac{1}{KH}\sum_{k=1}^K\sum_{h=1}^H\phi\left(s_h^{(k)},a_h^{(k)}\right)\varepsilon_{h,k}^\theta\right).
\end{align*}
Rewriting the above decomposition in a vector form, we get
\begin{align*}
    \nabla_\theta v_\theta-\widehat{\nabla_\theta v_\theta}=&\nabla_\theta\Bigg(\left(\widehat{\nu}^\theta\right)^\top\widehat{\Sigma}^{-1}\frac{1}{KH}\sum_{k=1}^K\sum_{h=1}^H\phi\left(s_h^{(k)},a_h^{(k)}\right)\varepsilon_{h,k}^\theta\\
    &+\frac{\lambda}{KH}\left(\widehat{\nu}^\theta\right)^\top\widehat{\Sigma}^{-1}w^\theta+\left(\left(\widehat{\nu}^\theta\right)^\top\widehat{\Sigma}^{-1}-\left(\nu^\theta\right)^\top\Sigma^{-1}\right)\frac{1}{KH}\sum_{k=1}^K\sum_{h=1}^H\phi\left(s_h^{(k)},a_h^{(k)}\right)\varepsilon_{h,k}^\theta\Bigg),
\end{align*}
which is the desired result. 
\end{proof}

\subsubsection{Proof of Lemma \ref{e1_finite_product_homo}}
\begin{proof}
Note that, 
\begin{align*}
    \langle E_1,t\rangle=&\left\langle\nabla_\theta\left(\left(\nu^\theta\right)^\top\Sigma^{-1}\frac{1}{KH}\sum_{k=1}^K\sum_{h=1}^H\phi\left(s_h^{(k)},a_h^{(k)}\right)\varepsilon_{h,k}^\theta\right),t\right\rangle\\
    =&\left\langle\left(\nabla_\theta\nu^\theta\right)^\top\Sigma^{-1}\frac{1}{KH}\sum_{k=1}^K\sum_{h=1}^H\phi\left(s_h^{(k)},a_h^{(k)}\right)\varepsilon_{h,k}^\theta,t\right\rangle+\left\langle\left(\nu^\theta\right)^\top\Sigma^{-1}\frac{1}{KH}\sum_{k=1}^K\sum_{h=1}^H\phi\left(s_h^{(k)},a_h^{(k)}\right)\nabla_\theta\varepsilon_{h,k}^\theta,t\right\rangle.
\end{align*}
Let $e_{h,k}=\left\langle\nabla_\theta\left(\left(\nu^\theta\right)^\top\Sigma^{-1}\phi\left(s_h^{(k)},a_h^{(k)}\right)\varepsilon_{h,k}^\theta\right),t\right\rangle$, we have 
\begin{align*}
    \vert e_{h,k}\vert&\leq\sqrt{C_1dm}\Vert t\Vert\frac{1}{1-\gamma}\max_{j\in[m]}\sqrt{\left(\nabla^j_\theta\nu^\theta\right)^\top\Sigma^{-1}\nabla^j_\theta\nu^\theta}+2G\sqrt{C_1dm}\Vert t\Vert\frac{1}{(1-\gamma)^2}\sqrt{\left(\nu^\theta\right)^\top\Sigma^{-1}\nu^\theta}=B\sqrt{C_1dm}\Vert t\Vert.
\end{align*}
We have
\begin{align*}
    &\sum_{k=1}^K\sum_{h=1}^H\textrm{Var}[e_{h,k}\vert\mathcal{F}_{h,k}]=\sum_{k=1}^K\sum_{h=1}^H\mathbb{E}\left[\left.\left\langle\nabla_\theta\left(\left(\nu^\theta\right)^\top\Sigma^{-1}\phi\left(s_h^{(k)},a_h^{(k)}\right)\varepsilon_{h,k}^\theta\right),t\right\rangle^2\right\vert\mathcal{F}_{h,k}\right]=HKt^\top\Lambda_\theta t.
\end{align*}
We pick $\sigma^2=HKt^\top\Lambda_\theta t$, the Bernstein’s inequality implies that for any $\varepsilon\in\mathbb{R}$, 
\begin{align*}
    \mathbb{P}\left(\left\vert\sum_{k=1}^K\sum_{h=1}^He_{h,k}\right\vert\geq\varepsilon\right)\leq 2\exp\left(-\frac{\varepsilon^2/2}{\sigma^2+\sqrt{C_1dm}\Vert t\Vert B\varepsilon/3}\right).
\end{align*}
Therefore, if we pick $\varepsilon=\sigma\sqrt{2\log(2/\delta)}+2\log(2/\delta)\sqrt{C_1dm}\Vert t\Vert B/3$, we get
\begin{align*}
    \mathbb{P}\left(\left\vert\sum_{k=1}^K\sum_{h=1}^He_{h,k}\right\vert\geq\varepsilon\right)\leq\delta,
\end{align*}
i.e., we have with probability $1-\frac{\delta}{4}$, 
\begin{align*}
    \left\vert\frac{1}{HK}\sum_{k=1}^K\sum_{h=1}^H e_{h,k}\right\vert\leq\sqrt{\frac{2t^\top\Lambda_\theta t\log(8/\delta)}{HK}}+\frac{2\log(8/\delta)\sqrt{C_1dm}\Vert t\Vert B}{3HK}
\end{align*}
\end{proof}

\subsubsection{Proof of Lemma \ref{e2_homo}}
\begin{proof}
For an arbitrarily given $\theta_0$, let $\Sigma_{\theta_0}=\mathbb{E}^{\pi_{\theta_0}}[\phi(s,a)\phi(s,a)^\top\vert s\sim\xi_{\theta_0},a\sim\pi_{\theta_0}(\cdot\vert s)]$, we have
\begin{align*}
    &\left(\left(\widehat{\nu}^\theta\right)^\top\widehat{\Sigma}^{-1}-\left(\nu^\theta\right)^\top\Sigma^{-1}\right)\frac{1}{HK}\sum_{k=1}^K\sum_{h=1}^H\phi\left(s_h^{(k)},a_h^{(k)}\right)\varepsilon_{h,k}^\theta\\
    =&\sum_{h^\prime=1}^\infty\gamma^{h^\prime-1}\left(\nu^\theta_1\right)^\top\left(\left(\widehat{M}_\theta\right)^{h^\prime-1}\widehat{\Sigma}^{-1}-\left(M_\theta\right)^{h^\prime-1}\Sigma^{-1}\right)\frac{1}{HK}\sum_{k=1}^K\sum_{h=1}^H\phi\left(s_h^{(k)},a_h^{(k)}\right)\varepsilon_{h,k}^\theta\\
    =&\sum_{h^\prime=1}^\infty\gamma^{h^\prime-1}\left(\Sigma_{\theta_0}^{-\frac{1}{2}}\nu^\theta_1\right)^\top\left(\left(\Sigma_{\theta_0}^{\frac{1}{2}}\widehat{M}_\theta\Sigma_{\theta_0}^{-\frac{1}{2}}\right)^{h^\prime-1}\Sigma_{\theta_0}^{\frac{1}{2}}\Sigma^{-\frac{1}{2}}\Sigma^{\frac{1}{2}}\widehat{\Sigma}^{-1}\Sigma^{\frac{1}{2}}-\left(\Sigma_{\theta_0}^{\frac{1}{2}}M_\theta\Sigma_{\theta_0}^{-\frac{1}{2}}\right)^{h^\prime-1}\Sigma_{\theta_0}^{\frac{1}{2}}\Sigma^{-\frac{1}{2}}\right)\\
    &\cdot\Sigma^{-\frac{1}{2}}\frac{1}{HK}\sum_{k=1}^K\sum_{h=1}^H\phi\left(s_h^{(k)},a_h^{(k)}\right)\varepsilon_{h,k}^\theta.
\end{align*}
Taking derivatives on both sides, and let $\theta_0=\theta$, we get
\begin{align*}
    &\nabla_\theta^j E_2=\nabla_\theta^j\left(\left(\left(\widehat{\nu}^\theta\right)^\top\widehat{\Sigma}^{-1}-\left(\nu^\theta\right)^\top\Sigma^{-1}\right)\frac{1}{HK}\sum_{k=1}^K\sum_{h=1}^H\phi\left(s_h^{(k)},a_h^{(k)}\right)\varepsilon_{h,k}^\theta\right)=E_{21}^j+E_{22}^j+E_{23}^j,
\end{align*}
where 
\begin{align*}
    E_{21}^j=&\sum_{h^\prime=1}^\infty\gamma^{h^\prime-1}\left(\Sigma_\theta^{-\frac{1}{2}}\nu^\theta_1\right)^\top\left(\left(\Sigma_\theta^{\frac{1}{2}}\widehat{M}_\theta\Sigma_\theta^{-\frac{1}{2}}\right)^{h^\prime-1}\Sigma_\theta^{\frac{1}{2}}\Sigma^{-\frac{1}{2}}\Sigma^{\frac{1}{2}}\widehat{\Sigma}^{-1}\Sigma^{\frac{1}{2}}-\left(\Sigma_\theta^{\frac{1}{2}}M_\theta\Sigma_\theta^{-\frac{1}{2}}\right)^{h^\prime-1}\Sigma_\theta^{\frac{1}{2}}\Sigma^{-\frac{1}{2}}\right)\\
    &\cdot\Sigma^{-\frac{1}{2}}\frac{1}{HK}\sum_{k=1}^K\sum_{h=1}^H\phi\left(s_h^{(k)},a_h^{(k)}\right)\nabla_\theta^j\varepsilon_{h,k}^\theta\\
    E_{22}^j=&\sum_{h^\prime=1}^\infty\gamma^{h^\prime-1}\left(\Sigma_\theta^{-\frac{1}{2}}\nabla_\theta^j\nu^\theta_1\right)^\top\left(\left(\Sigma_\theta^{\frac{1}{2}}\widehat{M}_\theta\Sigma_\theta^{-\frac{1}{2}}\right)^{h^\prime-1}\Sigma_\theta^{\frac{1}{2}}\Sigma^{-\frac{1}{2}}\Sigma^{\frac{1}{2}}\widehat{\Sigma}^{-1}\Sigma^{\frac{1}{2}}-\left(\Sigma_\theta^{\frac{1}{2}}M_\theta\Sigma_\theta^{-\frac{1}{2}}\right)^{h^\prime-1}\Sigma_\theta^{\frac{1}{2}}\Sigma^{-\frac{1}{2}}\right)\\
    &\cdot\Sigma^{-\frac{1}{2}}\frac{1}{HK}\sum_{k=1}^K\sum_{h=1}^H\phi\left(s_h^{(k)},a_h^{(k)}\right)\varepsilon_{h,k}^\theta\\
    E_{23}^j=&\sum_{h^\prime=1}^\infty\gamma^{h^\prime-1}\left(\Sigma_\theta^{-\frac{1}{2}}\nu^\theta_1\right)^\top\\
    &\cdot\left(\left.\nabla_\theta^j\left(\Sigma_{\theta_0}^{\frac{1}{2}}\widehat{M}_{\theta}\Sigma_{\theta_0}^{-\frac{1}{2}}\right)^{h^\prime-1}\right\vert_{\theta_0=\theta}\Sigma_\theta^{\frac{1}{2}}\Sigma^{-\frac{1}{2}}\Sigma^{\frac{1}{2}}\widehat{\Sigma}^{-1}\Sigma^{\frac{1}{2}}-\left.\nabla_\theta^j\left(\Sigma_{\theta_0}^{\frac{1}{2}}M_\theta\Sigma_{\theta_0}^{-\frac{1}{2}}\right)^{h^\prime-1}\right\vert_{\theta_0=\theta}\Sigma_\theta^{\frac{1}{2}}\Sigma^{-\frac{1}{2}}\right)\\
    &\cdot\Sigma^{-\frac{1}{2}}\frac{1}{HK}\sum_{k=1}^K\sum_{h=1}^H\phi\left(s_h^{(k)},a_h^{(k)}\right)\varepsilon_{h,k}^\theta.
\end{align*}
Therefore, using the result of Lemma \ref{decomp_homo}, we get
\begin{align*}
    \vert E_{21}^j\vert\leq&\sum_{h^\prime=1}^\infty\gamma^{h^\prime-1}\left\Vert\Sigma_\theta^{-\frac{1}{2}}\nu^\theta_1\right\Vert\left\Vert\Sigma_\theta^{\frac{1}{2}}\Sigma^{-\frac{1}{2}}\right\Vert\left(\left(1+\left\Vert\Sigma_\theta^{\frac{1}{2}}\left(\Delta M_\theta\right)\Sigma_\theta^{-\frac{1}{2}}\right\Vert\right)^{h^\prime-1}\left(1+\left\Vert\Sigma^{\frac{1}{2}}\left(\Delta\Sigma^{-1}\right)\Sigma^{\frac{1}{2}}\right\Vert\right)-1\right)\\
    &\cdot\left\Vert\Sigma^{-\frac{1}{2}}\frac{1}{HK}\sum_{k=1}^K\sum_{h=1}^H\phi\left(s_h^{(k)},a_h^{(k)}\right)\nabla_\theta^j\varepsilon_{h,k}^\theta\right\Vert\\
    =&\left\Vert\Sigma_\theta^{-\frac{1}{2}}\nu^\theta_1\right\Vert\left\Vert\Sigma_\theta^{\frac{1}{2}}\Sigma^{-\frac{1}{2}}\right\Vert\left(\left(1-\gamma-\gamma\left\Vert\Sigma_\theta^{\frac{1}{2}}\left(\Delta M_\theta\right)\Sigma_\theta^{-\frac{1}{2}}\right\Vert\right)^{-1}\left(1+\left\Vert\Sigma^{\frac{1}{2}}\left(\Delta\Sigma^{-1}\right)\Sigma^{\frac{1}{2}}\right\Vert\right)-\left(1-\gamma\right)^{-1}\right)\\
    &\cdot\left\Vert\Sigma^{-\frac{1}{2}}\frac{1}{HK}\sum_{k=1}^K\sum_{h=1}^H\phi\left(s_h^{(k)},a_h^{(k)}\right)\nabla_\theta^j\varepsilon_{h,k}^\theta\right\Vert\\
    \leq&\frac{1}{1-\gamma}\left\Vert\Sigma_\theta^{-\frac{1}{2}}\nu^\theta_1\right\Vert\left\Vert\Sigma_\theta^{\frac{1}{2}}\Sigma^{-\frac{1}{2}}\right\Vert\left(\left(1-\frac{\left\Vert\Sigma_\theta^{\frac{1}{2}}\left(\Delta M_\theta\right)\Sigma_\theta^{-\frac{1}{2}}\right\Vert}{1-\gamma}\right)^{-1}\left(1+\left\Vert\Sigma^{\frac{1}{2}}\left(\Delta\Sigma^{-1}\right)\Sigma^{\frac{1}{2}}\right\Vert\right)-1\right)\\
    &\cdot\left\Vert\Sigma^{-\frac{1}{2}}\frac{1}{HK}\sum_{k=1}^K\sum_{h=1}^H\phi\left(s_h^{(k)},a_h^{(k)}\right)\nabla_\theta^j\varepsilon_{h,k}^\theta\right\Vert\\
    \vert E_{22}^j\vert\leq&\sum_{h^\prime=1}^\infty\gamma^{h^\prime-1}\left\Vert\Sigma_\theta^{-\frac{1}{2}}\nabla_\theta^j\nu^\theta_1\right\Vert\left\Vert\Sigma_\theta^{\frac{1}{2}}\Sigma^{-\frac{1}{2}}\right\Vert\left(\left(1+\left\Vert\Sigma_\theta^{\frac{1}{2}}\left(\Delta M_\theta\right)\Sigma_\theta^{-\frac{1}{2}}\right\Vert\right)^{h^\prime-1}\left(1+\left\Vert\Sigma^{\frac{1}{2}}\left(\Delta\Sigma^{-1}\right)\Sigma^{\frac{1}{2}}\right\Vert\right)-1\right)\\
    &\cdot\left\Vert\Sigma^{-\frac{1}{2}}\frac{1}{HK}\sum_{k=1}^K\sum_{h=1}^H\phi\left(s_h^{(k)},a_h^{(k)}\right)\varepsilon_{h,k}^\theta\right\Vert\\
    \leq&\frac{1}{1-\gamma}\left\Vert\Sigma_\theta^{-\frac{1}{2}}\nabla_\theta^j\nu^\theta_1\right\Vert\left\Vert\Sigma_\theta^{\frac{1}{2}}\Sigma^{-\frac{1}{2}}\right\Vert\left(\left(1-\frac{\left\Vert\Sigma_\theta^{\frac{1}{2}}\left(\Delta M_\theta\right)\Sigma_\theta^{-\frac{1}{2}}\right\Vert}{1-\gamma}\right)^{-1}\left(1+\left\Vert\Sigma^{\frac{1}{2}}\left(\Delta\Sigma^{-1}\right)\Sigma^{\frac{1}{2}}\right\Vert\right)-1\right)\\
    &\cdot\left\Vert\Sigma^{-\frac{1}{2}}\frac{1}{HK}\sum_{k=1}^K\sum_{h=1}^H\phi\left(s_h^{(k)},a_h^{(k)}\right)\varepsilon_{h,k}^\theta\right\Vert\\
    \vert E_{23}^j\vert\leq&\sum_{h^\prime=2}^\infty(h^\prime-1)\gamma^{h^\prime-1}G\left\Vert\Sigma_\theta^{-\frac{1}{2}}\nu^\theta_1\right\Vert\left\Vert\Sigma_\theta^{\frac{1}{2}}\Sigma^{-\frac{1}{2}}\right\Vert\\
    &\cdot\left(\left(1+\left\Vert\Sigma_\theta^{\frac{1}{2}}\left(\Delta M_\theta\right)\Sigma_\theta^{-\frac{1}{2}}\right\Vert\right)^{h^\prime-2}\left(1+\left\Vert\Sigma_\theta^{\frac{1}{2}}\left(\frac{\nabla_\theta^j\left(\Delta M_\theta\right)}{G}\right)\Sigma_\theta^{-\frac{1}{2}}\right\Vert\right)\left(1+\left\Vert\Sigma^{\frac{1}{2}}\left(\Delta\Sigma^{-1}\right)\Sigma^{\frac{1}{2}}\right\Vert\right)-1\right)\\
    &\cdot\left\Vert\Sigma^{-\frac{1}{2}}\frac{1}{HK}\sum_{k=1}^K\sum_{h=1}^H\phi\left(s_h^{(k)},a_h^{(k)}\right)\varepsilon_{h,k}^\theta\right\Vert\\
    \leq&\frac{G}{(1-\gamma)^2}\left\Vert\Sigma_\theta^{-\frac{1}{2}}\nu^\theta_1\right\Vert\left\Vert\Sigma_\theta^{\frac{1}{2}}\Sigma^{-\frac{1}{2}}\right\Vert\\
    &\cdot\left(\left(1-\frac{\left\Vert\Sigma_\theta^{\frac{1}{2}}\left(\Delta M_\theta\right)\Sigma_\theta^{-\frac{1}{2}}\right\Vert}{1-\gamma}\right)^{-2}\left(1+\left\Vert\Sigma_\theta^{\frac{1}{2}}\left(\frac{\nabla_\theta^j\left(\Delta M_\theta\right)}{G}\right)\Sigma_\theta^{-\frac{1}{2}}\right\Vert\right)\left(1+\left\Vert\Sigma^{\frac{1}{2}}\left(\Delta\Sigma^{-1}\right)\Sigma^{\frac{1}{2}}\right\Vert\right)-1\right)\\
    &\cdot\left\Vert\Sigma^{-\frac{1}{2}}\frac{1}{HK}\sum_{k=1}^K\sum_{h=1}^H\phi\left(s_h^{(k)},a_h^{(k)}\right)\varepsilon_{h,k}^\theta\right\Vert,
\end{align*}
where $\Delta\Sigma^{-1}=\widehat{\Sigma}^{-1}-\Sigma$ and we use the fact $\left\Vert\Sigma_\theta^{\frac{1}{2}}M_\theta\Sigma_\theta^{-\frac{1}{2}}\right\Vert\leq 1$ and $\left\Vert\Sigma_\theta^{\frac{1}{2}}\left(\nabla_\theta^j M_\theta\right)\Sigma_\theta^{-\frac{1}{2}}\right\Vert\leq G$ from Lemma \ref{ineq_homo}. Now, define $\alpha=6\sqrt{\kappa_1}(4+\kappa_2+\kappa_3)\sqrt{\frac{C_1d\log\frac{16dmH}{\delta}}{K}}$ and pick
\begin{align*}
    K\geq 36\kappa_1(4+\kappa_2+\kappa_3)^2\frac{C_1d}{(1-\gamma)^2}\log\frac{16dmH}{\delta},\quad\lambda\leq C_1d\sigma_{\textrm{min}}(\Sigma)\cdot\log\frac{8dmH}{\delta},
\end{align*}
we get $\alpha\leq \frac{1-\gamma}{2}$. Using the results of Lemma \ref{dsig2_homo} and Lemma \ref{dm2_homo}, we get
\begin{align}
    \label{sig_homo}
    \left\Vert\Sigma^{\frac{1}{2}}\left(\Delta\Sigma^{-1}\right)\Sigma^{\frac{1}{2}}\right\Vert\leq 4\sqrt{\frac{C_1d\log\frac{8dmH}{\delta}}{K}}\leq\alpha\leq 1, 
\end{align}
and 
\begin{align}
    \label{dm_homo}
    \left\Vert\Sigma_\theta^{\frac{1}{2}}\left(\Delta M_\theta\right)\Sigma_\theta^{-\frac{1}{2}}\right\Vert&\leq\alpha,\\
    \label{ddm_homo}
    \left\Vert\Sigma_\theta^{\frac{1}{2}}\left(\frac{\nabla_\theta^j\left(\Delta M_\theta\right)}{G}\right)\Sigma_\theta^{-\frac{1}{2}}\right\Vert&\leq\alpha,\quad\forall j\in[m].
\end{align}
Meanwhile, the event $\mathcal{E}_\varepsilon$ implies
\begin{align}
    \label{ep1_homo}
    \left\Vert\Sigma^{-\frac{1}{2}}\frac{1}{KH}\sum_{k=1}^K\sum_{h=1}^H\phi\left(s_h^{(k)},a_h^{(k)}\right)\varepsilon_{h,k}^\theta\right\Vert&\leq \frac{4\sqrt{d}}{1-\gamma}\sqrt{\frac{\log\frac{32dmH}{\delta}}{KH}}\\
    \label{ep2_homo}
    \left\Vert\Sigma^{-\frac{1}{2}}\frac{1}{KH}\sum_{k=1}^K\sum_{h=1}^H\phi\left(s_h^{(k)},a_h^{(k)}\right)\nabla_\theta^j\varepsilon_{h,k}^\theta\right\Vert&\leq \frac{8\sqrt{d}G}{(1-\gamma)^2}\sqrt{\frac{\log\frac{32dmH}{\delta}}{KH}}.
\end{align}
Combining the results of \eqref{sig_homo}, \eqref{dm_homo}, \eqref{ddm_homo}, \eqref{ep1_homo}, \eqref{ep2_homo} and use a union bound, we have with probability $1-3\delta$, 
\begin{align*}
    \vert E_{21}^j\vert\leq&\frac{16\alpha\sqrt{d}G}{(1-\gamma)^4}\left\Vert\Sigma_\theta^{-\frac{1}{2}}\nu^\theta_1\right\Vert\left\Vert\Sigma_\theta^{\frac{1}{2}}\Sigma^{-\frac{1}{2}}\right\Vert \sqrt{\frac{\log\frac{32dmH}{\delta}}{KH}}\\
    \vert E_{22}^j\vert\leq&\frac{8\alpha\sqrt{d}}{(1-\gamma)^3}\left\Vert\Sigma_\theta^{-\frac{1}{2}}\nabla_\theta^j\nu^\theta_1\right\Vert\left\Vert\Sigma_\theta^{\frac{1}{2}}\Sigma^{-\frac{1}{2}}\right\Vert\sqrt{\frac{\log\frac{32dmH}{\delta}}{KH}}\\
    \vert E_{23}^j\vert\leq&\frac{16G\alpha\sqrt{d}}{(1-\gamma)^4}\left\Vert\Sigma_\theta^{-\frac{1}{2}}\nu^\theta_1\right\Vert\left\Vert\Sigma_\theta^{\frac{1}{2}}\Sigma^{-\frac{1}{2}}\right\Vert\sqrt{\frac{\log\frac{32dmH}{\delta}}{KH}}
\end{align*}
where we use the fact $(1-\frac{\alpha}{1-\gamma})^{-1}(1+\alpha)\leq 1+\frac{3\alpha}{1-\gamma}$ whenever $\alpha \leq \frac{1}{2(1-\gamma)}$. Summing up the above terms and using the definition of $\alpha$, we get
\begin{align*}
    \vert E_2^j\vert\leq \frac{240\sqrt{\kappa_1}(2+\kappa_2+\kappa_3)\sqrt{C_1}d}{(1-\gamma)^3}\left(\left\Vert\Sigma_\theta^{-\frac{1}{2}}\nabla_\theta^j\nu^\theta_1\right\Vert+\frac{G}{1-\gamma}\left\Vert\Sigma_\theta^{-\frac{1}{2}}\nu^\theta_1\right\Vert\right)\left\Vert\Sigma_\theta^{\frac{1}{2}}\Sigma^{-\frac{1}{2}}\right\Vert\frac{\log\frac{32dmH}{\delta}}{KH},\quad\forall j\in[m],
\end{align*}
which finished the proof. 
\end{proof}

\subsubsection{Proof of Lemma \ref{e3_homo}}
\begin{proof}
Similar to the decomposition in the proof of Lemma \ref{e2_homo}, we have
\begin{align*}
    \left\vert E_3^j\right\vert=&\frac{\lambda}{KH}\left\vert\nabla_\theta^j\left(\left(\widehat{\nu}^\theta\right)^\top\widehat{\Sigma}^{-1}w^\theta\right)\right\vert\\
    \leq&\frac{\lambda}{KH}\sum_{h=1}^\infty\gamma^{h-1}\left\Vert\Sigma_\theta^{-\frac{1}{2}}\nu^\theta_1\right\Vert\left\Vert\Sigma_\theta^{\frac{1}{2}}\Sigma^{-\frac{1}{2}}\right\Vert\left(1+\left\Vert\Sigma_\theta^{\frac{1}{2}}\left(\Delta M_\theta\right)\Sigma_\theta^{-\frac{1}{2}}\right\Vert\right)^{h-1}\left(1+\left\Vert\Sigma^{\frac{1}{2}}\left(\Delta\Sigma^{-1}\right)\Sigma^{\frac{1}{2}}\right\Vert\right)\left\Vert\Sigma^{-1}\right\Vert\left\Vert\Sigma^{\frac{1}{2}}\nabla_\theta^j w^\theta\right\Vert\\
    +&\frac{\lambda}{KH}\sum_{h=1}^\infty\gamma^{h-1}\left\Vert\Sigma_\theta^{-\frac{1}{2}}\nabla_\theta^j\nu^\theta_1\right\Vert\left\Vert\Sigma_\theta^{\frac{1}{2}}\Sigma^{-\frac{1}{2}}\right\Vert\left(1+\left\Vert\Sigma_\theta^{\frac{1}{2}}\left(\Delta M_\theta\right)\Sigma_\theta^{-\frac{1}{2}}\right\Vert\right)^{h-1}\left(1+\left\Vert\Sigma^{\frac{1}{2}}\left(\Delta\Sigma^{-1}\right)\Sigma^{\frac{1}{2}}\right\Vert\right)\left\Vert\Sigma^{-1}\right\Vert\left\Vert\Sigma^{\frac{1}{2}}w^\theta\right\Vert\\
    +&\frac{\lambda}{KH}\sum_{h=1}^\infty\gamma^{h-1}(h-1)G\left\Vert\Sigma_\theta^{-\frac{1}{2}}\nu^\theta_1\right\Vert\left\Vert\Sigma_\theta^{\frac{1}{2}}\Sigma^{-\frac{1}{2}}\right\Vert\\
    &\cdot\left(1+\left\Vert\Sigma_\theta^{\frac{1}{2}}\left(\Delta M_\theta\right)\Sigma_\theta^{-\frac{1}{2}}\right\Vert\right)^{h-2}\left(1+\left\Vert\Sigma_\theta^{\frac{1}{2}}\left(\frac{\nabla_\theta^j\left(\Delta M_\theta\right)}{G}\right)\Sigma_\theta^{-\frac{1}{2}}\right\Vert\right)\left(1+\left\Vert\Sigma^{\frac{1}{2}}\left(\Delta\Sigma^{-1}\right)\Sigma^{\frac{1}{2}}\right\Vert\right)\left\Vert\Sigma^{-1}\right\Vert\left\Vert\Sigma^{\frac{1}{2}}w^\theta\right\Vert\\
    \leq&\frac{\lambda}{KH}\sum_{h=1}^\infty\gamma^{h-1}\left(1+\alpha\right)^h\left\Vert\Sigma^{-1}\right\Vert\left\Vert\Sigma_\theta^{\frac{1}{2}}\Sigma^{-\frac{1}{2}}\right\Vert\left(\left\Vert\Sigma_\theta^{-\frac{1}{2}}\nu^\theta_1\right\Vert\left\Vert\Sigma^{\frac{1}{2}}\nabla_\theta^j w_h^\theta\right\Vert+\left\Vert\Sigma_\theta^{-\frac{1}{2}}\nabla_\theta^j\nu^\theta_1\right\Vert\left\Vert\Sigma^{\frac{1}{2}}w^\theta\right\Vert+G(h-1)\left\Vert\Sigma_\theta^{-\frac{1}{2}}\nu^\theta_1\right\Vert\left\Vert\Sigma^{\frac{1}{2}}w^\theta\right\Vert\right)\\
    \leq&\frac{2\lambda}{KH(1-\gamma)}\left\Vert\Sigma^{-1}\right\Vert\left\Vert\Sigma_\theta^{\frac{1}{2}}\Sigma^{-\frac{1}{2}}\right\Vert\left(\left\Vert\Sigma_\theta^{-\frac{1}{2}}\nu^\theta_1\right\Vert\left\Vert\Sigma^{\frac{1}{2}}\nabla_\theta^j w^\theta\right\Vert+\left\Vert\Sigma_\theta^{-\frac{1}{2}}\nabla_\theta^j\nu^\theta_1\right\Vert\left\Vert\Sigma^{\frac{1}{2}}w^\theta\right\Vert+\frac{2G}{1-\gamma}\left\Vert\Sigma_\theta^{-\frac{1}{2}}\nu^\theta_1\right\Vert\left\Vert\Sigma^{\frac{1}{2}}w^\theta\right\Vert\right),
\end{align*}
where $\alpha$ is defined in the same way as that in the proof of Lemma \ref{e2_homo}. Similarly, we have $\alpha\leq\frac{1-\gamma}{2}$ and we have
\begin{align*}
    \left\Vert\Sigma^{\frac{1}{2}}\nabla_\theta^j w^\theta\right\Vert^2&=\mathbb{E}\left[\frac{1}{H}\sum_{h=1}^H\left(\nabla_\theta^j Q^\theta\left(s_h^{(1)},a_h^{(1)}\right)\right)^2\right]\leq \frac{G^2}{(1-\gamma)^4}\\
    \left\Vert\Sigma^{\frac{1}{2}}w^\theta\right\Vert^2&=\mathbb{E}\left[\frac{1}{H}\sum_{h=1}^H\left(Q^\theta\left(s_h^{(1)},a_h^{(1)}\right)\right)^2\right]\leq\frac{1}{(1-\gamma)^2}.
\end{align*}
We conclude
\begin{align*}
    \vert E_3^j\vert\leq&\frac{2\lambda}{KH(1-\gamma)^2}\left\Vert\Sigma^{-1}\right\Vert\left\Vert\Sigma_\theta^{\frac{1}{2}}\Sigma^{-\frac{1}{2}}\right\Vert\left(\left\Vert\Sigma_\theta^{-\frac{1}{2}}\nu^\theta_1\right\Vert\frac{3G}{1-\gamma}+\left\Vert\Sigma_\theta^{-\frac{1}{2}}\nabla_\theta^j\nu^\theta_1\right\Vert\right)\\
    \leq&\frac{6C_1d\sigma_{\textrm{min}}(\Sigma)\cdot\log\frac{8dmH}{\delta}
}{KH(1-\gamma)^2}\left\Vert\Sigma^{-1}\right\Vert\left\Vert\Sigma_\theta^{\frac{1}{2}}\Sigma^{-\frac{1}{2}}\right\Vert\left(\left\Vert\Sigma_\theta^{-\frac{1}{2}}\nu^\theta_1\right\Vert\frac{G}{1-\gamma}+\left\Vert\Sigma_\theta^{-\frac{1}{2}}\nabla_\theta^j\nu^\theta_1\right\Vert\right).
\end{align*}
which has finished the proof. 
\end{proof}

\end{document}